\newif\ifColor

\documentclass[11pt,twoside,a4paper]{book}
\usepackage[square, sort, numbers]{natbib}
\usepackage{bibentry}
 \nobibliography*

\usepackage{breakcites} 
\usepackage{microtype} 

\title{Graphs for deep learning latent representations}
\author{Carlos Lassance}
\date{}
\usepackage{subcaption}
\expandafter\def\csname ver@subcaption.sty\endcsname{}

\usepackage{zref-abspage}

\usepackage[vcentering,dvips]{geometry}
\usepackage[utf8]{inputenc}
\usepackage{tikz}
\usepackage{pgfplotstable}
\usepackage{makecell}
\pgfplotsset{compat=1.15}
\usepgfplotslibrary{statistics}
\usetikzlibrary{arrows}

\usepackage[T1]{fontenc}

\usepackage{amsmath}
\usepackage{amsfonts}
\usepackage{amsthm}
\usepackage{multirow}
\usepackage{colortbl}
\usepackage{booktabs}
\usepackage{graphicx}

\usepackage{bm}

\usepackage{float}

\usepackage{makeidx}\makeindex
\usepackage[nottoc]{tocbibind}
\usepackage{caption}

\usepackage{fancyhdr}
\usepackage{tikz}
\usetikzlibrary{arrows}
\usetikzlibrary{shapes}
\usepackage{pdfpages}
\usetikzlibrary{external}
\usepackage{amssymb}
\usepackage[chapter]{algorithm}

\usepackage[noend]{algpseudocode}

\theoremstyle{definition}

\newtheorem{proposition}{Proposition}
\newtheorem{theorem}{Theorem}

\newtheorem{remark}{Remark}

\usepackage{array}
\newcolumntype{P}[1]{>{\centering\arraybackslash}p{#1}}

\newcommand{\changefont}{
    \fontsize{9}{11}\selectfont
}
\def\eqref#1{equation~\ref{#1}}
\def\Eqref#1{Equation~\ref{#1}}

\def\1{\bm{1}}
\newcommand{\dataset}{\mathcal{D}}
\newcommand{\trainset}{\mathcal{D_{\mathrm{train}}}}
\newcommand{\supportset}{\mathcal{D_{\mathrm{support}}}}
\newcommand{\relevant}{\mathcal{D_{\mathrm{relevant}}}}
\newcommand{\queryset}{\mathcal{D_{\mathrm{query}}}}
\newcommand{\validset}{\mathcal{D_{\mathrm{valid}}}}
\newcommand{\easyset}{\mathcal{D_{\mathrm{easy}}}}
\newcommand{\hardset}{\mathcal{D_{\mathrm{hard}}}}
\newcommand{\unclearset}{\mathcal{D_{\mathrm{unclear}}}}
\newcommand{\testset}{\mathcal{D_{\mathrm{test}}}}
\newcommand{\classifier}{\mathcal{C}}
\newcommand{\featureextractor}{\mathcal{F}}
\newcommand{\featuremaps}{F}

\def\vb{{\bm{b}}}
\def\vc{{\bm{c}}}

\def\vlambda{{\bm{\lambda}}}

\def\vs{{\bm{s}}}

\def\vx{{\bm{x}}}
\def\vy{{\bm{y}}}

\def\evlambda{{\lambda}}

\def\mA{{\bm{A}}}

\def\mD{{\bm{D}}}

\def\mF{{\bm{F}}}

\def\mH{{\bm{H}}}

\def\mL{{\bm{L}}}

\def\mS{{\bm{S}}}

\def\mW{{\bm{W}}}
\def\mX{{\bm{X}}}

\def\mLambda{{\bm{\Lambda}}}

\DeclareMathAlphabet{\mathsfit}{\encodingdefault}{\sfdefault}{m}{sl}
\SetMathAlphabet{\mathsfit}{bold}{\encodingdefault}{\sfdefault}{bx}{n}

\def\gG{{\mathcal{G}}}

\def\adjmatrix{{\mathbf{\mathcal{A}}}}
\def\allocationmatrix{{\mathbf{\mathcal{T}}}}

\def\identity{{\mathbf{\mathcal{I}}}}

\def\sA{{\mathbb{A}}}

\def\sD{{\mathbb{D}}}
\def\sE{{\mathbb{E}}}
\def\sF{{\mathbb{F}}}

\def\sH{{\mathbb{H}}}

\def\sM{{\mathbb{M}}}
\def\sN{{\mathbb{N}}}

\def\sT{{\mathbb{T}}}
\def\sU{{\mathbb{U}}}
\def\sV{{\mathbb{V}}}

\def\sX{{\mathbb{X}}}

\def\emLambda{{\Lambda}}
\def\emAdjacency{{\mathcal{A}}}

\def\emD{{D}}

\def\emH{{H}}

\def\emS{{S}}

\def\emW{{W}}

\newcommand{\R}{\mathbb{R}}

\newcommand{\normltwo}{L_2}

\newcommand{\normmax}{L_\infty}
\newcommand{\sfilter}{\mathbf{\mathfrak{s}}}
\newcommand{\xfilter}{\mathbf{\mathfrak{s}}}

\usepackage[pdffitwindow=false,
pdfview=FitH,
pdfstartview=FitH,
pagebackref=true,
breaklinks=true,
\ifColor
colorlinks,
\fi
bookmarks=true,
plainpages=false]{hyperref}

\DeclareRobustCommand{\[}{\begin{equation}}
\DeclareRobustCommand{\]}{\end{equation}}

\usepackage[inline]{enumitem}
\newlist{inlinelist}{enumerate*}{1}
\setlist*[inlinelist,1]{label=\roman*),itemjoin={{, }},itemjoin*={{, and }}}

\usepackage{etoc}
\usepackage{titlesec}
\theoremstyle{definition}
\newtheorem{definition}{Definition}[section]

\usepackage[utf8]{inputenc}
\usepackage{pgfplots}
\usepackage{adjustbox}
\usepackage{framed}

\DeclareUnicodeCharacter{2212}{−}
\usepgfplotslibrary{groupplots,dateplot}
\usetikzlibrary{patterns,shapes.arrows}
\pgfplotsset{compat=newest}
\usepackage{svg}

\usetikzlibrary{decorations.markings}
\tikzset{
    set arrow inside/.code={\pgfqkeys{/tikz/arrow inside}{#1}},
    set arrow inside={end/.initial=>, opt/.initial=},
    /pgf/decoration/Mark/.style={
        mark/.expanded=at position #1 with
        {
            \noexpand\arrow[\pgfkeysvalueof{/tikz/arrow inside/opt}]{\pgfkeysvalueof{/tikz/arrow inside/end}}
        }
    },
    arrow inside/.style 2 args={
        set arrow inside={#1},
        postaction={
            decorate,decoration={
                markings,Mark/.list={#2}
            }
        }
    },
}

\definecolor{imtatlantique}{rgb}{0.005,0.715,0.867}

\allowdisplaybreaks

\begin{document}
\setlength{\parindent}{2em}
\setlength{\parskip}{1em}
\newlength{\figwidth}
\setlength{\figwidth}{26pc}
\newlength{\notationgap}
\setlength{\notationgap}{1pc}

\frontmatter

\includepdf{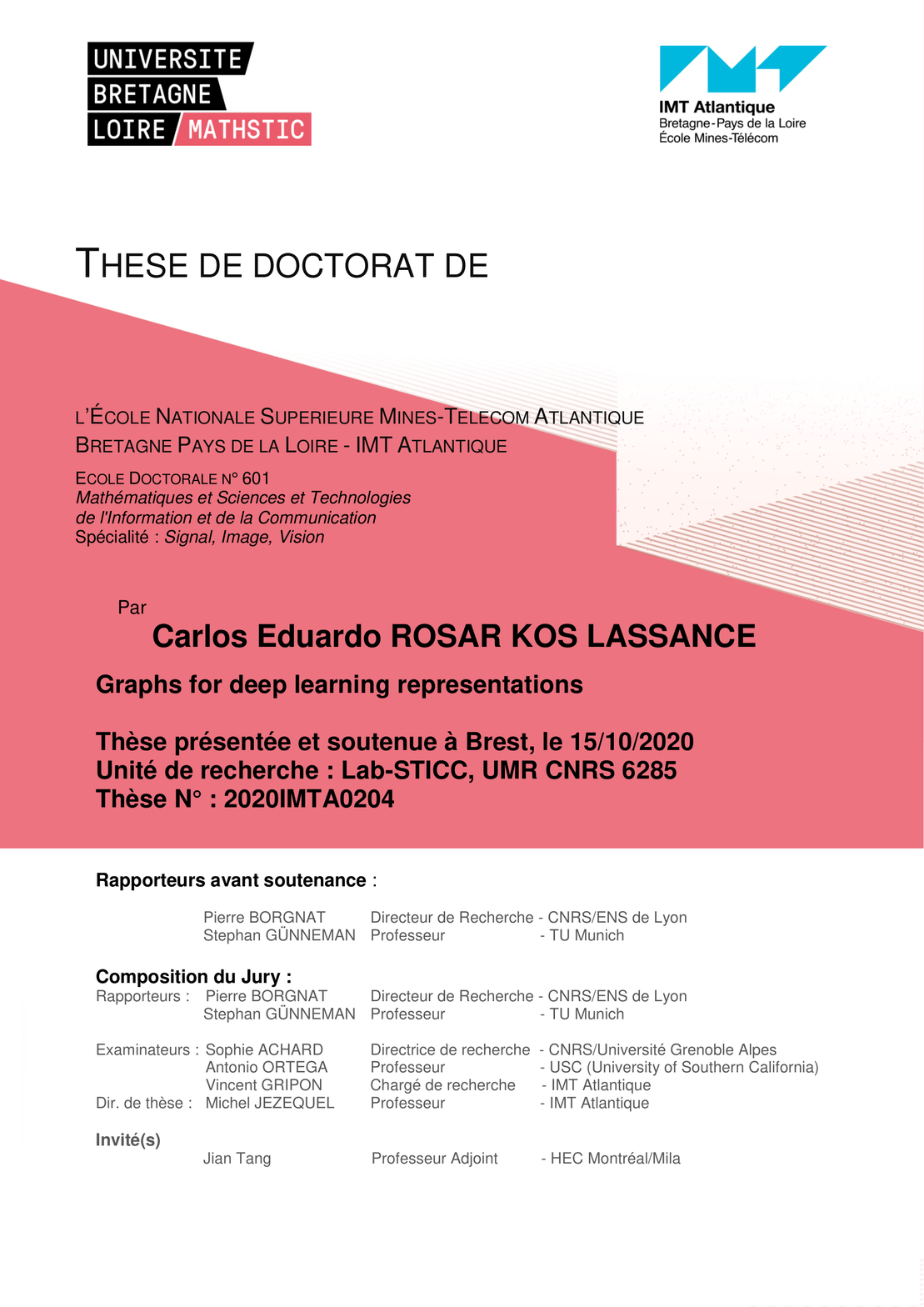}
\tableofcontents


\chapter*{Acknowledgements}\addcontentsline{toc}{chapter}{Acknowledgements}\markboth{\MakeUppercase{Acknowledgements}}{\MakeUppercase{Acknowledgements}}
This is for me the most important part of this thesis, as it would not be possible to do this alone. I can only hope to acknowledge everyone that helped me in this pursuit, but for the ones that are not cited here, it is 100\% my fault as I have left this part to the very end of my writing and I hope you forgive me.

First, I have to thank my whole family. My parents (Jacqueline ROSAR KOS LASSANCE and Carlos Alberto KOS LASSANCE JUNIOR) that supported me in all my decisions and that greatly invested in my education, both academic/formal but most importantly on my social/informal education. I would not be anywhere near where I am now without you. 

The same could be said about my brother (Luiz Carlos BANDEIRA LASSANCE) and sisters (Tatiana BANDEIRA LASSANCE and Patricia BANDEIRA LASSANCE BURNS), without you in my life I would not be able to develop in the person that I am now.

Second, I have to thank my supervisor (Vincent GRIPON) and my thesis director (Michel JEZEQUEL). Vincent teached me more than I could help to understand and both gave me the full support I needed to develop my thesis in the appropriate time. I also need to thank them for their patience, when I was down and thought that I was not doing any work. I can only hope to be able to pay this forward to the next generation.

Third, I have to thank my PUC-Rio professor (Sergio LIFSCHITZ) for all the things he teached me during my time at PUC-Rio, for nudging me into considering Télécom Bretagne (now IMT Atlantique) for my double-degree and for his friendship. 

In this vein, I also need to thank all my co-authors for the invaluable discussions and for improving my knowledge of the field. The same has to be said about my lab-mates who had to endure endless hours of me nagging about the most innocous things. 

I also have to thank my international collaborators. First Antonio ORTEGA from USC, who was almost a second supervisor during my thesis and helped immensely to advance on my thesis and my writing. Second, Jian TANG from Mila and HEC-Montréal, who took me under an one year internship and helped me develop a lot of skills in the graph neural network domain. Third, Yasir LATIF and Ravi GARG from University of Adelaide that teached me a lot on the domain of VBL (Visual-based localization) and allowed me extend my competences to another field of Deep Learning. Finally, Gonzalo MATEOS from University of Rochester, who helped me a lot with graph inference and took me under a one month visit of his lab.

To all my friends, everywhere in the world, who helped me not obsess about the thesis and to be able to arrive here at the end.

I also have to thank all the members of the jury for accepting to be a part of this process. I was eager to read their comments and discuss the work with them and it was with major excitement that I read their reports and answered their questions during my thesis defense.

Finally, I feel that I have to acknowledge Claude BERROU for believing in my capabilities and allowing me to work with him in the context of Deep Learning when I knew almost nothing of the domain. I would have never worked in the field if not for this.

\etocsettocstyle{}{} 
\chapter*{Résumé}\addcontentsline{toc}{chapter}{Résumé}\markboth{\MakeUppercase{Résumé}}{\MakeUppercase{Résumé}}
\localtableofcontents

\section*{Abstract}\addcontentsline{toc}{section}{Abstract}

Ces dernières années, les méthodes d'apprentissage profond ont atteint l'état de l'art dans une vaste gamme de tâches d'apprentissage automatique, y compris la classification d'images et la traduction automatique. Ces architectures sont assemblées pour résoudre des tâches d'apprentissage automatique de bout en bout. Afin d'atteindre des performances de haut niveau, ces architectures nécessitent souvent d'un très grand nombre de paramètres. Les conséquences indésirables sont multiples, et pour y remédier, il est souhaitable de pouvoir compreendre ce qui se passe à l'intérieur des architectures d'apprentissage profond. Il est difficile de le faire en raison de: \begin{inlinelist} \item la dimension élevée des représentations \item la stochasticité du processus de formation.\end{inlinelist} Dans cette thèse, nous étudions ces architectures en introduisant un formalisme à base de graphes, s'appuyant notamment sur les récents progrès du traitement de signaux sur graphe (TSG). À savoir, nous utilisons des graphes pour représenter les espaces latents des réseaux neuronaux profonds. Nous montrons que ce formalisme des graphes nous permet de répondre à diverses questions, notamment: \begin{inlinelist} \item mesurer des capacités de généralisation \item réduire la quantité de des choix arbitraires dans la conception du processus d'apprentissage \item améliorer la robustesse aux petites perturbations ajoutées sur les entrées \item réduire la complexité des calculs. \end{inlinelist}

\section*{Introduction}\addcontentsline{toc}{section}{Introduction}

Ces dernières années, les réseaux de neurones profonds (DNN) ont explosé en popularité, créant un nouveau domaine appelé «Apprentissage Profond»~\citep{dlbook}. Si le concept de réseaux de neurones~\citep{rosenblatt1958perceptron} et les DNN~\citep{rumelhart1986learning} sont tous deux assez anciens, ils n'ont commencé à gagner en popularité que ces dernières années. Ce changement est dû aux deux avancées en matériel, spécialement les cartes graphiques (GPU)~\citep{hacene2019processing} et aux premières victoires dans les défis de vision par ordinateur comme AlexNet~\citep{krizhevsky2012imagenet} gagnant le LSRVC 2012- Imagenet~\citep{russakovsky2015imagenet} et DanNet~\citep{cirecsan2013mitosis} remportant le «Contest on Mitosis Detection in Breast Cancer Histological Images»~\citep{roux2013mitosis}. La figure~\ref{resu:fig:example_alexnet_dannet} présente des exemples d'images issues de ces concours. 

\begin{figure}[ht]
\begin{center}
\includegraphics[width=0.45\columnwidth,height=0.45\columnwidth]{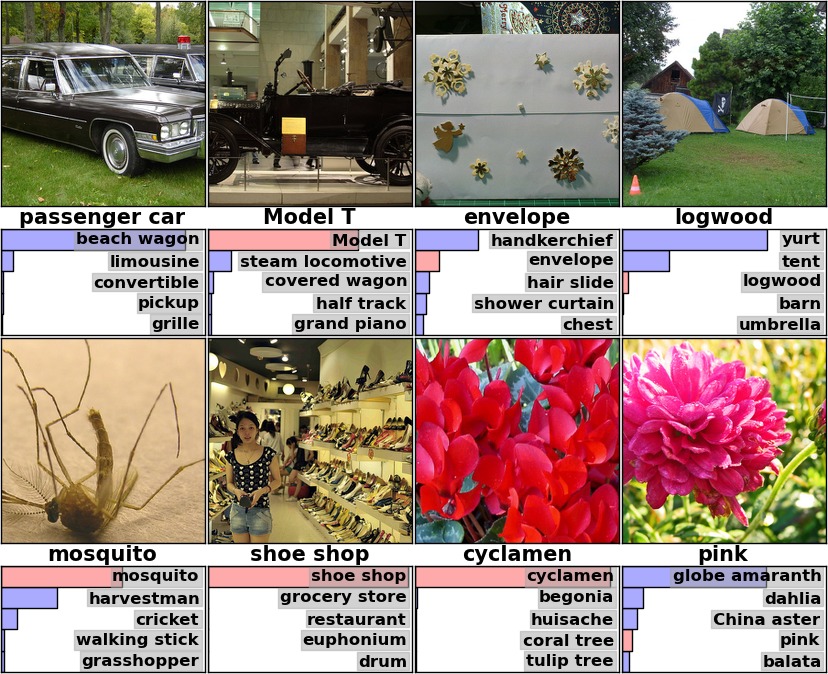}
\includegraphics[width=0.45\columnwidth,height=0.45\columnwidth]{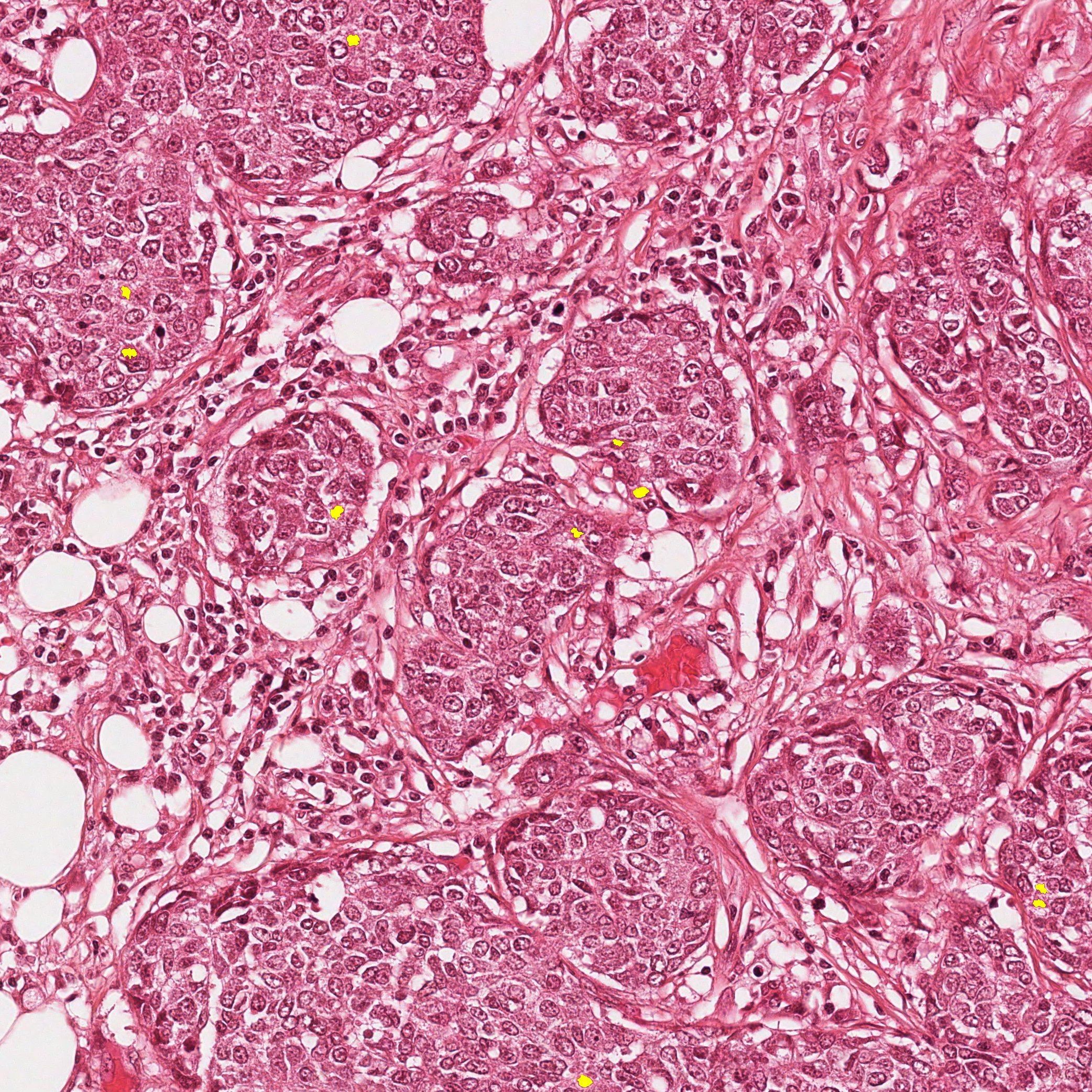}
\end{center}
\caption{\textbf{Gauche:} Huit images du «test set» de ILSVRC-2010 et les cinq étiquettes les plus probables par AlexNet. L'étiquette correcte est écrite au dessous de chaque image, et la probabilité que le réseau a assigné pour la bonne étiquette est montrée par la barre rouge (seulement si la bonne étiquette est dans le top5). Cette image est tirée de~\citep{krizhevsky2012imagenet} @2012 Neural Information Processing Systems Foundation. \textbf{Droite:} Image tirée de~\citep{roux2013mitosis} pour la détéction de mitose, où les parties jaunes sont les parties que le DNN doit détécter}\label{resu:fig:example_alexnet_dannet}
\end{figure}

Ces réseaux ont été construits sur la base de deux grands principes : \begin{inlinelist} \item un a-priori convolutif \item des représentations hiérarchiques apprises par rétropropagation du gradient d'erreur\end{inlinelist}. Le premier guide la forme de base du réseau, afin d'imposer l'invariance aux translations et le partage des poids. Le second assure que c'est à la méthode d'optimisation d'adapter le réseau à un extracteur de caractéristiques suivi d'un classificateur, sans contrôle spécifique de l'évolution de ce dernier. En effet, dans~\citep{lecun1995convolutional} les auteurs disent : «A potentially more interesting scheme is to eliminate the feature extractor, feeding the network with ``raw'' inputs (e.g. normalized images), and to rely on backpropagation to turn the first few layers into an appropriate feature extractor». Par conséquent, nous pouvons considérer le traitement d'une entrée dans un DNN comme la génération d'une séquence de \textbf{représentations intermédiaires} qui font partie des \textbf{espaces latents} du DNN.

Pour mieux comprendre ce que nous appelons un DNN et les représentations intermédiaires, nous allons illustrer les DNN dans la figure~\ref{resu:fig:example_architectures} et les représentations intermédiaires dans la figure~\ref{resu:fig:visualization}. Notons comment les représentations intermédiaires sont adaptées à diverses résolutions et concepts abstraits~\citep{zeiler2014visualizing}, par exemple dans la figure~\ref{resu:fig:visualization} nous pouvons dire que la couche 2 est spécialisée pour détécter des coins et des bords tandis que la couche 5 est spécialisée pour des objets entiers. 
\begin{figure}[ht]
    \begin{center}
    \includegraphics[height=0.5\textheight]{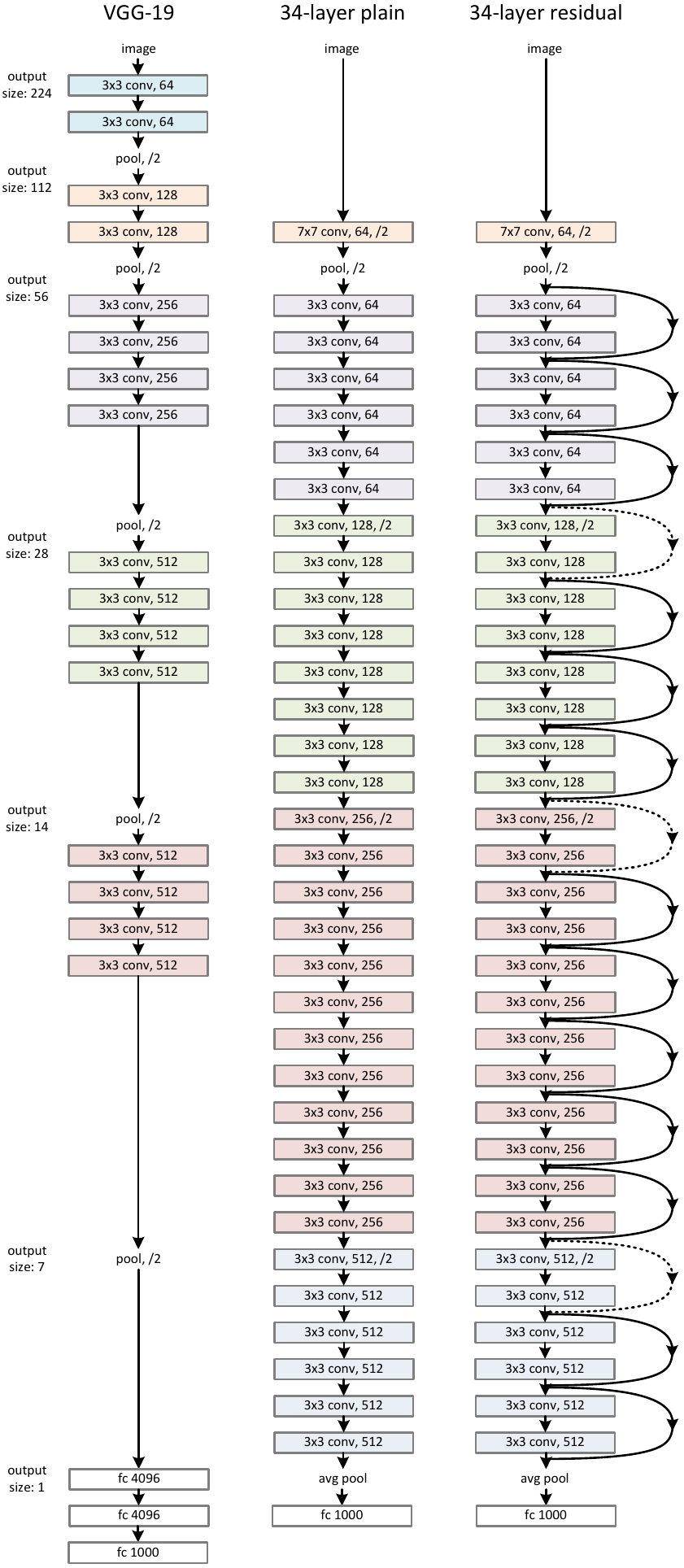}
    \caption{Examples de architectures de DNNs. \textbf{Gauche}: VGG-19~\citep{simonyan2014very} (19.6 milliards de FLOPs). \textbf{Centre}: un réseau avec 34 couches (3.6 milliards de FLOPs). \textbf{Droite}: un réseau résiduel~\citep{he2016deep} avec 34 couches (3.6 milliards FLOPs). Image retiré de~\citep{he2016deep} ©2016 IEEE.}
    \label{resu:fig:example_architectures}
    \end{center}
\end{figure}
            
\begin{figure}[ht]
    \begin{center}
      \begin{subfigure}[ht]{.49\linewidth}
        \includegraphics[width=\linewidth]{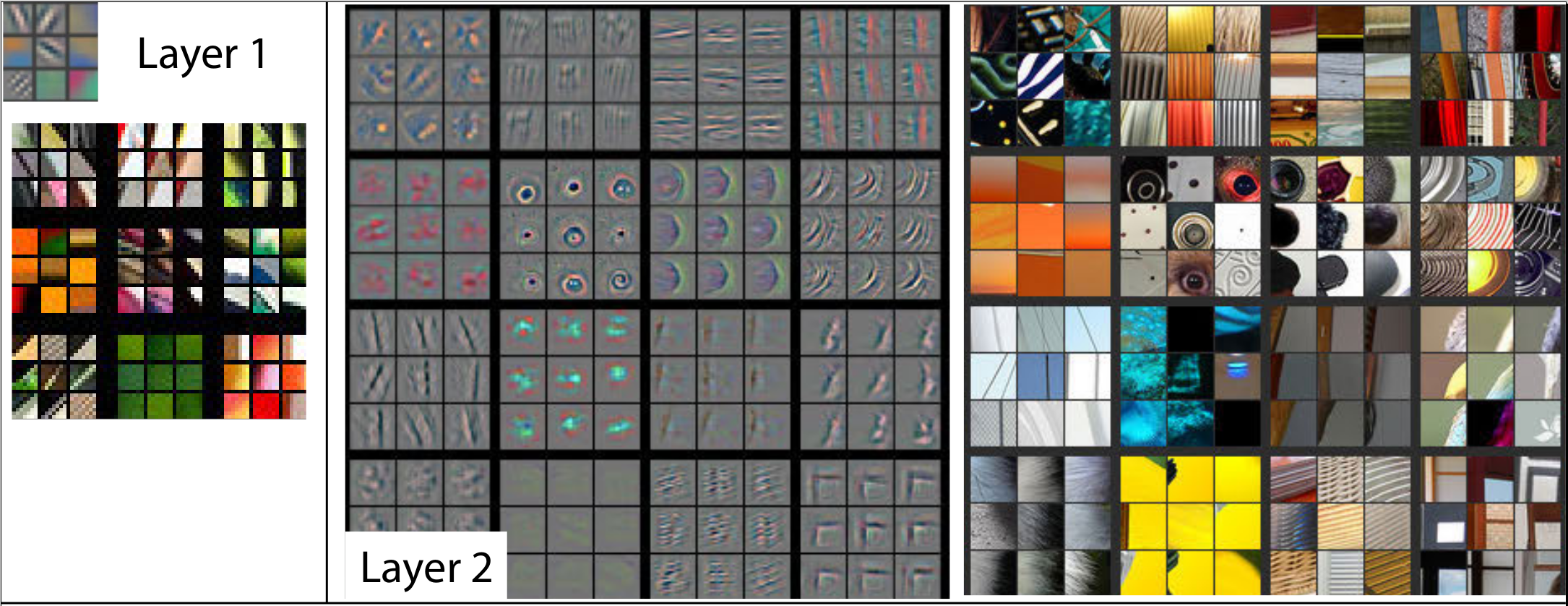}
      \end{subfigure}
      \begin{subfigure}[ht]{.49\linewidth}
        \includegraphics[width=\linewidth]{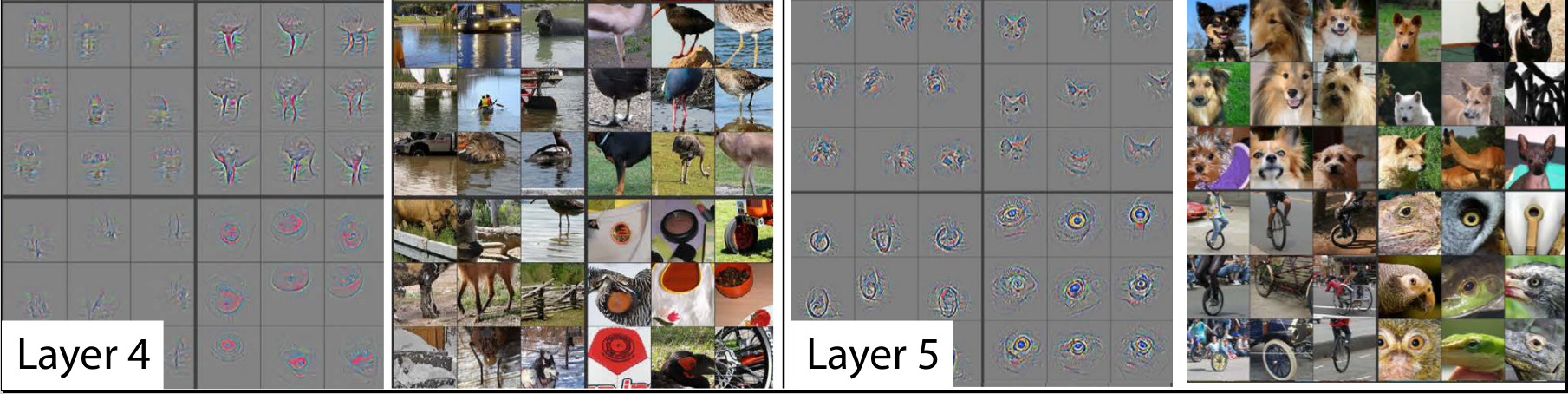}
      \end{subfigure}
      \caption{Visualisation des caractéristiques dans un modèle entrainé. Pour les couches (2,4,5), les 9 principales activations sont représentées à partir d'un sous-ensemble aléatoire de cartes d'éléments dans les données de validation, projetées dans l'espace des pixels en utilisant une approche de réseau déconvolutif. Les reconstructions sont des échantillons du modèle : il s'agit de modèles reconstruits à partir de l'ensemble de données de validation qui provoquent des activations élevées dans une carte de caractéristiques donnée. Pour chaque carte de caractéristiques, nous montrons également les patchs d'images correspondants. On note : (i) le regroupement important dans chaque carte de caractéristiques, (ii) une plus grande invariance dans les couches supérieures et (iii) l'exagération des parties discriminantes de l'image. Il est préférable de visualiser l'image sous forme numérique. Image et legende adaptées de~\citep{zeiler2014visualizing}.}\label{resu:fig:visualization}
    \end{center}
  \end{figure}    

\subsection*{Contexte et motivation}

Comme nous l'avons dit précédemment, les architectures d'apprentissage profond sont capables d'atteindre l'état de l'art dans de nombreux défis dans le domaine de l'apprentissage machine. Elles le font parce qu'elles sont capables d'exploiter la quantité colossale d'informations disponibles. Elles sont souvent présentées comme un cas extrême de méthodes basées sur les données, où il n'a pas de connaissance sur la forme de la fonction à trouver. En tant que telles, elles souffrent de quelques incovénients:
\begin{enumerate}
    \item Elles contiennent de nombreux paramètres qui sont réglés à l'aide de routines \\ d'optimisation complexes qui dépendent à la fois de leur initialisation et des données d'entraînement. En conséquence, elles sont souvent déployées comme des boîtes noires associant des entrées à des sorties. Il existe peu de théories capables de fournir des résultats exploitables sur les mécanismes de ces boîtes noires;
    \item Il est tout à fait habituel d'observer une optimalité de pareto entre la complexité des modèles et les performances sur les tâches considérées. Autrement dit, pour atteindre l'état de l'art, les modèles nécessitent un grand nombre de paramètres, de calculs et de mémoire~\citep{hacene2019processing};
    \item Il est juste de dire que les méthodes d'apprentissage profond ont connu un grand succès grâce à leurs performances expérimentales. Les modèles proposés ont connu plusieurs générations de complexification depuis le renouvellement du domaine au début des années 2010. Il existe donc un écart croissant entre ce que la théorie de l'apprentissage profond peut expliquer et ce que les solutions pratiques actuelles mettent en œuvre pour résoudre les problèmes.
\end{enumerate}

\subsubsection*{DNN - une «boîte noire»}

Comme nous l'avons présenté au tout début de ce résumé, le changement de paradigme consistant à passer de caractéristiques et de modèles triés sur le volet à des architectures d'apprentissage profond a été le principe directeur des recherches récentes dans le domaine. Si des caractéristiques conçues par un expert humain sont considérées comme bien maitrisées ou interprétables, les DNN n'ont de leur côté aucun contrôle explicite, ce qui conduit à un très haut degré de liberté et à des solutions qui sont basées à 100\% sur des données. De manière empirique, il a été constaté que les DNN ont tendance à dépasser les performances des systèmes experts au fur et à mesure que la quantité de données disponibles augmente.

Si cela a conduit à diverses améliorations de l'apprentissage machine et a permis à l'apprentissage profond d'être l'état de l'art de la plupart des tâches d'apprentissage machine, cela entraîne divers inconvénients. Par exemple, dans les domaines où l'interprétabilité est essentielle, comme le domaine médical~\citep{miotto2018deep}, la précision des modèles d'apprentissage profond pourrait ne pas être un argument assez fort pour permettre son adoption mondiale. Ouvrir la «boîte noire» et comprendre les mécanismes sous-jacents de chaque étape d'une architecture d'apprentissage profond est un problème encore ouvert que nous abordons partiellement dans cette thèse. 

\subsubsection*{Quantité de paramètres et de calculs}

Un autre problème qui découle de l'utilisation des DNN est la quantité de paramètres et les exigences en calculs. Les physiciens citent fréquemment la célèbre déclaration de von Neumann : ``... with four parameters I can fit an elephant, and with five I can make him wiggle his trunk.''~\citep{dyson2004meeting} pour soutenir que les modèles d'apprentissage machine ont tendance à être sur-paramétrés. D'un autre côté, nous avons maintenant des praticiens de l'apprentissage profond qui dépassent fréquemment les millions et parfois même les milliards de paramètres~\citep{brown2020language}. Il est ainsi très compliqué de gérer la liberté dont disposent les modèles d'apprentissage profond et de faire en sorte qu'ils apprennent le comportement souhaité.

En effet, si les éléments constitutifs des architectures d'apprentissage approfondi sont des fonctions très simples, la quantité totale de paramètres rend très difficile l'interprétation exacte de ce qui est traité dans un DNN. En outre, la complexité de calcul de ces modèles s'est accrue, ce qui nécessite non seulement du matériel spécialisé (comme les GPU) mais aussi une grande consommation d'énergie. L'étude de modèles d'apprentissage profond et la réduction de la quantité de paramètres et de calculs nécessaires sont non seulement nécessaires du point de vue des connaissances sous-jacentes (par exemple, l'interprétabilité et la robustesse), mais aussi en raison de problèmes sociétaux plus complexes tels que le coût environnemental du matériel et de l'énergie nécessaires pour développer et utiliser ces modèles, ainsi que pour rendre les systèmes d'apprentissage profond accessibles à la plupart des gens.

\subsubsection*{Manque de compréhension théorique}

Comme la plupart des recherches dans le domaine de l'apprentissage profond sont fortement axées sur les applications, l'expérimentation et les résultats de référence sont bien souvent la partie la plus importante des articles récents. Bien qu'il soit très important de comparer les différentes méthodes utilisées dans la littérature, en particulier dans un domaine qui évolue aussi rapidement que l'apprentissage approfondi, cela peut présenter plusieurs inconvénients: 
\begin{enumerate}
\item Les améliorations peuvent provenir de petits détails de mise en œuvre plutôt que de la théorie sous-jacente. Considérons par exemple les problèmes de reproductibilité dans le domaine de l'apprentissage du renforcement~\citep{henderson2018deep}, où il est parfois impossible de reproduire un résultat sans regarder l'implémentation directe du code au lieu de regarder simplement l'article;
\item Les améliorations peuvent ne pas être en accord avec la théorie qui leur a été présentée, par exemple les couches de «batch normalization» (BN)~\citep{ioffe2015batch} ont été proposées pour traiter le décalage des covariables («covariate shift»), ce que d'autres chercheurs soutiennent qu'elles ne sont pas adaptées pour traiter~\citep{santurkar2018does, zhang2018residual, de2020batch}, mais les couches BN sont toujours la pierre fondamentale de diverses architectures;
\item Des comparaisons injustes entre les méthodes peuvent se produire en raison de la combinaison de différentes méthodes ou de la référence elle-même. Par exemple, dans l'apprentissage des métriques, il a été démontré que les méthodes traditionnelles peuvent être plus performantes que les méthodes plus récentes si elles sont correctement formées (c'est-à-dire dans des conditions d'égalité avec les méthodes récentes)~\citep{roth2020revisiting}.
\end{enumerate}

\subsection*{Graphes pour représenter les espaces latents des réseaux neuronaux profonds}

En résumé, les lacunes présentées dans les paragraphes précédents peuvent être considérées comme provenant de la principale force de l'apprentissage profond, à savoir : le fait que les modèles sont capables d'utiliser leur grande liberté pour apprendre des fonctions très complexes, alors que le fonctionnement sous-jacent de chaque étape individuelle n'est pas compris ou explicitement contrôlé. Dans cette thèse, nous proposons d'attaquer ces inconvénients en nous concentrant sur l'étude des représentations intermédiaires dans les DNN. En effet, l'étude des représentations intermédiaires peut être considérée comme une «ouverture de la boîte noire des DNN» et vise à mieux comprendre ce qui est traité à chaque étape. Notez que l'étude des représentations intermédiaires n'augmente pas la complexité du calcul (il faut de toute façon les calculer) et il a déjà été démontré qu'elles contiennent des informations importantes, par exemple la compression des DNN par la distillation des connaissances~\citep{hinton2014distillation,romero2015fitnets}.

L'étude des représentations intermédiaires et de l'effet global de chaque couche intermédiaire devrait permettre une compréhension plus fine des différentes propriétés des DNN, telles que la robustesse~\citep{lassance2019robustness} et la généralisation globale~\citep{gripon2018insidelook}. Dans cette thèse, nous nous concentrons sur l'étude des représentations latentes/intermédiaires des DNN. Comme nous l'avons vu dans les paragraphes précédents, les principaux inconvénients de l'apprentissage profond viennent de sa force centrale : le fait qu'il est capable d'exploiter pleinement les données disponibles sans aucune contrainte forte. Cela tend à conduire à des réseaux très complexes où il est difficile de comprendre à quoi chaque partie est destinée, ainsi qu'à des difficultés pour évaluer si la fonction apprise est une bonne approximation de celle qui est visée.

Pour contrer cet inconvénient, nous analysons et proposons de nombreuses méthodes qui exploitent les connaissances intrinsèques des représentations intermédiaires des DNN. Dans la suite de ce document, nous présentons les définitions nécessaires pour comprendre les domaines explorés de l'apprentissage profond et les méthodes que nous proposons. En particulier, nous nous concentrons sur trois de ces domaines \begin{inlinelist} \item apprentissage de la représentation et du transfert des éléments \item compression des architectures \item surapprentissage (généralisation et robustesse). \end{inlinelist} Afin d'effectuer l'analyse nécessaire et de concevoir les méthodes, nous utilisons le cadre du \textbf{traitement de signal sur graphe (TSG)}~\citep{shuman2013emerging}. 

Nous avons choisi le cadre du TSG car il étend l'analyse harmonique traditionnelle aux domaines irréguliers représentés par des graphes. Les graphes présentent un avantage unique dans le domaine de l'apprentissage profond car ils exploitent les relations des données elles-mêmes. Cela est très conforme à la philosophie de l'apprentissage profond où les données sont essentielles. Par conséquent, l'utilisation des graphes nous permet d'étudier les DNN et leurs représentations intermédiaires car elle fournit un support pour les relations qui sont générées à chaque représentation intermédiaire. Cela facilite l'étude des représentations intermédiaires des DNN, car nous pouvons examiner les données et leurs relations au lieu de l'espace irrégulier de grande dimension. 

Afin d'illustrer et de donner une idée de ce que sont les graphes d'une représentation intermédiaire, nous décrivons et illustrons un exemple dans les paragraphes suivants.

\subsubsection*{Exemple: Illustration des graphes de répresentation intermédiaire}

Considérons que nous avons un DNN déjà entrainé sur un ensemble de données. Nous construisons trois graphes de similarité où les sommets correspondent aux échantillons et les arêtes relient les échantillons les plus similaires. Nous le construisons en utilisant un petit sous-ensemble de l'ensemble de données. Le premier graphe utilise les représentations de l'espace initial (l'espace des images) et les deux derniers utilisent les représentations intermédiaires du DNN. Ces représentations proviennent d'une couche intermédiaire et d'une des couches finales. Sur le plan qualitatif, nous nous attendons à ce que les échantillons qui appartiennent à une même classe soient plus faciles à séparer à mesure que nous nous enfonçons dans l'architecture considérée. Ceci serait en accord avec la définition citée de~\citep{lecun1995convolutional}. Nous décrivons cet exemple dans la figure~\ref{resu:fig-example}. Comme prévu, nous pouvons voir qualitativement la différence de séparation entre l'espace image et les espaces latents, comme on peut le constater par la quantité d'arêtes entre les éléments de classes distinctes (nous gardons le même nombre $k$ de voisins sur chaque graphe) et aussi par la séparation géométrique lorsque nous utilisons des «Laplacian eigenmaps»~\cite{belkin2003laplacian} pour placer les différents échantillons dans un espace 2D régulier.

\begin{figure}[ht]
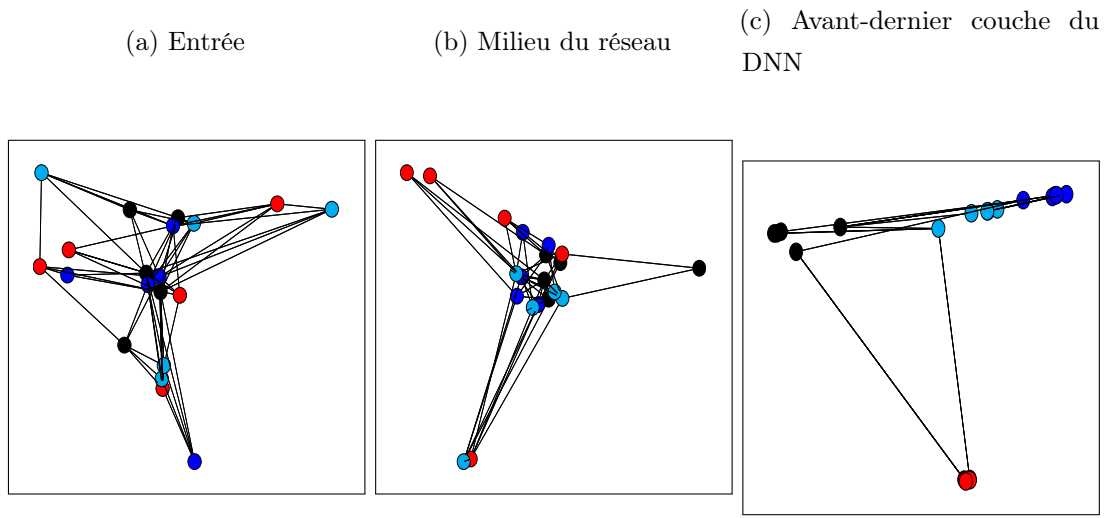

    \begin{center}
        \begin{subfigure}[ht]{0.32\linewidth}
            \centering
            \caption{Entrée}
            \vspace{0.5cm}
            \begin{framed}
                \begin{adjustbox}{width=\linewidth, height=\linewidth}
  \tikzsetnextfilename{chapter1/tikz/example_input}%
  \input{chapter1/tikz/example_input.tex}%

                \end{adjustbox}
            \end{framed}
        \end{subfigure}
        \begin{subfigure}[ht]{0.32\linewidth}
            \centering
            \caption{Milieu du réseau}
            \vspace{0.5cm}
            \begin{framed}
                \begin{adjustbox}{width=\linewidth, height=\linewidth}
  \tikzsetnextfilename{chapter1/tikz/example_middle}%
  \input{chapter1/tikz/example_middle.tex}%

                \end{adjustbox}
            \end{framed}
        \end{subfigure}
        \begin{subfigure}[ht]{0.32\linewidth}
            \centering
            \caption{Avant-dernier couche du DNN}
            \vspace{0.5cm}
            \begin{framed}
                \begin{adjustbox}{width=\linewidth,height=\linewidth}
  \tikzsetnextfilename{chapter1/tikz/example_output}%
  \input{chapter1/tikz/example_output.tex}%

                \end{adjustbox}
            \end{framed}
        \end{subfigure}
    \end{center}
    \caption{Exemple de répresentation sous forme de graphes, de l'espace des entrées (gauche) à l'avant dernière couche du DNN (droite). Les différentes couleurs des sommets représentent la classe de l'objet. Pour faciliter la visualisation, nous ne représentons que les arêtes entre des exemples de classes distinctes. Notez qu'il y a beaucoup plus de bords à l'entrée (a) et que le nombre de bords diminue au fur et à mesure que l'on s'enfonce dans l'architecture (b et c). }
    \label{resu:fig-example}
\end{figure}

Notez que ces représentations illustrent clairement le principe de démêlage dans les réseaux neuronaux profonds. En effet, les DNN peuvent être considérés comme une cascade d'opérations qui transforment l'espace d'entrée dans lequel les échantillons d'une même classe sont mélangés avec d''autres, en espaces latents qui sont progressivement mieux alignés sur la tâche considérée à laquelle le DNN a été entrainé.

\subsection*{Contributions}

Nous considérons que cette thèse comporte trois types de contributions différentes. Premièrement, nous avons cherché à rendre toutes les productions de la thèse aussi ouvertes/libres que possible. Cela nous est très cher, car le but de cette thèse n'était pas de générer une ``application de premier ordre'' ou une preuve de concept, mais d'étudier les orientations de recherche que nous pensons être d'intérêt pour la communauté. Pour atteindre cet objectif, nous avons mis à disposition la plupart de notre production textuelle (par exemple des articles) soit sur des sites d'archives bien connus (tels que \url{arxiv.org}), soit sur les pages web personnelles des auteurs. Nous avons également mis à disposition, lorsque cela était possible, le code utilisé pour les expériences et les preuves de concepts sur le site de contrôle de version \url{github.com}. Nous fournissons une liste des contributions à la fin de cette section.

Deuxièmement, nous avons fait un effort pour communiquer nos résultats et diffuser les connaissances via des présentations, via la conception de cours et via l'enseignement. Nous avons présenté nos résultats lors de diverses conférences et ateliers internationaux et nationaux, afin de promouvoir et de discuter de nos résultats avec la communauté scientifique. Nous avons également consacré une partie de la thèse à l'enseignement et à la conception de cours dans les domaines de la théorie des graphes et de l'apprentissage automatique, y compris les cours ouverts massifs en ligne (MOOC). En fait, au cours de mon doctorat, j'ai eu l'occasion de contribuer à la création de deux cours. Le premier est un MOOC intitulé «Advanced Algorithmics and Graph Theory with Python», disponible sur le platorm EdX et qui a rassemblé plus de 10 000 étudiants de plus de 50 pays depuis son lancement en 2018. Le second est un cours d'introduction à l'IA moderne qui est conçu pour les étudiants de l'IMT Atlantique. Dans les deux cas, j'ai participé à la fois à la conception générale et aux parties techniques des cours.

Enfin, nous nous sommes consacrés à l'étude des technologies qui, selon nous, devraient contribuer à la société dans son ensemble. Prenons par exemple deux des domaines que nous avons étudiés : la compression et la robustesse des DNN. Dans le premier, nous visons à réduire la consommation totale d'énergie (et donc les émissions de carbone) des réseaux neuronaux, indépendamment de la tâche en aval. Dans le second, nous visons à accroître la confiance générale que nous pouvons avoir dans les résultats générés par un DNN.

Dans les paragraphes suivants, nous présentons une liste des contributions de cette thèse :

\begin{itemize}
    \item \bibentry{lassance2018laplacian}
    \item \bibentry{lassance2018matching}
    \item \bibentry{lassance2018predicting}
    \item \bibentry{lassance2019improved}
    \item \bibentry{lassance2019robustness}
    \item \bibentry{lassance2020benchmark}
    \item \bibentry{lassance2020deep}
    \item \bibentry{grelier2018graph}
    \item \bibentry{bontonou2019smoothness}
    \item \bibentry{bontonou2019formalism}
    \item \bibentry{bontonou2019comparing}
\end{itemize}

Dans le suvi de cette résumé, nous décrivons via un sommaire rapide chacun des chapitres de cette thèse.

\section*{Chapitre 2: Concepts en apprentissage profond}\addcontentsline{toc}{section}{Chapitre 2: Concepts en apprentissage profond}

Dans ce premier chapitre, nous présentons les réseaux neuronaux profonds (DNN), en nous concentrant plus particulièrement sur les architectures résiduelles. Nous introduisons également le concept de représentations intermédiaires dans les DNN, qui sera un élément central des chapitres suivants de cette thèse. Ces représentations intermédiaires peuvent être utilisées afin d'effectuer un apprentissage par transfert, tel que présenté dans la section~\ref{chap2:feature_extraction} et également afin d'abstraire les DNN en tant qu'extracteurs de caractéristiques suivis de classificateurs.

Nous présentons divers problèmes pour lesquels les DNN sont pertinents. Ces problèmes vont être étudiés plus en détail dans les prochains chapitres et comprennent les tâches suivantes: \begin{inlinelist} \item localisation basée sur la vision \item classification des images \item classification des tâches neurologiques \item classification des documents \end{inlinelist}. 

Nous faisons également une revue de littérature sur la compression des réseaux de neurones, notamment les méthodes de distillation et les couches de convolution plus efficaces. Nous présentons SAL, une contribution sur le sujet des couches de convolution efficaces, qui a fait l'objet de l'article de conférence suivant:
\begin{itemize}
\item \bibentry{hacene2019attention}
\end{itemize}

En outre, nous introduisons le concept de robustesse d'un classificateur, et nous démontrons empiriquement comment il peut être lié à la capacité de bien fonctionner en présence d'entrées corrompues. Ce concept de robustesse et ses expériences empiriques ont été publiés dans l'article de conférence suivant:
\begin{itemize}
\item \bibentry{lassance2019robustness}
\end{itemize}

\section*{Chapitre 3: Concepts en traitment des signaux sur graphe}\addcontentsline{toc}{section}{Chapitre 3: Concepts en traitment de signaux sur graphe}

Dans ce chapitre, nous introduisons les concepts de graphes et de signaux de graphes, ainsi que les outils nécessaires du cadre du traitement des signaux de graphes (TSG). Ces concepts et outils nous permettent d'analyser les représentations latentes profondes et d'en tirer de nouvelles contributions destinées à la communauté d'apprentissage machine et qui seront présentées dans les chapitres suivants.

Parmi les outils présentés dans cette section figurent la transformée de Fourier sur graphe (GFT) et l'analyse de la fluidité des signaux de graphes. Nous abordons également les méthodes permettant de déduire des graphes à partir de données pour lesquelles la structure de support des graphes n'est pas explicitement disponible, y compris une nouvelle contribution :

\begin{itemize}
  \item \bibentry{lassance2020benchmark}
\end{itemize}

Nous dérivons ensuite des filtres de graphe, qui émulent les filtres traditionnels de traitement du signal dans le domaine des graphes. Ces filtres de graphe seront utilisés pour relier les couches convolutives et les couches convolutives de graphes dans le prochain chapitre. Nous présentons également deux applications de filtres de graphes qui nous permettent de réduire la quantité de bruit des éléments extraits à l'aide de DNN et d'améliorer les performances des tâches en aval, notamment : l'apprentissage avec peu d'exemples ; la classification des images ; la localisation visuelle (VBL) et l'extraction d'images (IR). L'application de localisation visuelle a fait l'objet d'une contribution :
  
  \begin{itemize}
    \item \bibentry{lassance2019improved}
  \end{itemize}

\section*{Chapitre 4: Réseaux de neurones profonds pour des signaux sur graphe}\addcontentsline{toc}{section}{Chapitre 4: Réseaux de neurones profonds pour des signaux sur graphe}

Dans ce chapitre, nous approfondissons le domaine des réseaux neuronaux profonds définis sur des graphes. Nous nous sommes appuyés sur les concepts des chapitres précédents afin de définir les méthodes récentes dans un cadre de filtrage sur graphes unique que nous avons présenté par ordre croissant de complexité dans la section~\ref{chap4:definitions}. Bien que ce cadre ne soit pas exactement nouveau, nous l'avons étendu à d'autres méthodes et avons introduit une discussion sur les inconvénients de ces méthodes.

Nous avons ensuite discuté des applications des DNN définis sur les graphes dans le contexte de la classification supervisée des signaux des graphes dans la section~\ref{chap4:classify_graph_signals}. Nous avons discuté des contributions récentes qui montrent les inconvénients des approches actuelles dans ce domaine et avons ensuite présenté deux de nos contributions. Leur objectif est de combler l'écart entre les convolutions des graphes et les convolutions 2D/3D classiques. Nos deux contributions introduites ont été publiées dans des conférences :
\begin{itemize}
    \item \bibentry{grelier2018graph}
    \item \bibentry{lassance2018matching}
\end{itemize}

Enfin, nous discutons des applications dans le contexte de la classification semi-supervisée des sommets d'un graphe. Nous discutons d'abord du problème de l'évaluation équitable des différentes méthodes GNN sur cette tâche. Bien qu'il ne s'agisse pas d'un problème nouveau dans le domaine, les travaux récents présentent encore les deux écueils les plus courants : \begin{inlinelist} \item l'utilisation d'une seule répartition train/validité/essai dont il a déjà été démontré qu'elle faussait les résultats~\citep{shchur2018pitfalls} \item des expériences ne comparant pas équitablement les méthodes, par exemple, la méthode A est plus performante que la méthode B, mais cela est principalement dû à l'ajout du dropout plutôt qu'à la méthode elle-même\end{inlinelist}. Notez que ces problèmes ne sont pas nécessairement dus à une faute de connaissance ou à une malveillance, mais surtout à un simple problème de quantité des calcul qui seraient nécessaires pour tout exécuter correctement. En effet, nous proposons un cadre afin de résoudre à la fois les problèmes i) et ii), mais nous montrons que nous ne pourrions jamais exécuter la version optimale dans un délai raisonnable. Nous avons ajouté un cadre plus souple et présentons nos résultats sur l'ensemble de données de Cora.

\section*{Chapitre 5: Répresentations latentes de réseaux profonds sur graphe}\addcontentsline{toc}{section}{Chapitre 5: Répresentations latentes de réseaux profonds sur graphe}

Dans ce chapitre nous présentons principalement nos contributions dans le domaine de l'utilisation de graphes pour représenter les espaces latents de réseaux de neurones profonds. Bien que ce domaine ne soit pas très développé, nous espérons que nos contributions pourront apporter un éclairage et permettre de le développer davantage, car nous pensons qu'il y a beaucoup de contributions intéressantes à poursuivre. 

Nous présentons d'abord le travail qui a été le début de notre intérêt pour le domaine~\citep{gripon2018insidelook}, dans lequel les auteurs ont montré qu'il était possible de caractériser différents comportements de DNN en analysant l'évolution de la fluidité d'un signal sur un graphe. Nous nous sommes ensuite appuyés sur ces travaux pour proposer une mesure qui est empiriquement corrélée avec la performance de généralisation des DNN. Cette mesure a fait l'objet d'une contribution à une conférence :

\begin{itemize}
    \item \bibentry{lassance2018predicting}
\end{itemize}

Nous nous sommes ensuite concentrés sur les utilisations possibles de la fluidité du signal sur le graphe pendant la formation des réseaux de neurones. Nous avons d'abord montré que nous sommes capables de former de bons extracteurs de caractéristiques en entraînant le réseau à minimiser la fluidité des signaux des indicateurs d'étiquettes sur les graphes générés par leurs sorties. Cette nouvelle fonction objectif possède trois caractéristiques importantes qui ne sont pas présentes dans l'entropie croisée et nous démontrons à l'aide d'expériences que nous sommes capables d'obtenir des réseaux plus robustes, sans perdre trop de performance de généralisation. Deuxièmement, nous proposons d'utiliser un régularisateur afin de contrôler l'évolution de la fluidité des signaux indicateurs des étiquettes sur les graphes qui sont générés par les représentations intermédiaires des DNN. Nous montrons que ces régulariseurs sont non seulement théoriquement conformes à notre définition de la robustesse (Definition~\ref{chap2:def_robustness}), mais aussi que nous pouvons démontrer empiriquement leur efficacité lorsqu'ils sont comparés (ou ajoutés) à d'autres méthodes dans la littérature. Ces deux utilisations de la fluidité d'un signal sur graphe ont fait l'objet de contributions, l'une à une conférence et l'autre est en cours d'examen dans une revue :

\begin{enumerate}
    \item \bibentry{bontonou2019smoothness}
    \item \bibentry{lassance2018laplacian}
\end{enumerate}

Enfin, nous présentons une méthode qui ne s'appuie pas sur le cadre du TSG, mais qui nous permet d'utiliser le cadre du TSG sur des techniques préalablement définies. En d'autres termes, nous avons spécialisé le cadre RKD en GKD, dont nous avons démontré empiriquement et analytiquement qu'il permettait d'améliorer les performances des réseaux compressés. Ce travail d'introduction a été publié lors d'une conférence :

\begin{itemize}
    \item \bibentry{lassance2020deep}
\end{itemize}

\section*{Conclusion}

L'idée principale que nous avons poursuivie au cours des trois dernières années était de remédier à certaines lacunes des architectures d'apprentissage profond en examinant leurs représentations intermédiaires. Pour effectuer nos analyses, nous avons utilisé le cadre du traitement du signal des graphes, dans lequel les graphes sont utilisés pour représenter la topologie d'un domaine complexe (ici : les espaces latents des architectures d'apprentissage profond). Nous avons considéré les applications d'apprentissage profond dans trois domaines d'apprentissage machine : \begin{inlinelist} \item apprentissage et transfert des représentations \item compression des architectures d'apprentissage profond \item étude du sur-apprentissage (generalisation et robustesse) \end{inlinelist}.
\mainmatter

\chapter{Introduction}\label{chap1}
\localtableofcontents

In recent years, Deep Neural Networks (DNN) have exploded in popularity, creating a new domain called "Deep Learning"~\citep{dlbook}. While both the concept of neural networks~\citep{rosenblatt1958perceptron} and DNNs~\citep{rumelhart1986learning} are quite old, they only started gaining popularity in recent years. This change was due to both advances in hardware, specially Graphic Processing Unit (GPUs)~\citep{hacene2019processing} and with the first victories in computer vision challenges such as AlexNet~\citep{krizhevsky2012imagenet} winning the 2012 CVPR \textbf{L}arge \textbf{S}cale \textbf{V}isual \textbf{R}ecognition \textbf{C}hallenge (LSRVC-Imagenet)~\citep{russakovsky2015imagenet} and DanNet~\citep{cirecsan2013mitosis} winning the Contest on Mitosis Detection in Breast Cancer Histological Images of ICPR 2012~\citep{roux2013mitosis}. Figure~\ref{fig:example_alexnet_dannet} depict example images from these competitions. 

\begin{figure}[ht]
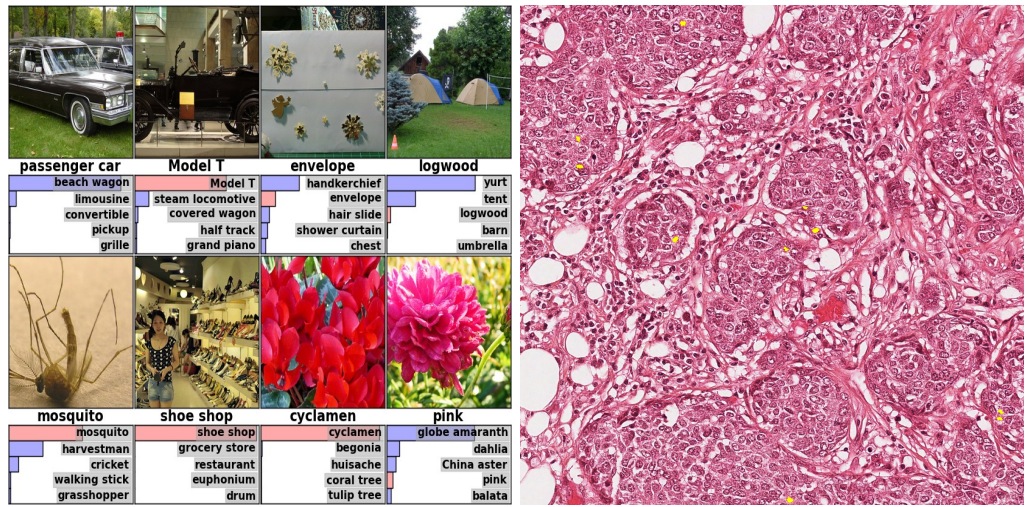

\begin{center}
\includegraphics[width=0.45\columnwidth,height=0.45\columnwidth]{chapter1/AlexNetExamples.jpg}
\includegraphics[width=0.45\columnwidth,height=0.45\columnwidth]{chapter1/DanNetExample.jpg}
\end{center}
\caption{\textbf{Left:}Eight ILSVRC-2010 test images and the five labels considered most probable by AlexNet. The correct label is written under each image, and the probability assigned to the correct label is also shown with a red bar (if it happens to be in the top 5). Image extracted from~\citep{krizhevsky2012imagenet} @2012 Neural Information Processing Systems Foundation. \textbf{Right} Example image from~\citep{roux2013mitosis} for mitosis detection, where the yellow parts represent the parts that should be detected by the DNN}\label{fig:example_alexnet_dannet}
\end{figure}

These networks were built upon two main principles: \begin{inlinelist} \item convolutional priors \item hierarchical backpropagation-learned representations \end{inlinelist}. The former guides the base form of the network, in order to enforce invariance to shifts and weight sharing. The latter informs that the network should receive the input as-is and it is thus the job of the optimization method to adapt the network to a feature extractor followed by a classifier, without any specific control of the network evolution. Indeed, in~\citep{lecun1995convolutional} the authors say: ``A potentially more interesting scheme is to eliminate the feature extractor, feeding the network with ``raw'' inputs (e.g. normalized images), and to rely on backpropagation to turn the first few layers into an appropriate feature extractor". Therefore, we can look at the processing of an input in a DNN as the generation of a sequence of \textbf{intermediate representations} that are part of the \textbf{latent spaces} of the DNN.

To better understand what we call a DNN and an intermediate representation, let us illustrate these concepts. We first depict in Figure~\ref{fig:example_architectures} some typical deep neural networks. Note how they tend to follow a mostly sequential structure, with few shortcuts.

\begin{figure}[ht]
    \begin{center}
    \includegraphics[height=0.5\textheight]{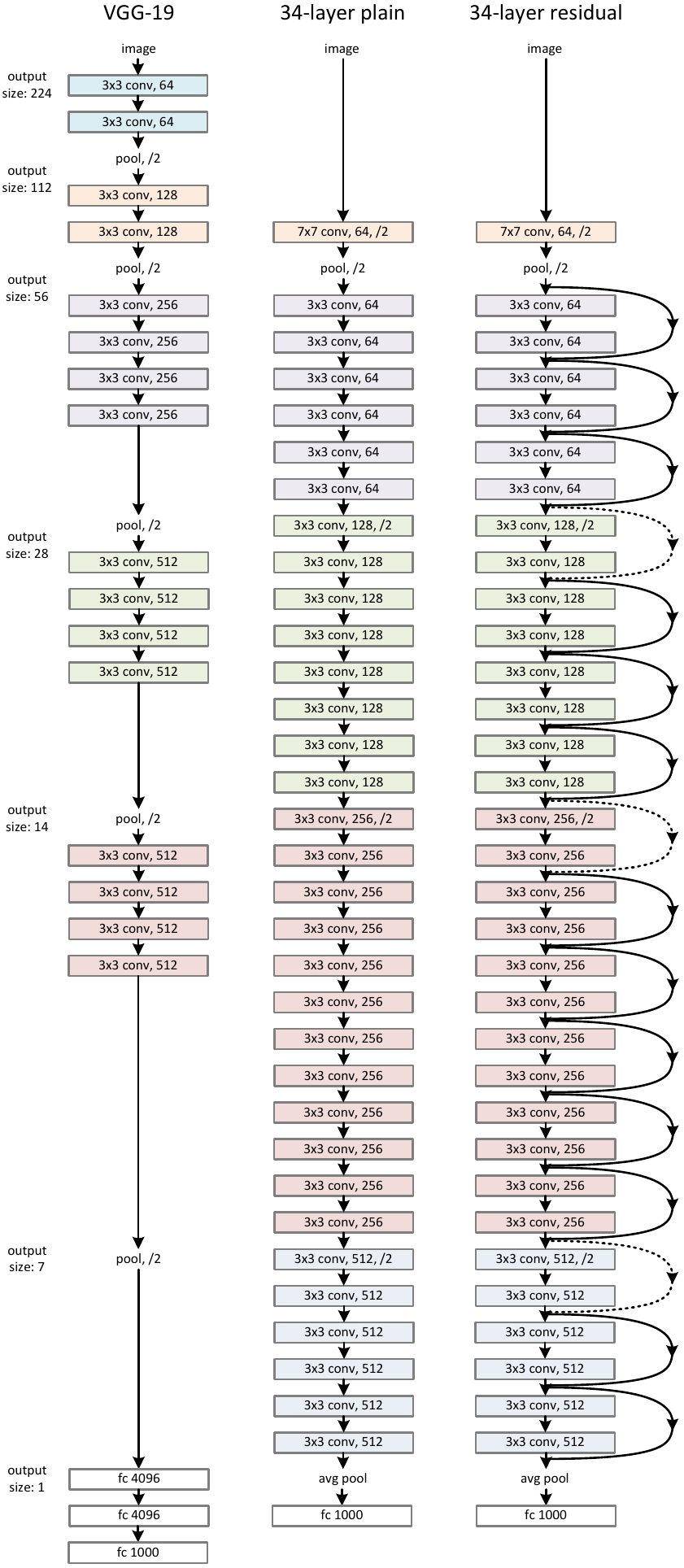}
    \caption{Examples of network architectures. \textbf{Left}: the VGG-19 model~\citep{simonyan2014very} (19.6 billion FLOPs). \textbf{Middle}: a plain network with 34 parameter layers (3.6 billion FLOPs). \textbf{Right}: a residual network~\citep{he2016deep} with 34 parameter layers (3.6 billion FLOPs). Image extracted from~\citep{he2016deep} ©2016 IEEE.}
    \label{fig:example_architectures}
    \end{center}
\end{figure}
            
Now, in Figure~\ref{fig:visualization} we depict the intermediate representation evolution from one layer to the next in the same architecture. Note how the intermediate representations are adapted to various abstract resolutions and concepts~\citep{zeiler2014visualizing}, for example in Figure~\ref{fig:visualization} we can say that layer 2 responds to corners/edges while layer 5 responds to entire objects. 

\begin{figure}[ht]
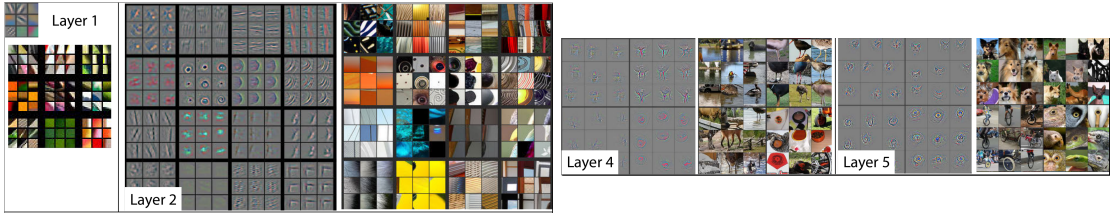

    \begin{center}
      \begin{subfigure}[ht]{.49\linewidth}
        \includegraphics[width=\linewidth]{chapter1/Visualizations1.PNG}
      \end{subfigure}
      \begin{subfigure}[ht]{.49\linewidth}
        \includegraphics[width=\linewidth]{chapter1/Visualizations2.PNG}
      \end{subfigure}
      \caption{Visualization of features in a fully trained model. For layers (2,4,5) the top 9 activations are shown from a random subset of feature maps across the validation data, projected down to pixel space using a deconvolutional network approach. The reconstructions are {\em not} samples from the model: they are reconstructed patterns from the validation set that cause high activations in a given feature map. For each feature map, we also show the corresponding image patches. Note: (i) the strong grouping within each feature map, (ii) greater invariance at higher layers and (iii) exaggeration of discriminative parts of the image. Best viewed in electronic form. Figure and caption adapted from~\citep{zeiler2014visualizing}.}\label{fig:visualization}
    \end{center}
  \end{figure}    

\section{Context and motivation}

As we said in the previous section, deep learning architectures are able to reach state-of-the-art performance in many challenges in the field of machine learning. They do so because they are able to exploit the colossal amount of information contained in the training data. They are often presented as an extreme case of data-driven methods (i.e. a discriminative approach), where very few priors are given about the function to be found.

As such, they suffer from the same shortcomings as most discriminative approaches:
\begin{enumerate}
    \item They contain a lot of parameters that are tuned using complex optimization routines that depend on both their initialization and on the training data. As a consequence, they are often seen as black boxes associating inputs with outputs. There is little theory able to provide exploitable results about the inside of these black boxes;
    \item It is quite usual to observe a pareto optimality between complexity of the models and performance on the considered tasks. Said otherwise, in order to reach state-of-the-art accuracy, models require a huge number of parameters, computations and memory~\citep{hacene2019processing};
    \item It is fair to say that deep learning methods have known a great success thanks to their experimental performance. Proposed models have known several generations of complexifications since the renewal of the domain in the early 2010s. There is therefore an increasing gap between what the theory of deep learning can explain and what current practical solutions implement to solve complex problems.
\end{enumerate}

\subsection{DNN as a black box}

As we have introduced at the very beginning of this document, changing the paradigm from hand-picked features and models to deep learning architectures has been the guiding principle of recent deep learning research. The former is seen as well-behaved or interpretable as the features and models are designed to solve the task in a very specific way. The latter however thrives in the fact that no explicit control is performed\footnote{safe for some priors such as shift-invariance or temporal connections.}, which leads to a very high degree of freedom and to solutions that are 100\% data based. Empirically it has been found that the latter tends to overtake the former as the amount of available data increases.

While this has led to various improvements in machine learning and has allowed deep learning to be the de-facto state-of-the-art of most machine learning tasks, this leads to various drawbacks. For example, in domains where interpretability is key, such as the medical domain~\citep{miotto2018deep}, the accuracy of deep learning models might not be an argument strong enough to allow its global adoption. To ``open'' the black box and understand the underlying mechanisms of each step of a deep learning architecture is still an open problem that we partially tackle in this thesis. 

\subsection{Amount of parameters and computational complexity}

Another problem that arises from the use of DNNs is the amount of parameters and the computational requirements of recent DNN models. Physicists frequently cite the famous von Neumann statement ``... with four parameters I can fit an elephant, and with five I can make him wiggle his trunk.''~\citep{dyson2004meeting} to argue that machine learning models tend to be overparametrized. On the other hand we now have deep learning practicioners frequently surpass the millions and sometimes even billions of parameters~\citep{brown2020language}. Dealing with the amount of freedom that deep learning models have and enforcing that they learn the desired behaviour is very complicated.

Indeed, while the building blocks of deep learning architectures are very simple functions, the total amount of parameters make it very hard to interpret exactly what is being processed inside a DNN. Moreover, the computational complexity of these models have been increasing, which not only requires dedicated hardware (such as GPUs) but a large amount of energy consumption. Studying deep learning models and reducing the amount of parameters and computations needed is not only necessary from an underlying knowledge aspect (e.g. interpretability and robustness) but also due to more complex societal problems such as the environmental cost of material and energy necessary for developing and using those models, as well as making deep learning systems accessible to most.

\subsection{Lack of theoretical understanding}

As most of the research on the domain of deep learning is highly application focused, experimentation and benchmark results are now seen as the most important part of recent papers. While it is very important in order to compare different methods in the literature, especially in a domain that evolves as quickly as deep learning, it may come with several drawbacks: 
\begin{enumerate}
\item Improvements may come from small implementation details instead of the underlying theory. Consider for example the reproducibility concerns in the domain of reinforcement learning~\citep{henderson2018deep}, where sometimes it is not possible to reproduce a result without looking at the direct code implementation instead of just looking at the paper;
\item Improvements may not agree with the theory they were presented with, for example batch normalization (BN)~\citep{ioffe2015batch} layers were proposed to deal with covariate shift, what other researchers argue that they are not suited to address~\citep{santurkar2018does, zhang2018residual, de2020batch}, but the BN layers are still the cornerstone of various architectures;
\item Unfair comparisons between methods may happen due to the combination of different methods or from the benchmark itself. For example, in metric learning it has been shown that traditional methods may outperform more recent methods if they are properly trained (i.e. in equality of conditions with the recent methods)~\citep{roth2020revisiting}.
\end{enumerate}

\subsection{Limitations of deep discriminative models}

In summary, the shortcomings presented in the previous paragraphs can be seen as originating from the main strength of deep learning, that is: the fact that the models are able to use their high amount of freedom to learn very complex functions, while the underlying functioning of each individual step is not understood or explicitly controlled. In this thesis we propose to attack these three drawbacks by concentrating on the study of intermediate representations in DNNs. Indeed, studying the intermediate representations can be seen as ``opening'' the black box of the DNN and aiming to better understand what is being processed at each step. Note that looking at the intermediate representations does not increase the computational complexity (one has to compute them anyway) and it has already been shown that they contain important information, e.g., compression of DNNs via knowledge distillation~\citep{hinton2014distillation,romero2015fitnets}.

Studying the intermediate representations and the overall effect of each intermediate layer should lead to a more fine-grain understanding of different properties of the DNNs, such as robustness~\citep{lassance2019robustness} and overall generalization~\citep{gripon2018insidelook}. In the following paragraphs, we describe how we propose to leverage these representations in order to contribute to deep learning research.

\section{Graphs for deep learning latent representations}

In this thesis, we focus on studying the latent/intermediate representations of DNNs. As we have discussed in the previous paragraphs, the main drawbacks of deep learning come from its central strength: the fact that it is able to fully leverage the available data without any strong constraint. This tends to lead to very complex networks where it is difficult to understand what each part is meant for, as well as difficulties in assessing whether the learned function is a good approximation of the targeted one.

To counter this drawback we analyze and propose many different methods that exploit the intrinsic knowledge from the intermediate representations of DNNs. In the following of this document we present the definitions needed to understand the explored domains of deep learning and the methods we propose. In particular, we focus on three such domains \begin{inlinelist} \item representation and transfer learning \item compression of architectures \item overfitting (generalization and robustness). \end{inlinelist} In order to perform the needed analysis and design the methods we use the framework of \textbf{Graph Signal Processing (GSP)}~\citep{shuman2013emerging}. 

We have chosen the GSP framework as it extends traditional harmonic analysis to irregular domains represented by graphs. Graphs have a unique advantage in the deep learning domain as they exploit the relationships from the data itself. This is very inline with the philosophy of deep learning where data is key. Therefore, using graphs allows us to study the DNNs and intermediate representations as it provides a support for the relationships that are generated at each intermediate representation. This facilitates the study of the intermediate representations of DNNs, as we can look at the data and its relationships instead of the highly-dimensional irregular space.

In order to illustrate and to give an idea of what are intermediate representation graphs, we describe and depict an example in the following paragraphs.

\subsection{Example: Depiction of intermediate representation graphs}

Consider that we have a pre-trained DNN on an image dataset. What does it look like if we create our graph representations in such a scenario? In order to create such a depiction, we construct three similarity graphs where vertices correspond to samples and edges connect samples that are the most similar. We build it by using a small subset of the training dataset. The first graph uses the representations from image space and the latter two use the intermediate representations of the DNN. Such representations come from an intermediate layer and one of the final (end) layers. What we expect to see qualitatively is that the samples that belong to a same class will be easier to separate as we go deeper in the considered architecture. This would be in hand with the quoted definition from~\citep{lecun1995convolutional} in the first page of this thesis. We depict this example in Figure~\ref{fig-example}. As expected, we can qualitatively see the difference in separation from the image space to the latent spaces, as can be noted by the amount of edges between elements of distinct classes (we keep the same number $k$ of neighbors at each graph) and also by the geometric separation when using Laplacian eigenmaps~\cite{belkin2003laplacian} to position the different samples in a regular 2D space.

\begin{figure}[ht]
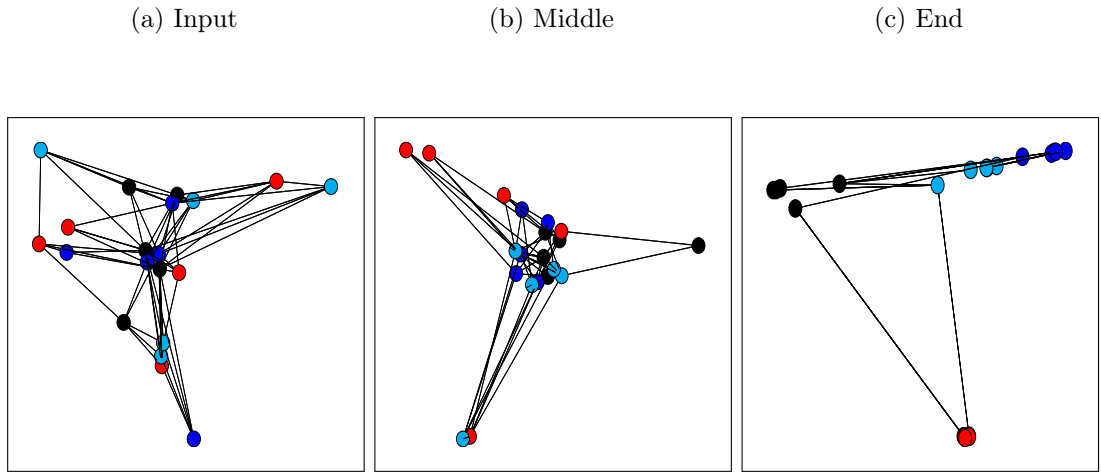

    \begin{center}
        \begin{subfigure}[ht]{0.32\linewidth}
            \centering
            \caption{Input}\label{figure-example-left}
            \vspace{0.5cm}
            \begin{framed}
                \begin{adjustbox}{width=\linewidth, height=\linewidth}
  \tikzsetnextfilename{chapter1/tikz/example_input}%
  \input{chapter1/tikz/example_input.tex}%

                \end{adjustbox}
            \end{framed}
        \end{subfigure}
        \begin{subfigure}[ht]{0.32\linewidth}
            \centering
            \caption{Middle}
            \vspace{0.5cm}
            \begin{framed}
                \begin{adjustbox}{width=\linewidth, height=\linewidth}
  \tikzsetnextfilename{chapter1/tikz/example_middle}%
  \input{chapter1/tikz/example_middle.tex}%

                \end{adjustbox}
            \end{framed}
        \end{subfigure}
        \begin{subfigure}[ht]{0.32\linewidth}
            \centering
            \caption{End}
            \vspace{0.5cm}
            \begin{framed}
                \begin{adjustbox}{width=\linewidth,height=\linewidth}
  \tikzsetnextfilename{chapter1/tikz/example_output}%
  \input{chapter1/tikz/example_output.tex}%

                \end{adjustbox}
            \end{framed}
        \end{subfigure}
    \end{center}
    \caption{Graph representation example, from the input space (left) to the output of the network (right). The different vertex colors represent the class of the object. To help the visualization, we only depict the edges between examples of distinct classes. Note how there are many more edges at the input (a) and how the number of edges decrease as we go deeper in the architecture (b and c). }
    \label{fig-example}
\end{figure}

Note that these representations clearly illustrate the principle of disentangling in deep neural networks. Indeed, DNNs can be thought of as a cascade of operations that smoothly transform the input space in which samples from a same class are likely to be spread among other ones, to latent spaces that progressively are better aligned with the considered task the network is trained upon.

\subsection{Contributions}

We consider this thesis to have three different types of contributions. First, we have seeked to turn all productions of the thesis as open acess/free as possible. This is very dear to us, as the goal of this thesis was not to generate a so called ``top notch application'' or proof-of-concept, but to investigate research directions that we believe to be of interest to the community. To achieve this goal we have made available most of our textual production (e.g. articles) either at well known archival websites (such as \url{arxiv.org}) or at the authors personal webpages. We have also made available, when possible, the code responsible for the experiments and proofs of concepts on the version control site \url{github.com}. We provide a list of contributions at the end of this subsection.

Second, we have done an effort to communicate our results and disseminate knowledge via presentations, course designs and teaching. We have presented our findings in various international and national conferences and workshops, in order to promote and discuss our results with the scientific community. We have also dedicated a part of the thesis to the teaching and course designs in the domains of graph theory and machine learning, including open Massive Online Open Courses (MOOCs). As a matter of fact, during my PhD I had the opportunity to contribute to the creation of two courses. The first one is a MOOC entitled ``Advanced Algorithmics and Graph Theory with Python'' that is available on the EdX platorm and that has gathered more than 10k students from 50+ countries since its launch in 2018. The second one is an introductory course to modern AI that is designed for students at IMT Atlantique. In both cases, I have participated to both the overall design and to the techincal parts of the courses.

Finally, we have dedicated ourselves to the study of technologies that we believe should contribute to society at large. For example consider two of the domains that we have studied: compression and robustness of DNNs. In the former we aim to reduce the total energy consumption (and therefore carbon emissions) of neural networks, independently of the downstream task. In the latter, we aim to increase the general confidence that we can have on the results generated by a DNNs.

In the following paragraphs, we present a list of the academic contributions of this thesis:

\begin{itemize}
    \item \bibentry{lassance2018laplacian}
    \item \bibentry{lassance2018matching}
    \item \bibentry{lassance2018predicting}
    \item \bibentry{lassance2019improved}
    \item \bibentry{lassance2019robustness}
    \item \bibentry{lassance2020benchmark}
    \item \bibentry{lassance2020deep}
    \item \bibentry{grelier2018graph}
    \item \bibentry{bontonou2019smoothness}
    \item \bibentry{bontonou2019formalism}
    \item \bibentry{bontonou2019comparing}
\end{itemize}
  
In addition to these contributions, we are currently working on a book chapter dedicated to the use of graphs to represent latent spaces of DNNs.

\section{Document structure}

Overall this manuscript is divided into four parts. First, we present an introduction of the two main subjects of this thesis: \begin{inlinelist} \item Deep Learning (Chapter~\ref{chap2}) \item Graph Signal Processing (Chapter~\ref{chap3}) \end{inlinelist}. In these two chapters we define concepts that were essential to the thesis in our own words and introduce contributions that are directly linked to either deep learning or GSP. 

Second we introduce and discuss the domain of ``Deep Learning for inputs supported on graphs'' in Chapter~\ref{chap4}. This domain combines both Deep Learning and GSP concepts and studies the application of deep learning methods to inputs that are defined in the graph domain. As we did not have a specific focus on these types of applications during the thesis, we present it more as an overview using a proposed mathematical framework and discuss some applications.

Third, we introduce the main contribution of this thesis, which is the development of the domain of graph based methods for improving deep learning. We do so by studying the intermediate representations of deep neural networks using the formalism from GSP. In contrast with the previous part, we do not need for the inputs to be defined over a graph domain to be able to deploy our methods. In this part we are actually introducing a new domain of research: the study of general intermediate representations of DNNs using the GSP framework.

Finally, we present a summary in Chapter~\ref{conclusion}, including a quick recall of our contributions and the research directions that are now open for future work.

\chapter[Deep Neural Networks]{Classification and feature extraction with Deep Neural Networks}\label{chap2}
\localtableofcontents
\vspace{1.0cm}

In this chapter, we introduce the concepts of classification and feature extraction using Deep Neural Networks (DNNs). This chapter is organized as follows: first in Section~\ref{chap2:definition}, we introduce and define neural networks, then in Section~\ref{chap2:layers}, we introduce the layers used in the scope of this work and in Section~\ref{chap2:datasets} we introduce the datasets considered in this thesis. Finally, we introduce compression tools in Section~\ref{chap2:compression} and robustness definitions in Section~\ref{chap2:robustness}.

\section{Definitions}\label{chap2:definition}

Deep Neural Networks (DNNs) contain the term ``neural'' as they are loosely inspired by the functioning of brain neurons. However, it is fair to say that this inspiration is becoming less important in the recent developments in the field. This is why in this Chapter we adopt a network-based definition of these models.

So let us consider a DNN architecture. Such an architecture is mathematically described by its ``network function'' $f$. We call $f$ ``deep'' as it is obtained through a long cascading sequence of intermediate functions from its input to its output. More precisely, $f$ receives an input tensor $\vx$, which typically represents the pixel values of an image, and outputs a corresponding tensor $f(\vx)$ which dimensions and interpretation depends on the task for which the network was initially designed. 

There exists a lot of ways to obtain deep neural network functions, but the simplest to formalize mathematically consists in a composition of layer functions:
\begin{equation}\label{chap2:compositional_equation}
    f = f^{\ell_\text{max}}\circ f^{\ell_\text{max}-1} \circ \dots \circ f^1.
\end{equation}
Here, each function $f^\ell$, called ``layer function'', is highly constrained. Indeed, each $f^\ell$ is typically defined as a parametrized linear function followed by a non-parametric non-linear function. 

In modern literature, it is rare to encounter such constructions of deep neural network functions. Instead, many authors use residual networks~\citep{he2016deep}. Indeed, residual networks have been demonstrated to reach state-of-the-art performance in many challenges in the context of classification. Residual networks (\textbf{Resnet}) are composed of blocks of layers, as depicted in Figure~\ref{fig:resnet_simple}. Note that even in the case of Resnets, the core idea remains that network functions are built as an assembly of layer functions.

\begin{figure}[ht!]
  \begin{center}
  \tikzsetnextfilename{chapter2/tikz/resnet_simple}%
  \input{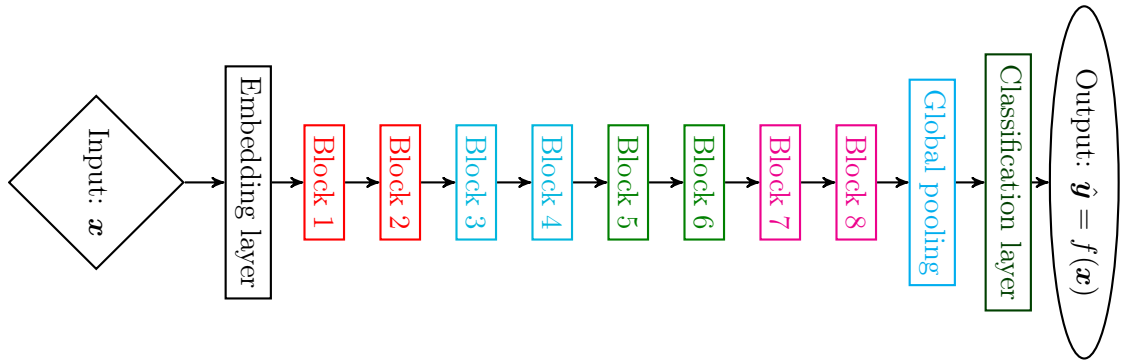}%

    \caption{Simplified depiction of a Resnet with eight residual blocks, divided into four block groups. We depict the residual blocks in Figure~\ref{fig:resnet_blocks}. The color of each block indicates their group block and that they have the same dimensions.}
    \label{fig:resnet_simple}
  \end{center}
\end{figure}

In the literature one can find many types of layers. The most notable ones are \begin{inlinelist} \item Fully connected \item Convolutional \item Pooling \item Normalization \item Graph Convolutional \item Recurrent Neural Network \item Long Short-Term Memory \end{inlinelist}. In Section~\ref{chap2:layers}, we will describe the first four items in more detail; item 5 will be introduced and detailed in Chapter~\ref{chap4}; items 6 and 7 are outside the scope of this work.

The outputs of layer functions are called intermediate representations.

\begin{definition}[intermediate representation]
We call the output of an intermediate function $f^\ell$ an intermediate representation. In other words, in the simple case of architectures that can be written using Equation~\ref{chap2:compositional_equation}, $\vx^\ell$ is the intermediate representation generated by applying the network function $f$ from $f^1$ until $f^\ell$ on $\vx$ where $\ell$ represents the depth in the DNN architecture.
\end{definition}

The goal of the intermediate representations is to capture the internal state of the DNN. In simple architectures that respect Equation~\ref{chap2:compositional_equation}, they also obey the Markov property (i.e., they fully capture the actual state of the input traversing the network and are sufficient to compute the output). Note that in more complex architectures this property does not necessarily holds.

Recall the concept of Resnets that are grouped in blocks of layers. These blocks define splits $f', f'',\dots$. The output of each block is fully characterized by its weights and its inputs. Intermediate representations obtained between blocks are therefore Markovian. However, this property is not valid inside a block, as a residual connection exists. We depict residual network blocks in Figure~\ref{fig:resnet_blocks}.

\begin{figure}[ht!]
      
    \begin{center}
      \begin{adjustbox}{max width=\linewidth}
  \tikzsetnextfilename{chapter2/tikz/resnet_blocks}%
  \input{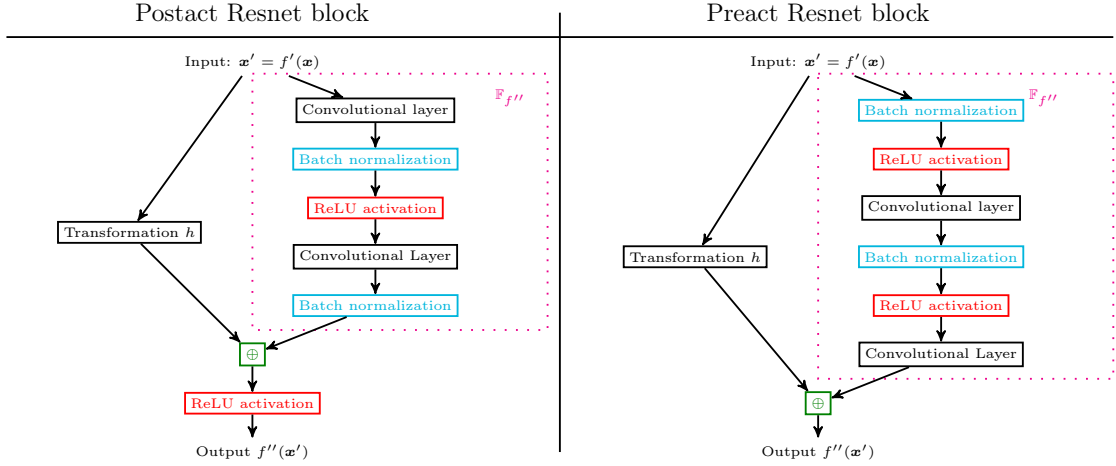}%

      \end{adjustbox}
    \caption{Depiction of the preact and postact residual network blocks proposed in~\citep{he2016deep}. The transformation $h$ ensures that $h(\vx')$ and $\sF_{f''}(\vx')$ have the same dimensions so that we can perform the sum operation between them. $h$ is normally implemented as either the identity operation (if both $\vx'$ and $\sF_{f''}(\vx')$ have the same dimensions) or as a convolution layer if the dimensions differ. In many cases, this convolution has only one parameter for each feature map.}
      \label{fig:resnet_blocks}
    \end{center}
\end{figure}

The function $f$ is characterized using tunable values called \textbf{parameters}. Initially, these parameters are randomly sampled in a distribution $\mathcal{N}$, so that the output of the network function can be interpreted as a random projection of data. In the context of classification, the most used framework is to output a class-wise classification score $\hat{\vy}$. Therefore, in order to train the network function to solve the classification task, it is common to use the \textbf{label indicator vector} $\vy$ associated with the input $\vx$.

\begin{definition}[label indicator vector]\label{chap2:label_indicator_vector}
  A binary vector with as many coordinates as the number of classes in the problem. Only the coordinate corresponding to the class of input $\vx$ is set to 1 while all the other coordinates are zeroed.
\end{definition}

In classification tasks, the goal of a DNN is to correctly classify all the inputs $\vx$ from the domain of possible inputs $\sD$. As it is impossible to collect all the possible inputs for most tasks, we use a subset $\dataset$ that we call dataset. Therefore, the parameters of the DNN are tuned during a learning phase using a dataset $\dataset$ and an objective function $\mathcal{L}$ that measures the discrepancy between the outputs of the network functions and expected label indicator vectors, i.e., discrepancy between $\vy$ and $\hat{\vy}$. We present the datasets used in this thesis in Section~\ref{chap2:datasets}. 
 
Usually, the objective function is a loss function, which is minimized over a subset of the dataset that we call ``\textbf{training set}'' ($\trainset$), composed of training examples $\sX$. This optimization is usually performed using variants of the stochastic gradient descent algorithm~\citep{bottou2010large}. As such, the network function $f$, which is typically composed of a vast number of parameters, is adjusted during the learning phase. We will see in Section~\ref{chap2:compression} exactly how vast is this number of parameters and some techniques aiming at reducing the number of computations and memory needed for both training and inference in neural networks. 

Note that in the context of vision, it is quite common to introduce a \textbf{data augmentation} scheme to go alongside our training.

\begin{definition}[data augmentation]\label{chap2:data_augmentation}
  We define data augmentation as the act of artificially generating new inputs $\vx_{DA}$ to increase the size of $\trainset$ by performing a set of transformations $\sH_{DA}$ on the inputs $\vx \in \trainset$.   
\end{definition}

Data augmentation is indeed an instrumental technique, as it allows us to increase the size of $\trainset$ without the cost of drawing new labeled samples from $\sD$. There are two main types of data augmentation, which we refer to as \textbf{domain-driven} and \textbf{data-driven}. The former uses the knowledge one has over the domain $\sD$ to design transformations $h$ which are known to generate valid new inputs without compromising the nature of their corresponding class. Typical domain-driven data augmentation in image scenarios include randomly removing a small part of the image (also called \emph{random crop}) and horizontal flipping.

Data-driven data augmentation, on the other hand, uses information from the dataset to generate new inputs. Doing so allows us to have a significant advantage as one does not need to be a specialist on the domain $\sD$ to propose the data augmentation scheme. However, it may lead to training on inputs $\vx_{da}$ that are outside of the domain $\sD$ and possibly inputs $\vx_{da}$ that are misclassified. 

Two of the most used data-driven techniques are \emph{autoaugment}~\citep{cubuk2019autoaugment} where one tries various data augmentation schemes at random and keeps the ones that work the best in terms of final accuracy of the model and \emph{mixup}~\citep{zhang2018mixup} where both the input $\vx_1$ and its desired output $\vy_1$ are interpolated with another example which input is $\vx_2$ and output is $\vy_2$ to generate a new input $\vx_{da}$ and its associated output $\vy_{da}$. Note that both data-driven techniques come with drawbacks. Autoaugment still requires some domain knowledge (to design the data augmentation schemes that are tested). On the other hand, mixup may be incompatible with some datasets and requires adapting the objective function. We present some examples of data augmented samples in Table~\ref{chap2:table_mixup} (mixup) and in Figure~\ref{chap2:fig_autoaugment} (autoaugment).

\begin{table}[ht!]
  \begin{center}
  \caption{Examples of mixup based data-augmentation samples. Figures extracted from~\citep{yun2019cutmix} ©2019, IEEE.}
  \begin{tabular}{cccc}
   & $\vx$  &  $\vx_{da}$\citep{zhang2018mixup} & $\vx_{da}$ \citep{yun2019cutmix} \\
  \multirow{4}{*}{Image} &  \multirow{4}{*}{\includegraphics[width=0.185\linewidth]{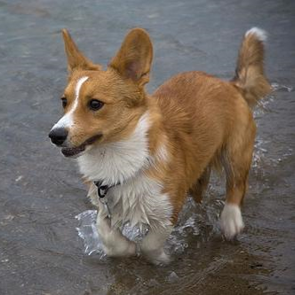}}
   &  \multirow{4}{*}{\includegraphics[width=0.185\linewidth]{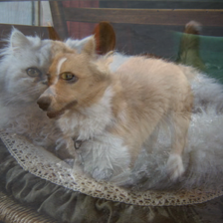}}
   &  \multirow{4}{*}{\includegraphics[width=0.185\linewidth]{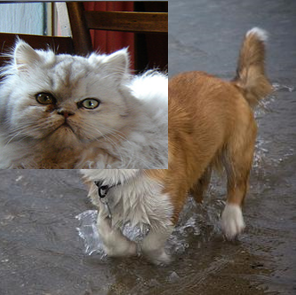}} \\ 
   & & &  \\ 
   & & &  \\
   & & &  \\
   & & &  \\ \midrule
  \multirow{2}{*}{Label indicator vector} 
  &\multirow{2}{*}{Dog 1.0}
  &\multirow{2}{*}{\begin{tabular}[c]{@{}c@{}}Dog 0.5\\ Cat 0.5\end{tabular}}
  &\multirow{2}{*}{\begin{tabular}[c]{@{}c@{}}Dog 0.6\\ Cat 0.4\end{tabular}} \\ 
    & & &  \\  \midrule
  \end{tabular}
  \label{chap2:table_mixup}
\end{center}
\end{table}

\begin{figure}[ht!]
  \centering
  \includegraphics[width=1.0\linewidth]{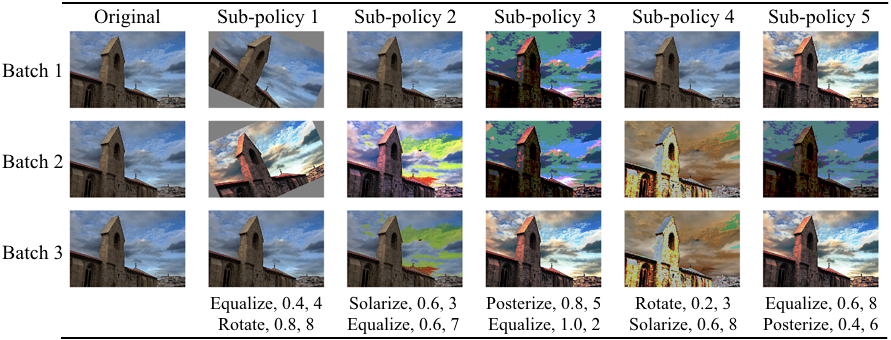}
  \caption{Examples of autoaugment based data-augmentation samples. Figure extracted from~\citep{cubuk2019autoaugment}: ©2019, IEEE.
  }
  \label{chap2:fig_autoaugment}
\end{figure}

After the learning phase comes the \emph{inference phase}. During inference, one first fixes the weights of the network and then evaluates its performance. Note that training and inference phases can be performed iteratively, and most of the time alternate. At the end of the training, we obtain the final architecture with its corresponding weights. This architecture obtains a score (most often the score consists in measuring the accuracy of the model on a dataset that is distinct from the training set).

Most of the time, the whole process of obtaining the final score of an architecture is repeated several times for two main reasons: \begin{inlinelist} \item to verify that the score is robust against a different initialization of the parameters and sampling of the training examples, \item to search for more efficient \textbf{hyperparameters} of the network.\end{inlinelist}

\begin{definition}[hyperparameters]
  Hyperparameters are a set of parameters that are fixed during training. For example, the parameters concerning the architecture of the DNN (e.g., the number of layers in the network) and the training methodology are considered to be hyperparameters of a DNN.
\end{definition} 

During the inference phase there are two types of evaluation we can consider: \begin{inlinelist} \item memorization \item generalization \end{inlinelist}. In the former, the goal is to verify that the network can correctly classify the inputs used during the training phase, i.e., the inputs in $\trainset$. \emph{Generalization}, on the other hand, aims at evaluating the capacity of the network to extrapolate from the seen inputs and correctly classify unseen ones, i.e., not presented during the training. 
  
Generalization is a fundamental property for DNNs, as it is not possible to access the full domain $\sD$. This is why alongside of the training set it also common to define a \textbf{validation set} $\validset$ and a \textbf{test set} $\testset$. The main goal is that all sets come from the same distribution $\dataset$, but do not intersect $\trainset \cap \validset = \varnothing$, $\validset \cap \testset = \varnothing$ and $\trainset \cap \testset = \varnothing$. We illustrate $\sD$ and its relationship with the dataset in Figure~\ref{fig:dataset_domain}. 

\begin{figure}[ht!]
  \begin{center}
    \begin{adjustbox}{max width=\linewidth}
  \tikzsetnextfilename{chapter2/tikz/dataset_domain}%
  \input{chapter2/tikz/dataset_domain.tex}%

    \end{adjustbox}
    \caption{Diagram depicting the relationship between $\sD$, $\dataset$, $\trainset$, $\validset$ and $\testset$. }
    \label{fig:dataset_domain}
  \end{center}
\end{figure}

The \emph{validation set} is to be used during the iterative phases of training and inference. In this case, the validation set helps one make decisions about the architecture and the training process of the network. On the other hand, the \emph{test set} should not be used to influence the decisions during the iterative phases of the process. The test set should serve only as an a posteriori measure of performance.

This setting with validation and test sets has been proposed to avoid the shortcoming of \textbf{overfitting}.

\begin{definition}[overfitting]\label{chap2:overfitting}
  In this work, we define overfitting as a three-fold phenomenon:
  \begin{enumerate}
    \item \emph{Overfit to $\trainset$:} a network that has an excellent performance on memorization, but an inferior one on generalization is said to be overfitted to $\trainset$. Ideally, we would aim for both evaluations having similar performance.
    \item \emph{Overfit to $\validset$:} on the other hand, having an excellent generalization performance to the $\validset$ is not sufficient. By using $\validset$ to define the hyperparameters, the trained network can be biased to $\validset$. For this reason, it is recommended to have two distinct sets for evaluating generalization, one used for tuning hyperparameters ($\validset$) and a second one as an external measure of performance ($\testset$).
    \item \emph{Overfit to $\dataset$:} finally, even if the network shows excellent generalization to both $\validset$ and $\testset$, it could still be overfitted to the dataset, i.e., the network would not generalize to the rest of the domain $\sD$. Another way to say that a network is overfitted to $\dataset$ is to say that it is not robust. We delve into more details on this problem in Section~\ref{chap2:robustness}.
  \end{enumerate}
\end{definition}

Given this definition of overfitting, it is quite surprising that in the recent literature, most works ignore the use of a $\validset$. Instead, they mostly use the union of $\trainset$ and $\validset$ as the de facto training set, and the $\testset$ is used to tune the hyperparameters. Fortunately, this seems not to be biasing the DNNs to the $\testset$ as described in a recent work~\citep{recht2019imagenet} that draws a new $\testset_2$ from the same $\sD$ and shows that the networks that performed the best in $\testset$ also were the best in $\testset_2$. However, this finding only covers the first two definitions of overfitting. 

As the domain $\sD$ grows, it is harder to represent it accurately with the subset $\dataset$. Therefore it is often the case that the construction of the dataset biases neural networks, i.e., they may disregard features that are not present or underrepresented on $\dataset$. We discuss the robustness of neural networks to inputs that are in the domain $\sD$ but are not represented in $\dataset$ in Section~\ref{chap2:robustness}.

\subsection{Residual networks (Resnet)}\label{chap2:subsection_resnet}

There is a vast number of DNN architectures proposed in recent literature~\citep{he2016deep,zagoruyko2016wide,vaswani2017attention,pham2018efficient,tan2019efficientnet}. In this work we mostly use residual networks~\citep{he2016deep}, that we introduce in more detail this subsection. This architecture was chosen for three main reasons:
\begin{enumerate} 
  \item Ease of training: as we are going to introduce in the next paragraphs, residual networks are easy to train and to find the correct hyperparameters;
  \item Standard of the literature: residual networks are used very often in the literature, therefore choosing this architecture allows us an easier and fairer comparison with recent works;
  \item Certified correct implementation: as residual networks are both easy to train and widely used in recent literature, there exist standardized implementations of residual networks for almost all languages and frameworks, removing the possible fail-point that our implementation is incorrect or differs from the works we compare to.
\end{enumerate}

Residual networks are named because they introduce residual connections between layers in the DNN. The main goal of residual connections is to ease the training of deep neural networks. Residual network paths are formed using the original input and a set of sequential layers. The input goes through the set of sequential layers generating an output. This output is then summed with a simple and direct transformation of the original input. This additive path, called residual path, is also commonly called a block. More formally, we define a residual block $f''$ that receives an input $\vx'$ by:
\begin{equation}\label{eq:residualconnection}
  f''(\vx') = g(h(\vx') +  \sF_{f''}(\vx')),
\end{equation} 
where $h$ is a simple transformation that ensures the sum is performed between two tensors with the same dimensions, $g$ is an optional activation function and $\sF_{f''}$ is the set of sequential layers that belongs to block $f''$.

One key interest of residual networks is that a block can easily implement the identity function, and as such deeper architectures can emulate shallower ones. Among other interests, being able to behave as a shallower network eases the training process. In other words, residual DNN can behave like a shallower network at the start of the training in order to be able to warm up, and when their weights are well-conditioned, they can start to use their entire depth and reap the benefits of deeper networks. In the literature, it is not rare to see residual networks with hundreds of layers.

The property of behaving like a shallower network and its influence in training deep residual networks was demonstrated in recent works, where authors~\citep{de2020batch,zhang2018residual} show that one can remove the normalizations of the DNN by starting the network biased for the shallower paths. In Figure~\ref{fig:resnet_blocks_code} we depict the two residual network blocks used in this work (\emph{preact} and \emph{postact} blocks) and the residual connection. Note that these two types of blocks are not the only ones used in the literature and that the preact block is slightly different from the original presentation in~\citep{he2016deep} and in Figure~\ref{fig:resnet_blocks}. This choice was made because most of the recent literature shifted to this design.

\begin{figure}[ht!]
  \begin{center}
    \begin{adjustbox}{max width=\linewidth}
  \tikzsetnextfilename{chapter2/tikz/resnet_blocks_code}%
  \input{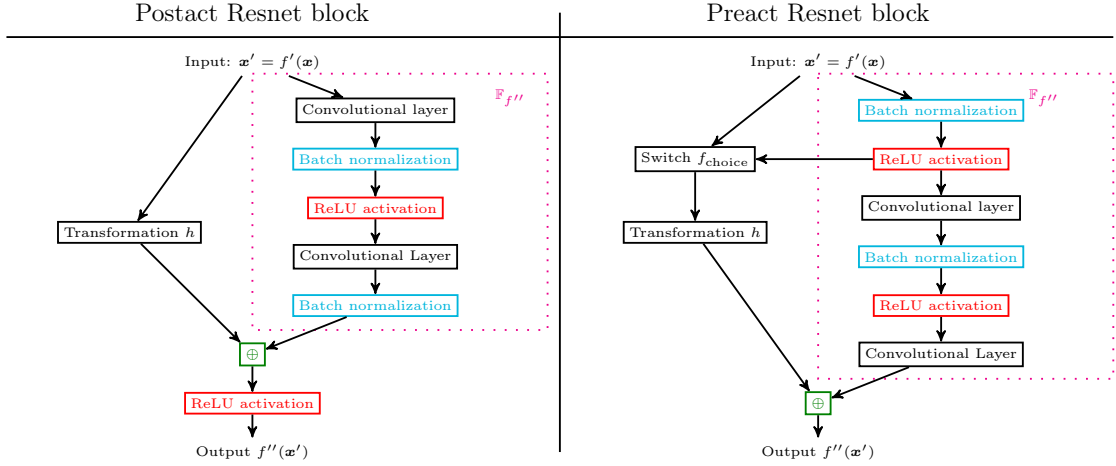}%

    \end{adjustbox}
    \caption{Depiction of the preact and postact residual network blocks used in this work. The switch $f_\text{choice}$ decides the input of $h$ based on the size of the input $\vx'$ and the output from $\sF_{f''}(\vx')$: if both have the same dimensions the direct path is chosen, if not the indirect (the one that starts on the right and then goes to the left) is chosen. We note that the blocks depicted here are different from the ones in Figure~\ref{fig:resnet_blocks} as most works in recent literature shifted to this design.}
    \label{fig:resnet_blocks_code}
  \end{center}
\end{figure}

As shown in Figure~\ref{fig:resnet_simple}, Resnets start with a first convolutional layer to increase the feature map size of the input to $\featuremaps_\text{initial}$, sometimes called embedding layer. The residual blocks then follow the embedding layer. After the residual blocks, it is common to add a pooling layer. This pooling layer is sometimes called global if the 3D intermediate representation ($[\featuremaps,w,h]$) is downsampled to an 1D representation ($[\featuremaps]$). A fully connected layer then follows this downsampling. This FC layer is sometimes called the classification layer. 

To determinate the depth of Resnets and, therefore, their nomenclature, one has to look at the number of blocks of the network and how they are constructed. Note that Resnet blocks do not always treat the same feature map size, and it is quite common to split these blocks into groups, where each block may either end or start with a downsampling operation. We use the nomenclature Resnet$n$-$w_i$ where $n$ is the number of layers, and $w_i$ is a widen factor first defined in WideResnets~\citep{zagoruyko2016wide}. 

The amount of layers $n$ is characterized by the following equation $n= 2 (cd + 1)$, where $c$ is the number of convolutional layers per block, and $d$ is the number of blocks. Unfortunately this notation can be misleading because two Resnet18-$w$ could have very different block configurations (e.g, $[2,2,2,2]$, $[1,3,1,3]$ and $[3,3,2]$ would be valid Resnet18-$w$ configurations). 

It is also standard to define a common feature map amount for all convolutions of the same block group and to define a feature map scaling that doubles at each downsampling operation. We recall that WideResnets~\citep{zagoruyko2016wide} also add a widen factor $w_i$ that is a multiplier to the feature map amounts from the first convolution to the first block, allowing the network to be wider. We note that WideResnets follow a different nomenclature ($n=2 (cd + 2)$). As we try to be consistent on our nomenclature, networks in this work are always presented following $2(cd + 1)$, which may lead to confusion if one is familiarized with WideResnets (e.g., WideResnet28-10 will be presented as WideResnet26-10 in this work).  

In more detail, in this work, we only use Resnets with blocks with two convolutions, as presented in Figure~\ref{fig:resnet_blocks_code} and groups with an equal amount of blocks. We also limit our configurations to have either 3 or 4 groups of blocks. Finally, the residual networks used in this work will use either $\featuremaps_\text{initial}=16$ if there are three groups of blocks and $\featuremaps_\text{initial}=64$ if there are four groups. We depict Resnet18-$w_i$ in Figure~\ref{fig:resnet_18} and Resnet20-$w_i$ in Figure~\ref{fig:resnet_20} to illustrate these architectures as they will be widely used in our experiments.

\begin{figure}[ht!]
  \begin{center}
    \begin{adjustbox}{max width=\linewidth}
  \tikzsetnextfilename{chapter2/tikz/resnet_18}%
  \input{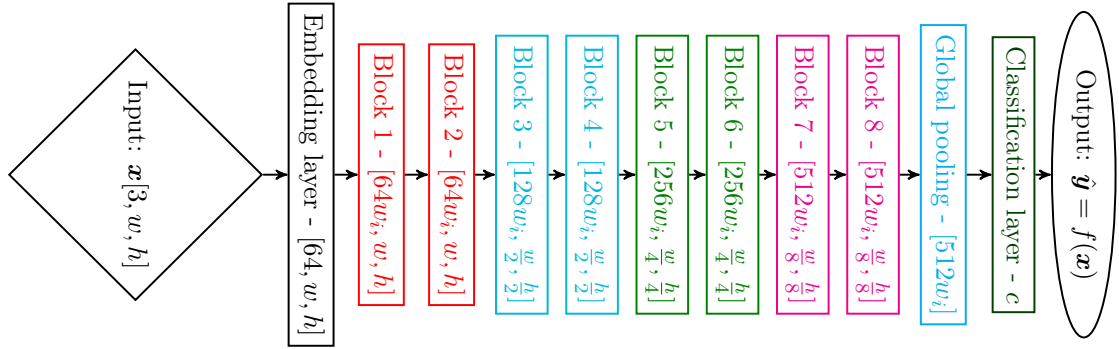}%

    \end{adjustbox}
    \caption{Depiction of the Resnet18-$w_i$ used in this work. We write after the layer/block the dimensions of the output. If the output dimensions of the block differ from the input, the first convolution of the block will be the one responsible for the change, by either increasing the number of output feature maps or performing strided convolutions to reduce width and/or height.}
    \label{fig:resnet_18}
  \end{center}
\end{figure}

\begin{figure}[ht!]
  \begin{center}
    \begin{adjustbox}{max width=\linewidth}
  \tikzsetnextfilename{chapter2/tikz/resnet_20}%
  \input{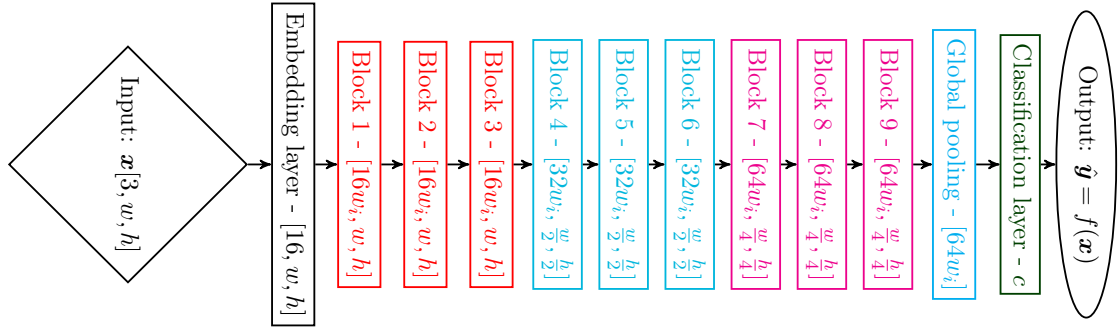}%

    \end{adjustbox}
    \caption{Depiction of the Resnet20-$w_i$ used in this work. We write after the layer/block the dimensions of the output. If the output dimensions of the block differ from the input, the first convolution of the block will be the one responsible for the change, by either increasing the number of output feature maps or performing strided convolutions to reduce width and/or height.}
    \label{fig:resnet_20}
  \end{center}
\end{figure}

\subsection{Feature extraction with DNNs}\label{chap2:feature_extraction}

One of the most significant advantages of DNNs is that they can extract relevant features from the input, as seen in Figure~\ref{fig:visualization}. This property of DNNs seems to be in line with the way biological neural classification operates. Indeed, works such as~\citep{hubel1962receptive} and the more modern~\citep{hansen2007topographic} show that there is a hierarchical organization of transformation from raw images to concepts in the brain. These works were used as a first reasoning basis for DNNs and are the historical reason for terms such as neurons and artificial neural networks. In this work, we try to avoid the word neuron, using the network nomenclature instead, i.e., we use node instead of neuron. 

In this vein, we can also define the network function $f$ by a feature extractor ($\featureextractor$) followed by a a classifier ($\classifier$), such that $f(\vx) = \classifier(\featureextractor(\vx))$. 

As a consequence, it is possible to re-use networks in order to perform \emph{transfer learning}~\citep{tan2018transfersurvey, zhai2019largescale}. In transfer learning we first train a network $f_1$ on a first dataset $\dataset_1$. Then when a subsequent dataset $\dataset_2 \cap \dataset_1 = \varnothing$ appears, one can either re-use the feature extractor $\featureextractor_1$ and train a new classifier $\classifier_2$ or can fine-tune the entire DNN using $f_1$ as a starting point. The former approach is preferable if $|\trainset_2|$ is small~\citep{arandjelovic2016netvlad, mangla2019manifold}, while the latter is used when the $|\trainset_2|$ is sufficiently large, e.g., it is quite common for the state of the art of medium-sized image datasets to be achieved by fine-tuning a network trained using a large image dataset~\citep{alex2019large}. 

The decision of training from scratch or performing transfer learning depends on the cardinality of both $\dataset_1$ and $\dataset_2$ and on the type of classification that is performed. For example, image classification tasks seem to do better with fine-tuning~\citep{alex2019large,zhai2019largescale}, while image retrieval tasks seem to be better with re-using the feature extractor~\citep{arandjelovic2016netvlad,liu2019guided}.

\subsection{Classifiers for DNNs}

Multiple classifiers can be used in DNNs, depending on the goal/task at hand. In this subsection, we introduce some of the most used classifiers for DNNs.

Please note that one does not need to split $\featureextractor$ and $\classifier$ exactly at the last layer of a DNN, as is usually done in the literature. As a matter of fact, it can be beneficial in some cases to have deep classifiers (e.g., ~\citep{alex2019large,zhai2019largescale}).

Throughout this section, we consider given a feature extractor $\featureextractor$, and we introduce various ways to perform downstream classification. Note that when training a DNN end-to-end the classifier choice will directly impact the training of the entire network.

\subsubsection{Logistic Regression (LR)}

The logistic regression is the most used form of the classifier for DNNs. It applies a linear transformation to the input so that the output is of the same dimension as the number of classes in the problem, provided we are facing a classification one. Its goal is to optimize an objective function using a logarithmic model of the probabilities for each class. 

In neural networks, it is common to generate the pseudo-probabilities of logistic regression using the softmax function:
\begin{equation}\label{chap2:softmax_equation}
\hat{\vy}_i = \text{softmax}(\vx')_i \; = \; \frac{e^{\frac{x'_i}{T}}}{\sum_{x'_j \in \vx'}{e^{\frac{x'_j}{T}}}} \;,
\end{equation} 
where $T$ is a temperature parameter set to 1 in most cases, $\hat{\vy}$ is the output of the network, and $\vx'$ is the output of the last layer (commonly called classification layer). The logarithmic model then uses the \textbf{cross-entropy loss} as the objective function $\mathcal{L}$:

\begin{equation}\label{chap2:simplified_ce}
\mathcal{L}_{\text{cross-entropy}} \; = \; -\sum_i{\vy \log \hat{\vy}}\;.
\end{equation}
where $\vy$ is the label indicator vector of $\vx$. We depict three logistic regression classifiers under different values of $T$ in Figure~\ref{chap2:logistic_regression_figure}. Note that higher values of temperature create softer decisions and lower values lead to strict decisions.

\begin{figure}[ht]
  \begin{center}
  \tikzsetnextfilename{chapter2/tikz/logistic_regression}%
  \input{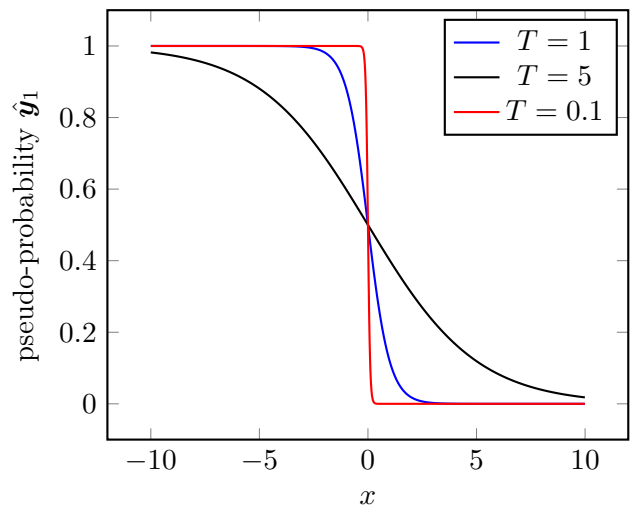}%

    \caption{Logistic regression of a problem with two classes ($x\leq0 \in c_1$ and $x>0 \in c_2$) under different temperature ($T \in \{0.1,1,5\}$) conditions. We depict the output $\hat{y}_1$, and consider that the decision is performed using the $\arg\max$ function.}
    \label{chap2:logistic_regression_figure}
  \end{center}
\end{figure}

\subsubsection{Support Vector Machine (SVM)}

Another widely used classifier in machine learning is the Support Vector Machine~\citep{tang2013svm}. The primary motivation of a SVM is to generate a set of hyperplanes to separate the classes. These hyperplanes may be applied directly to the raw data or after the data has been transformed using a kernel $k$. The latter is also called the kernel-trick, and its main contribution is that it allows one to learn nonlinear classifiers using convex optimization techniques that are guaranteed to converge efficiently~\citep{dlbook}. 

The most used kernel $k$ is the radial basis function (\textbf{RBF}) defined as:
\begin{equation}\label{chap2:eq-rbf}
k(\vx_1,\vx_2) = e^{-\gamma \| \vx_1 - \vx_2 \|^2_2} \; , 
\end{equation}
where $\vx_1$ and $\vx_2$ are two inputs and $\gamma$ is an adjustable parameter. The RBF kernel can also be seen as a dot product in an infinite-dimensional space~\citep{dlbook}. Note that the output of $k(\vx_1, \vx_2)$ is a similarity measure between $\vx_1$ and $\vx_2$.

As the DNN already generates a suitable nonlinear feature extractor, it is more common to apply the hyperplane separators directly. In this case, given $\hat{\vy}$ the output of the network, the objective function $\mathcal{L}_{\text{LinearSVM}}$ is defined as: 

\begin{equation}
  \mathcal{L}_{\text{LinearSVM}} = max(0,m-\vy\hat{\vy})
\end{equation}
where $m$ is the classifier margin, and $\vy$ is a modified binary label indicator vector, where one uses -1 to indicate that it does not belong to a class instead of using 0. In other words, the objective of the network is to output at least $m$ for the coordinate corresponding to the class of the input and at most $-m$ for the other classes. Figure~\ref{chap2:svm_figure} depicts both linear and RBF kernel SVM classifications.

\begin{figure}[ht]
  \begin{center}
  \begin{adjustbox}{max width=\columnwidth}    
  \tikzsetnextfilename{chapter2/tikz/svm}%
  \input{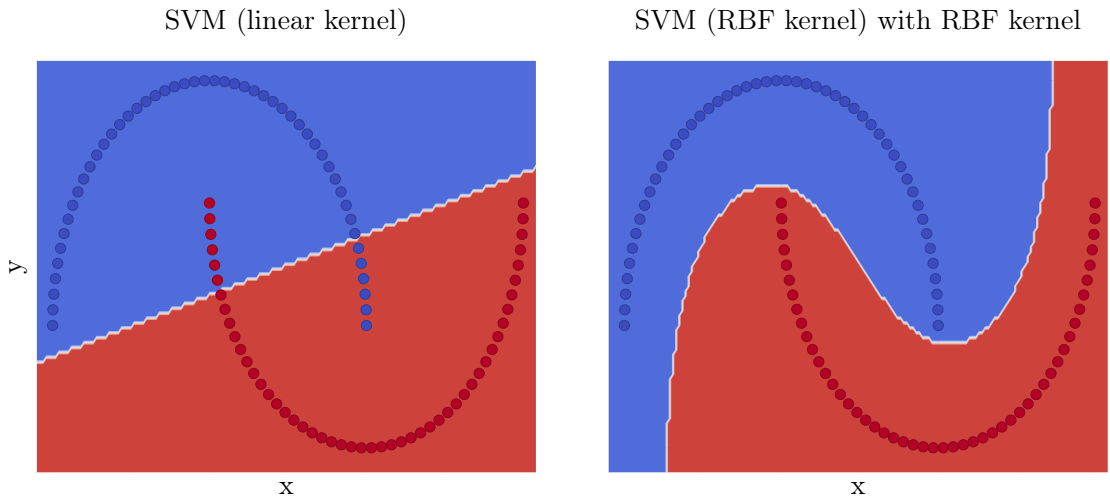}%

  \end{adjustbox}
  \caption{SVM classification applied to the ``two moons'' dataset. The goal is to classify the blue and the red dots. We depict the hyperplanes applied to the raw data (Linear SVM) and using an RBF kernel (RBF SVM). Note how the RBF SVM can completely separate the two classes, while the Linear SVM misclassifies 12 out of the 100 points.}
    \label{chap2:svm_figure}
  \end{center}
\end{figure}

\subsubsection{$k$-Nearest Neighbor classifier ($k$-nn) }

Not all classifiers need to be optimized. Using $k$-nearest neighbors, the idea is to classify an input by looking at the closest training samples in the output domain of the feature extractor. Usually, the classification of an unseen input, commonly called \textbf{query}, is performed using the $k$ closest examples from the training set (also called \textbf{support set} in this scenario) and a majority vote. One of the advantages of not having a training phase for the classifier is that one can quickly create ensemble decisions by adding a classifier per layer, as seen in~\citep{papernot2018deep}. Also, this technique can be easily deployed in the context of streaming data. 

An interesting characteristic of a 1-nn classifier is that it has perfect memorization by definition. On the other hand it will probably have a poor generalization performance. We depict an example of a $1$-nn classifier and a 20-nn classifier in Figure~\ref{chap2:knn_figure}. We note the trade-off between the smoothness of the classifier border and memorization that is displayed in the image. The 1-nn classifier has perfect memorization but a very rough classification border which can sometimes be a problem for generalizing to unseen data, especially when using inputs that could be mislabeled. On the other hand the 20-nn classifier has a very smooth classification border, but fails to correctly classify some examples of the dataset. 
\captionsetup[subfigure]{labelformat=empty}

\begin{figure}[ht]
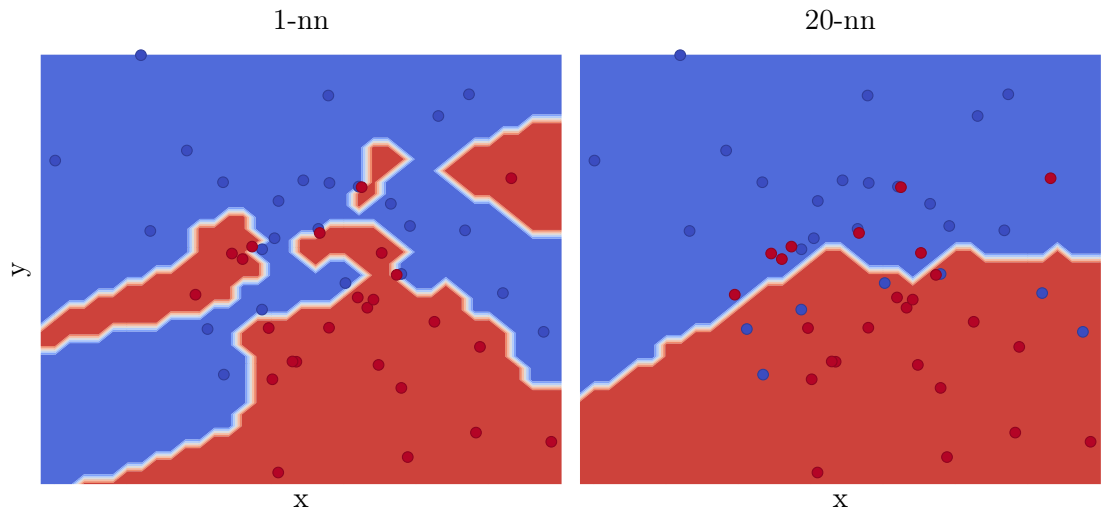

  \begin{center}
    \begin{subfigure}[b]{.5\linewidth}
      \centering
  \tikzsetnextfilename{chapter2/knn/1-nn}%
  \input{chapter2/knn/1-nn.tex}%

    \end{subfigure}%
    \begin{subfigure}[b]{.5\linewidth}
      \centering
  \tikzsetnextfilename{chapter2/knn/20-nn}%
  \input{chapter2/knn/20-nn.tex}%

    \end{subfigure}
    \caption{Example of 1-nn and 20-nn classification. The goal is to classify the blue and the red dots. Note how the 1-nn can completely memorize the dataset, but does so using a very non smooth border which can be a problem for generalizing to unseen data.}
    \label{chap2:knn_figure}
  \end{center}
\end{figure}

\subsubsection{Nearest Centroid Mean classifier (NCM)}

Another example of a classifier that does not need to be training is the NCM or Rocchio classifier~\citep{manning2008introduction}. In this case, one uses the mean representations for each cluster (one or more clusters per class) as the support set for a $k$-nn classifier. A comparison of the $k$-nn classifier and the NCM classifier is depicted in Figure~\ref{chap2:ncm_figure}. Note that the classification border of the NCM classifier is even smoother than that of the 20-nn classifier.

\begin{figure}[ht]
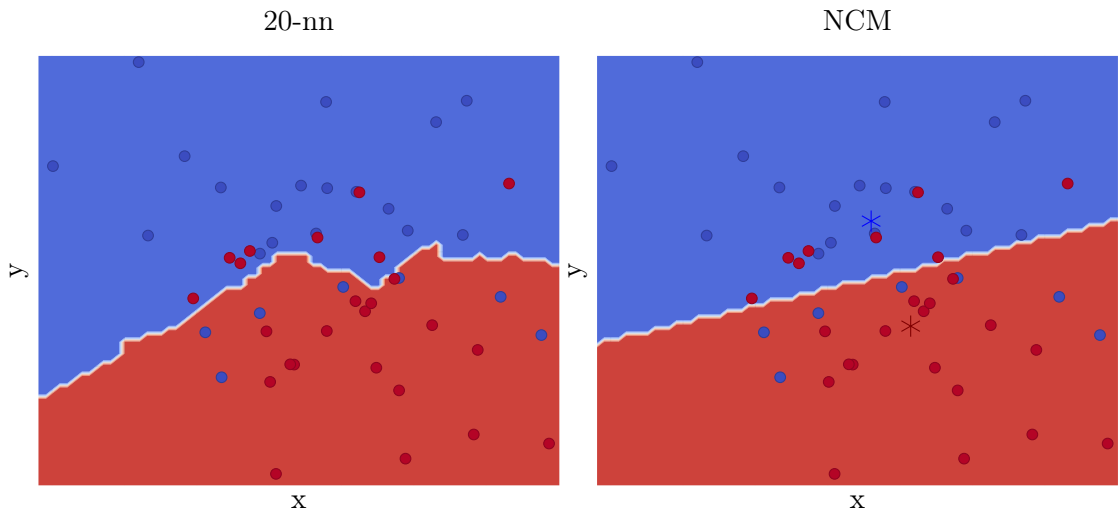

  \begin{center}
    \begin{subfigure}[b]{.5\linewidth}
      \centering
  \tikzsetnextfilename{chapter2/ncm/20-nn}%
  \input{chapter2/ncm/20-nn.tex}%

    \end{subfigure}%
    \begin{subfigure}[b]{.5\linewidth}
      \centering
  \tikzsetnextfilename{chapter2/ncm/ncm}%
  \input{chapter2/ncm/ncm.tex}%

    \end{subfigure}
    \caption{Comparison between the $k$-nn and NCL classifiers. The goal is to classify the blue and the red dots, with the asterisks representing the centroids of each class. Note that the classification border of the NCM classifier is even smoother than that of the 20-nn classifier.}
    \label{chap2:ncm_figure}
  \end{center}
\end{figure}

\section{Deep Learning Layers}\label{chap2:layers}

In this section, we introduce the layers that are used in this work. We recall that layers are the elementary functions on top of which the network function $f$ of a DNN is built. In the remaining of this section, we denote by $\vx'$ the input of a layer and by $f''$ the layer function. 

In this work we always consider 4D inputs $[b_\text{size},\featuremaps,w,h]$, where the first dimension $b_\text{size}$ indicates the number of samples treated concurrently as a \textbf{batch}, and the other three dimensions typically correspond to the number of feature maps, width and height of $\vx'$. 

Feature maps typically aggregate various components of the input (e.g., RGB images have 3 feature maps, one for each color: red, green, and blue). Sometimes, the three dimensions $\featuremaps, w, h$ can be flattened; in this case, we consider an input of dimension $d = w h \featuremaps$. Note that each coordinate of these dimensions is typically referred to as a \emph{node} or \emph{neuron} in the literature. 

Layers typically compose a nonlinear activation function $g$ with a linear function. Popular nonlinear activation functions include sigmoids, ReLUs, tanh\dots These functions are essential to ensure the expressivity of the deep learning models. Indeed, as the algebra of matrices is associative, removing the nonlinear functions would mean that deep models would be equivalent to a linear one-layer model. If one-layer models can only solve linearly separable problems, it has been known for a long time that the deep learning models are universal approximators under mild conditions~\citep{hornik1989approximator}.

\subsection{Fully connected layer (FC)}

A layer is said to be fully connected if each node $i$ of the input is connected to each node $j$ of the output with a proper weight $\emW_{i,j}$. In this case we can say that:
\begin{equation}
  f''(\vx') = g(\mW\vx' + \vb),
\end{equation}  
where $\mW$ is the trainable weight matrix, $\vb$ is a trainable bias added to the layer, and finally, $g$ is the activation function. In other words, we first perform a parametric linear transformation ($\mW\vx' + \vb$) and then we apply a non-parametrized non-linear function $g$. 

The fully connected layer is the most generic layer and the basis of deep learning. However, its use is severely limited when data to be processed is structured. For example, modern architectures for image classification usually only use one FC layer as the final classification layer. Figure~\ref{fig:resnet_18} and Figure~\ref{fig:resnet_20} depict two examples of the architectures used in this work, and both have only one FC layer. This limitation happens because FC layers ignore the intrinsic structure of the data. Indeed, the indexing of the inputs of a FC layer has no consequence on the accuracy of the trained model.

Theoretically, it would be possible for a DNN composed only of FC layers to understand the structure during training and limit its connections to obey this structure. However, it is tough to do so in practice~\citep{lassance2018matching}. We will explore this drawback in more detail in Chapter~\ref{chap4}. 

Examples of data with intrinsic structure range from images to citation networks. The most common solution to this problem is to use convolutional layers that we introduce in the next subsection. Another drawback of FC layers is their huge amount of trainable parameters. Considering that the input has dimension $d_\text{input}$ and the output has $d_\text{output}$ dimensions, the amount of trainable parameters of the FC layer is $(d_\text{input}+1)d_\text{output}$. We represent the FC layer in Figure~\ref{chap2:fc_figure}.

\begin{figure}[ht]
  \begin{center}
  \tikzsetnextfilename{chapter2/tikz/fc}%
  \input{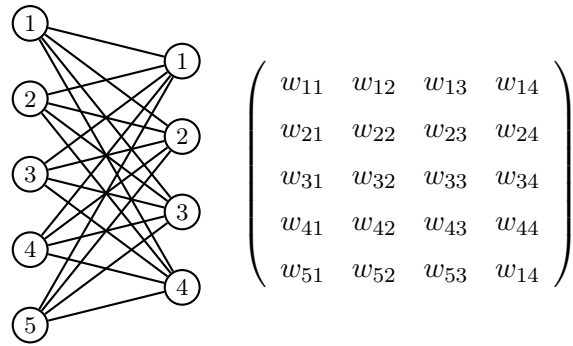}%

    \caption{Representation of a FC layer with $d_\text{input}=5$ and $d_\text{input}=4$.}
    \label{chap2:fc_figure}
  \end{center}
\end{figure}

\subsection{Convolutional layer}\label{chap2:convolutional_layer}

Convolutional layers~\citep{lecun1995convolutional} can be associated with the principle of filtering from signal processing. Instead of generating one full representation of the data by connecting every node of the input to every node of the output the idea of convolutional layers is to convolve several small filters over the coordinates of the input, generating multiple representations of the data on the output. We depict a convolutional layer in Figure~\ref{chap2:conv_unrolled_figure} (unrolled) and Figure~\ref{chap2:conv_rolled_figure} (compressed).

\begin{figure}[ht]
  \begin{center}
  \tikzsetnextfilename{chapter2/tikz/conv}%
  \input{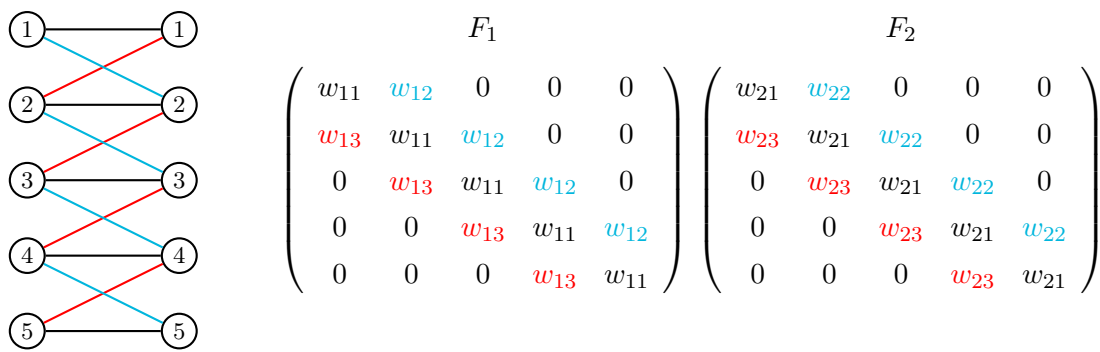}%

    \caption{Representation of an unrolled 1D convolutional layer with $f_w$ = 3, $s=1$, $\featuremaps_{input}=1$, $\featuremaps_{output}=2$, $d_\text{input}=d_\text{output}=5$. }
    \label{chap2:conv_unrolled_figure}
  \end{center}
\end{figure}

\begin{figure}[ht]
  \begin{center}
  \tikzsetnextfilename{chapter2/tikz/conv_compressed}%
  \input{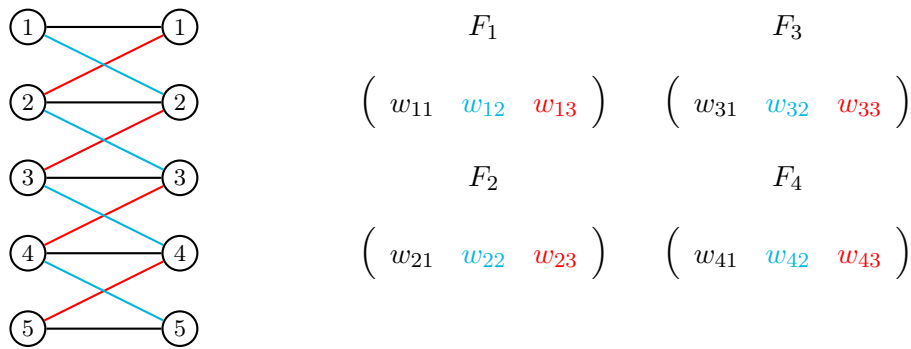}%

    \caption{Representation of a 1D convolutional layer with $f_w$ = 3, $s=1$, $\featuremaps_{input}=1$, $\featuremaps_{output}=4$, $d_\text{input}=d_\text{output}=5$. }
    \label{chap2:conv_rolled_figure}
  \end{center}
\end{figure}

We call \textbf{feature map} each one of those representations ($f'_\text{maps}$), i.e., each intermediate representation from a convolutional layer is composed of multiple feature maps of dimension $d_\text{output}$. In this work we focus on 2D convolutions (except for the illustrative figures~\ref{chap2:conv_unrolled_figure},~\ref{chap2:conv_rolled_figure}, and~\ref{chap2:conv_strided}), and this will be reflected on our notations, but it should be straightforward to extend our work to any dimensional convolutions. In 2D convolutions the filters can be organized as a 4 dimensional weight tensor with dimensions $[\featuremaps_{input},\featuremaps_{output},f_w,f_h]$.  

As these small ($f_w$ by $f_h$) filters are convolved on the input, they respect the intrisic structure of the data and allow the operation to be \textbf{translation equivariant}. 

\begin{definition}[translation equivariance]\label{def:equivariance}
  A function $f'$ is said to be translation equivariant if applying a translation $h$ on the input $\vx'$ is equivalent to applying a consistent mapping $M_h$ on the output of $f'$. Formally we define translation equivariance by:
  \begin{equation}\label{eq:equivariance}
    \forall\vx \in \sD: f'(h(\vx)) \approx M_h(f'(\vx)) 
  \end{equation}
\end{definition}

Being translation equivariant is an essential feature of convolutional layers. It allows for patterns to be recognized even if they are shifted on the data, e.g., in images, the position of the object we want to classify should not change the overall class of the image. Note that translation equivariance and invariance are different features, even if some works use the terms interchangeably. Both translation equivariance and translation invariance are desired features in image classification.

As it was the case of FC layers, we can characterize a convolutional layer by a parametric linear transformation followed by a non-parametrized non-linear function $g$:

\begin{equation}
  f''(\vx') = g(\sH \circledast \vx' + \vb),
\end{equation}
where $\circledast$ is the convolution operator and $\sH$ is the set of $f_w$ by $f_h$ filters of the convolutional layer. Note that this convolution has a parameter $s$ called \emph {stride} that specifies the gap between the center of each convolution. In this work we consider two values of stride, $s=[1,2]$. If $s=1$, then the convolution is performed pixel by pixel, this means that every 2D coordinate is used as the center of every filter. On the other hand, if $s=2$, then the convolution is performed using a quarter of the pixels, and the 2D representation of each feature map is therefore downsampled to a quarter of its original value. In this work we call convolutions with $s=2$ \emph{strided convolutions}, even if all convolutions are by definition strided. We depict a strided convolution in Figure~\ref{chap2:conv_strided}.

\begin{figure}[ht]
  \begin{center}
  \tikzsetnextfilename{chapter2/tikz/strided_conv}%
  \input{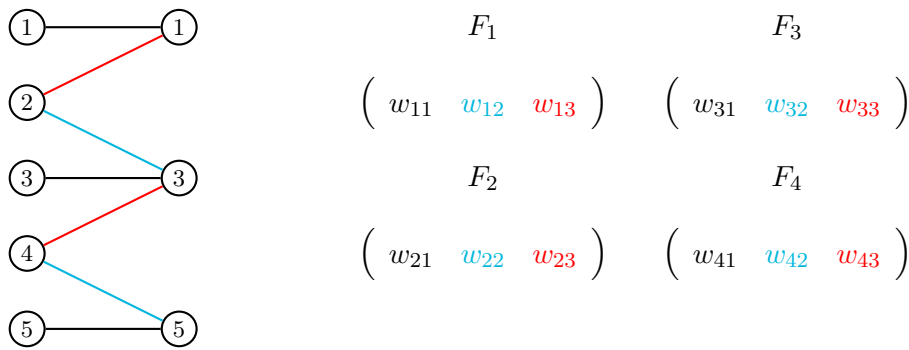}%

    \caption{Representation of a strided 1D convolutional layer with $f_w$ = 3, $s=2$, $\featuremaps_{input}=1$, $\featuremaps_{output}=4$, $d_\text{input}=5$, $d_\text{output}=3$. }
    \label{chap2:conv_strided}
  \end{center}
\end{figure}

We ignore border effects that could be caused by misaligned $w$ and $f_w$ or $h$ and $f_h$ thanks to a zero padding. Zero padding means that if there would be a problem of misalignment between the dimensions we add zeros to the border in order to ensure that either $w_{input}=w_{output}$ or $w_{input}=\frac{w_{output}}{2}$ (in case of $s=2$). The same is done for $h$. 

Works such as~\citep{bontonou2019formalism,vialatte2018convolution} aim to use a similar representation for both FC and convolution layers. This common representation is discussed in more detail in Chapter~\ref{chap4} as this abstraction is fundamental to understanding graph convolutions. 

In the case of convolutional layers the amount of parameters is: $(\featuremaps_{input} f_w f_h + 1) \featuremaps_{output}$, which is typically vastly smaller when compared to the FC layer as $\featuremaps << d$ for both input and output.

\subsubsection{Pooling layer}

Pooling layers have two primary goals.

The first one is to downsample the intermediate representations, allowing to use more coarse or simplified representations as we go deeper in the DNN. Using downsampled representations increases the real filter size of the network (e.g., a 3x3 filter on a downsample representation has a larger visual receptive field than a 3x3 filter on the full representation). In this work, this downsampling operation is executed with a 2D square of side $p$ and an aggregation function $\text{agg}$. The two most common aggregation functions are the max operator (i.e., the output of the pooling operation is the maximum value inside each 2D square) and the average operator (i.e., the output of the pooling operation is the average value inside each 2D square).

The second goal of pooling layers is to provide a weak translation invariance in the radius of the square $p$.

\subsection{Normalization layer}

It is common to add normalization layers as the DNN grows deeper. Normalization layers are added in order to ease the training of the network, even if the theoretical reasoning behind these layers is not exactly solid. In this work, we consider only the most used normalization layer that is the \textbf{batch normalization} layer. 

\subsubsection{Batch normalization layer (BN)}

Batch normalization (BN) layers were first proposed in order to reduce the internal covariate shift~\citep{ioffe2015batch}. Covariate shift happens when the distribution in a learning system changes~\citep{shimodaira2000improving}, which can be a big problem for DNNs as the deeper the network is, the easier it is for a small change in the distribution to affect the deeper layers. The goal is to normalize the intermediate representations $\vx'$ by applying $f_\text{BN}$ so that each feature map has a fixed mean and standard deviation. To do so, we first perform $f''$ to normalize a batch of inputs $\mX'$ to have zero mean and standard deviation equal to one:

\begin{equation}
 \vx'' = f''(\mX') = \frac{\mX' - \mathbf{\mu}_{\mX'}}{\mathbf{\sigma}^2_{\mX'}} 
\end{equation}
where $\mu_{\mX'}$ is a vector with the mean value for each feature map over the batch, and $\sigma^2_{\mX'}$ is a vector with the standard deviation for each feature map over the batch. Note that during the inference phase, one cannot consider that it will receive a batch of inputs, therefore in this work, we always use the running mean and running standard deviation computed during the training phase. Now that the intermediate representation is normalized, we can fix the mean and standard deviation to trainable parameters $\mathbf{\mu}_{f_\text{BN}}$ and $\mathbf{\sigma}^2_{f_\text{BN}}$ as follows:

\begin{equation}
  f_\text{BN}(\mX') = \mathbf{\sigma}^2_{f_\text{BN}} f''(\mX') + \mathbf{\mu}_{f_\text{BN}}
\end{equation}

Various recent works contest the covariate shift claim. In this work, we concentrate on three such studies~\citep{santurkar2018does, zhang2018residual, de2020batch}, but many others exist in recent literature. The first one is a more theoretical paper that argues that BN helps the network to converge by smoothing the optimization landscape instead.

The latter two are more empirical, and while they have a similar argument, they are more focused on the application in residual networks. The authors explain that the BN layers smooth the optimization landscape by biasing the residual network to follow the shallower paths at the start of the training. The smoothing of the optimization landscape is done by giving less importance (sometimes even ignoring) to most of the blocks in the Resnet. They propose to remove the BN layers and to either change the initialization of the parameters or to add a parameter $\alpha$ to the residual connection in order to mimic this behavior and show similar results to batch normalized networks.
 
\section{Tasks and Datasets}\label{chap2:datasets}

In this section, we first introduce the tasks as subsections, and then the datasets used in this thesis for each task as their subsubsections.

\subsection{Image Classification}\label{chap2:image_classification}

Image classification is one of the most common task in computer vision~\citep{dlbook}. In this work, we consider only single-labeled classifications, where the goal is to classify the most prominent object in the image, even if more than one object is visible. In this task, we mostly focus on a measure of performance which is called the top-$k$ classification accuracy. Top-$k$ accuracy is the proportion of samples for which the expected output is among the $k$ highest ranked ones in the obtained output.

\subsubsection{CIFAR-10/100}\label{chap2:cifar10}

CIFAR-10 and CIFAR-100~\citep{krizhevsky2009learning} are tiny (32x32 pixels) image datasets extracted from the 80 million tiny images dataset~\citep{torralba2008tinyimages}. They are mostly used because they offer a good trade-off between complexity (i.e., trivial solutions and DNNs with only FC layers do not provide good performance) and training times. This trade-off means that while trivial solutions, such as KNN on the pixel domain and DNNs with only FC layers, do not provide optimal performance, it is still possible to train near state of the art DNNs with even 3 to 4-year-old GPUs in less than a day. These two characteristics allow for quick/low-cost training and idea iteration on a significant problem. 

The 10 and 100 after the dataset names specify the number of classes of the problem, but this does not mean that CIFAR-10 is a subset of CIFAR-100. The 100 classes of CIFAR-100 may be divided into 20 supergroups of 5 classes. CIFAR-10 classes, on the other hand, may be divided into 2 supergroups: \begin{inlinelist} \item transportation methods \item animals.\end{inlinelist} Note that it is infrequent to use these supergroups in the literature. 

The CIFAR datasets come with a standard train/test split, and it is up to the authors to split the $\trainset$ to define the $\validset$ set. As we discussed in Section~\ref{chap2:definition}, most authors opt for directly training on $\trainset \cup \validset$ using the $\testset$ to optimize the hyperparameters, which of course is not completely fair. Both datasets are composed of 60,000 images, being 50,000 images on the training set (5,000 per class for CIFAR-10 and 500 per class for CIFAR-100), and 10,000 images on the test set (1,000 per class for CIFAR-10 and 100 per class on CIFAR-100). We illustrate some of the examples from the datasets in Figure~\ref{chap2:tab-cifar10}.

\begin{figure}[ht]
  \begin{center}
    \begin{subfigure}[b]{.24\linewidth}
      \centering
      \caption{airplane}
      \includegraphics[]{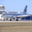} \includegraphics[]{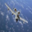} 
    \end{subfigure}%
    \begin{subfigure}[b]{.24\linewidth}
      \centering
      \caption{automobile}
      \includegraphics[]{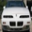} \includegraphics[]{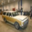} 
    \end{subfigure}%
    \begin{subfigure}[b]{.24\linewidth}
      \centering
      \caption{bird}
      \includegraphics[]{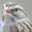} \includegraphics[]{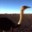} 
    \end{subfigure}%
    \begin{subfigure}[b]{.24\linewidth}
      \centering
      \caption{cat}
      \includegraphics[]{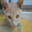} \includegraphics[]{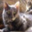} 
    \end{subfigure}%
    \\
    \begin{subfigure}[b]{.24\linewidth}
      \centering
      \caption{deer}
      \includegraphics[]{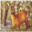} \includegraphics[]{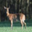} 
    \end{subfigure}%
    \begin{subfigure}[b]{.24\linewidth}
      \centering
      \caption{dog}
      \includegraphics[]{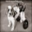} \includegraphics[]{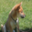} 
    \end{subfigure}%
    \begin{subfigure}[b]{.24\linewidth}
      \centering
      \caption{frog}
      \includegraphics[]{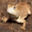} \includegraphics[]{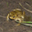} 
    \end{subfigure}%
    \begin{subfigure}[b]{.24\linewidth}
      \centering
      \caption{horse}
      \includegraphics[]{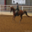} \includegraphics[]{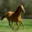} 
    \end{subfigure}%
    \\
    \begin{subfigure}[b]{.24\linewidth}
      \centering
      \caption{ship}
      \includegraphics[]{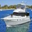} \includegraphics[]{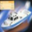} 
    \end{subfigure}%
    \begin{subfigure}[b]{.24\linewidth}
      \centering
      \caption{truck}
      \includegraphics[]{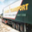} \includegraphics[]{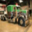} 
    \end{subfigure}%
    \caption{Illustrative examples from the CIFAR-10 dataset.}
    \label{chap2:tab-cifar10}
  \end{center}
\end{figure}

\subsubsection{SVHN}\label{chapt2:SVHN}

The Street View House Numbers (SVHN)~\citep{netzer2011reading} is a dataset composed of images taken from House numbers, where the goal is to identify the most prominent number on the cropped image, where sometimes more than one number appears on the image. The dataset was first proposed as a harder alternative to the MNIST dataset~\citep{lecun1998gradient} on the digit recognition task. SVHN is composed of 10 classes (1 for each digit) and comes with a standard split of 73,257 training images and 26,032 test images. Note that an additional 531,131 extra images are available and can be incorporated into the dataset. We depict some examples of this dataset in Figure~\ref{chap2:fig-svhn}.

\begin{figure}[ht]
  \begin{center}
    \includegraphics{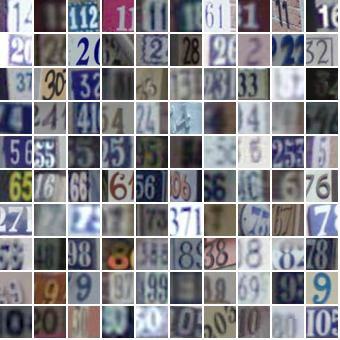}
    \caption{Illustrative examples from the SVHN dataset, extracted from the dataset website~\url{http://ufldl.stanford.edu/housenumbers/}}
    \label{chap2:fig-svhn}
  \end{center}
\end{figure}

\subsubsection{Imagenet and variants}\label{chap2:imagenet}

As introduced in Chapter~\ref{chap1}, one of the reasons for the quick expansion of DNNs and deep learning was that these systems won the CVPR2012 Large Scale Visual Recognition (LSVR) challenge~\citep{krizhevsky2012imagenet,russakovsky2015imagenet}. This challenge has been used as the de-facto benchmark since then, and most works (including this one) still use the $\trainset$ and $\validset$ from that competition for benchmarking purposes. In other words, while it is common place to call this dataset Imagenet, a more appropriate name would be LSVR2012 Imagenet, as Imagenet is the database where the images were extracted. Note that it is uncommon to use the $\testset$ of this dataset.

The Imagenet dataset is one of the most complex challenges in computer vision given not only its high number of classes (1,000) and images (1.2 million images on the training set and 50,000 images on the validation set) but also the high image resolution of the dataset (it varies from 75x56 to 4,288x2,848). Treating images of high resolution takes considerably more time to process, but tends to lead to better results. In this work, due to computational constraints, we use the Imagenet32 variant~\citep{chrabaszcz2017downsampled}, which downscales all images to the same size of the CIFAR datasets (32x32) allowing us to have a good trade-off between the cost of training and the relevance of the problem. We depict some examples of Imagenet in Figure~\ref{chap2:fig-imagenet} and Figure~\ref{fig:example_alexnet_dannet}.

\begin{figure}[ht]
  \begin{center}
    \begin{subfigure}[b]{.3\linewidth}
      \centering
      \includegraphics[height=.6\linewidth, width=\linewidth]{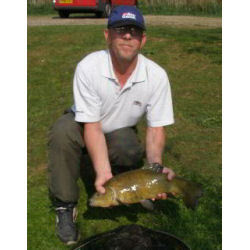}
      \caption{tench/tinca (fish)}
    \end{subfigure}%
    \begin{subfigure}[b]{.3\linewidth}
      \centering
      \includegraphics[height=.6\linewidth, width=\linewidth]{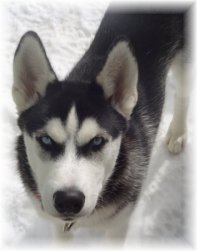}
      \caption{eskimo dog/husky}
    \end{subfigure}
    \begin{subfigure}[b]{.3\linewidth}
      \centering
      \includegraphics[height=.6\linewidth, width=\linewidth]{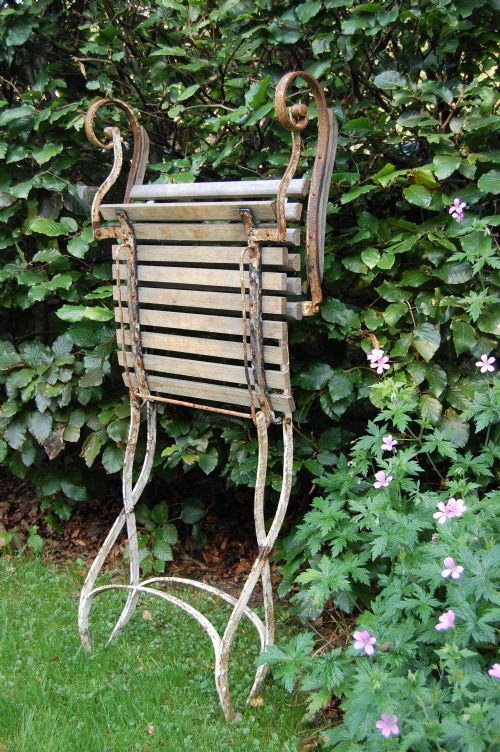}
      \caption{folding chair}
    \end{subfigure}
    \caption{Illustrative examples from the Imagenet dataset.}
    \label{chap2:fig-imagenet}
  \end{center}
\end{figure}

\subsection{Image retrieval}\label{chap2:ir}

Image retrieval differs from image classification as the focus is mostly on the mean average precision (mAP) instead of the top-$k$ accuracy. 
\begin{definition}[mean average precision (mAP)]
  To define mAP, we first need to define average precision (AP). To compute the average precision of an an example $\vx_q \in \queryset$, we use the sum of the precision at each correctly retrieved image of the support set ($\supportset$) based on its ranking $k$:
  \begin{equation}
    AP(\vx_q) = \frac{\sum_{k=1}^{|\supportset|}{P(k) \; r(k) }}{|\relevant_q|}
  \end{equation}
  where $P(k)$ is the precision at ranking $k$ (the percentage of correctly retrieved images until rank $k$), $\relevant_q$ is the subset of $\supportset$ containing the relevant examples to $\vx_q$ and $r(k)$ is an indicator function returning 1 if the item at ranking $k$ is in $\relevant$, and 0 otherwise. Given the average precision we can now define the mAP over a query set $\queryset$ as: 
  \begin{equation}
    mAP(\queryset) = \frac{\sum_{\vx_q \in \queryset}{AP(\vx_q)}}{|\queryset|}
  \end{equation}

\end{definition}

In other words, mAP grows when the top ranked retrieved images from the support set ($\supportset$) are from closer location and/or present the same objects as the considered query image $\vx_q \in \queryset$. Note that doing simple classification and then outputting all images from the found class would be heavily penalized by this measure in case of misclassification. This is why most methods rely on alternative solutions.

\subsubsection{Revisited oxford5k and paris6k}

The revisited oxford5k ($\mathcal{R}\text{Oxf}$) and revisited paris6k($\mathcal{R}\text{Par}$) were first introduced in~\citep{radenovic2018revisiting}, in order to better represent the image retrieval task when compared to their original versions~\citep{philbin2007object,philbin2008lost}. Images are divided into $\supportset$ and $\queryset$. For each object that we want to retrieve (13 for Oxford and 12 for Paris), the images of the support set are divided, depending on the quality of the image and how apparent are the objects we want to retrieve, into three sets: \begin{inlinelist} \item easy ($\easyset_q$) \item hard ($\hardset_q$) \item unclear ($\unclearset_q$) \end{inlinelist}. The task is then divided into two difficulties:

  \begin{enumerate}
\item Medium: Easy and hard images have to be retrieved and unclear images are disregarded (i.e., are not taking into account for computing mAP). More formally: $\supportset_q = \supportset - \unclearset_q$  and $\relevant_q = \hardset_q \cup \easyset_q$.
\item Hard: Only the hard images have to be retrieved while easy and unclear images are disregarded. More formally $\supportset_q = \supportset - (\unclearset \cup \easyset)$ and $\relevant_q = \hardset_q$.
\end{enumerate}

The datasets are composed of 70 query images ($\queryset$) and a $\supportset$ of 4,993 images for the Oxford dataset and 6,332 for the Paris one. We depict example images from the datasets in Figure~\ref{chap2:fig-oxford} (oxford5k) and in Figure~\ref{chap2:fig-paris} (paris6k).  

\begin{figure}[ht]
  \begin{center}
    \begin{subfigure}[b]{.33\linewidth}
      \centering
      \includegraphics[height=\linewidth, width=\linewidth]{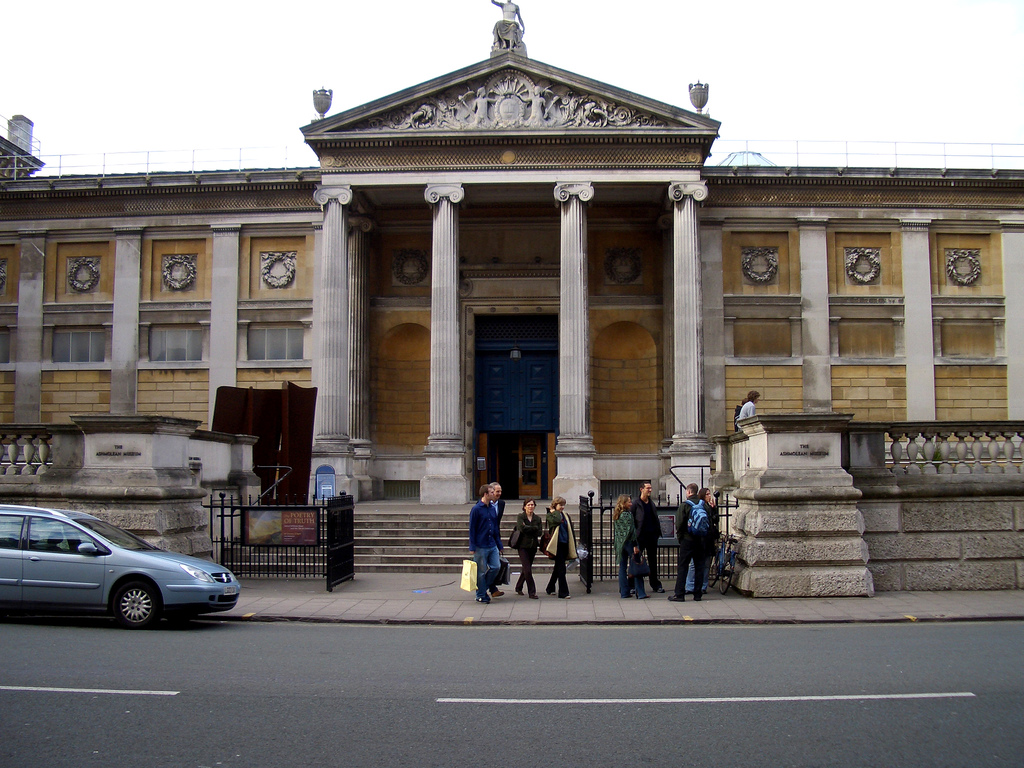}
      \caption{Ashmolean museum}
    \end{subfigure}%
    \begin{subfigure}[b]{.33\linewidth}
      \centering
      \includegraphics[height=\linewidth, width=\linewidth]{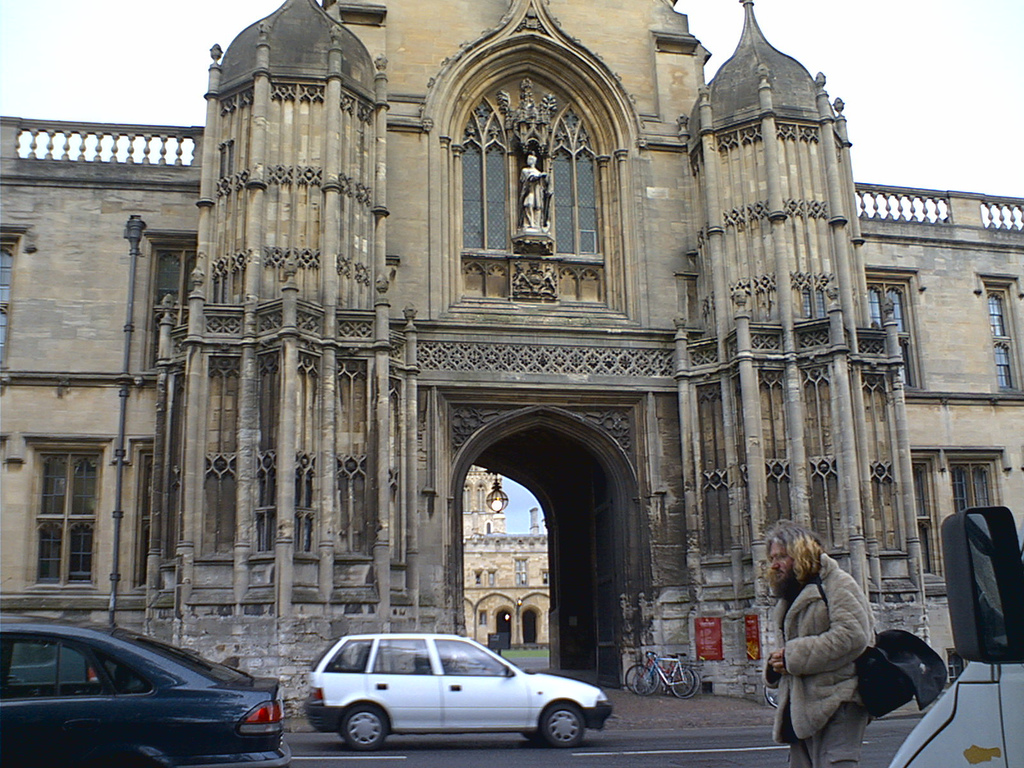}
      \caption{Christ church}
    \end{subfigure}
    \begin{subfigure}[b]{.33\linewidth}
      \centering
      \includegraphics[height=\linewidth, width=\linewidth]{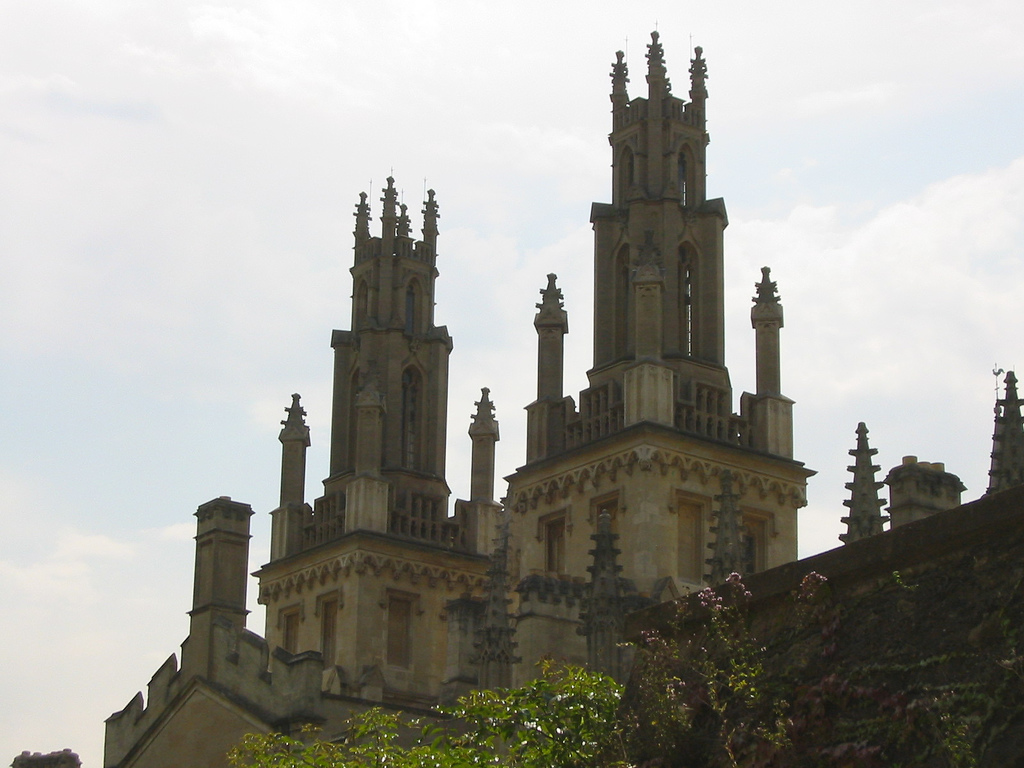}
      \caption{Magdalen college}
    \end{subfigure}
    \caption{Illustrative examples from the revisited oxford5k dataset.}
    \label{chap2:fig-oxford}
  \end{center}
\end{figure}

\begin{figure}[ht]
  \begin{center}
    \begin{subfigure}[b]{.33\linewidth}
      \centering
      \includegraphics[height=\linewidth, width=\linewidth]{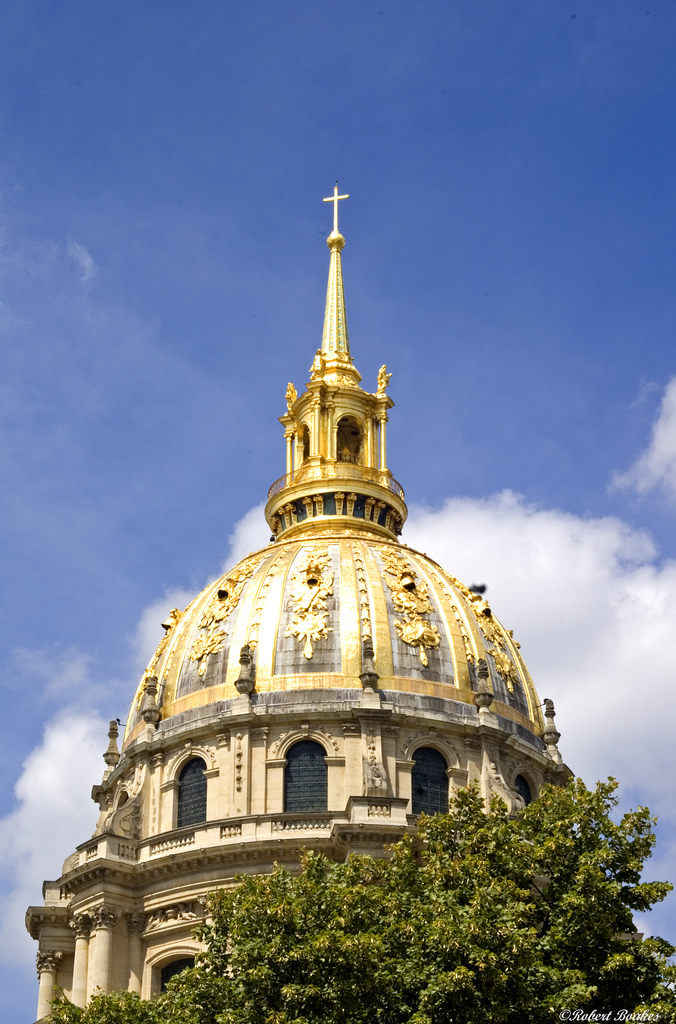}
      \caption{Les invalides}
    \end{subfigure}%
    \begin{subfigure}[b]{.33\linewidth}
      \centering
      \includegraphics[height=\linewidth, width=\linewidth]{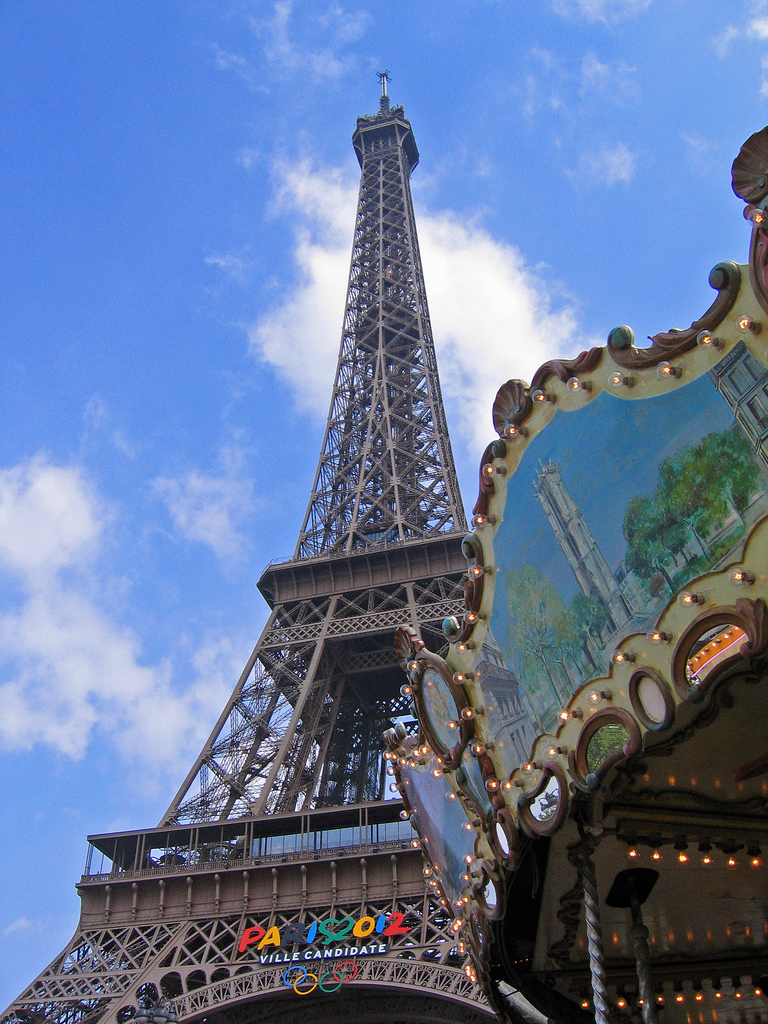}
      \caption{Tour Eiffel}
    \end{subfigure}
    \begin{subfigure}[b]{.33\linewidth}
      \centering
      \includegraphics[height=\linewidth, width=\linewidth]{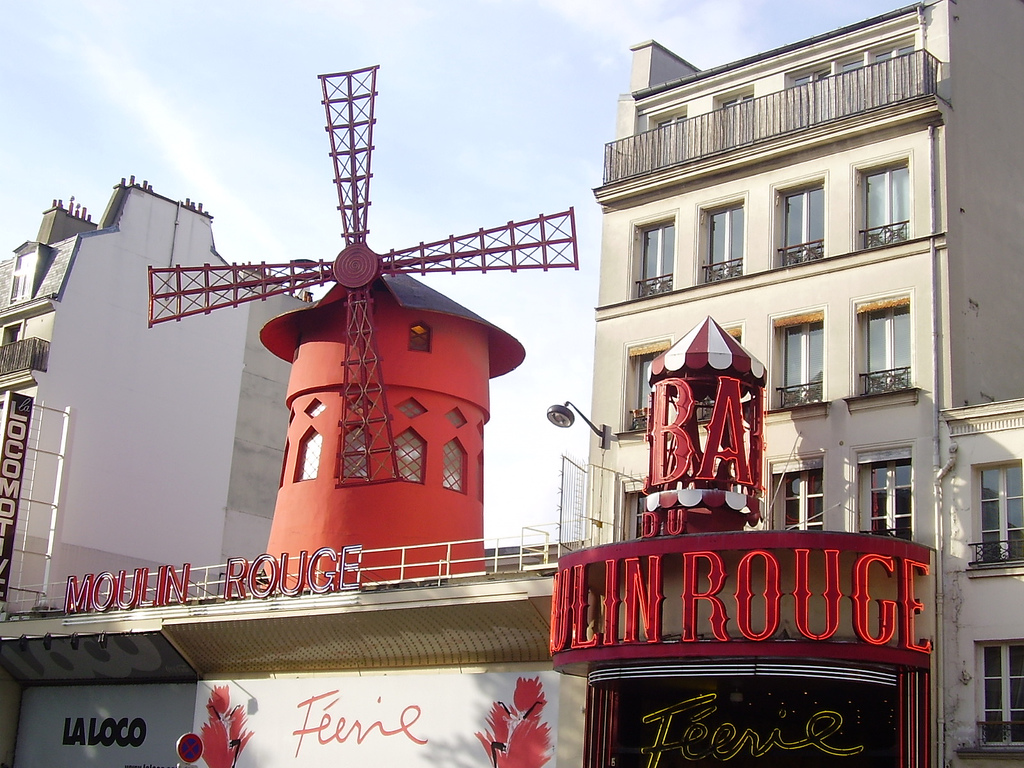}
      \caption{Moulin Rouge}
    \end{subfigure}
    \caption{Illustrative examples from the revisited paris6k dataset.}
    \label{chap2:fig-paris}
  \end{center}
\end{figure}

\subsection{Vision-based localization}\label{chap2:vbl}

Vision-Based Localization (VBL) is at the same time closely related to the image retrieval problem as we want to find the landmarks on the image that allows us to recognize the location and to the image classification problem as it could be seen as a regression task where given an image we want to return the physical location of the camera. In a more formal way, VBL refers to the problem of retrieving both the location and orientation (pose) of the camera based on a query image. We refer the readers to~\citep{piasco2018vblsurvey} for a review on the subject.

\subsubsection{Adelaide and Sydney datasets}

In this work, we focus mostly on the datasets we introduced in~\citep{lassance2019improved}. These datasets were constructed by collecting images from the Mapillary API\footnote{\url{https://www.mapillary.com/developer/api-documentation/}}, which contains data that was publicly sourced over time using dashcams (i.e., cameras mounted on the windshield of vehicles). The images are extracted from videos, meaning that they can be divided into sequences. We extracted two sets of images from the road imagery of Australian cities. The first (Adelaide) covers the Central Business District (CBD) area of Adelaide, Australia. The second set is collected around the Greater Sydney region and covers an area of around 200km$^2$.  Since the data is publicly sourced, there are some extra difficulties in these datasets (that are kept to simulate real life scenarios better):
\begin{enumerate}
  \item There are viewpoint, illumination and dynamic changes. This is expected to happen in real-life scenarios too (e.g., winter images would have snow while summer images may be brighter),
  \item In the Sydney dataset, some of the sequences were generated using different equipment (e.g., panoramic cameras) and different positioning (e.g., some of the images are not of dashcams but cameras mounted on the side windows) from the ones used in traditional VBL problems. 
\end{enumerate}

In addition to imagery, the collected data provides sequence information and GPS. We depict the GPS tracks in Figure~\ref{chap2:fig_adelaide} and Figure~\ref{chap2:fig_sydney}

\begin{figure}
  \centering
  \includegraphics[width=0.5\columnwidth]{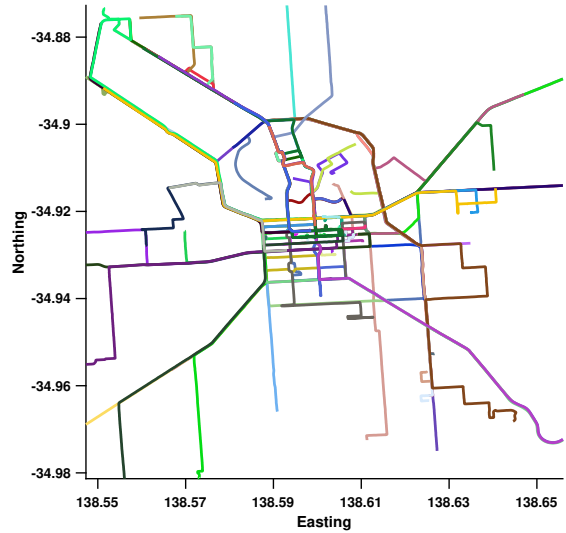}
  \caption{GPS Tracks of image sequence collected around Adelaide CBD from Mapillary. Figure and caption extracted from~\citep{lassance2019improved}}
  \label{chap2:fig_adelaide}.
\end{figure}

\begin{figure}
  \centering
  \includegraphics[width=0.5\columnwidth]{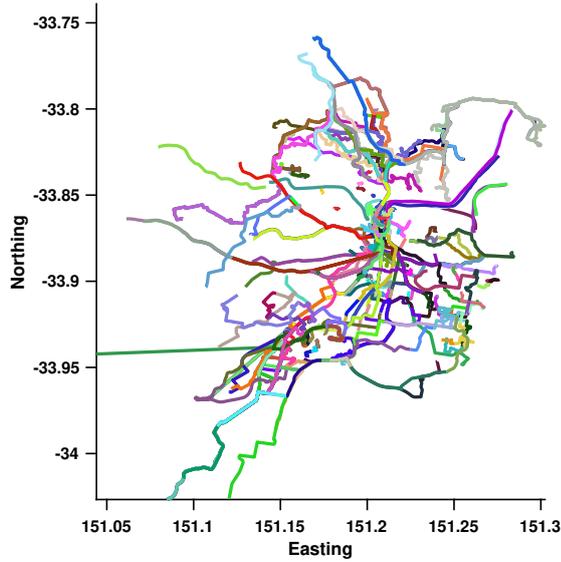}
  \caption{GPS Tracks of image sequence collected around Sydney from Mapillary. Figure and caption extracted from~\citep{lassance2019improved}}
  \label{chap2:fig_sydney}.

\end{figure}

The Adelaide dataset is divided into one support set and two query sets ($\validset$ and $\testset$). For the Sydney database, we split the images into one support and two query sets ($\easyset$ and $\hardset$). The split into easy and hard was performed due to the greater difficulty of the Sydney set. Note that the definition of $\easyset$ and $\hardset$ is not the same as in the case of Image retrieval problems. Statistics for each dataset are presented in Table~\ref{dataset-summary}. 

\begin{table}[ht!]
\centering
\caption{Summary of the VBL datasets.}
\begin{tabular}{|c|cc|}
\hline
City             & \multicolumn{2}{c|}{Adelaide} \\ \hline
                 & \# Sequences     & \# Images    \\ \hline
$\supportset$ & 44              & 24,263       \\ 
$\validset$ & 4               & 2,141        \\ 
$\testset$       & 5               & 1,481        \\ \hline
                 & \multicolumn{2}{c|}{Sydney} \\ \hline
                 & \# Sequences     & \# Images    \\ \hline
$\supportset$ & 284              & 117,860       \\ 
$\easyset$       & 5               & 1,915      \\ 
$\hardset$       & 5        & 2,285   \\ \hline
\end{tabular}
\label{dataset-summary}
\end{table}

\subsection{Neurological task classification}

In this work we also consider datasets composed of fMRI (functional Magnetic Resonance Imaging) scans. The main goal of using this type of data is to study architectures that leverage the underlying structure, even though the latter is not as simple as a standard euclidean 2D space. We will delve into more details on this in Chapter~\ref{chap4}.

\subsubsection{Pines}

In~\citep{lassance2018matching}, we introduce the use of the Pines dataset in the context of deep neural networks for the first time. The task of the PINES dataset is to identify the emotional rating a subject gives to a picture by using the fMRI scan of its brain. There are 182 subjects in the study~\citep{chang2015sensitive}. To generate the dataset, we fetched first-level statistical maps (beta images) of each individual with minimal and maximal ratings from \url{https://neurovault.org/collections/1964/}. Final volumes used for classification contain 369 signals per sample distributed on a 16mm cube. The dataset is composed of a $\trainset$ containing 1,949 samples and a $\testset$ of 1,010 samples. The samples are divided into two classes, one for the minimal rating and the other for the maximal rating. 

\subsection{Document classification on citation networks}\label{chap2:citation_networks}

In this work, we also consider the problem of classifying scientific papers on a citation network. All datasets are constructed using the bag-of-words~\citep{sivic2005discovering} method. The bag-of-words approach consists of using indicator vectors over a dictionary as the features of each document (i.e., the amount of times each word of the dictionary appears in the document). In this work, we use binarized versions of the bag-of-words vectors as it is commonplace in recent literature. All the datasets presented for this task are accompanied by a citation graph connecting documents that either cite or are cited by other documents. 

The datasets presented in this subsection have no definitive split into $\trainset$ and $\testset$, but it is common to use the split from~\citep{yang2016revisiting} for benchmarking. We delve into more detail on this choice and these datasets in Chapter~\ref{chap4}. In particular, we will show that they suffer from many limitations and biases, despite being the cornerstone of benchmarking in the field of graph-supported semi-supervised learning~\citep{shchur2018pitfalls}.

\subsubsection{cora}\label{chap2:cora}

A dataset of machine learning papers, composed of 2,708 documents, divided into seven classes, with a dictionary of 1,433 words (i.e., 1,433 features per document). This dataset was first proposed in~\citep{sen2008collective}, using articles from the cora database~\citep{mccallum2000automating}.

\subsubsection{citeseer}

The citeseer dataset is composed of 3,312 documents with a dictionary of 3,703 words and divided into six classes. This dataset was also first proposed in~\citep{sen2008collective} using data from the citeseer database~\citep{giles1998citeseer}.

\subsubsection{pubmed}

A dataset of diabetes medicine research~\citep{namata2012query}. This dataset is composed of 19,717 papers with a dictionary of 500 words and divided into 3 classes depending on the type of diabetes addressed in the publication.

\section{Compression}\label{chap2:compression}

In the previous sections we have introduced DNNs and the tasks and datasets on which we are going to apply them in this work. Now let us take a step back and consider the computational complexity of DNNs. Indeed, to achieve state of the art results, CNNs will often rely on a large number of trainable parameters, and considerable computational complexity. This is why there has been a lot of interest in the past few years towards the compression of CNNs, so that they can be deployed onto embedded systems or in real-time settings. The purpose of this section is to analyze the techniques that allow for the efficient compression of DNNs in order to reduce their computational complexity and memory footprint, while maintaining a high level of accuracy.

Prominent areas in neural network compression include distilling knowledge from a larger teacher network to a smaller (student) network~\citep{hinton2014distillation,romero2015fitnets,koratana2019lit,park2019rkd,lassance2019deep}, binarizing (or quantifying) weights and activations~\citep{hubara2016bnn,bulat2019xnor}, pruning network connections during or before training~\citep{li2017pruning,ardakani2017sparsely}, changing the way convolutions are performed~\citep{wu2018shift,hacene2019attention} and many others. 

Most of the works in this area focus only on reducing complexity and footprint in the inference phase. Indeed, many of these methods increase the cost of the training phase. This choice is justified because the training of the network is considered to be done in an unconstrained environment, while the inference phase would be run in a highly constrained environment, which is applicable to most (yet not all) practical applications.

In this work, we focus only on distillation and in restricting the possible convolutions in DNNs and refer the reader to~\citep{hacene2019processing} for a more detailed review on this subject. 

\subsection{Distillation}\label{chap2:distillation}

Distillation based approaches aim at distilling knowledge from a pre-trained larger network that we call teacher to a smaller yet to be trained network called student. More formally, let $T$ and $S$ denote the architectures of the teacher and student. The goal is to transfer knowledge from $T$ to $S$. For presentation simplicity, we assume that both architectures always generate the same number of intermediate representations, even if they are not from the same depth in the network architecture. 

Distillation methods can be divided into two groups depending on how they perform the distillation: \begin{inlinelist} \item Individual Knowledge Distillation (IKD) \item Relational Knowledge Distillation (RKD) \end{inlinelist}. IKD methods consider each example of the training set separately, while RKD methods use information between sets of examples in order to distill knowledge. More formally, we can define the objective function of the student networks trained with knowledge distillation as:

\begin{equation}
  \mathcal{L} = \mathcal{L}_\text{task} + \lambda_{\text{KD}} \cdot \mathcal{L}_\text{KD}\;,\label{distill_loss}
\end{equation}
where $\mathcal{L}_\text{task}$ is typically the same loss that was used to train the teacher (e.g., cross-entropy), $\mathcal{L}_\text{KD}$ is the distillation loss and $\lambda_{\text{KD}}$ is a scaling parameter to control the importance of the distillation with respect to that of the task. We now describe some IKD techniques. The first we describe is often called HKD (Hinton Knowledge Distillation)~\citep{hinton2014distillation}, and it is focused on the output $\hat{\vy}$ of the networks, forcing $S$ to mimic the output of $T$. The knowledge acquired by the teacher during the training phase is therefore diffused throughout the student network due to the backpropagation of weights.

Recent works~\citep{romero2015fitnets,koratana2019lit} advise that even if the knowledge is diffused through backpropagation, it is better also to force $S$ to mimic the intermediate representations of $T$ in order to improve the performance of $S$. However, as IKD treats each example individually, it can only do layer-wise mimicking if the student and the teacher have intermediate data representations with the same dimension~\citep{koratana2019lit}. This drawback restricts the architecture choice of the student networks. In an attempt to avoid this limitation, the authors of~\citep{romero2015fitnets} propose to include affine transformations to ensure that the intermediate representations have the same size. The authors introduced extra layers meant to perform distillation during training. These transformations are therefore discarded during the inference of $S$, which could be a problem as these transformations will encode part of the knowledge that comes from the teacher network. In Figure~\ref{chap2:fig-lit}, we depict the architecture from~\citep{koratana2019lit} where the output of each block is compared.

\begin{figure}[ht]
  \centering
  \includegraphics[width=\linewidth]{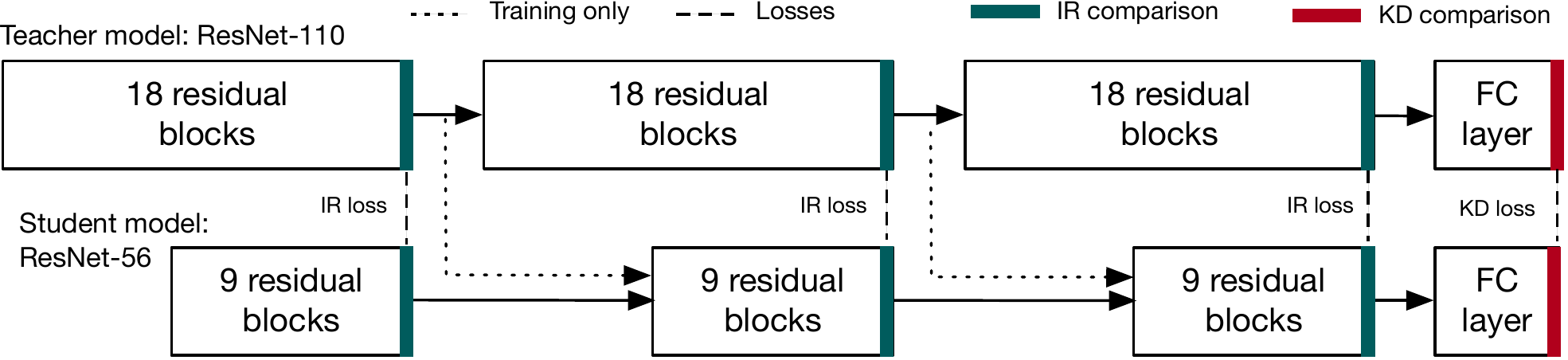}
  \caption{Depiction of the LIT architecture, IR comparison refers to a sample-wise L2 loss between the blocks, while KD comparison refers to the loss proposed in~\citep{hinton2014distillation}. Figure extracted from~\citep{koratana2019lit} ©2019 PMLR.}
  \label{chap2:fig-lit}
\end{figure}

More formally, we denote $\mX \in \trainset$ a batch of input examples and $\sX'$ the set of intermediate representations generated using $\mX$ that are used for inferring knowledge. As previously described in Section~\ref{chap2:definition}, when an input $\vx$ is processed, a series of inner representations $\vx', \vx'', \dots$ is generated. IKD approaches will directly compare the inner representations of both teacher and student when processing $\vx$. We therefore define the IKD loss from~\citep{hinton2014distillation,romero2015fitnets,koratana2019lit} as follows: 
\begin{equation}
\mathcal{L}_\text{IKD} = \sum_{\vx' \in \sX'}{\mathcal{L}_d(\vx'^T,\vx'^S)},
\end{equation}
where $\vx'^A$ is the intermediate representation of architecture $A$ and $\mathcal{L}_d$ is in most cases a measure of the distance between its arguments, which requires that they have the same dimensions.

However, forcing architectures for teacher and student to have intermediate representations with the same dimensions is not always desirable. Indeed, recent recommendations on DNN architecture design~\citep{tan2019efficientnet} show that efficient neural network scaling considers three main aspects: \begin{inlinelist} \item network depth (number of layers); \item network width (number of $\featuremaps$ per layer); \item resolution ($w$ by $h$ size of the input and intermediate representations).\end{inlinelist}. The two latter points are directly related to the dimensions of the intermediate representations and are therefore incompatible with IKD techniques.
  
To mitigate this drawback, recent works such as~\citep{park2019rkd,lassance2019deep} have introduced distillation that can be performed in both dimension-agnostic fashion and without adding extra transformations to the architecture of the DNNs. These methods are a part of the RKD group, where the focus is in the relative distances between the intermediate representations rather than on their exact positions. We delve into more details in RKD based compression in Section~\ref{chap5:gkd}, where we describe Graph Knowledge Distillation (GKD), a method that we proposed in~\citep{lassance2019deep} to improve performance of RKD based approaches.

\subsection{Shift attention layers (SAL)}~\label{chap2:SAL}

In~\citep{hacene2019attention}, we propose a more memory and computation efficient variation of convolutional layers that we call SAL, which we will detail in this subsection. Recently, the authors of~\citep{wu2018shift} have proposed to replace the convolution operator with the combination of shifts and 1x1 convolutions, an approach they called shift convolutions. In other words, shift convolutions propose to limit the filter construction of convolutions to just one weight per filter, while still keeping the original filter shape. In this first work, all the shifts were hand-crafted (i.e., decided arbitrarily before training). Note that previous works have shown that these shift convolutions are well suited for computationally constrained devices~\citep{hacene2018quantized}.

In order to increase the performance of the shift convolution, we introduced the Shift Attention Layer (SAL), which can be seen as a selective shift layer. SAL starts with vanilla convolution (i.e., with all the weights kept for each filter) and learns to transform it into a shift layer throughout the training of the network function. The introduced SAL use an attention mechanism~\citep{vaswani2017attention} that selects the best shift for each feature map of the architecture. Note that this could also be considered a pruning technique as we start with all the weights for each filter and then choose the one to keep. It can significantly outperform the original shift layers from~\citep{wu2018shift} at the cost of requiring more parameters during the training phase. We note that it still ends with fewer parameters during the inference phase.

We depict SAL in Figure~\ref{chap2:SAL_figure} and provide the code for reproducing our experiments at~\url{https://github.com/eghouti/SAL}. In the next paragraphs, we give more details about their core principle.

\begin{figure}[ht!]
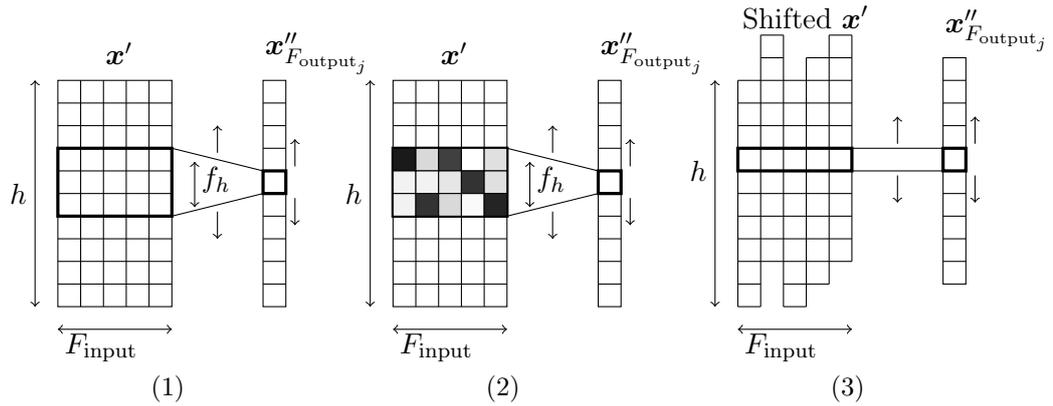

  \begin{center}
      \begin{subfigure}[b]{.3\linewidth}
        \centering
  \tikzsetnextfilename{chapter2/tikz/SAL1}%
  \input{chapter2/tikz/SAL1.tex}%

        \vspace{-0.9cm}
        \caption{(1)}
      \end{subfigure}%
      \begin{subfigure}[b]{.3\linewidth}
        \centering
  \tikzsetnextfilename{chapter2/tikz/SAL2}%
  \input{chapter2/tikz/SAL2.tex}%

        \vspace{-0.9cm}
        \caption{(2)}
      \end{subfigure}
      \begin{subfigure}[b]{.3\linewidth}
        \centering
  \tikzsetnextfilename{chapter2/tikz/SAL3}%
  \input{chapter2/tikz/SAL3.tex}%

        \vspace{-0.9cm}
        \caption{(3)}
      \end{subfigure}
    \caption{Overview of the proposed method: we depict here the computation for a single output feature map $\featuremaps_{\text{output}_i}$, considering a 1d convolution and its associated shift version. Panel (1) represents a standard convolutional operation: the weight filter $f \in \sH$ containing $f_h \featuremaps_\text{input}$ weights is moved along the spatial dimension ($h$) of the input to produce each output in $\vx''$. In panel (2), we depict the attention tensor $\mA$ on top of the weight filter: the darker the cell, the most important the corresponding weight has been identified to be. At the end of the training process, $\mA$ should contain only binary values with a single 1 per slice $\mA_{f,\cdot}$. In panel (3), we depict the corresponding obtained shift layer: for each slice along the input feature maps ($\featuremaps_\text{input}$), the cell with the highest attention is kept and the others are discarded. As a consequence, the initial convolution with a feature height $f_h$ has been replaced by a convolution with a feature height 1 on a shifted version of the input $\vx'$. As such, the resulting operation in panel (3) is exactly the same as the shift layer introduced in~\citep{wu2018shift}, but here the shifts have been trained instead of being arbitrarily predetermined. Figure and caption adapted from~\citep{hacene2019attention}}.
    \label{chap2:SAL_figure}
  \end{center}
\end{figure}

\subsubsection{Methodology}

We propose to add a selective tensor $\mA$ to standard convolutional layers, in order to identify which weight should be kept for each filter $f \in \sH$ where $\sH$ is the set of filters of the convolutional layer. As such, we introduce $\mA \in  \mathbb{R}^{\sH \times (f_w \times f_h)}$ a tensor that for each filter in $\sH$ contains a matrix of the same size of the filter ($f_w$ by $f_h$). Each submatrix of $\mA$ is then normalized so that each value is between 0 and 1, with the sum of values in the matrix being equal to 1. The values in  each submatrix from $\mA$ represent how important is the corresponding weight $\mW$ in the filter from $\sH$. During the training process, the values of $\mA$ are pushed to binarization, until the end of the training process when they are binarized. With the weights in $\mA$ binarized, only the corresponding weight for each filter $f \in \sH$ are kept.

More precisely, each slice $\mA_{\sH_i,\cdot,\cdot}$ is normalized with the softmax function from Equation~\ref{chap2:softmax_equation} with a temperature $T$. The temperature is decreased smoothly along the training process in order to force the binarization of the softmax outputs. Note that in order to force the mask $\mA$ to be selective, we first normalize each slice $\mA_{\sH_i,\cdot,\cdot}$ so that it has a standard deviation ($\sigma^2$) of 1. We summarize the training process of one SAL layer in Algorithm~\ref{chap2:algo_SAL}. At the end of training, the selected weight in each filter $\sH_i$ corresponds to the maximum value in $\mA_{\sH_i,\cdot,\cdot}$.

\begin{algorithm}[ht!]
\caption{Pseudo-algorithm for one SAL layer, adapted from~\citep{hacene2019attention}}
\textbf{Inputs}: Input tensor $\vx'$,\\ Initial softmax temperature $T$, Constant $\alpha < 1$.
\begin{algorithmic}
\For{ each training iteration}
\State $T \leftarrow \alpha T$
\State $\mA' \leftarrow \mA$
\For{$i:=1$ to $|\sH|$}
\State $\mA'_{\sH_i,\cdot,\cdot} \leftarrow \frac{\mA_{\sH_i,\cdot,\cdot}}{\sigma^2(\mA_{\sH_i,\cdot,\cdot})}$

\State $\mA'_{\sH_i,\cdot,\cdot} \leftarrow Softmax ({\mA'_{\sH_i,\cdot,\cdot}})$

\State $\sH'_i \leftarrow \sH_i \cdot \mA'_{\sH_i}$

\EndFor

\State Compute standard convolution as described in Section~\ref{chap2:convolutional_layer} using input tensor $\vx'$ and the set of filters $\sH'$ instead of $\sH$ .
\State  Update $\sH$ and $\mA$ via back-propagation.

\EndFor
\end{algorithmic}
\label{chap2:algo_SAL}
\end{algorithm}

Note that an advantage of SAL is that the number of filters per direction is not fixed as it was the case in the vanilla shift layers. However, this advantage comes with a drawback: it increases memory usage in order to retain which shift kept for each feature map. We note that this drawback is taken into account in our experiments.

\subsubsection{Experiments on CIFAR-10/100}

Now, we present the benchmarking protocol and a performance comparison of SAL and other shift layers on CIFAR-10/100. We refer the reader to~\citep{hacene2019attention} for a more extensive experimental discussion, including tests on the Imagenet dataset and comparison against pruning methods. 

In order to promote a fair comparison with the other shift layers~\citep{wu2018shift,jeon2018constructing}, we use the same hyperparameters when possible, for example: 
\begin{itemize} 
  \item Epochs: 300 epochs;
  \item Learning rate scheme: the learning rate starts at $0.1$ and is divided by $10$ after each $100$ epochs;
  \item Batch size: 128;
  \item Temperature $T$: $T$ starts at 6.7 and after each parameter update (step) it is updated so that it ends at $T_\text{final}=0.02$.
\end{itemize}

We present in Table~\ref{chap2:table_sal} a comparison of SAL against the vanilla shift layer in terms of accuracy and number of parameters needed during inference. We observe that our method achieves a better accuracy with fewer parameters than the baseline and other shift-module based methods. 

\begin{table}[ht]
  \begin{center}      
    \caption{Comparison of accuracy and number of parameters between 3x3 convolution, vanilla shift~\citep{wu2018shift}, interpolation shift~\citep{jeon2018constructing}, and ours on CIFAR10 and CIFAR100.}
    \label{chap2:table_sal}
    {\renewcommand{\arraystretch}{1.3}%
      \begin{adjustbox}{max width=\columnwidth}        
        \begin{tabular} { | c | c || c | c | c | c | c |} 
    
        \cline{4-7}
        \multicolumn{3}{l|}{} &  \multicolumn{2}{c|}{CIFAR10} & \multicolumn{2}{c|}{CIFAR100} \\
        \hline
        Network & $\featuremaps_\text{initial}$ & Convolutional layer & Accuracy & Params (M) &  Accuracy & Params (M) \\
      \hline
        \hline
        Resnet20  & 16 & 3x3 Convolution & $94.66\%$ & $1.22$ & $73.7\%$ & $1.24$ \\
        \hline
        Resnet110 & 16 & Vanilla shift~\citep{wu2018shift}   & $93.17\%$  & $1.2$ &  $72.56\%$ & $1.23$ \\
        \hline
      
        Resnet20  & 88& Interpolation shift~\citep{jeon2018constructing}   & $94.53\%$  & $0.99$ &  $76.73\%$ & $1.02$ \\
        \hline
        Resnet20  & 83  & SAL (ours)   & $\mathbf{95.52}\%$  & $\mathbf{0.98} $ &  $\mathbf{77.39}\%$ & $\mathbf{1.01}$ \\
        \hline
     
      \end{tabular}
      \end{adjustbox}

      }
  \end{center}
\end{table}

\section{Robustness}\label{chap2:robustness}

As we previously discussed in Chapter~\ref{chap1}, DNNs can provide state-of-the-art performance in many machine learning challenges. This success can be justified based on their universal approximation properties~\citep{hornik1989approximator}, which allow them to approximate any function that associates each training set input to its corresponding class. However, this is also a double-edged sword, as the resulting function may not handle well domain shifts (i.e., it does not generalize well to previously unseen inputs). We have previously described this phenomenon in Defintion~\ref{chap2:overfitting}.

Indeed, adversarial attacks, i.e., imperceptible changes to the input explicitly built to fool the network function~\citep{szegeny2013intriguing,goodfellow2014adversarial}, illustrate the risks of overfitting to $\dataset$. More realistic scenarios include isotropic noise~\citep{mallat2016understanding} and standard corruptions~\citep{hendrycks2019robustness} that are also likely to produce similar misclassifications. Robustness to such deviations is, therefore, a key challenge, especially in applications that are very sensitive to errors, such as autonomous vehicles or robotic-assisted surgery.

Note that by using a definition linked to $\dataset$, robustness is therefore defined as the resiliency of the network to \textbf{corrupted inputs} $\hat{\vx}$.
\begin{definition}[corrupted input $\hat{\vx}$]\label{chap2:corrupted_inputs}
We aim to train DNNs to be robust to corrupted inputs $\hat{\vx} \not \in \dataset$. These inputs are defined in such a way that there exists a $\vx \in \dataset$ and that $\|\hat{\vx}-\vx \| \approx \varepsilon$ where $\|\cdot\|$ is a measure of distance and $\epsilon$ is a small enough threshold. The most common measures of distance in the literature are the $\normltwo$~\citep{qian2019l2nonexpansive} and the $\normmax$~\citep{madry2018towards}, but it may also refer to more abstract concepts such as the same image but with different levels of contrast/brightness~\citep{hendrycks2019robustness}. 
\end{definition}

We note that recent works have also studied the concept of deep neural network robustness in the context of graph neural networks. In this case we have to consider not only corrupted inputs $\hat{\vx}$ but also corrupted support graphs. We will not delve into the concept of graph neural network robustness, but we refer the reader to~\citep{bojchevski2019graphrobustness} for a more in depth discussion. More details about graph-supported deep neural network methods are available in Chapter~\ref{chap4}.

Given our definition of the corrupted inputs $\hat{\vx}$, a common approach to increasing the robustness of DNNs is to concentrate on the Lipschitz constant of the network. Recall that a function $f$ is said to be $\alpha$-Lipschitz with respect to a norm $\|\cdot\|$ if $\|f(\vx_i)-f(\vx_j)\| \leq \alpha \|\vx_i-\vx_j\|, \forall \vx_i,\vx_j$. Provided $\alpha$ is small, such a function is robust to small deviations around correctly classified inputs, as it holds that: $\|f(\vx+\varepsilon)-f(\vx)\| \leq \alpha\|\varepsilon\|$. One example of such a method that focuses in the Lipschitz constant of $f$ is Parseval Networks~\citep{cisse2017parseval}, where the authors softly enforce the network $\normltwo$ and $\normmax$ Lipschitz constants to be bounded. Another example is~\citep{qian2019l2nonexpansive} where the authors propose to bound only the $\normltwo$ norm of the network. 

However, imposing a small Lipschitz constraint may be too restrictive of a constraint for the network function $f$. Indeed, the Lipschitz constant defines the slope of the function everywhere. Nonetheless, given the context of DNNs for classification, it is not unreasonable to expect sharp transitions in the output of $f$ if we are near the class boundaries. In other words, ideally, we would like for the smoothness properties of the network function to be location-dependent (e.g., different behavior close to class boundaries), meaning that global Lipschitz metrics may not be as meaningful. 

To illustrate our point, consider a function $f$ that outputs binary label indicator vectors. We can then compute the minimal Lipschitz constraint that allows for outputting these vectors, given the distance between each pair of examples, i.e., the closer a pair of examples of different classes is in the input space, the higher the Lipschitz constant of the network would have to be to allow for outputting label indicator vectors. We are interested in this measure, as the training objective of most classification DNNs is to be able to output this type of vector. We depict in Figure~\ref{chap2:Lipschitz_estimations} the proportion of pairs of training set inputs that are possible for a given Lipschitz constraint. The figure depicts information from the previously introduced datasets, CIFAR-10 and Imagenet32, using the $\normmax$ norm. We note that for the $\normmax$ norm a very high Lipschitz constraint is needed in order to correctly output binary label indicator vectors.

\begin{figure}[ht!]
  \begin{center}
  \tikzsetnextfilename{chapter2/tikz/incompatible_pairs}%
  \input{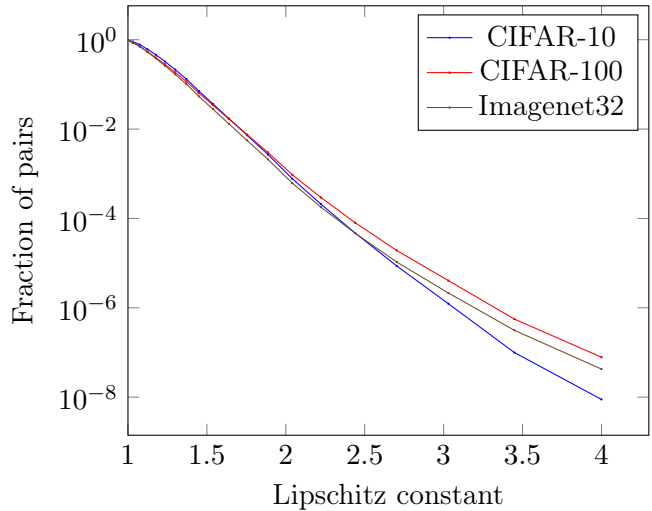}%

    \caption{Depiction of the proportion of pairs of training examples of distinct classes incompatible with a given Lipschitz constraint on the network function, for various datasets and the $\normmax$ norm. Figure and caption adapted from~\citep{lassance2019robustness} @2019 IEEE. }
    \label{chap2:Lipschitz_estimations}
  \end{center}
\end{figure}

In this section we describe our definition of robustness, that was first introduced in our previous work~\citep{lassance2019robustness}. This definition can be viewed as a \emph{localized} Lipschitz constant of the network function in $\trainset$. 

We ensure that any small deviation around a correctly classified training input should not dramatically impact the decision of the network function. Therefore our definition can be seen as a refinement over the previously proposed Lipschitz-based definitions, where we only enforce the Lipschitz constant on the previously defined $\hat{\vx}$. We will derive reasonable sufficient conditions to enforce the robustness of a deep learning architecture in the following paragraphs. Using experiments on CIFAR-10 (Section~\ref{chap2:cifar10}) and Imagenet32 (Section~\ref{chap2:imagenet}) we demonstrate that our proposed definition of robustness is correlated to the empirical robustness observed in a series of existing network training methods~\citep{madry2018towards,cisse2017parseval,qian2019l2nonexpansive,lassance2018laplacian}.

Note that the example depicted in Figure~\ref{chap2:Lipschitz_estimations} illustrates that for such a sharp network function, a global Lipschitz constraint is not meaningful: unless the Lipschitz constant is large (e.g., greater than 4) imposing a constraint will prevent the training error from converging to zero.  This example also suggests two related principles that can lead to better robustness and motivate our proposed robustness metric: \begin{inlinelist} \item robust network functions should be able to yield sharp transitions in boundary regions \item smoothness metrics should be localized. \end{inlinelist}

\subsection{Defining robustness}

We recall that our objective is to train a function $f$, which maps data from an input space $\sD$ into a softmax decision for classification. Therefore $f$ is a function from an input vector space (or tensor space) to $\mathbb{R}^c$, where $c$ is typically the number of classes. We are interested in the robustness of the network function~$f$. We introduce here a notion of robustness that should account for:
\begin{enumerate}
    \item A restricted domain $R$ on which it is defined,
    \item A locality $r$ around each point in $R$ on which it should be enforced.
\end{enumerate}

More formally, we define robust DNN behavior as follows: 

\begin{definition}[$\alpha$-robustness]
\label{chap2:def_robustness}
We say a network function $f$ is \textbf{$\alpha$-robust} over a domain $R$ and for $r>0$, and denote $f\in {\rm Robust}_\alpha(R,r)$, if:
\begin{equation}
     \|f(\vx+\varepsilon) - f(\vx)\| \leq \alpha \|\varepsilon\|,\forall \vx \in R, \forall \varepsilon \text{ s.t. } \|\varepsilon\| < r\;.
\label{chap2:eq_robustness}
\end{equation}
\end{definition}

In words, $f\in \rm{Robust}_\alpha(R,r)$ if $f$ is locally $\alpha$-Lipschitz within a radius $r$ of any point in domain $R$. Note that the following holds: $f \in \rm{Robust}_\alpha(\Omega,+\infty)$ if and only if $f$ is $\alpha$-Lipschitz. Note that as it was previously discussed, we are interested in enforcing robustness for a small radius $r$ around $\dataset$ that is still inside of $\sD$. 

We also define: $\alpha_{\lim}(f,R,r) = \inf\{\alpha: f\in \rm{Robust}_\alpha(R,r)\}$, where $\alpha_{\lim}(f,r)$ represents the minimum value $\alpha$ for which a region of radius $r$ is robust. Therefore our robustness definition can leverage a trade-off between the smoothness slope $\alpha$ and a radius $r$.

\emph{Illustrative example}: Figure~\ref{chap2:example_robustness} (Left) depicts the evolution of $\alpha_{\lim}(\sigma , R,r)$ as function of $r$ for the sigmoid function $\sigma: x\mapsto \frac{1}{1 + \exp(-x)}$ and $R = \{-10,10\}$. Observe that the sigmoid function yields an almost 0-Lipschitz constant around the two points $-10$ and $10$ and for a very small radius $r$. When the radius starts to increase, the Lipschitz constraint needs to be less restrictive (as the Lipschitz constant increases). The fact that $\alpha$ is almost 0 when $r$ is small is an illustration of robustness around $R$. The sharp transition occurring for $r\approx 10$ corresponds to a possible boundary between classes.

\begin{figure}[ht]
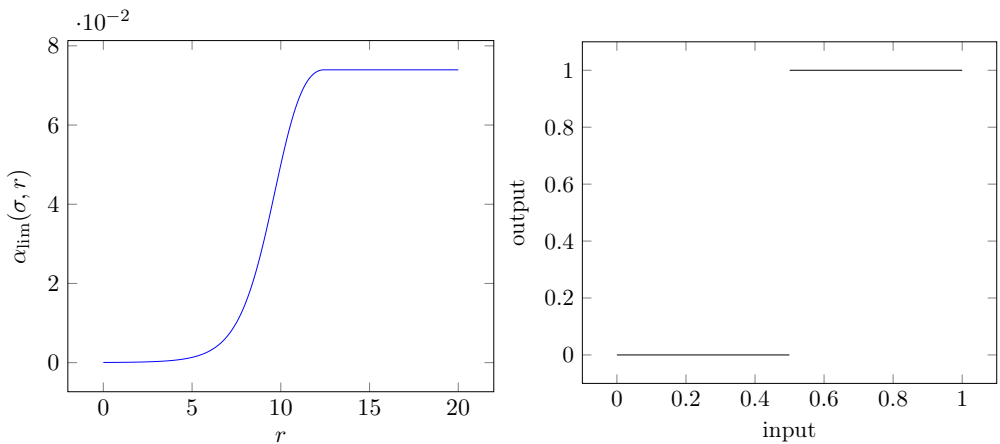

  \begin{center}
    \begin{subfigure}[b]{.45\linewidth}
      \centering
      \begin{adjustbox}{max width=\columnwidth}        
  \tikzsetnextfilename{chapter2/tikz/example_robustness_left}%
  \input{chapter2/tikz/example_robustness_left.tex}%

      \end{adjustbox}
  \end{subfigure}%
    \begin{subfigure}[b]{.45\linewidth}
      \centering
      \begin{adjustbox}{max width=\columnwidth}        
  \tikzsetnextfilename{chapter2/tikz/example_robustness_right}%
  \input{chapter2/tikz/example_robustness_right.tex}%

      \end{adjustbox}
    \end{subfigure}
    \caption{\textbf{Left}: Evolution of $r \mapsto \alpha_{\lim}(\sigma, R, r)$. \textbf{Right}: Representation of the decision of a hyperplane separator between points $0$ and $1$. Figure and caption extracted from~\citep{lassance2019robustness} @2019 IEEE.}
    \label{chap2:example_robustness}
  \end{center}
\end{figure}

\subsection{Relation with Lipschitz constants}

We now describe the relation between our definition of robustness and the Lipschitz constant of $f$. Note that by particularization, if $f$ is $\alpha$-Lipschitz then $f\in \rm{Robust}_\alpha(R,r),\forall r$. On the other hand, $f\in \rm{Robust}_\alpha(R,r)$ for some $r$ does not imply that $f$ is $\alpha$-Lipschitz. An example of this would be a trivial classification problem where $\trainset$ is composed of two distinct vectors $\vx_1$ and $\vx_2$ of distinct classes. A network function $f$ that uses a hyperplane to separate the space into two halves has no Lipschitz constant because $\alpha\approx\infty$ close to the hyperplane, despite $\alpha_{\lim}(f, R, \|\mathbf{x}-\mathbf{y}\|_2/2) = 0$. See Figure~\ref{chap2:example_robustness} (Right) for a 1D depiction of this example.

Note that this relation to the Lipschitz constant is a fundamental result, given that the best Lipschitz constant $\alpha$ of a function $f$ is going to be constrained by $\dataset$, i.e., if two training points of different classes are very close to each other then a zero training error classifier will by definition have a large Lipschitz constant near those points (as suggested by~Figure~\ref{chap2:Lipschitz_estimations}). The proposed robustness described in Definition~\ref{chap2:def_robustness} is also going to be limited by the construction of $\dataset$, but given a small enough $r$ it is easy to imagine that we will be able to reach any small $\alpha$. Indeed, denote by~$\vc^\vx$ the class corresponding to training example $\vx$. Then, if $f$ matches a 1-nn classifier, we obtain that: 
\begin{equation}\label{chap2:0alpha}
f\in \rm{Robust}_0\left(\min_{\genfrac{}{}{0pt}{2}{\vx_1,\vx_2\in T} {\vc^{\vx_1}\neq\vc^{\vx_2}}}{\frac{\|\vx_1 - \vx_2\|}{2}}\right)\;,
\end{equation}
and thus any small value for $\alpha$ is achievable within a small radius around examples.

\subsection{Compositional robustness}

Note that directly enforcing a robustness criterion on the entire DNN function $f$ can be hard in practice, because of the numerous intermediate representations that are implied in the process. This is why several works in the literature consider each layer of the architecture separately~\citep{cisse2017parseval,qian2019l2nonexpansive,lassance2018laplacian}.

Here we will focus on the simple case where $f$ is obtained as the composition of intermediate functions $f^\ell$. Note that the results presented here could easily be extended to more generic cases. So we denote by $f^\ell$ the function corresponding to layer $\ell$, by $F^\ell$ the function $F^\ell = f^\ell \circ \dots \circ f^1$, and we suppose that $F^{\ell_{\max}} = f$ is the DNN function. We define \emph{layer-robustness} as:
\begin{definition}[layer-robustness]
We say that an intermediate function $f^{\ell+1}$ is $\alpha$-robust over $F^\ell(R)$ and for $r>0$ at depth ${\ell+1}$ and we denote $f^{\ell+1} \in Robust_\alpha(F^\ell(R),r)$ if:

\begin{equation}
\begin{array}{l}\|f^{\ell+1}(\vx^{\ell}+\varepsilon)-f^{\ell+1}(\vx^{\ell})\|\leq \alpha\|\varepsilon\|,\\ \vspace{0.3cm} \forall \vx^{\ell} \in F^{\ell}(R), \forall \varepsilon \text{ s.t. } \|\varepsilon\|< r.
\end{array}
\end{equation}

Note that we consider $f^0$ to be the identity function.
\end{definition}

There is a direct relationship between robustness of functions $f^{\ell}$ at the various layers of the architecture and that of $f$, as expressed in the following proposition.

\begin{proposition}
\label{parts_proposition}
Suppose that: \begin{equation}
f^{\ell+1} \in \rm{Robust}_{\alpha^{\ell+1}}(F^\ell(R),r \prod_{\lambda\leq \ell}{\alpha^{\lambda}}),\forall \ell \text{ s.t. } 0\leq \ell < \ell_{max}\;. 
\end{equation}

Denote $\alpha = \prod_{\lambda\leq \ell_{\max}}{\alpha^\lambda}$, then 
\begin{equation}
f = f^{\ell_{\max}} \circ \dots \circ f^{1} \in \rm{Robust}_\alpha(R,r)\;.
\end{equation}
\end{proposition}

\begin{proof}
  Let us fix $\vx\in R$.  We proceed by induction. Let us show that if:
  \begin{equation}
  F^{\ell} \in \rm{Robust}_{\prod_{\lambda\leq \ell}{\alpha^\lambda}}(R,r)\;,
  \end{equation}
  then 
  \begin{equation}
  F^{\ell+1} \in \rm{Robust}_{\prod_{\lambda\leq \ell+1}{\alpha^\lambda}}(R,r)\;.
\end{equation}
  Indeed, let us fix  $\varepsilon \text{ s.t. } \|\varepsilon\| < r$, then:
  \begin{equation}
      \|F^{\ell+1}(\vx+\varepsilon) - F^{\ell+1}(\vx)\|= \|f^{\ell+1}(F^{\ell}(\vx+\varepsilon)) - f^{\ell+1}(F^{\ell}(\vx))\|\;.
  \end{equation}
  Note that as $F^{\ell} \in \rm{Robust}_{\prod_{\lambda\leq \ell}{\alpha^\lambda}}(R,r)$, it holds that:
  \begin{equation}
  \|F^{\ell}(\vx + \varepsilon) - F^{\ell}(\vx)\| \leq \prod_{\lambda\leq {\ell}}{\left(\alpha^\lambda\right)} \|\varepsilon\|;.
  \end{equation}
  So we can write:
  \begin{equation}
  F^{\ell}(\vx+\varepsilon) = F^{\ell}(\vx)+ \varepsilon'\;,
  \end{equation} where: 
  \begin{equation}
\|\varepsilon'\| \leq \prod_{\lambda\leq {\ell}}{\left(\alpha^{\lambda}\right)} \|\varepsilon\| \leq r \prod_{\lambda\leq {\ell}}{\left(\alpha^{\lambda}\right)}   \;.
  \end{equation}
  Finally, we obtain:
  \begin{equation}
    \begin{array}{lll} 
    &      & \|F^{\ell+1}(\vx+\varepsilon) - F^{\ell+1}(\vx)\|\\
    & =    & \|f^{\ell+1}(F^{\ell}(\vx)+\varepsilon') - f^{\ell+1}(F^{\ell+1}(\vx^{\ell}))\| \\
    & \leq & \alpha^{\ell+1}\|\varepsilon'\| \leq \prod_{\lambda\leq \ell+1}{\left(\alpha^\lambda\right)} \|\varepsilon\| \;.
    \end{array}  
  \end{equation}
\end{proof}

We note that conditioning the intermediate function $f^{\ell+1}$ is less strict if all the previous layers were already yielding small values of $\alpha$. In other words, the demanded radius for $f^{\ell+1}$ robustness is smaller. We thus observe there would be multiple possible strategies to enforce the compositional robustness of $f$ in practice: \begin{inlinelist} \item forcing all layers to provide similar robustness \item focusing only on a few layers of the architecture. \end{inlinelist} Note that most proposed methods in the literature~\citep{cisse2017parseval,qian2019l2nonexpansive,lassance2018laplacian} opt for enforcing the former property as the latter would be too restrictive and would probably prevent the learning procedure from converging. We note that ii) would also be incompatible with the previously described Resnet architecture as it would restrict the ability of residual connections to ignore individual blocks.

\subsection{Sources of noise}

We now present some of the sources of noise that we study in this work and that fulfill Definition~\ref{chap2:corrupted_inputs}. These sources may be deliberate (adversarial attacks), or they could be just circumstantial (Gaussian noise added to a $\vx \in \dataset$). 

\subsubsection{Adversarial attacks}

In the literature, several methods have been proposed to measure the robustness of network functions. The first set of approaches~\citep{szegeny2013intriguing, goodfellow2014adversarial, madry2018towards} proposes to generate perturbations that both maximize the training loss and minimize the distance from the original inputs. This perturbation is generated by backpropagating the gradients through the networks to the inputs. These adversarially generated images are very potent against unprotected DNNs even for very small $\varepsilon$. 

Note that in this same vein, various methods were proposed to increase the robustness of DNNs by using adversarial examples as a data augmentation procedure. They do so by adding these adversarial examples to $\trainset$~\citep{goodfellow2014adversarial, madry2018towards}. Thus, during the training phase, the network function becomes increasingly robust to the corresponding corruptions. However, no guarantee exists that by increasing robustness to a specific type of corruption leads to better performance on other types of corruption, as discussed in~\citep{hendrycks2019robustness,engstrom2018evaluating}.

In this work, we consider three gradient-based adversarial attacks: 
\begin{enumerate}
\item FGSM method~\citep{goodfellow2014adversarial}, that we consider a mean case of adversarial noise, where the adversary can only use one forward and one backward pass to generate the perturbations. This approach is called Fast Gradient Sign Method (FGSM);
\item The DeepFool method~\citep{moosavi2016deepfool}, that is considered to be a worst case scenario, where the adversary can use multiple forward and backward passes to try to find the smallest perturbation that will fool the network;
\item The PGD method~\citep{madry2018towards}, that can be seen a compromise between the mean case and the worst case, where the adversary can do a predefined number of forward and backward passes with a perturbation threshold limit.
\end{enumerate} 

A main criticism about adversarial attacks is that they require access to the network function $f$ and its derivative. This is highly improbable for many application cases. As such, some authors prefer to concentrate their studies on natural classifier-agnostic corruptions of data, as described in the next section.

\subsubsection{Corrupted inputs}
\label{chap2:robustness_benchmark}
Multiple types of corruption may be applied to the inputs during the capture of the data. Images are especially affected by this, as just changing the camera lens or sensor could generate specific artifacts that change the distribution of $\sD$ and $\dataset$. To deal with this, we introduce 15 common image corruptions, with five levels of severity each. These corruptions were first organized as a benchmark in ~\citep{hendrycks2019robustness}, with releases for the CIFAR-10 and Imagenet datasets. The goal of this benchmark is to compare the performance on the corrupted $\testset$ and the clean $\testset$ in order to isolate the original clean set performance from the analysis. We depict the 15 different corruptions in Figure~\ref{chap2:figure_dogs}.

\begin{figure}
  \begin{center}
  \tikzsetnextfilename{chapter2/tikz/dogs}%
  \begin{tabular}{ccccc}
    \footnotesize{Gaussian Noise}&\footnotesize{Shot Noise}& \footnotesize{Impulse Noise} & \footnotesize{Defocus Blur} & \footnotesize{Frosted Glass Blur}\\ 
      \includegraphics[scale=0.3]{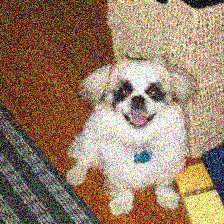}
      &\includegraphics[scale=0.3]{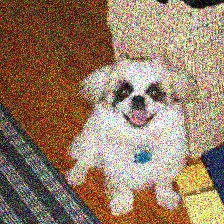}
      &\includegraphics[scale=0.3]{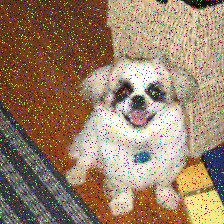}
      &\includegraphics[scale=0.3]{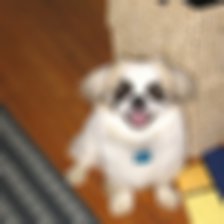}
      &\includegraphics[scale=0.3]{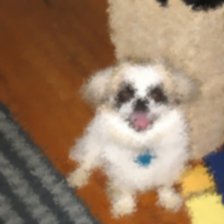} \\ 
      \footnotesize{Motion Blur}&\footnotesize{Zoom Blur}& \footnotesize{Snow} & \footnotesize{Frost} & \footnotesize{Fog}\\ 
      \includegraphics[scale=0.3]{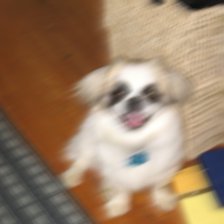}
      &\includegraphics[scale=0.3]{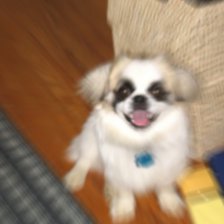}
      &\includegraphics[scale=0.3]{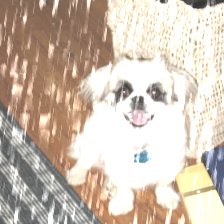}
      &\includegraphics[scale=0.3]{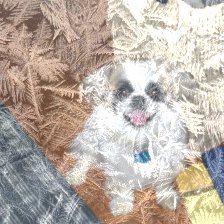}
      &\includegraphics[scale=0.3]{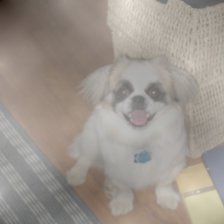} \\ 
      \footnotesize{Brightness}&\footnotesize{Contrast}& \footnotesize{Elastic} & \footnotesize{Pixelate} & \footnotesize{JPEG}\\
      \includegraphics[scale=0.3]{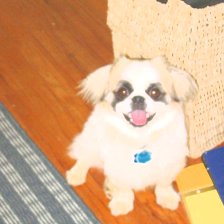}
      &\includegraphics[scale=0.3]{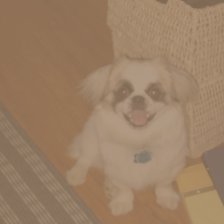}
      &\includegraphics[scale=0.3]{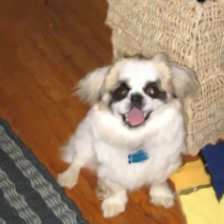}
      &\includegraphics[scale=0.3]{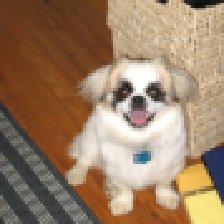}
      &\includegraphics[scale=0.3]{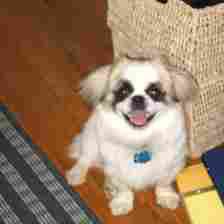} \\ 
\end{tabular}

    \caption{The 15 different types of corruptions from~\citep{hendrycks2019robustness} applied to a random Imagenet dog image.}
    \label{chap2:figure_dogs}
  \end{center}
\end{figure}

To measure the robustness of the networks to these corruptions, we use the relative mean corrupted error (relative mCE). To compute this measure of performance, one has to follow several steps:
\begin{enumerate}
  \item Take a trained classifier $f$ and compute the clean $\testset$ top-1 error rate. We denote this measure clean error rate: $E^f_\text{clean}$;
  \item Now test the classifier $f$ on each corruption $c$ and every severity $s$. We denote this by $E_{s,c}^f$;
  \item As different corruptions have different difficulty levels even with the 5 severities $s$, we now normalize each score using a pre-trained baseline network that we call $f_V$. This normalized score is called the corruption error or CE and it is defined by the following equation: 
  \begin{equation}
    \text{CE}_c^f = \frac{\sum_{s=1}^5 E_{s,c}^f}{\sum_{s=1}^5 E_{s,c}^{f_V}} \;
  \end{equation}
  \item We can therefore summarize model robustness by averaging over the different $\text{CE}_c$ for the same $f$. This leads to the mean corruption error (mCE);
  \item Finally, to ensure that gains in robustness come from the strengths of network $f$ and not as a by product of better performance on the clean $\testset$ we compute the relative CE 
  \begin{equation}
    \text{Relative CE}_c^f = \frac{ \sum_{s=1}^5 E_{s,c}^f - E^f_\text{clean} } { \sum_{s=1}^5 E_{s,c}^{f_V} - E^{f_V}_\text{clean} } \; .
  \end{equation}   
  As it was the case CE, we can also average the relative corruption errors to obtain the relative MCE.   
\end{enumerate}

Note that this is the standard measure of robustness for this benchmark as introduced in~\citep{hendrycks2019robustness}.

\subsection{Existing methods in the lens of our definition of robustness}

We now present four of the prior works in the literature from the perspective of our proposed robustness measure, namely: Parseval networks (P)~\citep{cisse2017parseval}, $\normltwo$ non-expansive networks (L2NN)~\citep{qian2019l2nonexpansive}, Laplacian networks (L)~\citep{lassance2018laplacian} and Projected Gradient Descent adversarial training (PGD)~\citep{madry2018towards}. See Table~\ref{chap2:robustness_summary} for a summary.

We recall that in~\citep{cisse2017parseval}, networks are trained to be $\alpha$-Lipschitz for both $\normmax$ and $\normltwo$. They do so by adding a regularizer to the training scheme that tries to make each weight matrix of the DNN a Parseval tight frame~\citep{kovavcevic2008introduction}. Among the four methods we consider, this Parseval was the only one that led to improved performance on the clean test set. However, we note that \citep{cisse2017parseval} does not seem to strictly enforce the $\alpha$-Lipschitz constraint, as it does not consider batch normalization layers. This is why it can achieve excellent memorization (which theoretically, as seen in Figure~\ref{chap2:Lipschitz_estimations}, could only be achieved if $\alpha$ is large). This also explains why this method achieves worse results in robustness than L2NN~\citep{qian2019l2nonexpansive}. In terms of our proposed definition of robustness, this is a global method that targets the $\text{Robust}_\alpha(r)$ metric for $r \rightarrow +\infty$, penalizing large slopes in the network function between any two points. We will see that more localized approaches (targeting finite $r$) achieve improved robustness when compared to Parseval. We denote this method P in the remainder of this work.

L2NN~\citep{qian2019l2nonexpansive}, on the other hand, enforces the network to be $\alpha$-Lipschitz only in terms of the $\normltwo$ norm but does it with a stricter criterion. The Lipschitz condition is built into the structure of the network itself. The authors from~\citep{qian2019l2nonexpansive} admit that enforcing a global $\alpha$-Lipschitz constant is by itself too hard and that the distances between examples should not collapse throughout the network architecture. As such, they also limit the contraction of space. L2NN seems to be the most robust method against $\normltwo$ attacks of the four methods we consider. It has also been shown to combine well with PGD training. However, it is also the method that was the worst-performing on the clean test set.

In~\citep{lassance2018laplacian}, we applied a regularization at each ReLU activation in the architecture to enforce that the average distance between examples of different classes remains almost constant from layer to layer. We exploit the smoothness of the label indicator vector across the graph generated by intermediate representations at a given layer to enforce this smooth transformation. We detail this contribution in Chapter~\ref{chap5:laplacian}. In terms of Definition~\ref{chap2:def_robustness}, this method focuses on pairs of examples of distinct classes and tries to restrict changes in their $\normltwo$ distance. Thus, \citep{lassance2018laplacian} indirectly penalizes changes in local smoothness. If we consider Definition ~\ref{chap2:def_robustness} with $f(.)$ chosen to be the function that assigns to each example its correct label, and we do not allow the average $r$ between opposite class examples to change much, then the corresponding $\alpha$ will change slowly with the training.

Finally, PGD adversarial training~\citep{madry2018towards} is a data augmentation procedure that generates adversarial examples during the training phase, as described in the previous subsection. Using the PGD data augmentation leads to a min-max game between the network and the examples generation. It works mostly on the domain $\trainset$, as it increases its size and also decreases the difference between $\trainset$ and a noisy test domain. The data augmentation scheme leads to less domain shift to corrupted inputs, but on the other hand, it increases the domain shift to clean images. As a result, the networks perform well against noise (isotropic or adversarial) but have problems with clean examples. 

\begin{table}[ht!]
    \centering
    \begin{tabular}{c|c|c|c|c}
\hline
         Method & Domain ($R$) & Slope ($\alpha$) & Locality ($r$) & Metric \\
        \Xhline{2\arrayrulewidth}
         
         P & $\Omega$ & Yes & No & $\normltwo$ + $\normmax$\\
         L2NN & $\Omega$ & Yes & No & $\normltwo$\\
         L & $T$ & Approx. & Yes & $\normltwo$ + cos\\
         PGD & augmented $T$ & No & Yes & $\normmax$\\ \hline
    \end{tabular}
    \caption{Summary of the methods and the notions of the introduced robustness they consider. Table extracted from~\citep{lassance2019robustness} @2019 IEEE.}
    \label{chap2:robustness_summary}
\end{table}

\subsection{Empirical evaluation of the proposed robustness metric}

We perform experiments to evaluate the relevance of the proposed robustness definition (\ref{chap2:def_robustness}) empirically. Vanilla (V), Parseval (P), and Laplacian (L) refer to the networks that were trained in~\citep{lassance2018laplacian}, PGD, and L2NN refer to the networks trained in their original papers~\citep{madry2018towards,qian2019l2nonexpansive}. Note that this direct comparison with the baseline is not entirely fair, given that the networks and hyperparameters for different papers are not the same. For example, PGD has more layers and parameters and uses non-adversarial data augmentation during training, while L2NN does not use residual architectures. 

We depict in Figure~\ref{chap2:incompatibility_robustness}, as a function of $\alpha$, the ratio between \begin{inlinelist} \item the number of examples within distance $d$ of each other that are not $\alpha$-robust for $\normmax$ \item the total number of example pairs.\end{inlinelist} Note that $d$ can be roughly interpreted as a diameter ($2r$) in our definition of robustness.  As in the previous Figure~\ref{chap2:Lipschitz_estimations}, the output of the network function is a label indicator vector of the corresponding classes.  Note that for each choice of $d$, the curve is initially flat. The flat section means that \emph{all the pairs within $d$ are $\alpha$-robust}. Note that the number of pairs of examples in distinct classes that are closer than $d$ drops very fast and becomes negligible for $d=0.3$. In other words, for $d=0.3$, it is theoretically possible to find a robust $f$ that is compatible with {\em almost all} pairs of the training set.

\begin{figure}[ht]
  \begin{center}
    \begin{adjustbox}{max width=\columnwidth}      
  \tikzsetnextfilename{chapter2/tikz/incompatibility_robustness}%
  \input{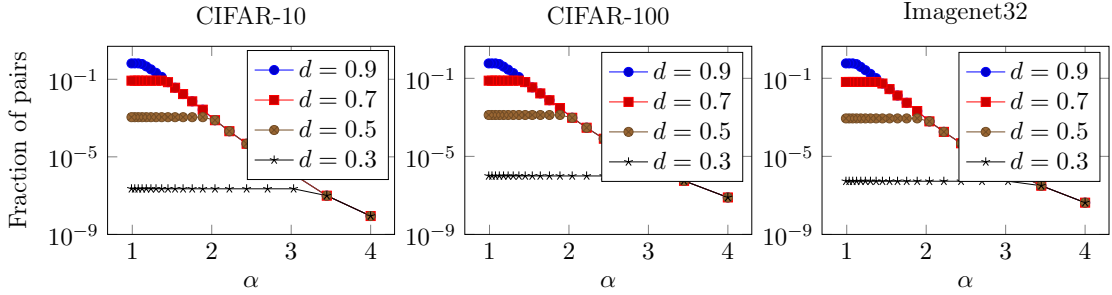}%

  \end{adjustbox}
  \caption{Depiction of the proportion of pairs of training examples of distinct classes incompatible with Definition~\ref{chap2:def_robustness} for the $\mathcal{L}_{\infty}$ norm, as a function of $\alpha$ and for various values of $d$.}
  \label{chap2:incompatibility_robustness}
  \end{center}
\end{figure}

In Figure~\ref{chap2:ralpha} we depict the evolution of $\alpha_{\lim}(r)$ as a function of $r$ for the various methods. We use 100 training examples with 1000 Gaussian noise realizations as a proxy to estimate $\alpha_{\lim}(\cdot)$ on the CIFAR-10 dataset. Note that for all methods, $\alpha$ increases as a function of $r$ and then achieves its maximum value. 

\begin{figure}[ht]
  \begin{center}
  \tikzsetnextfilename{chapter2/tikz/ralpha}%
  \input{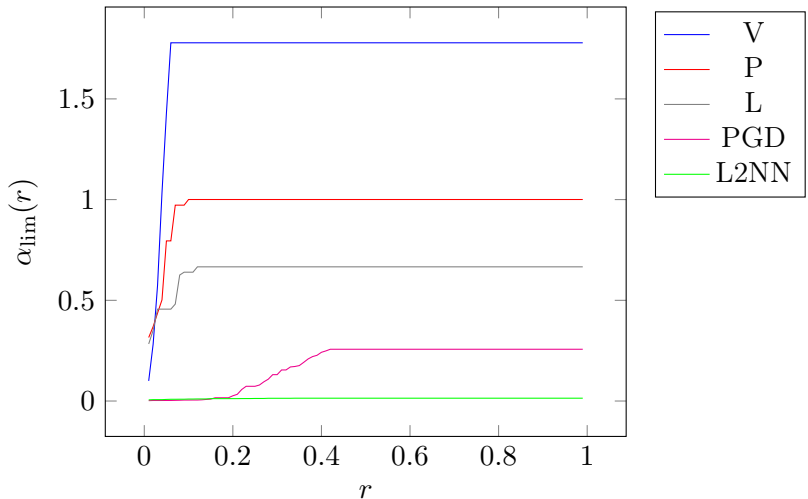}%

    \caption{Estimations of $\alpha_{\lim}(r)$ obtained for different radius $r$ over training examples with the $\normmax$ norm. Figure extracted from~\citep{lassance2019robustness} @2019 IEEE.}
    \label{chap2:ralpha}
  \end{center}
\end{figure}

Vanilla (V) is the fastest to saturate and achieves the largest value of $\alpha$ for two reasons \begin{inlinelist} \item sharp transitions in the network function over short distances are allowed \item the network function produces outputs closest to the one-hot-bit encoded vector (since V can achieve zero error on the training set).\end{inlinelist} In contrast, the other methods grows at a slower pace and reach smaller maximum values of $\alpha_\text{lim}$. This behavior indicates that transitions are not as sharp and that some examples may be misclassified. 

Moreover, the fact that, for both P and L, $\alpha$ saturates at larger values of $r$ suggests that the margin between the examples and the boundary is increased compared to Vanilla. L2NN and PGD saturate at the lowest $\alpha$ values. 

We also observe a transition for PGD occurring at around $r=0.3$ (which is an excellent radius given the aforementioned result that theoretically $d=0.3$ would be the limit), whereas L2NN remains almost constant. We believe that the L2NN behavior is due to the fact L2NN enforces a strong Lipschitz constraint, on the $\normltwo$ norm, everywhere on the function. Therefore the network function described by the L2NN network is almost linear between all training samples, leading to a network that is not as accurate on the clean set. As seen in Figure~\ref{chap2:Lipschitz_estimations}, this lower value of $\alpha$ creates strong incompatibilities with the training dataset, which could be the reason why L2NN achieves the worst performance on the clean set (c.f. Table~\ref{chap2:robustness_accuracy_table}).

We now compare methods in terms of robustness on the recently proposed benchmark of image corruptions~\citep{hendrycks2019robustness}. PGD achieves the best trade-off between accuracy and robustness, as seen in Table~\ref{chap2:robustness_accuracy_table}. We note that for PGD, our robustness metric has a relatively small value for $\alpha_{\max}$, and the slope starts at $0.2\leq r \leq0.4$, which corresponds to appropriate values of $d$ as seen in Figure~\ref{chap2:Lipschitz_estimations}. The results described in both Table~\ref{chap2:robustness_accuracy_table} and the behavior described in Figure~\ref{chap2:ralpha}, seem to suggest that improved robustness is achievable when the network function is smooth locally near the examples, i.e., network functions that favor Definition~\ref{chap2:def_robustness} were empirically more robust. 

\begin{table}[ht]
  \begin{center}
  \caption{Test set error on the CIFAR-10 dataset under different image conditions. Table and caption extracted from~\citep{lassance2019robustness} @2019 IEEE.}
  \begin{tabular}{c|ccccc}
  \hline
  Dataset             & V        & P           & L       & PGD             & L2NN          \\ 
  \Xhline{2\arrayrulewidth}
  Clean           & 11.9\%  & \textbf{10.2\%}  & 13.2\%  & 12.8\%          & 20.9\%                   \\
  MCE       & 31.6\%  & 30.5\%           & 31.3\%  & \textbf{18.8\%} & 28.5\%                 \\
  relativeMCE & 100    & 103             & 92    & \textbf{30}            & 39                  \\  \hline
  \end{tabular}
  \label{chap2:robustness_accuracy_table}
  \end{center}
\end{table}

  Finally, we depict in Figure~\ref{chap2:figure_snr} the relative robustness performance of each method under the same Gaussian Noise parameters used to generate Figure~\ref{chap2:ralpha}\footnote{We do not report the results for P in the Imagenet32 dataset since we did not find right hyperparameters to obtain a good accuracy on the clean test set. Also, PGD and L2NN results are not reported in the case of CIFAR-100 and Imagenet32 as pre-trained networks were not available.}.

\begin{figure}[ht]
  \begin{center}
  \begin{adjustbox}{max width=\columnwidth}    
  \tikzsetnextfilename{chapter2/tikz/figure_snr}%
  \input{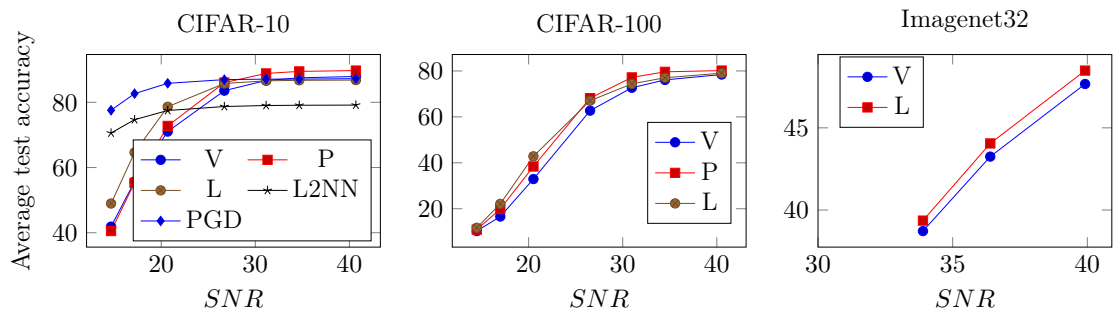}%

  \end{adjustbox}
  \caption{Average test set accuracy under Gaussian noise for various datasets and methods. Figure and caption extracted from~\citep{lassance2019robustness} @2019 IEEE.}
    \label{chap2:figure_snr}
  \end{center}
\end{figure}

\section{Summary of the chapter}

In this first chapter, we introduced Deep Neural Networks (DNNs), with a specific focus on Residual architectures. We also introduced the concept of intermediate representations in DNNs, which will be a central part of the following chapters of this thesis. These intermediate representations can be used in order to perform transfer learning as introduced in Section~\ref{chap2:feature_extraction} and also in order to abstract DNNs as feature extractor followed by classifiers.

We presented various problems where DNNs are relevant. These problems are going to be further studied in the next chapters and include the following tasks: \begin{inlinelist} \item image retrieval \item vision based localization \item image classification  \item neurological task classification \item document classification\end{inlinelist}. 

We also reviewed some of the literature in neural network compression, including distillation methods and more efficient convolution layers. We introduced SAL, a contribution on the subject of efficient convolution layers, which was subject to the following archival paper:
\begin{itemize}
\item \bibentry{hacene2019attention}
\end{itemize}

Further, we introduced the concept of robustness of a classifier, and demonstrated empirically how it can be connected with an ability to perform well in presence of corrupted inputs. This concept of robustness and its empirical experiments were published in the following conference paper:
\begin{itemize}
\item \bibentry{lassance2019robustness}
\end{itemize}

In the following chapter, we introduce the framework of Graph Signal Processing (GSP) which defines a series of concepts and analytical tools. These tools are going to allow us to analyse the intermediate representations of DNNs, and are going to be the cornerstone for our contributions in both Chapter~\ref{chap4} and Chapter~\ref{chap5}. Indeed, in Chapter~\ref{chap4} we use the tools of GSP to introduce Graph Neural Networks. In Chapter~\ref{chap5}, the framework of GSP will allow us to propose analytical tools to better understand the inner workings of DNNs. We also introduce improvements in the accuracy, robustness and compression of these architectures.
\chapter{Graph Signal Processing}\label{chap3}
\localtableofcontents
\vspace{1.0cm}

We present in this Chapter the concepts of graphs and Graph Signal Processing (\textbf{GSP}) that are exploited in this work in order to study the topology of intermediate representations of DNNs. Studying the topology of DNNs will allow us to propose new methods to improve DNNs in Chapter~\ref{chap4} and Chapter~\ref{chap5}. This chapter is organized as follows: first we define graphs, graph signals and GSP in Section~\ref{chap3:definitions}. Then we introduce the two types of graphs that we consider in this work in Section~\ref{chap3:graph_uses}, graph topology inference from data~\ref{chap3:graph_inference} and graph filters~\ref{chap3:graph_filter}. We refer the reader to~\citep{shuman2013emerging} for a more detailed introduction to GSP.

\section{Definitions}\label{chap3:definitions}

In general, \emph{graphs} are used as a formalism to represent data and its relationships. More precisely we define a graph as:

\begin{definition}[graph]
A graph $\gG$ is a tuple of sets $\langle \sV , \sE \rangle$, such that:  
\begin{enumerate}
 \item The set $\sV$ is composed of vertices $v_1, v_2, \dots $;
 \item The set $\sE$ is composed of pairs of vertices of the form ($v_i$,$v_j$) called edges.    
\end{enumerate}
\end{definition}

It is common to represent the set $\sE$ using an edge-indicator symmetric adjacency matrix $\adjmatrix \in \R^{|\sV|\times |\sV|}$. Note that being symmetric means that in this work we consider only undirected graphs, which allow us to simplify most of our notations. As a quick recall to the reader, in an undirected graph, there is no distinction between edges $(v_i,v_j)$ and $(v_j, v_i)$. 

In some cases, the matrix $\adjmatrix$ is weighted (it takes values other than 0 or 1) because it not only represents the fact that a pair of vertices $(v_i,v_j) \in \sE$ but also the weight associated with that representation, where typically a value closer to 0 corresponds to a vanishing relationship. In other words, each element $\emAdjacency_{i,j}$ represents the weight of the edge between $v_i$ and $v_j$. 

As is the case in most works in the literature, we consider only graphs with nonnegative weights, and say that two vertices are not connected if their weight is equal to $0$.

We can use the $\adjmatrix$ to define the diagonal \textbf{degree matrix} $\mD$ of the graph as follows:
\begin{equation}
    \emD_{i,j} = \left\{ \begin{array}{cl}\displaystyle{\sum_{j' \in \sV}{\emAdjacency_{i,j'}}} & \text{if } i = j\\ 0 & \text{otherwise}\end{array}\right.\;.
\end{equation} 
We also define the r-neighborhood $\sN_r(v)$ of a vertex $v \in \sV$ as the set of vertices that are at most $r$-hop away from a vertex $v \in \sV$, that is to say that $v'\in \sN_r(v)$ if and only if it exists a sequence of at most $r$ vertices $v_{i_1}, v_{i_2}, \dots, v_{i_\rho}$, such that $v_{i_\rho} = v'$, $(v, v_{i_1})\in \sE$ and $(v_{i_j},v_{i_{j+1}})\in \sE, \forall j$.

A graph typically represents a relation between its vertices. When those vertices are associated with measures (typically scalars), we talk of graph signals. In this thesis we only consider signals supported on the vertices of the graph, but there are also studies that consider signals supported on the edges~\citep{schaub2018flow}. In other words, we consider graph signals $\vs \in \R^{|\sV| \times F}$ where F is the number of realizations of the signal $\vs$ for each vertex of the graph. In Figure~\ref{chap3:fig-examples-ml} we depict examples of graphs in various machine learning scenarios. 

\begin{figure}[ht]
  \begin{center}
    \begin{subfigure}[ht]{.48\linewidth}
      \centering
      \resizebox{\linewidth}{.25\textheight}{%
  \tikzsetnextfilename{chapter3/tikz/karate}%
  \input{chapter3/tikz/karate.tex}%
}
      \caption{Zach's karate club community graph~\citep{girvan2002community}.}
  \end{subfigure}
  \begin{subfigure}[ht]{.48\linewidth}
      \centering
      \resizebox{\linewidth}{.25\textheight}{%
  \tikzsetnextfilename{chapter3/tikz/toronto}%
  \input{chapter3/tikz/toronto.tex}%
}
      \caption{Toronto road network~\citep{irion2016efficient}.}
  \end{subfigure}        
  \begin{subfigure}[ht]{\linewidth}
      \centering
      \begin{subfigure}[ht]{.48\linewidth}
            \includegraphics[width=\linewidth]{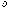}
      \end{subfigure}
      \begin{subfigure}[ht]{.48\linewidth}
        \resizebox{\linewidth}{!}{%
  \tikzsetnextfilename{chapter3/tikz/usps}%
  \input{chapter3/tikz/usps.tex}%
}
      \end{subfigure}
      \caption{Left: USPS digit~\citep{friedman2001elements}; Right: Grid graph representation of the digit image.}
  \end{subfigure}
      \caption{Depiction of various examples of graphs commonly used in machine learning problems.}
      \label{chap3:fig-examples-ml}
  \end{center}
\end{figure}

As a natural representation of complex data structures, graphs and graph signals are ubiquitous, in particular in the field of machine learning. In this thesis we focus on two main uses of graphs: \begin{inlinelist}\item Graphs that model the inner dependencies of observations \item Graphs that model the relationship between data samples\end{inlinelist}. More details on this subtle distinction are available in Section~\ref{chap3:graph_uses}.

In the following paragraphs we introduce operations on graphs that are going to be useful in the rest of this thesis.

\subsection{Translations on graphs}\label{chap3:graph_translation}

In the previous chapter we introduced the ideas of translation-equivariance and translation-invariance and their uses in convolutional networks and computer vision. In the same vein, graphs and graph signals can be subjected to translationse. In this subsection we will introduce the concept of graph translations, as was previously defined in the laboratory~\citep{pasdeloup2017extending}, and define how they are used in this thesis.

First off, we must start by recalling what is the translation operator. Indeed in a discrete euclidean space, translations are quite straightforward to define. For example consider a signal $s(t)$ that evolves over time. This signal $s(t)$ is defined in an 1D euclidean space and can therefore be translated either forward (``advancing in time'') or backwards (``going back in time''). In the same vein, if our signal $s$ is supported on an 2D euclidean space, as it is the case with images, it is quite direct to infer four types of translations, by sending all pixels downwards, upwards, to the right or to the left. 

In order to transfer this same translation concept from euclidean spaces to graphs, we choose to use in this work the translation definition and inference methods introduced in~\citep{pasdeloup2017extending} that we extend in~\citep{lassance2018matching}. 

\begin{definition}[translation]\label{chap3:def-translation}
A \emph{translation} on a graph is a function $\phi: \sU \to \sV$, where $\sU \subset \sV$ and that is:
\begin{enumerate}
  \item \emph{injective}: $\forall v,v' \in \sU, \phi(v) = \phi(v') \Rightarrow v = v',$
  \item \emph{edge-constrained}: $\forall v \in \sU, (v,\phi(v)) \in \sE,$
  \item \emph{strongly neighborhood-preserving}:  $\forall v,v' \in \sU, (v,v')\in \sE \Leftrightarrow (\phi(v),\phi(v')) \in \sE.$\\
\end{enumerate}
\end{definition} 

We also define the \emph{loss} of a translation as the cardinal $|\sV-\sU|$ that counts the vertices that are not a part of the translation. We also say that two translations $\phi$ and $\phi'$ are \emph{aligned} if $\exists v\in \sU, \phi(v) = \phi'(v)$. 

Ideally we would also add that translations should be lossless, i.e., $|\sV|=|\sU|$. Unfortunately it is not possible to guarantee that a given graph $\gG$ will be able to admit lossless translations~\citep{pasdeloup2017extending}. Therefore in this work we compromise by considering only minimal translations:

\begin{definition}[minimal-translation]
A translation is said to be minimal if there is no aligned translation with a strictly smaller loss.
\end{definition} 

We depict in Figure~\ref{chap3:fig-minimal-translations} the minimal translations of a grid-graph representing a 2D discrete euclidean space. We note that the inferred translations are exactly the same as previously defined for a 2D euclidean space. We will use this property in Section~\ref{chap4:classify_graph_signals} to define convolutions on graphs based on the 2D convolutional layers described in the previous chapter. 

\begin{figure}[ht]
  \begin{center}
  \tikzsetnextfilename{chapter3/tikz/translations}%
  \input{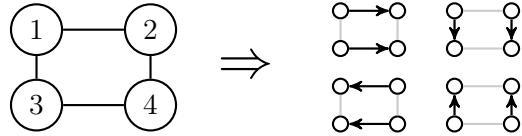}%

      \caption{Depiction of the minimal translations of a 2x2 grid graph.}
      \label{chap3:fig-minimal-translations}
  \end{center}
\end{figure}

\subsection{Graph Fourier Transform (GFT)}

In the domain of signal processing, the \textbf{frequency} is one of the most important concepts to analyse a signal and a starting point to introduce many useful tools including filtering. The frequency of a signal can be simplified as its rate of change, or in other words how it varies from one sample to another. Generally speaking, any signal can be decomposed into a continuous sum of sines/cosines by performing the Fourier transform. This decomposition can also be considered as expressing the signal in the frequency domain, as each sine/cosine is characterized by a proper frequency, and the original representation of the signal as its expression on the time domain.

Both the classical and graph Fourier transforms were first introduced to deal with signals defined on time, and then a posteriori extended to deal with other domains. The classical Fourier transform leads itself well to clearly structured domains such as the discrete Euclidean space, while tries to extend the framework developed using the Fourier transform to more loosely defined structures that can be supported on graphs. 

We present an example of the Fourier transform and a simple filtering application in a series of figures, from Figure~\ref{chap3:fig-examples-signals} to Figure~\ref{chap3:fig-examples-result}. First, in the left part of Figure~\ref{chap3:fig-examples-signals} we depict an original signal ($\vs(t)$), and in the right part of the Figure we depict the same signal with added white noise ($\breve{\vs}(t)$). We then use the Fourier transform to decompose both signals in the frequency domain ($f$) that we depict in Figure~\ref{chap3:fig-examples-fft}. We can observe that, since the white noise has no specific frequency, it has very little impact on the Fourier transform of the signal. As such, it would for example be much easier to perform classification in the frequency domain.

\begin{figure}[ht]
  \begin{center}
    \begin{subfigure}[ht]{.48\linewidth}
      \centering
      \begin{adjustbox}{max width=\linewidth}
  \tikzsetnextfilename{chapter3/tikz/signal}%
  \input{chapter3/tikz/signal.tex}%

      \end{adjustbox}
  \end{subfigure}
  \begin{subfigure}[ht]{.48\linewidth}
      \centering
      \begin{adjustbox}{max width=\linewidth}
  \tikzsetnextfilename{chapter3/tikz/noisy_signal}%
  \input{chapter3/tikz/noisy_signal.tex}%

      \end{adjustbox}
  \end{subfigure}        
  \caption{Depiction of a signal $\vs(t)$ and its noisy version $\breve{\vs}(t)$.}
  \label{chap3:fig-examples-signals}
  \end{center}
\end{figure}

Note how the original signal is mostly defined in the low-frequency side, while the noisy signal has many more high-frequency components. It is therefore straightforward to say that performing a filtering operation to remove the high frequency components in the noisy signal $\breve{\vs}(t)$ will allow us to retrieve a signal $\hat{\vs}(t)$ that is more inline with the original one. The retrieved signal $\hat{\vs}(t)$ is also commonly called \textbf{denoised signal}. We depict the original signal, the noisy signal and the retrieved signal in Figure~\ref{chap3:fig-examples-result}. Note that the retrieved signal is much more inline with the original signal, even if it is not exactly the same. Indeed, this denoising operation had the unfortunate effect of lowering the variations of the signal even for small frequencies.

\begin{figure}[ht]
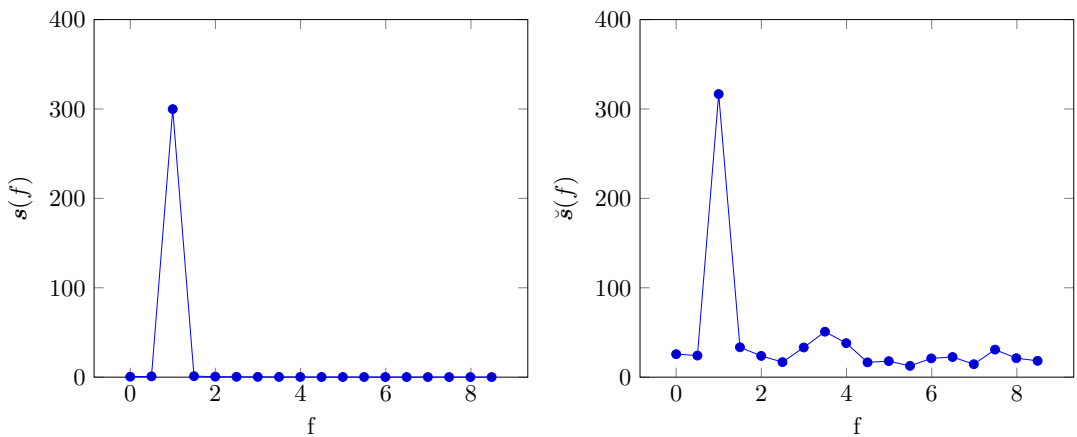

  \begin{center}
    \begin{subfigure}[ht]{.48\linewidth}
      \centering
      \begin{adjustbox}{max width=\linewidth}
  \tikzsetnextfilename{chapter3/tikz/fft_signal}%
  \input{chapter3/tikz/fft_signal.tex}%

      \end{adjustbox}
  \end{subfigure}
  \begin{subfigure}[ht]{.48\linewidth}
      \centering
      \begin{adjustbox}{max width=\linewidth}
  \tikzsetnextfilename{chapter3/tikz/fft_noisy_signal}%
  \input{chapter3/tikz/fft_noisy_signal.tex}%

      \end{adjustbox}
  \end{subfigure}        
  \caption{Depiction of the Fourier transform of the signal $\vs(t)$ and its noisy version $\breve{\vs}(t)$.}
  \label{chap3:fig-examples-fft}
  \end{center}
\end{figure}

\begin{figure}[ht]
  \begin{center}
    \begin{subfigure}[ht]{.3\linewidth}
      \centering
      \begin{adjustbox}{max width=\linewidth}
  \tikzsetnextfilename{chapter3/tikz/signal}%
  \input{chapter3/tikz/signal.tex}%

      \end{adjustbox}
  \end{subfigure}
  \begin{subfigure}[ht]{.3\linewidth}
      \centering
      \begin{adjustbox}{max width=\linewidth}
  \tikzsetnextfilename{chapter3/tikz/noisy_signal}%
  \input{chapter3/tikz/noisy_signal.tex}%

      \end{adjustbox}
  \end{subfigure}        
  \begin{subfigure}[ht]{.3\linewidth}
    \centering
    \begin{adjustbox}{max width=\linewidth}
  \tikzsetnextfilename{chapter3/tikz/retrieved_signal}%
  \input{chapter3/tikz/retrieved_signal.tex}%

    \end{adjustbox}
\end{subfigure}        
\caption{Depiction of a signal $\vs(t)$, its noisy version $\breve{\vs}(t)$ and the denoised version $\sfilter(t)$.}
  \label{chap3:fig-examples-result}
  \end{center}
\end{figure}

As we have previously defined, the goal of the \textbf{Graph Fourier Transform (GFT)} is to extend the same type of analysis and tools to signals that are described in the domain of the vertices of graphs (i.e. graph signals) instead of the time domain. In the case of graph signals the rate of change will not be evaluated as time evolves\footnote{at least in this thesis, we note that other works consider graph signals that vary on both graph and time domain.} but as the vertices evolve (i.e. the relationship between a vertex and its neighbors). In the case of GFT the transform is defined using the \textbf{graph Laplacian}.

The graph Laplacian $\mL$ is defined as:
\begin{equation}
    \mL = \mD - \adjmatrix \; ,
\end{equation} 
where $\mD$ is the diagonal degree matrix of the graph defined by $\adjmatrix$. As the Laplacian matrix is both real and symmetric it can be eigendecomposed into:
\begin{equation}\label{chap3:laplacian_spectral}
    \mL = \mF \mLambda \mF^\top \; ,
\end{equation} 
where $\mF$ are the eigenvectors of $\mL$ and $\mLambda$ are the eigenvalues in crescent order of magnitude. The GFT of a graph signal $\vs$ can then be defined as:
\begin{equation}
    \tilde{\vs} = \mF^\top s \; ,
\end{equation} 
where $\tilde{\vs} \in \R^{|\sV|}$ is the graph signal $\vs$ decomposition in the frequency domain, with each dimension corresponds to a specific frequency. Note that the inverse GFT can be similarly defined as:
\begin{equation}
    \vs = \mF \tilde{\vs} \; .
\end{equation} 

To illustrate the GFT, let us retake our previous example that used the Fourier transform in the time domain. The time domain could be represented using a simple line graph where each vertex corresponds to a specific time sample and is connected to the consecutive time samples. In that case, the GFT is very similar to the usual Fourier transform in the time domain. This can be seen in Figure~\ref{chap3:fig-signal-graph}, where we depict the same signal from Figure~\ref{chap3:fig-examples-signals} in the left part, we then show a discretization of the signal $\vs(t)$ in the center part, where each sample would correspond to a vertex in the line graph, and finally the GFT transform in the right part. As we had done in the previous example, we add white noise to both representations of our signal, and depict the noisy version, the sampled noisy version of the signal and its GFT in Figure~\ref{chap3:fig-noisy-graph}. We can observe very similar behaviors than in the previous ``classical'' Fourier domain.

\begin{figure}[ht]
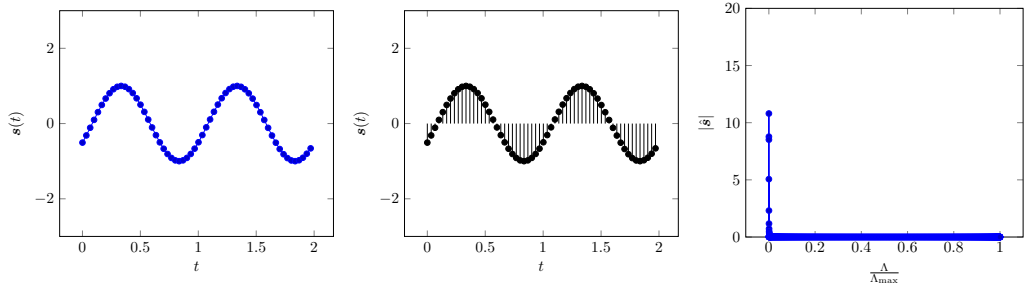

  \begin{center}
    \begin{subfigure}[ht]{.3\linewidth}
      \centering
      \begin{adjustbox}{max width=\linewidth}
  \tikzsetnextfilename{chapter3/tikz/signal}%
  \input{chapter3/tikz/signal.tex}%

      \end{adjustbox}
  \end{subfigure}
  \begin{subfigure}[ht]{.3\linewidth}
      \centering
      \begin{adjustbox}{max width=\linewidth}
  \tikzsetnextfilename{chapter3/tikz/signal_sampled}%
  \input{chapter3/tikz/signal_sampled.tex}%

      \end{adjustbox}
  \end{subfigure}        
  \begin{subfigure}[ht]{.3\linewidth}
    \centering
    \begin{adjustbox}{max width=\linewidth}
  \tikzsetnextfilename{chapter3/tikz/gft_signal}%
  \input{chapter3/tikz/gft_signal.tex}%

    \end{adjustbox}
\end{subfigure}        
\caption{Depiction of a signal $\vs(t)$, its sampled version and the GFT of the graph signal.}
  \label{chap3:fig-signal-graph}
  \end{center}
\end{figure}

\begin{figure}[ht]
  \begin{center}
    \begin{subfigure}[ht]{.3\linewidth}
      \centering
      \begin{adjustbox}{max width=\linewidth}
  \tikzsetnextfilename{chapter3/tikz/noisy_signal}%
  \input{chapter3/tikz/noisy_signal.tex}%

      \end{adjustbox}
  \end{subfigure}
  \begin{subfigure}[ht]{.3\linewidth}
      \centering
      \begin{adjustbox}{max width=\linewidth}
  \tikzsetnextfilename{chapter3/tikz/sampled_noisy_signal}%
  \input{chapter3/tikz/sampled_noisy_signal.tex}%

      \end{adjustbox}
  \end{subfigure}        
  \begin{subfigure}[ht]{.3\linewidth}
    \centering
    \begin{adjustbox}{max width=\linewidth}
  \tikzsetnextfilename{chapter3/tikz/gft_noisy_signal}%
  \input{chapter3/tikz/gft_noisy_signal.tex}%

    \end{adjustbox}
\end{subfigure}        
\caption{Depiction of a noisy signal $\breve{\vs}(t)$, its sampled version and the GFT of the graph signal.}
  \label{chap3:fig-noisy-graph}
  \end{center}
\end{figure}

We are then able to perform the same filtering operation from before, and depict the results in Figure~\ref{chap3:fig-result-gft}. In the left we have the original signal, in the middle the signal that was filtered in the time domain and in the right the signal that was filtered in the graph domain. Note that the results are not exactly the same, as our graph signal is a discretization of the real time signal, but the results are very close. We will present in the next sections some of the uses of the GFT and its abstractions.  

\begin{figure}[ht]
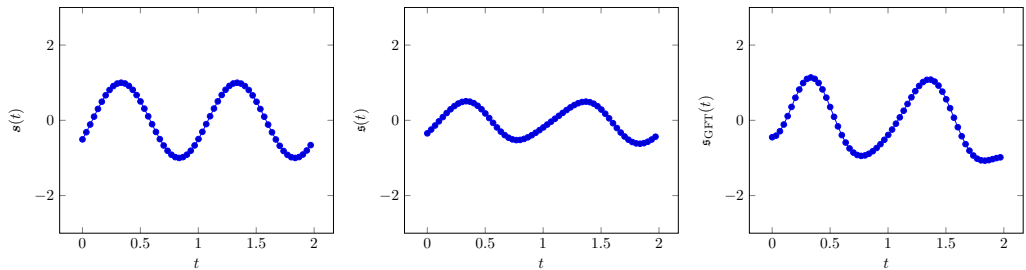

  \begin{center}
    \begin{subfigure}[ht]{.3\linewidth}
      \centering
      \begin{adjustbox}{max width=\linewidth}
  \tikzsetnextfilename{chapter3/tikz/signal}%
  \input{chapter3/tikz/signal.tex}%

      \end{adjustbox}
  \end{subfigure}
  \begin{subfigure}[ht]{.3\linewidth}
      \centering
      \begin{adjustbox}{max width=\linewidth}
  \tikzsetnextfilename{chapter3/tikz/retrieved_signal}%
  \input{chapter3/tikz/retrieved_signal.tex}%

      \end{adjustbox}
  \end{subfigure}        
  \begin{subfigure}[ht]{.3\linewidth}
    \centering
    \begin{adjustbox}{max width=\linewidth}
  \tikzsetnextfilename{chapter3/tikz/retrieved_signal_gft}%
  \input{chapter3/tikz/retrieved_signal_gft.tex}%

    \end{adjustbox}
\end{subfigure}        
\caption{Depiction of a signal $\vs(t)$ and its retrieved time FT ($\sfilter(t)$) and GFT ($\sfilter_\text{GFT}(t)$) filtered signals.}
  \label{chap3:fig-result-gft}
  \end{center}
\end{figure}

\subsection{Smoothness of graph signals}\label{chap3:smoothness}

Now that we have introduced graphs and the GFT, we can now talk about the \textbf{smoothness} of graph signals. This concept is the cornerstone for most of the methods we introduce in Chapter~\ref{chap5}. In this subsection we will first define the concept of smoothness of graph signals and then we present an illustrative example of an application of graph signal smoothness for noise detection/removal.

\subsubsection{Definition}

The smoothness $\mathcal{\sigma}$ of a graph signal $\vs$ supported on a graph $\gG$ is defined as:
\begin{equation}
    \sigma = \vs \mL \vs^\top.
\end{equation} 
This is also called the quadratic form of the Laplacian. We note that while we call this measure ``smoothness'' to be coherent with GSP literature a more adequate name would be ``anti-smoothness'' as lower values of the smoothness metric are said to be very smooth in respect to the graph and higher values of the smoothness metric are said to be very rough (or unsmooth) in respect to the graph. If we consider the eigendecomposition of the Laplacian matrix (c.f. Equation~\ref{chap3:laplacian_spectral}), we can also rewrite $\sigma$ as:

\begin{equation}
    \sigma = \vs \mL \vs^\top = \sum_{i=1}^{|\sV|} \emLambda_{i,i} \vs_{i} \; ,
\end{equation} 
where $\mLambda$ is a diagonal matrix of the eigenvalues of $\mL$ in crescent order. In this way we can see that a smooth signal will be one that is aligned with the first eigenvectors of the Laplacian. Finally, we observe that the smoothness is strongly related to the rate of change of the signal values from one vertex to its neighbors, by rewriting $\sigma$ as:  
\begin{equation}
    \sigma = \vs \mL \vs^\top = \emAdjacency_{i,j} (\vs_i - \vs_j)^2 \;.
\end{equation} 
Considering the smoothness as the rate of change is a very useful abstraction, especially when the signal $\vs$ is binary. Indeed if $\vs$ is binary, the smoothness can be simplified as the sum of the weights between nodes with different values in $\vs$. If we consider the example where each entry of $\vs$ is a binary label indicator vector first defined in Definition~\ref{chap2:label_indicator_vector} and the graph vertices correspond to samples, the smoothness will be the sum of the weights of edges connecting samples of different classes and the smoothest graph possible would be one that connects only examples of the same class\footnote{A graph that has $\sE=\emptyset$ would also have $\sigma=0$, but we do not consider this type of graph in this thesis.}. In other words, smoothness is a measure of discrepancy between a signal and a graph structure.

In the following paragraphs we introduce an application of graph signal smoothness as an illustrative example.
 
\subsubsection{Illustrative example}

Now that we have defined what is the smoothness of a graph signal, an illustrative example is in order to illustrate its usefulness. Let us reuse the example signal from Figure~\ref{chap3:fig-signal-graph} where a signal in the time domain is discretized and converted into a graph signal on a line graph. In Figure~\ref{chap3:fig-smoothness-graph} we depict the original graph signal in the left part, a noisy version of the graph signal in the center and a low-pass filtered version of the noisy graph signal in the right and their smoothness. Note how both the original graph signal and the filtered version are smoother than the noisy version of the signal. 

\begin{figure}[ht]
  \begin{center}
    \begin{subfigure}[ht]{.3\linewidth}
      \centering
      \begin{adjustbox}{max width=\linewidth}
  \tikzsetnextfilename{chapter3/tikz/signal_sampled}%
  \input{chapter3/tikz/signal_sampled.tex}%

      \end{adjustbox}
      \caption{$\sigma = 0.132$}
    \end{subfigure}
  \begin{subfigure}[ht]{.3\linewidth}
      \centering
      \begin{adjustbox}{max width=\linewidth}
  \tikzsetnextfilename{chapter3/tikz/sampled_noisy_signal}%
  \input{chapter3/tikz/sampled_noisy_signal.tex}%

      \end{adjustbox}
      \caption{$\sigma = 1270.902$}
    \end{subfigure}        
  \begin{subfigure}[ht]{.3\linewidth}
    \centering
    \begin{adjustbox}{max width=\linewidth}
  \tikzsetnextfilename{chapter3/tikz/retrieved_signal_gft}%
  \input{chapter3/tikz/retrieved_signal_gft.tex}%

    \end{adjustbox}
    \caption{$\sigma = 0.144$}
  \end{subfigure}        
\caption{Depiction of a sampled signal $\vs$, its noisy version $\breve{\vs}$ and its filtered version $\sfilter$, with their respective smoothness ($\sigma$) values.}
  \label{chap3:fig-smoothness-graph}
  \end{center}
\end{figure}

In other words the smoothness of a graph signal can also be used to detect if there is noise present in it. In the following of this document we will introduce how this measure can be used to infer a graph from a signal in Section~\ref{chap3:graph_inference}, how it can be used to improve the robustness of DNNs in Section~\ref{chap5:laplacian} and finally how it can be used to train DNNs for classification in Section~\ref{chap5:smoothness_loss}. 

\section{Graphs for samples of features}\label{chap3:graph_uses}

In this thesis we focus on two main uses of graphs: \begin{inlinelist}\item Graphs that model the inner dependencies of observations \item Graphs that model the relationship between data samples.\end{inlinelist}. In the following paragraphs we detail exactly what we understand as each type of graph and give practical examples to illustrate each case.

\subsection{Modelling inner dependencies of observations with graphs}\label{chap3:inner_dependencies}

One of the uses of graphs is modelling the inner dependencies of the observations. In this case each vertex is a coordinate of an observation and the relationship between vertices encode the relationship between the different coordinates. Encoding the relationship between different coordinates can also be seen as exploiting the intrisic structure of the data samples. This is mostly used for supervised classification of graphs (e.g., protein-protein iteration) and supervised classification of graph signals (e.g., classification of scrambled images as described in Section~\ref{chap4:classify_graph_signals}).

For example consider that our observations/elements of interest are images. In this case we can create a grid-graph to emulate the intrisic euclidean 2D structure. The grid representation creates a graph that model the inner dependencies between the pixels (coordinates of an observation) and therefore allow one to extract important information from the intrisic structure (c.f. Section~\ref{chap4:classify_graph_signals}). We depict in Figure~\ref{chap3:fig-grid-graph} an image and its grid graph representation. 

\begin{figure}[ht]
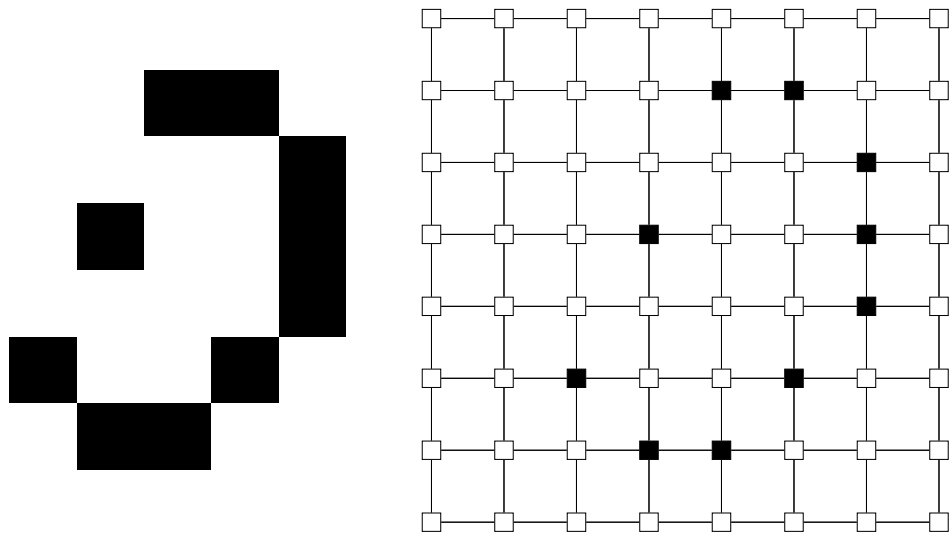

  \begin{center}
      \begin{subfigure}[ht]{.48\linewidth}
            \includegraphics[width=\linewidth]{chapter3/usps.png}
      \end{subfigure}
      \begin{subfigure}[ht]{.48\linewidth}
        \resizebox{\linewidth}{!}{%
  \tikzsetnextfilename{chapter3/tikz/usps}%
  \input{chapter3/tikz/usps.tex}%
}
      \end{subfigure}
      \caption{Depiction of an image (left) and its grid graph representation (right).}
      \label{chap3:fig-grid-graph}
  \end{center}
\end{figure}

\begin{definition}[2D grid graph]\label{chap3:def-gridgraph}
We call (2D) \emph{grid graph} a graph whose vertices are of the form $\{1,\dots,\ell\}\times \{1,\dots,h\}$ where $\ell,h\geq 3$, and edges are added between vertices at Manhattan distance one from each other.     
\end{definition}

This grid graph image representation has been exploited in multiple GSP works as a basis for explanation and visualization~\citep{pasdeloup2017extending, shuman2013emerging}. In the case of an image grid graph, each pixel is a vertex $v_i$ and the RGB values form a graph signal $\mS \in \R^{|w h| \times 3}$ where $w$ and $h$ represent the image width and height ($|\sV|$) respectively and 3 is the number of the RGB color channels ($|F|$). In other words, we have $|\sV| = |w h|$ coordinates and $|F|=3$ the number of observations or realizations per coordinate.

Using a low-pass graph filter (that we introduce more formally in Section~\ref{chap3:graph_filter}) it is possible to exploit the structure of the pixels to remove noise from the image as shown in Figure~\ref{chap3:fig-denoising}. Note how by removing the high frequencies of the graph signal, we obtain an image that is more inline with the original image. Also note that the smoothness value of the recovered graph signal is the same as the original image, while the images are very different. Indeed as the smoothness of a graph signal is a global measure and not a localized one it is easy to see that multiple image configurations are possible for the same value of smoothness.

\begin{figure}[ht]
  \begin{center}
      \begin{subfigure}[ht]{.3\linewidth}
        \includegraphics[width=\linewidth]{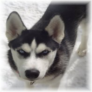}
        \caption{$\sigma = \sigma_1$}
      \end{subfigure}
      \begin{subfigure}[ht]{.3\linewidth}
        \includegraphics[width=\linewidth]{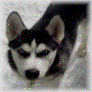}
        \caption{$\sigma = 1.69\sigma_1$}
      \end{subfigure}
      \begin{subfigure}[ht]{.3\linewidth}
        \includegraphics[width=\linewidth]{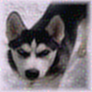}
        \caption{$\sigma = \sigma_1$}
      \end{subfigure}
      \caption{Depiction of a dog (left), a noisy realization of the image (center) and the graph filtered image (right).}
      \label{chap3:fig-denoising}
  \end{center}
\end{figure}

Finally, it is also possible to infer the graph structure solely from the signal. One naive approach is to compute a similarity metric between the signal coordinates to generate a weighted adjacency matrix and then threshold the $k$ most connected neighbors of each coordinate to generate the graph. Note that the adjacency matrix that is obtained by this approach needs to be then symmetrized. The resulting graph is commonly called a $k$-nn similarity graph. We delve into more details in graph inference in Section~\ref{chap3:graph_inference} 

\subsection{Graphs that model the relationship between data samples}\label{chap3:relationship_samples}

The other type of graphs that we consider in this work are graphs that model relationship between different observations. In this case each observation is a vertex and the relationship between the vertices encodes a relationship between them. This relationship depends on the task at hand, for example in the document classification datasets presented in Section~\ref{chap2:citation_networks}, an edge ($v_1,v_2$) will encode the fact that document $v_1$ cites $v_2$ or is cited by $v_2$. 

In a graph that models the relationship between data samples, the goal is to exploit the relationship between the different samples rather than the intrisic structure of each one. This allows us to consider other scenarios such as semi-supervised classification of vertices in the graph. We introduce an example of semi-supervised classification of vertices in the following paragraphs.

Consider that we have a subset of the CIFAR-10 dataset (c.f. Section~\ref{chap2:cifar10}), where we only consider the cat and truck classes, with 10 training examples per class (labeled examples) and 50 test examples per class (unlabeled examples). We can then apply a naive graph inference technique using the cosine similarity between the samples and threshold the $k$-neighbors in order to generate a graph. We depict the generated graphs in Figure~\ref{chap3:fig-cifar10-graphs}, where on the left we depict the graph masking the unlabeled examples, on the center the we depict the graph with labels retrieved using the label propagation algorithm (c.f. Section~\ref{chap3:benchmarking} for more details) and finally on the right we depict the ground-truth. 

\begin{figure}[ht]
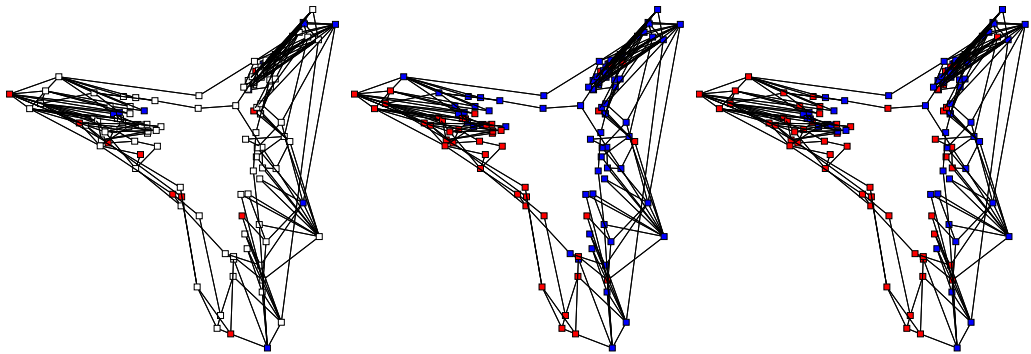

  \begin{center}
    \begin{subfigure}[ht]{.3\linewidth}
      \resizebox{\linewidth}{!}{%
  \tikzsetnextfilename{chapter3/tikz/cifar-10_graph}%
  \input{chapter3/tikz/cifar-10_graph.tex}%
}
    \end{subfigure}
    \begin{subfigure}[ht]{.3\linewidth}
      \resizebox{\linewidth}{!}{%
  \tikzsetnextfilename{chapter3/tikz/cifar-10_graph_retrieved}%
  \input{chapter3/tikz/cifar-10_graph_retrieved.tex}%
}
    \end{subfigure}
    \begin{subfigure}[ht]{.3\linewidth}
      \resizebox{\linewidth}{!}{%
  \tikzsetnextfilename{chapter3/tikz/cifar-10_graph_groundtruth}%
  \input{chapter3/tikz/cifar-10_graph_groundtruth.tex}%
}
    \end{subfigure}
  \caption{Depiction of a graph connecting samples from a subset of the CIFAR-10 dataset, where red represents images of cats, blue represents images of trucks and white represents unlabeled data. On the left we have masked the unlabeled data, while on the center we use a label propagation algorithm to retrieve the labels of the unlabeled data with 70\% accuracy and finally on the right we depict the ground truth labels.}
      \label{chap3:fig-cifar10-graphs}
  \end{center}
\end{figure}

Note how just by organizing the data in a graph structure and using Laplacian eigenmaps~\citep{belkin2003laplacian}, i.e. taking advantage of the first two eigenvectors of the laplacian\footnote{we consider the first two eigenvectors to be the eigenvectors associated with the two smallest nonzero eigenvalues.} to position the nodes in a 2D space, we can better understand our data and retrieve a very simple semi-supervised classification baseline. In the following of this thesis we will use this representation in order to solve various tasks.

\section{Inferring graph topology from signals}\label{chap3:graph_inference}

In the previous sections we defined graphs, graph signals and gave examples on how they are going to be used in this thesis. However even if a signal has an underlying structure, this support structure is not necessarily explicitly available.  If we do not have access to the adjacency matrix that supports a signal, inferring this graph topology is required. We published a large portion of the results in this section in~\citep{lassance2020benchmark}.

Many recent works have tackled the problem of inferring the graph topology using graph signals~\citep{pasdeloup2017characterization,kalofolias2018large,shekkizhar2019graph}; see also~\citep{gonzalo_spmag_19} for a recent tutorial treatment. Inferring a graph structure can be performed in a task-agnostic manner, i.e., labeling data is not used. Priors are used to relate observations to the sought task-agnostic graph structure: e.g. smoothness~\citep{kalofolias2018large} (c.f. Section~\ref{chap3:smoothness}), stationarity~\citep{pasdeloup2017characterization}, sparsity~\citep{shekkizhar2019graph}, and probabilistic~\citep{egilmez2017jstsp} as well as graph filtering-based~\citep{shafipour2018topoidnsTSP18} generative models, just to name a few. 

On the other hand, sometimes it is interesting to consider graphs that are specific to the task at hand. For example in~\citep{koopman2016information} the authors infer graphs for medical search and in~\citep{iscen2019label,lassance2019improved,hu2020exploiting}, the authors aim at improving the accuracy of various classification tasks using inferred graphs: semi-supervised learning, visual based localization and few-shot learning.

As it is more general, task-agnostic graph inference is of particular interest. In this context, the two most common goals are visualization~\citep{kalofolias2018large} and interpretation~\citep{anirudh2017influential} using the graph support. On the other hand, it is challenging to compare methods for which there is no ground truth, i.e., that are task agnostic. Unsurprisingly, many works rely on synthetic data to evaluate the ability of their proposed methods in unveiling the topology from the observations. We present a benchmark and delve into more details on the discussion of evaluating task-agnostic methods in Section~\ref{chap3:benchmark_graph}.

In the following subsections we present a simple framework of graph inference that we call ``naive baselines'' and we go into more details in the two methods that we are going to use in this work: \begin{inlinelist}\item Kalofolias~\citep{kalofolias2018large} \item NNK~\citep{shekkizhar2019graph}. \end{inlinelist}

\subsection{Naive baselines for graph topology inference}

In this subsection we present a simple framework that allows us to infer graph structures from data with 3 quick steps:

\begin{enumerate}
    \item Choosing a similarity measure to be applied to the features of each vertex in order to determine their connectivity. In more details, we consider cosine similarity, sampled covariance or an RBF kernel applied on the $L_2$ distance between considered items. The result is a square matrix homogeneous to the desired graph structure. Note that some of the similarity measures may require extra hyperparameters such as the $\gamma$ in an RBF kernel (c.f., Equation~\ref{chap2:eq-rbf}).
    \item Choosing a number of neighbors to be kept for each vertex. We simply use a $k$-nearest neighbor selection. Note that we symmetrize the resulting graph, so that each vertex has at least $k$ neighbors. Note that as it was the case with the similarity measure, in this step we are also adding an extra hyperparameter, which is the number of neighbors to keep.
    \item Normalizing the obtained graph adjacency matrix. Note that the decision of normalizing (and which normalization to use) or not the adjacency matrix adds another hyperparameter to the framework.
\end{enumerate}

Note that while this framework is very simple, the amount of decisions and possible combinations is very large. In this subsection as we want to investigate the maximum ability of the methods we test a very extensive number of combinations and always report the best one. In the remaining of this work (excluding the present Section), we will either use a fixed decision on how to infer our graphs from data or perform minimal experiments to show the impacts of this decision.

\subsection{Graph construction with Non Negative Kernel (NNK) regression}

We now consider a recently proposed graph inference method. We choose NNK (Non Negative Kernel regression)~\citep{shekkizhar2019graph}, due to its simplicity and its demonstrated results on semi-supervised learning tasks. This method can be interpreted as producing representations with orthogonal approximation errors, which in turn favors sparser representations. It has two parameters: $k$, the maximum degree for each vertex, and $\sigma$ the minimum value for an edge weight (threshold). In this work we test multiple values of $k$ and fix $\sigma=10^{-4}$ following~\citep{kalofolias2018large}. In our experiments we use the authors implementation from~\url{https://github.com/STAC-USC/PyNNK_graph_construction}.

\subsection{Graph learning from smooth signals}

Graphs can be inferred from graph signals using certain priors. In this subsection we consider that a signal should be smooth (c.f., Section~\ref{chap3:smoothness}) on its graph support. In this thesis we rely on a state-of-the-art approach in~\citep{kalofolias2018large}. It consists in a framework that infers the graph from an underlying set of smooth signals. As it was the case with NNK, it has two parameters: $k$ the desired mean sparsity and $\sigma$ the minimum value for an edge weight. In this work we test the same values for these two parameters as we did for NNK and keep the best combination. In our experiments we use the implementation from the GSP toolbox~\citep{perraudin2014gspbox}.

\subsection{Benchmarking graph topology inference}\label{chap3:benchmark_graph}

Now that we have presented the graph topology, we can discuss the problem of comparing these methods. Most methods have to rely on synthetic data in order to evaluate their efficacy and to put forth their pros and cons. While synthetic data are always useful to perform controlled scalability experiments as well as reveal the emerging statistical and computational trade-offs, this validation protocol comes with two shortcomings. First, the models used to generate synthetic data are likely to be biased in favor of the proposed methods. Second, the ability of the proposed method to handle hard real-world problems is often not demonstrated convincingly, e.g., very small datasets or toy data such as black and white digit recognition.

In order to address this problem, standardized benchmarks are required. The main challenge is that benchmarks are necessarily task-specific, and as such they do not encompass the whole potential offered by task-agnostic graph inference state-of-the-art methods. To fill in this gap, in this section we detail a broad collection of benchmarks, that we first introduced in~\citep{lassance2020benchmark} that are specifically designed to compare graph inference algorithms. To this end, we consider three timely problems arising with network data: \begin{inlinelist}\item unsupervised clustering of vertices \item semi-supervised classification of vertices (with or without vertex features) and \item graph signal denoising.\end{inlinelist} For each problem we introduce an easy-to-use dataset that we release publicly\footnote{\url{https://github.com/cadurosar/benchmark_graphinference}}.

We also introduce measures of performance to confront methods and better understand their adequacy to the aforementioned tasks. Furthermore, the released datasets comprise various types of signals, namely natural images, audio, texts, and traffic information. Note that we do not include brain data and protein-protein interactions that are two of the most interesting use-cases of graph inference and classification. Our choice is informed by recent developments in the literature~\citep{he2020deep,Errica2020A}, that have found no significant performance gains when graph-based machine learning techniques are brought to bear for some tasks in these areas. We will further discuss this in Chapter~\ref{chap4}.

Our benchmarks are divided into three tasks that encompass two types of machine learning problems. In Tasks 1 and 2, the graph model dependencies between observations, c.f. Section~\ref{chap3:relationship_samples}. In the second one (Task 3), the graph models relationships between features as seen in Section~\ref{chap3:inner_dependencies}. We expect some of the previously presented methods to perform better on the first series of tasks and others to be more adequate to the second one.

\subsubsection{Benchmarking tasks}\label{chap3:benchmarking}

\paragraph{Task 1: Unsupervised Clustering of Vertices (UCV): }

Consider a dataset composed of $|\sV|=N$ observations, each one containing $F$ features. Given a number of classes $c$, we consider the task of partitioning the $N$ observations into $c$ classes, such that the variability inside classes is smaller than the variability between classes. In practice, variability can be measured using various metrics. For the purpose of obtaining quantified benchmarks, we consider here that the observations belong to $c$ categories (e.g. classes of images or sounds), and that this information is not available when processing the considered methods. So, the performance of a considered method is evaluated by computing the Adjusted Mutual Information score~\citep{vinh2009information} based on the ground truth.

Note that this clustering problem can be treated without a graph structure. Examples are using $c$-means or DB-Scan algorithms. In the context of this work, we consider using spectral clustering. Spectral clustering consists in creating a graph linking the observations where the edges are inferred from the corresponding features. Then, vertices are projected using the first eigenvectors of the graph Laplacian and clustered using standard non-graph methods. In our work, we use the discretization method first proposed in~\citep{yu2003clustering} when features have been projected onto the first $c$ eigenvectors of the graph Laplacian except the very first one. We use the default SciKit-Learn~\citep{scikit-learn} implementation of spectral clustering and of the $c$-means algorithm in our experiments. 

\paragraph{Task 2: Semi-Supervised Classification of Vertices (SSCV): }

Consider a dataset composed of $|\sV|=N$ observations, each one containing $F$ features. Here, a portion of the $N$ observations are labeled. The task consists in inferring the labels of the other portion of observations. Again, we consider datasets where we have access to the ground truth, and artificially hide the labels of part of the observations when processing the data. The score consists in measuring the accuracy of the classification on initially unlabeled observations.

This problem can be solved without relying on graphs. For example, a common solution would consist in performing a supervised classification using only the labeled observations. In this work, we consider inferring a graph connecting observations from the features. Then, we use this graph in two settings. In the first setting, we want the graph to fully encompass the information contained in the features, and therefore perform label propagation. Label propagation consists in diffusing the labels from the known observations to the other ones using the inferred graph structure. In a second setting, we use both the graph structure and the features to perform classification. We use the methodology described in~\citep{wu2019simplifying}, called Simplified Graph Convolution (SGC), where the goal is to combine feature diffusion with logistic regression. We delve into more details for this type of model in Section~\ref{chap4:classify_graph_vertices}.

In more details, we use two layers of feature diffusion ($\hat{\vx} = \adjmatrix^2 \vx$), followed by a logistic regression. The models are trained for 100 epochs, using Adam optimization with a learning rate of $0.001$. We use the average over 100 runs of the accuracy using random splits of 5\% training set and 95\% test set. We always report the average accuracy and standard deviation. To propagate labels, we simply diffuse the label signal one time using the exponential of the adjacency matrix. We note that SGC models tend to use the ``normalized augmented adjancency matrix'' $\tilde{\adjmatrix} = \identity + \adjmatrix$ where $\identity$ is the identity matrix. This augmented adjacency matrix is then normalized $\tilde{\adjmatrix} \leftarrow \mD_{\tilde{\adjmatrix}}^{-1/2} \tilde{\adjmatrix} \mD_{\tilde{\adjmatrix}}^{-1/2}$. In our work we test both the adjacency matrix and the augmented adjacency matrix and their respective normalizations and we report the best possible combination in terms of mean accuracy.

\paragraph{Task 3: Denoising of Graph Signals (DGS): }

Consider a dataset comprising $N$ observations, each one consisting of $|\sV|=F$ features. Consider some additive noise generated according to a distribution $\mathcal{N}$. The task consists in recovering initial observations from their noisy versions. We measure performance by looking at the Signal to Noise Rate.

Here, the graph connects features of observations. The idea is to use the graph structure to easily segregate components of the noise from components of the initial signals. In our work, we use a Simoncelli low-pass filter (c.f., Section~\ref{chap3:graph_filter} for more details on graph filtering) on the graph to perform denoising. Note that this filter has a parameter $\tau \in [0,1]$ that we vary from 0 to 1 in increments of 0.025. We use the noisy signal realization with a SNR (Signal to Noise Ratio) of 7, from~\citep{irion2016efficient}, and report the best SNR found for each graph construction.

\subsubsection{Datasets}

For the only purpose of benchmarking graph inference methods, we introduce here a few datasets. For Tasks 1 and 2, we use datasets of images, audio and texts (documents). To reduce the difficulty of the tasks in the image and audio domains, we choose to use features extracted from pretrained deep neural networks. Task 3 (DGS) data comes from real life traffic information. Additional details are given in the coming paragraphs.

\paragraph{Image dataset - flowers102: }

For the image dataset we use the training set portion of the ``102 Category Flower Dataset'' (shortened as flowers102)~\citep{Nilsback08}. This split contains $N=1020$ images of $C=102$ classes of flowers (10 images per class). The features are extracted from the final pooling layer of the Inceptionv3 architecture~\citep{szegedy2016rethinking}, which has a size of $F=2048$ dimensions. Note that Inceptionv3 was trained on the 2012 split of ImageNet challenge, so that the features we obtain are a case of transfer learning. This should be one of the most challenging scenarios we consider, as it provides the highest number of classes and has the highest signal dimension to number of items ratio: 2.

\paragraph{Audio dataset - ESC-50: }

For audio data, we use ``ESC-50: Dataset for Environmental Sound Classification''~\citep{piczak2015dataset}. This dataset contains $c=50$ classes, with 40 audio signals each (2000 in total). It also contains 5 standard splits that are not used here (as we do unsupervised and semi-supervised classification). We use the feature extractor introduced in~\citep{kumar2018knowledge} to generate our dataset, that was trained on AudioSet. Similar to the images data, this can be considered as transfer. At the end we have $N=2000$ items with $F=1024$ dimensions each. The signal dimension to number of items ratio is 0.512.

\paragraph{Text dataset - cora: }

We use the cora dataset~\citep{sen2008collective} that we have presented in Section~\ref{chap2:cora}, which is composed of $N=2708$ scientific articles of $c=7$ different domains for document clustering or classification. The features come from a word indicator vector (i.e. bag of words) that indicates if one of the words in the dictionary ($F=1433$ in total) is present on the title or abstract of the document. The dictionary is built with the most common words in the dataset. The signal dimension to number of items ratio is: 0.53. Note that this dataset is classically used for graph semi-supervised learning as it comes with a citation graph. But in our work we completely disregard this graph. Comparisons between the ground truth graph and inferred ones could be an interesting addition to this work. But since the citation graph is not exactly redundant with the signals, it is expected that inferred graphs and citation ones are quite different. 

\subsection{Toronto traffic data denoising (Toronto)}
We use data from the road network of the city of Toronto, from~\citep{irion2016efficient}. It describes traffic volume data over a 24 hour period at intersections in the road network of Toronto for a total of $F=2202$ vertices and $N=1$ observation. Note that extra information is available, such as the position of each road and intersection, but our baselines only consider the raw signal data. This graph is depicted in Figure~\ref{chap3:fig-examples-ml}

\subsubsection{Empirical evaluation of benchmarks}

We now present the results for our benchmark evaluation. The tested graph topology methods and the parameters we vary are summarized in Table~\ref{chap3:summarytests}. For every test we only display the results obtained by the best combination and we further discuss the effects of parameter choice in Section~\ref{chap3:benchmark_discussion}.

\begin{table}[ht]
    \begin{center}
   \caption{Summary of the tested graph topology inference methods. Table and caption extracted from~\citep{lassance2020benchmark}.}
   \label{chap3:summarytests}
   
   \begin{adjustbox}{max width=\columnwidth}    

   \begin{tabular}{|c|c|c|c|c|}
   \hline
   Method     & Similarity/Distance               & $k$                                        & $\sigma$       & Adjacency matrices     \\ \hline
   Naive      & \multirow{2}{*}{Cosine, Covariance, RBF}           & \multirow{3}{*}{\parbox{3.2cm}{5, 10, 20, 30, 40, 50, 100, 200, 500, 1000}} & None                  & \multirow{3}{*}{\parbox{3.3cm}{$\adjmatrix$, $\mD_{\adjmatrix}^{-1/2} \adjmatrix \mD_\adjmatrix^{-1/2}$, $\tilde{\adjmatrix}$, $\mD_{\tilde{\adjmatrix}}^{-1/2} \tilde{\adjmatrix} \mD_{\tilde{\adjmatrix}}^{-1/2}$}} \\ \cline{1-1} \cline{4-4}
   NNK~\citep{shekkizhar2019graph}        &  &                                                    & \multirow{2}{*}{$10^{-4}$} &                           \\ \cline{1-2}
   Kalofolias~\citep{kalofolias2018large} & Square Euclidean distance         &                                                    &                       &                           \\ \hline
   \end{tabular}
  \end{adjustbox}
   \end{center}
   \end{table}

\paragraph{Task 1: }

For the UCV task, we display both the results obtained with the inferred graph structures and with a $c$-means baseline.  The results are presented in Table~\ref{chap3:tab-unsup}. We can see that both naive and NNK get the most consistent results, with Kalofolias having difficulties with the cora dataset.  

\begin{table}[ht]
 \begin{center}
\caption{Results for Task 1. Here we present the best AMI score for each inference method. Table and caption extracted from~\citep{lassance2020benchmark}.}
\label{chap3:tab-unsup}
  \begin{adjustbox}{max width=\columnwidth}
\begin{tabular}{|c|c|c|c|c|}
\hline
\multicolumn{1}{|l|}{Method}         & \multicolumn{1}{l|}{Inference/Dataset} & \multicolumn{1}{l|}{ESC-50} & \multicolumn{1}{l|}{cora} & \multicolumn{1}{l|}{flowers102} \\ \hline
\multicolumn{2}{|c|}{$C$-means}                                                  & 0.59                    & 0.10                    & 0.36                          \\ \hline
\multirow{3}{*}{Spectral clustering} & Naive                                     & \textbf{0.66}                      & \textbf{0.34}           & \textbf{0.45}                          \\ \cline{2-5}
                                     & NNK                                     & \textbf{0.66}             & \textbf{0.34}           & 0.44                          \\ \cline{2-5}
                                     & Kalofolias                              & 0.65                      & 0.27                    & 0.44                           \\ \hline
\end{tabular}
\end{adjustbox}
\end{center}
\end{table}

\paragraph{Task 2:}

For the SSCV task, the results are presented in Table~\ref{chap3:tab-semisup}. We can see that using a similarity graph as support helps when compared to a simple logistic regression. Note that unfortunately, this is not a 100\% fair comparison as the logistic regression is not able to exploit the unsupervised data. In this task we have two methods, Label Propagation and SGC. In the first one, Kalofolias presents the best results for both flowers102 and ESC-50, but still struggles with the cora dataset. In SGC both Kalofolias and NNK seem to not be able to improve that much over the naive baselines.

\begin{table}[ht]
    \begin{center}
   \caption{Results for Task 2. Here we present the best mean test accuracy and its standard deviation for each inference method. Table and caption extracted from~\citep{lassance2020benchmark}.}
   \label{chap3:tab-semisup}
   \begin{adjustbox}{max width=\columnwidth}    
   
   \begin{tabular}{|c|c|c|c|c|}
   \hline
   \multicolumn{1}{|l|}{Method}         & \multicolumn{1}{l|}{Inference/Dataset} & \multicolumn{1}{l|}{ESC-50} & \multicolumn{1}{l|}{cora} & \multicolumn{1}{l|}{flowers102} \\ \hline
   \multicolumn{2}{|c|}{Logistic Regression}                                      & 52.92\% $\pm 1.9$              & 46.84\% $\pm 1.6$                 & 33.51\% $\pm 1.7 $                        \\ \hline
   \multirow{3}{*}{Label Propagation}   & Naive                                     & 59.05\% $\pm 1.8$              & \textbf{58.86\%} $\pm 2.9$                 & 36.73\% $\pm 1.6$                         \\ \cline{2-5} 
                                        & NNK                                     & 57.44\% $\pm 2.2$              & 58.66\% $\pm 2.9$                 & 33.57\% $\pm 1.6$                         \\ \cline{2-5} 
                                        & Kalofolias                              & \textbf{59.16\%} $\pm 1.8$              & 58.60\% $\pm 3.4$                 & \textbf{37.01\%} $\pm 1.7$                      \\ \hline
   \multirow{3}{*}{SGC}                 & Naive                                     & 60.48\% $\pm 2.0$            & \textbf{67.19\%} $\pm 1.5$                 & \textbf{37.73\%} $\pm 1.5$                         \\ \cline{2-5}
                                        & NNK                                     & \textbf{61.38\%} $\pm 2.0$    & 66.58\% $\pm 1.5$                 & 36.81\% $\pm 1.5$                         \\ \cline{2-5} 
                                        & Kalofolias                              & 59.36\% $\pm 2.0$              & 66.28\% $\pm 1.5$                 & 37.5\% $\pm 1.5$                         \\ \hline
   \end{tabular}
  \end{adjustbox}
   \end{center}
   \end{table}

\paragraph{Task 3: }

For the graph signal denoising task, the results are presented in Table~\ref{chap3:tab-denoising}. In this scenario we are not able to use neither cosine or covariance similarity. We compare our results with the ones we would obtain using the ground truth road map graph. Our RBF baselines were able to reduce the amount of noise, but not at the same level as of the real road graph. The Kalofolias smooth graph was able to achieve a better SNR than the real road graph.

\begin{table}[ht]
 \begin{center}
\caption{Results for Task 3. Here we present the best test accuracy for each baseline. Table and caption extracted from~\citep{lassance2020benchmark}.}
\label{chap3:tab-denoising}
  \begin{adjustbox}{max width=\columnwidth}
\begin{tabular}{|c|c|c|c|c|}
\hline
\multirow{2}{*}{Best SNR}               & Road graph & Kalofolias & RBF NNK & RBF $k$-NN \\ \cline{2-5}
          & 10.32 & \textbf{10.41} & 9.99 & 9.80    \\ \hline
\end{tabular}
\end{adjustbox}
\end{center}
\end{table}

\subsubsection{Discussion}\label{chap3:benchmark_discussion}

Over all tasks we can extract some lessons on graph inference:
\begin{enumerate}
    \item \textbf{Similarity choice:} If we have multiple non-negative realizations of the signal, cosine seems the best choice. It has competitive results on all benchmarks and it does not come with a parameter (as does RBF with $\gamma$).
    \item  \textbf{Choosing parameter $k$}: The best amount of sparsity depends not only on the dataset and task, but on the similarity that was chosen. We consider the ESC-50 dataset as an example. In the spectral clustering the best $k$ value for the $k$-NN graph was 30 for cosine, 5 for RBF and 20 for covariance. We note that in the graph denoising task, the best case was to not perform $k$-neighbors thresholding.
    \item \textbf{Normalization:} Note that only our graph denoising task does not expect a normalized graph, therefore most of our better results used normalized graphs. On the graph denoising task, normalized and non-normalized graphs had similar results.
    \item \textbf{Cora dataset}: The cora dataset is challenging not only because it is not class-balanced, but also because its features are binary (a bag of words, containing 1 if the word is present in the article and 0 if not). This could be a reason for the bad performance of both NNK and Kalofolias in this dataset.
    \item \textbf{Sparse graphs in semi-supervised problems:} In the semi-supervised tasks, the test accuracy standard deviation over the splits was very high. This could possibly be caused by the fact the sparse graphs we use here have more than one connected component, meaning that sometimes there could be sections of the graph that do not have any labeled vertices. One possible future direction would be to integrate a graph sampling algorithm to the problem in order to select which vertices we should label, instead of doing so randomly.
    \item \textbf{Naive Baselines vs. optimization approaches:} Over our tests there was no clear winner between simply doing a naive $k$-NN approach and more advanced graph topology inference techniques. Kalofolias had very good performance on the Label Propagation and Denoising tasks, while NNK was consistent in SGC and Spectral Clustering, but both were not able to consistently beat the naive baseline. On the other hand, there was a clear advantage of both Kalofolias and NNK over the naive baselines when we consider the robustness of both methods to the parameter $k$ selection.
    \end{enumerate}

\section{Graph filters}\label{chap3:graph_filter}

In the previous sections we touched on graph filters and their applications, without properly defining them. We do so now in this section. Graph filters are modeled using the same abstraction as traditional signal processing filters and in this work we detail and derive three possible representations of these filters:
\begin{enumerate}
  \item directly defined in the spectral domain by a diagonal matrix;
  \item defined as a function of the filter spectral response;
  \item as diffusion operator based on the graph adjacency or Laplacian matrix. 
\end{enumerate}

We describe in the following paragraphs the three different types of definitions, the advantages/drawbacks of using each representation  and describe some of their applications. We refer the reader to~\citep{hammond2011wavelets, tremblay2018design} for a more in depth discussion on graph filters.

\subsection{Defining filters in the spectral domain}

The simplest and most general way to define a filter is simply to describe its response to each frequency. In the case of graph signals the frequencies are defined by the diagonal values of the eigenvalue matrix $\mLambda$. For notation simplicity we consider the frequency vector $\vlambda$ where $\evlambda_i = \emLambda_{i,i}$. Recall that we consider that the eigenvalues are ordered in crescent order of magnitude. 

We thus define a filter on a graph $\gG$ using a diagonal matrix $\mH_{\gG} \in \R^{|\sV|\times |\sV|}$, where we call each element $\emH_{i,i}$ the response of the filter to the frequency $\vlambda_i$. The filter can then be applied to a graph signal $\vs$ by first converting the signal to the frequency domain using the GFT:
\begin{equation}
\tilde{\vs} = \mF^\top \vs \;.
\end{equation}
Now that the signal is on the frequency domain, we can apply the filter and obtain the filtered signal $\tilde{\sfilter}$ by a simple matrix multiplication:
\begin{equation}
\tilde{\sfilter} = \mH_{\gG} \tilde{\vs} \;.
\end{equation}
Finally, we can convert the signal to the vertex domain using the inverse GFT:

\begin{equation}
  \sfilter = \mF \tilde{\sfilter} \;.
\end{equation}
  
Defining filters directly in the spectral domain generates filters that are graph-specific. Indeed, as the frequencies $\evlambda_i$ are discretized, the same filter $\mH$ will be applied to very different frequencies for two distinct graphs $\gG$ and $\gG'$. Therefore this type of filter tends to be mostly used to remove the lowest or highest frequencies of the graph, without considering their ``true'' value. Another drawback is that it is unlikely that this type of filter may be represented with a low order polynomial filter which impacts the complexity (i.e., the possibility of scaling to larger graphs) of applying the filter.

\subsection{Defining filters using their spectral response}

Another more straightforward way to define a filter is by its spectral response, i.e., as a function of the frequency. In this way the filter becomes less graph dependent and more general. One such design is the Simoncelli filter, that we depict in Figure~\ref{chap3:fig-simoncelli} and that is defined by the following function:

\begin{equation}
  h(\evlambda_i)=\begin{cases} 1 & \mbox{if }\evlambda_i\leq \frac{\tau}{2}\\
              \cos\left(\frac{\pi}{2}\frac{\log\left(\frac{\evlambda_i}{\tau}\right)}{\log(2)}\right) & \mbox{if }\frac{\tau}{2}<\evlambda_i\leq\tau\\
              0 & \mbox{if }\evlambda_i>\tau \end{cases},
\end{equation}
where $\tau\in [0,1]$ is a user-defined threshold and $\lambda_i$ the $i$-th Laplacian eigenvalue. We consider that the eigenvalues are always normalized by dividing by the largest one, so that $0\leq \lambda_i \leq 1$. Defining a filter by its spectral response allow for more universal filters (i.e., filters that do not heavily depend on the graph support) and also to more easily represent the filter by a low order polynomial function. Indeed in this thesis we use the PyGSP~\citep{pygsp} toolbox to implement this type of graph filters which uses the Chebyshev polynomial approximation in order to apply the filters.

\begin{figure}[ht]
  \begin{center}
  \tikzsetnextfilename{chapter3/tikz/simoncelli}%
  \input{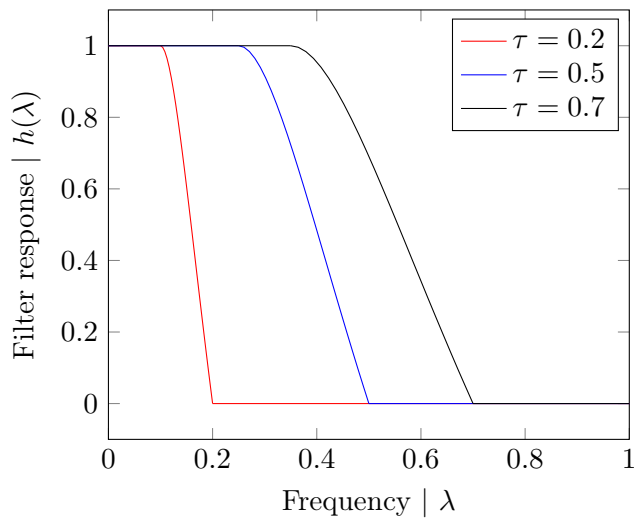}%

     \caption{Depiction of the Simoncelli filter spectral response for various values of $\tau$. }
     \label{chap3:fig-simoncelli}
  \end{center}
\end{figure}

\subsection{Defining filters using diffusion operators}

It is also possible to define a graph filter using a diffusion operator $\mS \in \R^{|\sV|\times |\sV|}$, sometimes also called graph shift operator. This operator is directly applied to a graph signal as follows:
\begin{equation}
  \sfilter = \mS \vs \;.
\end{equation} 
Note that by being applied with just a simple matrix multiplication, this type of filter is easily integrated in deep learning scenarios, where matrix multiplication is king. Indeed, most of the recent developments in graph convolutional layers use this design. We delve into this with more details in Section~\ref{chap4:classify_graph_vertices}.

We have previously introduced the concepts of transfer learning and using DNNs as feature extractors (c.f. Section~\ref{chap2:feature_extraction}). In the following subsections we consider the features extracted with the DNNs as graph signals where the graph support is not explicitly provided (i.e., we need to infer the graph structure). We then apply low-pass graph filters to reduce the amount of noise in the extracted features and improve the performance in a downstream task.

\subsection{Reducing the noise of DNN extracted features for supervised learning and few-shot learning}\label{chap3:reducing_noise}

In this section we analyze the extracted features in two scenarios: 
\begin{enumerate} 
\item \emph{Few-shot classification}: in this task a network has been trained on classes that do not belong to the problem (base) and we want to perform transfer learning to classify new (novel) classes;
\item \emph{Classification}: in this case both base and novel are the same, \emph{i.e} we aim to improve the performance of the original network. 
\end{enumerate} 

In the case of few-shot classification we evaluate our method three commonly used datasets for this problem:
\begin{enumerate}
    \item MiniImagenet~\citep{ravi2017optimization}: A subset of Imagenet, specifically designed for few-shot classification;
    \item CIFAR-FS\citep{bertinetto2018metalearning}: A reorganization of the CIFAR-100 dataset in order to transform into a few-shot problem (FS = Few-Shot);
    \item Caltech-UCSD Birds 200 (CUB)~\citep{welinder2010cub}: A dataset of bird classification.
\end{enumerate}

We also add a test called CUB-Cross, where the base classes come from the MiniImagenet dataset, while the novel classes come from the CUB dataset. Finally, we evaluate the classification problem using the previously described CIFAR-10 dataset.

We consider that the features that are extracted from the DNN are signals on graphs, where each node is an example of $\trainset$. The classes are separated in different graphs (1 per class). Our intuition is that by removing the high frequencies from the graph of each class, we can reduce the intra-class noise. With less intra-class noise, we believe that we will be able to increase the classification performance when using these features. In the following subsection we detail our methodology.

\subsubsection{Methodology}

Recall that we can split a deep neural network $f$ into a feature extractor ($\featureextractor$) followed by a a classifier ($\classifier$), such that $f(\vx) = \classifier(\featureextractor(\vx))$. We call $\hat{\vx}$ the feature tensor extracted using $\featureextractor$. We consider that the input tensor $\vx$ contains all elements of the training set $\trainset$. It can be divided in non-intersecting subsets $\vx_c$ which contain all elements of class $c$. Therefore, $\hat{\vx}_c = \featureextractor(\vx_c)$. 

Using the cosine similarity we can define the adjacency matrix for the graph of each class $c$ as $\adjmatrix_c$. We also call $L_c$ the normalized Laplacian matrix for each graph. This allow us to define low-pass graph class filters. In this section we test two of such filters, one defined directly on the spectral domain and one by its spectral response.

\paragraph{Filter on the spectral domain: } 

We define our low pass filters on the spectral domain using a diagonal matrix $\mH_{\gG_c}$ defined as:

\begin{equation}
     \mH_{\gG_c}[i,i] = \left\{ \begin{array}{cl}1.0 & \text{if } i \leq F_1\\ 0.2 & \text{if } F_1 < i \leq F_2\\ 0 & \text{otherwise}\end{array}\right.\;.
\end{equation}
Where $(F_1,F_2)$ are fixed for each scenario. In the few shot scenario we have graphs with 5 nodes and fixed the values to $(1,3)$. On the other hand, on the classification scenario graphs have 5000 nodes, so we have fixed our values to $(20,55)$. These filters were designed by hand and a possible extension of this work would be to have an automatic way of choosing these filters or even integrating them as parameters during the learning phase.

\paragraph{Filter defined by its spectral response: }

We use the previously defined Simoncelli filter, choosing an $\tau\in [0,1]$ that is inline with the task at hand. We use $\tau=0.025$ for the classification scenario (CIFAR-10 dataset) and for the few-shot scenario we use a different value of $\tau$ per dataset: \begin{inlinelist} \item MiniImagenet, $\tau=0.3$ \item CUB, $\tau=0.3$ \item CIFAR-FS, $\tau=0.3$ \item CUB-Cross, $\tau=0.35$ \end{inlinelist}.

In the following subsection we use empirical experiments to verify our intuition and show that denoising the representations with low-pass graph filters improves the classification accuracy on both scenarios, allowing us to even beat the state of the art in a competitive image classification benchmark.

\subsubsection{Empirical experiments}

We evaluate the proposed method using two scenarios, Few shot classification (transfer learning) and image classification (improving the original network). The code for all experiments is available at:~\url{https://github.com/cadurosar/graph_filter}.

\paragraph{Few shot classification - Transfer learning: }

In this first case we follow the framework from~\citep{mangla2019manifold} called few shot with backbone network. In this framework we first pre-train a DNN (backbone) in a self-supervised way on a bigger dataset of base classes. We then use the backbone to perform transfer learning on the novel classes. Recall that the base classes and novel classes do not intersect. In this scenario, in each iteration our training (support) set is composed of 5 examples (shots) for each of 5 classes (ways) that are taken at random from novel classes. The test (query) set is composed of 595 images of each class for the MiniImagenet and CIFAR-FS datasets and 15 images for CUB and CUB-Cross. We perform 100,000 iterations and report the mean accuracy and 95\% confidence intervals for each test. We use the pretrained networks that the authors published on~\url{https://github.com/nupurkmr9/S2M2_fewshot}. We test our method on 3 datasets (MiniImageNet, CUB, CIFAR-FS) for the in-domain Transfer Learning (TL) scenario and on one dataset for a cross domain TL. In the former the pre-trained network base classes come from the same dataset as do the novel classes, while on the latter the base classes come from MiniImagenet and the novel from CUB.

For each iteration we first extract the features $\hat{\vx}$ using the pre-trained feature extractor. We then infer 5 graphs of 5 nodes (as we have 5 classes and 5 examples per class), and apply the graph filter. This generates our filtered features $\xfilter$, that are classified with a simple 1-NN classifier. This is a stress test of the method as it has to be robust to a multitude of different graphs and different features. We compare our method against an 1-NN classifier on $\hat{\vx}$, a Nearest Mean Classifier (NMC) on $\hat{\vx}$, a Logistic Regression\footnote{using the default parameters of Scikit-Learn} (LR) trained on $\hat{\vx}$ and to the original results from~\citep{mangla2019manifold}. Note that none of these approaches change the test features. Results are described in Table~\ref{chap3:tab-resultsfilterfew}. We were able to improve the performance using our filter and an 1-NN classifier in almost all scenarios. We also note that NCM obtained results that are better than the results from the original paper in some scenarios.

\begin{table}[ht]
  \begin{center}
  \caption{Test error comparison on the few-shot learning task.}
  \label{chap3:tab-resultsfilterfew}
  \begin{adjustbox}{max width=\columnwidth}    

  \begin{tabular}{|c|c|c|c|c|c|}
  
  \hline
  \multicolumn{2}{|c|}{Method} & \multicolumn{3}{c|}{In-Domain TL} & Cross domain TL \\ \hline 
  $\mathcal{C}$  & Data                                     & MiniImageNet              & CUB                      & CIFAR-FS         & CUB-Cross                 \\ \hline
  Original Paper & $\hat{\vx}$                              & 16.82 $\pm$  0.11         & 9.15 $\pm$ 0.44          & 12.53 $\pm$ 0.13 & 29.56 $\pm$ 0.75          \\ \hline
  1-NN           & $\hat{\vx}$                              & 21.92 $\pm$ 0.04          & 11.06 $\pm$ 0.04         & 15.86 $\pm$ 0.04 & 36.49 $\pm$ 0.06          \\ \hline
  NCM            & $\hat{\vx}$                              & 16.73 $\pm$ 0.03          & 8.94 $\pm$ 0.03          & 12.58 $\pm$ 0.04 & 30.00 $\pm$ 0.05          \\ \hline
  LR             & $\hat{\vx}$                              & \textbf{16.49 $\pm$ 0.03} & 8.92 $\pm$ 0.03          & 12.62 $\pm$ 0.04 & 29.17 $\pm$ 0.05          \\ \Xhline{2\arrayrulewidth}
  1-NN           & Spectral Filter $\xfilter$               & 16.53 $\pm$ 0.03          & \textbf{8.86 $\pm$ 0.03} & 12.51 $\pm$ 0.04 & 29.13 $\pm$ 0.05          \\ \hline
  1-NN           & Simoncelli $\xfilter$                    & 16.55 $\pm$ 0.03          & 8.92 $\pm$ 0.03          & 12.53 $\pm$ 0.04 & \textbf{29.00 $\pm$ 0.05} \\ \Xhline{2\arrayrulewidth}
  LR             & $\text{concatenate}(\xfilter,\hat{\vx})$ & \textbf{16.29 $\pm$ 0.03} & \textbf{8.84 $\pm$ 0.03} & 12.5 $\pm$ 0.04 & \textbf{28.74 $\pm$ 0.05} \\ \hline
  \end{tabular}
  \end{adjustbox}
  \end{center}
  \end{table}

\paragraph{Image classification - Improving the original network results: }

On this second case we use the well known CIFAR-10 dataset and three pre-trained architectures, WideResNet 26-10~\citep{zagoruyko2016wide}, ShakeNet~\citep{gastaldi2017shake} and PyramidNet~\citep{han2017pyramid}. The first model is trained with traditional data augmentation techniques (namely random crop and horizontal flip) while the latter two\footnote{available at: \scriptsize{\href{https://github.com/kakaobrain/fast-autoaugment}{github.com/kakaobrain/fast-autoaugment}.}} use a stronger learned policy called fast-autoaugment~\citep{lim2019fast} on top of traditional data augmentation. We extract the features $\hat{\vx}$, create $k$-nearest neighbor graphs for each class ($k=10$), apply the graph filter on each graph and generate our filtered features $\xfilter$. We now compare the performance of a 1-NN classifier applied to the filtered features $\xfilter$ to the same classifiers as in the previous section. The results are described in Table~\ref{chap3:tab-resultsfiltercifar}. The 1-NN classifier on the filtered features was able to improve the performance over both the 1-NN classifier, NCM and even beats the performance of the original network. By using three networks we can show that our filter has to be adapted mostly to the dataset and not exactly to the features that are provided. We also note that we are able to beat the state-of-the art without retraining.

\begin{table}[ht]
  \begin{center}
  \caption{Test error comparison on the classification task. We note that this task does not have confidence intervals as the objective is to improve the original network that is expensive to train.}
  \label{chap3:tab-resultsfiltercifar}
  \begin{tabular}{|c|c|c|c|c|}
  
  \hline
  \multicolumn{2}{|c|}{Method} & \multicolumn{3}{c|}{DNN architecture} \\ \hline 
  $\mathcal{C}$  & Data          & WideResNet    & ShakeNet  & PyramidNet          \\ \hline
  Original Paper & $\hat{\vx}$   & 4.18          & 2.04      & 1.44     \\ \hline
  1-NN           & $\hat{\vx}$   & 4.19          & 2.05      & 1.46         \\ \hline
  LR           & $\hat{\vx}$     & 4.18          & 2.02      & 1.46         \\ \hline
  NMC            & $\hat{\vx}$        & 4.19          & 2.03      & 1.48       \\ \hline
  1-NN           & Spectral Filter $\xfilter$ & \textbf{4.09} & 2.03      & 1.39 \\ \hline
  1-NN           & Simoncelli $\xfilter$     & 4.12 & 2.02      & \textbf{1.37} \\ \hline
  \end{tabular}
  \end{center}
  \end{table}

In the following subsection we present a similar method that uses graph filters to improve results in the contexts of Visual-Based Localization (VBL) and Image Retrieval (IR).

\subsection{Improving VBL and IR using graph filter}\label{chap3:vbl}

We have previously described the tasks of Visual-Based Localization (VBL) and Image Retrieval(IR) in Section~\ref{chap2:vbl} and Section~\ref{chap2:ir} respectively. In this section we introduce our contribution~\citep{lassance2019improved} that aims at combining DNN feature extractors and low-pass graph filters to improve performance on these downstream tasks. 

Indeed, using Deep Learning (DL) methods for VBL approaches has recently received a lot of attention. DL can be used to directly map images to poses~\citep{brahmbhatt2018geometry,kendall2015posenet} or to generate latent representation (i.e., use the DNNs as feature extractors) that are resilient to appearance changes in the images~\citep{arandjelovic2016netvlad}. The former comes with major drawbacks. For example, such methods are unable to generalize to previously unseen locations. Furthermore, small differences in query poses can cause significant localization errors. Also, appending new locations to the dataset will require retraining the whole network. On the contrary, representation methods generalize well to new data without the need for this retraining. 

Therefore in this section we focus on situations where pose information from a visual query is inferred using $\supportset$, in which images are associated with a pose. This can be seen as an image retrieval problem where the aim is to find images in a set that might have been taken from the same location as that of the query image. Once a match or set of matches is found, the pose for the query image is computed as a function of the poses of the retrieved ones. 

The method we propose in this Section takes advantage of the additional information that might be available for each image in the support set, including GPS coordinates, consecutiveness in the acquisition process or similar latent representations. This is particularly interesting for a robotics setting, where images are almost always acquired sequentially from a camera mounted on a vehicle. This sequential nature of the acquisition process suggests that images closer in time should also have close representations. Additional information such as GPS coordinates, if available, can aid in encoding global relationships between images in the database. We show that by considering such relationships between images, localization accuracy can be increased.

Moreover, enhancements can be achieved using only minor adjustments to the inference process. Specifically, we exploit relationships via a graph filter on top of pre-learned deep representations extracted from deep neural networks~\citep{arandjelovic2016netvlad,radenovic2019finetune}. In this graph, each vertex is associated one-to-one with an image (i.e., a graph that models the relationship between data samples). Edges model relations between images and are derived from the additional source of information (e.g. temporal adjacency, GPS, similar latent representations). Interestingly, the proposed method can be seen as a fine-tuning of the representations that does not require additional learning, allowing this operation to be possibly executed on a resource constrained system.

In the following subsections we first introduce how we infer graphs from the extracted features, then we introduce the graph filter that will allow us to improve the performance on the downstream tasks, and finally we derive and discuss experiments. 

\subsubsection{Graph inference}

In order to make the graph filter improve the accuracy of VBL, we first need to be sure that the edges of the graph are well chosen to reflect the similarity between two images represented as vertices, as our main goal is to exploit extra information available in the database. In this work, we consider three different sources:

\begin{itemize}
    \item Metric distance (\texttt{dist}): the distance measured by the GPS coordinates between two vertices;
    \item Sequence (\texttt{seq}): the distance in time acquisition between two images (acquired as frames in videos);
    \item Latent similarity (\texttt{latent\_sim}): the cosine similarity between latent representations.
\end{itemize}

The matrix $\adjmatrix$ can therefore be derived from the three sources as: 
\begin{equation}
  \adjmatrix = \adjmatrix_{\texttt{dist}} + \adjmatrix_{\texttt{seq}} + \adjmatrix_{\texttt{latent\_sim}}.
\end{equation}

\subsubsection{Metric distance}

In order to transform the metric distance into a similarity, we use an exponential kernel. This is parametrized by a scalar $\gamma$ that controls the sharpness of the exponential and a threshold parameter $max_{distance}$ that cuts edges between distant vertices:
\begin{equation}
\adjmatrix_{\texttt{dist}}[i,j] = \left\{ \begin{array}{ll} e^{-\gamma dist_{i,j}} & \text{if } dist_{i,j} < max_{distance}\\ 0 & \text{otherwise}\end{array}\right..
\end{equation}

\subsubsection{Sequence}

To exploit the information of time acquisition of frames, we use the function $seq(k, \mu, \nu)$ which returns 1 if the frame distance between $\mu$ and $\nu$ is exactly $k$ and 0 otherwise. We then build a matrix $\adjmatrix_{\texttt{seq}}$ parametrized by scalars $\beta_k$ and $k_{max}$:

\begin{equation}
  \adjmatrix_{\texttt{seq}}[\mu \nu] = \sum_{k=1}^{k_{max}} \beta_{k} seq(k,\mu, \nu).
\end{equation}

\subsubsection{Latent similarity}

Finally, we define a matrix $\adjmatrix_{\texttt{latent\_sim}}$ for the latent representations cosine similarity. This is parametrized by a scalar $\alpha$ that controls the importance of the latent similarity. We only compute this similarity if either the distance similarity or the sequence similarity is nonzero:

\begin{equation}
\adjmatrix_{\texttt{latent\_sim}}[\mu \nu] = \left\{ \begin{array}{ll} \alpha sim(\mu,\nu) & \text{if } \adjmatrix_{dist}[\mu \nu] > 0 \\ & \text{or } \adjmatrix_{seq}[\mu \nu] > 0,\\ 0 & \text{otherwise}\end{array}\right.
\end{equation}

where $sim$ is the latent similarity function. In this work we use the cosine similarity ($sim_\text{cos}$), but any similarity function could be considered.

\subsubsection{Graph filter}

Given the signal $\vs$ and its normalized Laplacian matrix $\mL$, we define our graph low pass filter using a diffusion matrix $\mS$:
\begin{equation}
  \mS = \left(\identity - a\mL \right)^m,  
\end{equation} 
where $a=0.1$ and $m$ is an integer that we fix to 20. Note that when $m=0$ no filtering is performed ($\mS = \identity$).

\subsubsection{Experimental results}

We first present the results concerning VBL and then we extend our empirical test to the context of IR.

\paragraph{Visual-based localization: } In the context of VBL we can infer our graphs using all three subgraphs (distance, similarity and sequence). We thus have to first search and define the needed parameters. These parameters were obtained using a grid search and keeping the best score on the Adelaide validation query. We then use both the Adelaide test query and the Sidney dataset to ensure that the parameters are not overfitted to the validation query. Note that by using the same parameters for all cities we further validate the fact that our approach does not need to be updated when adding more cities. The parameters we use are $\gamma = 0.1, \; \beta_1 = 0.75, \; \beta_2 = 0.0625, \; \beta_3 = 0.015, \; k_{max} = 3, \; \alpha = 0.66, \; m=20$.

We test the graph filter in three different cases. First the extra data is available only for the support, second it is available only for the query and finally it is available in both cases. In each case we report two metrics, the median localization error over all the queries and the percentage of localizations that have less than 25m error.

First we perform the tests on the Adelaide dataset and present the results in table~\ref{chap3:results_adelaide}. The graph filter was able to increase performance, even when applied only on the query database, and as expected, adding the graph filter during both query and support gave the best results. Recall that the parameters were defined based on the validation query, under the case where the extra data is available only for the support database.

\begin{table}[ht]
\centering
\caption{Results under different graph filter conditions for the Mapiliary Adelaide dataset. GF means Graph Filtering.}
\begin{tabular}{|c|c|c|c|c|}
\hline
Measure                                & None & GF Support & GF Query & GF S+Q \\ \hline
\multicolumn{5}{|c|}{Validation}                           \\ \hline
acc $<$ 25m & 66.84\%           & 76.23\%                & 69.64\%              & \textbf{79.22\%}                \\ 
median distance      & 8.76m            & \textbf{6.90m}                 & 13.26m              & 8.90m                 \\ \hline
\multicolumn{5}{|c|}{Test}                           \\ \hline
acc $<$ 25m       & 44.63\%           & 50.44\%                & 46.25\%              & \textbf{52.13\%}                \\ 
median distance            & 110.66m          & 24.30m                & 42.03m              & \textbf{22.49m}                \\ \hline
\end{tabular}
\label{chap3:results_adelaide}
\end{table}

Second we validate that the operation can be used on other cities and that we do not need to perform an additional grid search for the new data. The results are presented in Table~\ref{chap3:results_sydney}. As expected the graph filter allowed us to get better performance in both median distance and accuracy, while using the parameters optimized for the Adelaide dataset. This is inline with our goal that is to have an operation that we do not have to retrain or re-validate parameters for a new dataset. We note that the performance of the hard query set is not inline with a good retrieval system (several kilometers from the correct point), but it is included to show that our method allows us to increase the performance both when the NetVLAD features are already very good for the task and when they are very bad. 

\begin{table}[ht]
\centering
\caption{Results under different graph filter conditions for the Mapiliary Sydney dataset. GF means Graph Filtering.}
\begin{tabular}{|c|c|c|c|c|}
\hline
Measure                                & None & GF Support & GF Query & GF S+Q \\ \hline
\multicolumn{5}{|c|}{Easy}                           \\ \hline
acc $<$ 25m & 49.45\%           & 55.69\%                & 57.21\%              & \textbf{64.33\%}                \\ 
median distance      & 28.25m            & 13.37m                 & 15.89m              & \textbf{11.70m}                 \\ \hline
\multicolumn{5}{|c|}{Hard}                           \\ \hline
acc $<$ 25m       & 13.87\%           & 17.29\%                & 17.29\%              & \textbf{22.28\%}                \\ 
median distance            & 4000m          & 3372m                & 3246m              & \textbf{2226m}                \\ \hline
\end{tabular}
\label{chap3:results_sydney}
\end{table}

Finally we perform ablation tests to ensure that each part of the graph is important, using the Adelaide test set. The results are presented in Table~\ref{chap3:ablation}. The table shows that different sources of information are important, with each one adding to increase in performance. Metric distance and sequence being the most important features and latent similarity being more of a complementary feature (this is expected, as it is being thresholded by the other two features). This is encouraging since in the absence of any other external information (GPS, etc), one can rely on the sequential nature of data collection to get a boost in localization performance. This information is readily available in a robotics setting. 

\begin{table}[ht]
\centering
\caption{Ablation study on the Mapiliary Adelaide test query. Table extracted from~\citep{lassance2019improved}}
\begin{tabular}{|c|c|c|c|c|}
\hline
$\adjmatrix_{\texttt{dist}}$ & $\adjmatrix_{\texttt{seq}}$ & $\adjmatrix_{\texttt{latent\_sim}}$ & median distance & acc $<$ 25m \\ \hline
     &     &     & 110.66m          & 44.63\%            \\ \hline
X    &     &     & 29.26m           & 49.42\%            \\ \hline
     & X   &     & 39.11m           & 47.47\%            \\ \hline
X    &     & X   & 28.41m           & 49.56\%            \\ \hline
X    & X   &     & 24.35m           & 50.17\%            \\ \hline
     & X   & X   & 37.34m           & 47.74\%            \\ \hline
X    & X   & X   & \textbf{24.30m}  & \textbf{50.44\%}   \\ \hline
\end{tabular}
\label{chap3:ablation}
\end{table}

We also want to demonstrate the effect of successive filtering. This is achieved by applying the filter $m$ times. Theoretically, this should help increase the performance until it hits a ceiling and then it should start to slowly decrease (as it enforces connected examples of the database to be too similar to each other). The results are presented in Figure.~\ref{chap3:ablation_m}. As can be seen, there is a clear pattern of increased performance until $m=20$ after which the performance starts to degrade. It should be noted that even for $m=40$ the graph filter still performs better than the baseline ($m=0$).

\begin{figure}[ht]
 \begin{center}
  \tikzsetnextfilename{chapter3/tikz/ablation_m}%
  \input{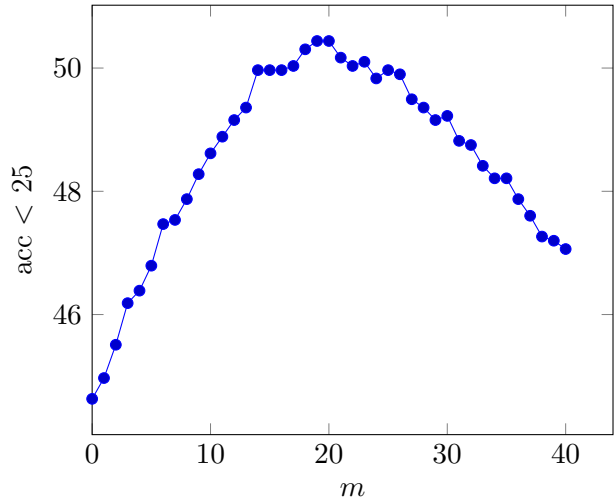}%

    \caption{Effect of the parameter $m$ on the retrieval accuracy under 25m for the Adelaide test query. Figure extracted from~\citep{lassance2019improved}. }
    \label{chap3:ablation_m}
 \end{center}
\end{figure}

\paragraph{Image retrieval: }

As a visual localization problem can be seen as an application of Image Retrieval, we test our method in classical Image Retrieval scenarios to verify its genericity. We use the revisited Oxford and revisited Paris datasets~\cite{radenovic2018revisiting} with the features from~\cite{radenovic2019finetune}. In this case we do not have the physical distance between the images to properly create $W_{\text{dist}}$ or the image sequence to generate $W_{\text{seq}}$. We therefore use the objects names as classes and our $W_{\text{dist}}$ is composed of only 1 (if $\mu and \nu$ are from the same object) and 0 otherwise. Note that in this way, we differ from traditional methods as they tend to not consider this additional information during training or testing and therefore comparison with other methods is not entirely fair. All the other parameters are the same as in the localization scenario.

In the scenario of Image Retrieval our approach can be categorized as diffusion-based. In the literature there are diffusion methods that are used during the ranking phase with $k$-NN graphs~\cite{iscen2017efficient} or that add an additional GCN~\cite{kipf2016semi} component that has to be trained in an unsupervised way~\cite{liu2019guided}. In summary, our main contribution is the graph construction (taking advantage of the class data that is available on the support set) and our smoothing/diffusion technique that is based on a low-pass filter.

The results are presented in Table~\ref{retrieval}. Our method was able to increase the mean average precision, with similar results to the approach from~\cite{iscen2017efficient}. When using in combination with~\cite{iscen2017efficient} we achieve a similar performance on the Paris dataset to a state of the art approach~\cite{liu2019guided} that requires training an additional GCN network.

\begin{table}[ht]
\begin{center}
\caption{mAP retrieval results comparison, results that do not include our filter are extracted as is from~\cite{liu2019guided}. Table extracted from~\citep{lassance2019improved}. }
\begin{tabular}{|c|c|c|c|c|c|}
\hline
                      & & \multicolumn{2}{c|}{rOxford}    & \multicolumn{2}{c|}{rParis}    \\ \hline
Features &  Ranking               & Medium         & Hard           & Medium         & Hard          \\ \hline
~\cite{radenovic2019finetune} & $sim_\text{cos}$              & 64.7           & 38.5           & 77.2           & 56.3          \\ \hline
~\cite{radenovic2019finetune} & \cite{iscen2017efficient}        & 69.8           & 40.5           & 88.9           & 78.5          \\ \hline
~\cite{radenovic2019finetune} + Our filter~\citep{lassance2019improved} & $sim_\text{cos}$ & 70.58 & 47.67 & 87.77 & 76.04 \\ \hline
~\cite{radenovic2019finetune} + Our filter~\citep{lassance2019improved} & \cite{iscen2017efficient} & 71.41 & 51.27 & 91.54 & 81.85 \\ \hline
~\cite{radenovic2019finetune} + ~\cite{liu2019guided} & ~\cite{liu2019guided}        & \textbf{77.8}           & \textbf{57.5}           & \textbf{92.4}           & \textbf{83.5}          \\ \hline
\end{tabular}
\label{retrieval}
\end{center}
\end{table}

In the previous paragraphs we showed that using techniques from Graph Signal Processing, the performance of visual based localization and image retrieval can be improved by incorporating additional available information. This additional information acts on the latent representation by making it smoother on a graph designed using all available information, leading to a boost in localization. One encouraging observation is that this additional information can take the form of a simple temporal relationship between surrounding images acquired in a sequence, and still lead to a significant increase in performance.

\section{Summary of the chapter}

In this chapter we have introduced the concepts of graphs and graph signals, alongside with the needed tools from the Graph Signal Processing (GSP) framework. These concepts and tools allow us to perform analysis of deep latent representations and to derive new contributions to the machine learning community that are going to be introduced in the following Chapters.

Some of the tools introduced in this section include the Graph Fourier Transform (GFT) and the analysis of the smoothness of graph signals. We also discuss methods of inferring graphs from data where the graph support structure is not explicitly available, including a novel contribution:

\begin{itemize}
  \item \bibentry{lassance2020benchmark}
\end{itemize}

Then we derived graph filters, that emulate traditional signal processing filters in the graph domain. These graph filters will be used to link convolutional layers and graph convolutional layers in the next chapter. We also present two applications of graph filters that allow us to reduce the amount of noise of features extracted using DNNs and improve the performance of downstream tasks including \begin{inlinelist}\item few-shot learning \item image classification \item visual-based localization (VBL) \item image retrieval (IR)\end{inlinelist}. The visual-based localization application was a subject of an archival contribution:
  
  \begin{itemize}
    \item \bibentry{lassance2019improved}
  \end{itemize}
  
In the following chapter we discuss DNNs that have their latent spaces supported on graphs. To do this we will introduce and discuss the field of Graph Neural Networks (GNNs). First we introduce the needed definitions in Section~\ref{chap4:definitions} using the concepts introduced in the current chapter and the previous one. These definitions will create a link between convolutional layers and graph convolutional layers that allow us to use an universal framework to represent both types of layers. 

We then discuss in Section~\ref{chap4:classify_graph_signals} the use of GNNs in the context of supervised classification of graph signals, first doing a sanity check using data that is defined in a regular 2D Euclidean space (e.g., images) and then using data that is defined in a slightly irregular 3D Euclidean space (e.g., neuroimaging) as a slightly more difficult test. Finally, we will then discuss applications of GNNs in a semi-supervised classification scenario and their pitfalls in Section~\ref{chap4:classify_graph_vertices}.
\chapter{Deep Learning for inputs supported on graphs}\label{chap4}
\localtableofcontents
\vspace{1.0cm}

In the previous chapters, we have introduced and discussed Deep Neural Networks and Graph Signal Processing. In this chapter, we build upon these methods and discuss DNNs that have their inputs supported on graphs. We first introduce the needed definitions in Section~\ref{chap4:definitions}, introducing the graph convolutional layers as graph filters. We also create a link between convolutional layers and graph convolutional layers. 

We then discuss in Section~\ref{chap4:classify_graph_signals} the use of GNNs in the context of supervised classification of graph signals. We do a quick review of the domain and then present a method to perform the supervised classification of graph signals. The method is evaluated first by a sanity check using data defined in a regular 2D Euclidean space (e.g., images). We then evaluate on a more real-world scenario by using data defined in a slightly irregular 3D Euclidean space (e.g., neuroimaging) as a slightly more difficult test. Finally, we will discuss applications of GNNs in a semi-supervised classification scenario and their pitfalls in Section~\ref{chap4:classify_graph_vertices}.

\section{Definitions}\label{chap4:definitions}

In this section, we will introduce some of the recent literature in Graph Neural Networks. Presenting all the methods in the literature is very complicated, given the speed of evolution of the field and the amount of already proposed methods~\citep{kipf2016semi,pasdeloup2017extending,lassance2018matching,grelier2018graph,velickovic2018graph,liao2019lanczosnet,isufi2020edgenets,wu2019simplifying,ruiz2019invariance,klicpera2019combining,balcilar2020bridging,defferrard2016convolutional}. In this thesis, we prefer to present this domain as six methodologies in a logical sequence of developments. Note that even if this can be seen as a logical sequence of developments, where at each time complexity increases, the methods themselves are not in chronological order. 

We first present in Section~\ref{chap4:sgc} the use of graph filter as feature extractors followed by a simple classifier. Note that this should not be considered a neural network, but it was introduced as so in~\citep{wu2019simplifying} that is called ``Simplifying Graph Convolutional Networks''. We then present a methodology that use this simple diffusion/filter on a graph inside each layer, leading to \emph{Graph Convolution Layers (GCL)}, and to networks that are called Graph Convolutional Networks (GCN)~\citep{kipf2016semi} in Section~\ref{chap4:gcn}. 

We then present two extensions to the GCL/GCN, first in Section~\ref{chap4:tagcn} we showcase methods that improve the GCL by applying multiple graph filters at the same layer in order to increase the degrees of freedom the layer, and second in Section~\ref{chap4:gat} we show how we can improve the GCNs by modifying the graph during the training, either via learning the weights or by adding additional information. 

Finally we present two additional extensions that are orthogonal to the first ones. In section~\ref{chap4:chebyshev} we show methods that can learn the graph filters directly instead of using predefined filters, and in Section~\ref{chap4:translation} we showcase methods that try to mimic traditional convolutional layers using the concept of graph translation introduced in Section~\ref{chap3:graph_translation}. 

Note that there are multiple ways to represent the above-mentioned methodologies~\citep{wu2020comprehensive,zhang2020deep,zhang2019graph,gilmer2017neural,vialatte2018convolution,bontonou2019formalism}. In this manuscript, we follow a different path to be more inline with the rest of the document. The reader should be informed that similar analysis were proposed in the first months of this year~\citep{isufi2020edgenets,balcilar2020bridging}.

\subsection{Using graph filters as feature extractors}\label{chap4:sgc}

As we present the considered methods in ascending order of complexity, we first recall the previously introduced concept of graph filters (c.f. Section~\ref{chap3:graph_filter}) and then introduce a very simple methodology that consists in using a graph filter as feature extractor. The most relevant part is that the graph filter does not intervene in the training phase of the classifier and can be seen as a fast pre-processing method. This methodology was popularized by SGC~\citep{wu2019simplifying} and has been shown to be very efficient not only in the context of inputs supported on graphs but also in contexts where the graph support has to be inferred such as visual based localization (c.f. Section~\ref{chap3:vbl} and few-shot learning~\citep{hu2020exploiting}.

In the case of SGC, the graph filter that is used is based on the adjacency matrix $\adjmatrix$. The first step is to add self-connections to $\adjmatrix$, generating the augmented adjacency matrix $\tilde{\adjmatrix}$: 
\begin{equation} 
    \tilde{\adjmatrix} = \identity + \adjmatrix \;.
\end{equation}
then the adjacency matrix is normalized in order to create the diffusion matrix $\mS$ as follows:
\begin{equation}\label{chap4:eq-augadj}
\mS = \tilde{\mD}^{-\frac{1}{2}} \tilde{\adjmatrix} \tilde{\mD}^{-\frac{1}{2}} = \dot{\tilde}{\adjmatrix} \; ,
\end{equation} 
where $\tilde{\mD}$ is the degree matrix of the augmented adjacency matrix. Note that it is common in the literature to represent $\mS$ as the $\tilde{\adjmatrix}$ itself, but here we prefer to use two different symbols. We can then apply the graph filter to the graph signal $\vx$ and generate the filtered signal $\xfilter$ as follows:
\begin{equation}\label{chap4:eq-SGC}
    \xfilter = \mS^m \vx \; ,
\end{equation}
where $m$ represents the amount of ``SGC layers'' that are applied to the signal. The filtered signal is then used to train a simple logistic regression classifier or even a $k$-neighbors classifier. In~\citep{wu2019simplifying} the authors have shown that for various applications the results are as good as multi-layered GCNs, in a fraction of the time (as the graph filter is only applied as a pre-processing). We will delve in more details on this in the following sections.

One advantage of considering this a filtered signal and a graph filter is that we can perform spectral analysis to better understand the underlying functioning of our feature extractor. Indeed we can reformulate~\eqref{chap4:eq-SGC} in order to represent our filter using the $h(\lambda)$ function and the graph Fourier transform. First we have to define the symmetric Laplacian of the diffusion matrix as:
\begin{equation}
    \tilde{\mL} = \identity - \mS \; .
\end{equation}
We can then redefine the filtered signal as:
\begin{align}
    \xfilter &= \mS^m \vx \\
             &= (\identity - \tilde{\mL})^m \vx \\
             &= (\identity - \mF \tilde{\mLambda} \mF^\top)^m \vx \\ 
             &= \mF (\identity - \tilde{\mLambda} )^m \mF^\top \vx \\
             &= \mF \mH \mF^\top \vx \; , \label{chap4:mh-filter}
\end{align}
where $\mH = (I - \tilde{\mLambda})^m$ is a diagonal matrix that determines the filter response for each discrete eigenvalue of $\tilde{\mL}$. We can then represent this filter by its spectral response with a function $h(\tilde{\lambda})$ as follows:
\begin{equation}
    h_\text{SGC}(\tilde{\lambda}) = (1-\tilde{\lambda})^m \;.
\end{equation}
We represent the spectral response in Figure~\ref{chap4:response_sgc} for various values of $m$. Note that while this should implement a low pass filter, it actually implements a band reject filter, where in the odd values the high frequencies are inverted. This behavior while unexpected is not necessarely bad, as recent theory in deep learning (both applied to computer vision and GSP) seem to converge to robustness is linked to the low frequencies of the signal, while generalization and overfitting are linked to the higher frequencies~\citep{wang2020high,ilyas2019adversarial,gama2019stability, yin2019fourier}. We will discuss this difference between the expected low pass and the obtained filter in more details in the application part of this chapter.

\begin{figure}[ht]
    \begin{center}
        \begin{adjustbox}{max width=.5\linewidth}
  \tikzsetnextfilename{chapter4/tikz/response_sgc}%
  \input{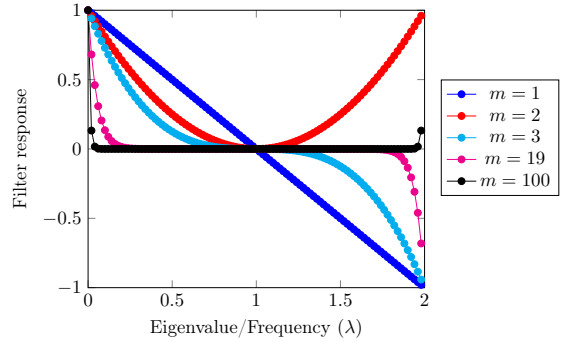}%

        \end{adjustbox}
    \end{center}
    \caption{Response of the filter used to extract features in SGC~\citep{wu2019simplifying}. Note that the eigenvalues here are the ones from the Laplacian of the augmented graph}
    \label{chap4:response_sgc}
\end{figure}

Note that this filter is based on the eigenvalues of the Laplacian of the augmented graph $\tilde{\adjmatrix}$ and not on the ones of the original graph $\adjmatrix$.

Let us now decompose the filter generated with the augmented graph in two filters $f_1$ and $f_2$:
\begin{align}
    \xfilter &= \mS \vx \\
    & = \frac{\vx}{\mD + \identity} + (\mD + \identity)^{-\frac{1}{2}} \adjmatrix (\mD + \identity)^{-\frac{1}{2}} \vx \; ; \\
    f_1(\vx) &= \frac{\vx}{\mD + \identity} \; ; \\
    f_2(\vx) &= (\mD + \identity)^{-\frac{1}{2}} \adjmatrix (\mD + \identity)^{-\frac{1}{2}} \vx \; ; \\
    \xfilter &= f_1(\vx) + f_2(\vx) \;.
\end{align}

By decomposing into $f_1$ and $f_2$ we can try to analyze them separately, indeed we can see that $f_1$ attenuates the signal of each sample $\vx$ by dividing it by its degree augmented by one. In the case of $f_2$, each node is affected differently depending on its degree and its neighbors. This means that the frequency profile of the filter will be perturbed, unless we have a regular graph (i.e. all nodes have the same degree).

Indeed if the graph represented by matrix $\adjmatrix$ is regular where each node has degree $d$, we can rewrite of the augmented graph $\tilde{\mL}$ as a function of the Laplacian of the original graph:
\begin{align}
    \mL & = \identity - \mD^{-\frac{1}{2}} \adjmatrix \mD^{-\frac{1}{2}} \\ 
    & = \identity - \frac{\adjmatrix}{d} \Rightarrow \adjmatrix = d\identity + d\mL \; ; \\
    \tilde{\mL} & = \identity - (\mD+\identity)^{-\frac{1}{2}} (\adjmatrix + \identity) (\mD + \identity)^{-\frac{1}{2}} \\
    & = \identity - \frac{\adjmatrix + \identity}{d+1} \\
    & = \frac{d \identity - \adjmatrix}{d+1} \\
    & = \frac{d\identity - (d\identity - d\mL) }{d+1} \\
    & = \frac{d \mL }{d+1} \; , \\
\end{align}
and by writing it as a function of the original graph, we can also write the eigenvalues of the augmented graph as a function of the eigenvalues of the original one:
\begin{equation}
    \tilde{\mLambda} = \frac{d}{d+1}\mLambda \; ,
\end{equation}
finally, we can now write the function $h(\lambda)$ as a function of the eigenvalues of the original graph as follows:
\begin{equation}
    h(\lambda) = \left(1 - \frac{d \lambda}{d+1}\right)^m \; ,
\end{equation}
which still has the same problem of not being exactly a low pass filter. As the maximum eigenvalue of $\mL$ is $\lambda_\text{max}=2$, this filter would only be low pass if $2d \leq d+1$ which is only possible for $d=0$ and $d=1$, which would yield graphs that are not fully connected and therefore are not interesting for a GCN application.

\subsubsection{Other filters proposed in the literature}

In the previous paragraphs we have described how one can use a graph filter as a preprocessing step for data, and have developed an analysis using the graph filter defined in~\citep{kipf2016semi,wu2019simplifying}. In this section, we will present some other possible choices of filters that were proposed in the literature under a common framework ($m$ is the attenuation factor and $\alpha$ controls the smoothness of the filter). First, we recall the filter used in Section~\ref{chap3:vbl}, which is a proper low pass filter if the paramater $\alpha$ is smaller than $0.5$ and that is defined as:
\begin{equation}
    \mS = (\identity - \alpha\mL)^m \; \Rightarrow  h_\text{VBL}(\lambda) = (1 - \alpha \lambda)^m \; \label{chap4:eq-VBL}
\end{equation}
Another more frequently used low pass filter is based on the Tikhonov regularization~\citep{shuman2013emerging} which aims at balancing the amount of change of the filter and the smoothness of the final filtered signal:
\begin{equation}
    \underset{h}{\mathrm{argmin}} \|| \xfilter - \vx \|| + \alpha \xfilter^\top \mL \xfilter \Rightarrow  h_\text{Tikhonov}(\lambda) = \frac{1}{1+\alpha \lambda} \; , \label{chap4:eq-Tikhonov}
\end{equation}
In~\citep{balcilar2020bridging} the authors propose to combine four different filters ($h_1$ to $h_4$), where $h_1$ is a low-pass filter, $h_2$ to $h_4$ are different band pass filters as follows:
\begin{align}
    h_\text{balcilar-lowpass}(\lambda) &= (\frac{\lambda_\text{max} - \lambda}{\lambda_\text{max}})^m \label{chap4:eq-balcilar} \\ 
    h_\text{balcilar-2}(\lambda) &= \exp(-\alpha (0.25\lambda_\text{max} - \lambda)^2) \\
    h_\text{balcilar-3}(\lambda) &= \exp(-\alpha (0.5\lambda_\text{max} - \lambda)^2) \\
    h_\text{balcilar-4}(\lambda) &= \exp(-\alpha (0.75\lambda_\text{max} - \lambda)^2) \\
\end{align}
Finally, there also exists a filter~\citep{klicpera2019combining} that is based on the personalized PageRank algorithm~\citep{page1999pagerank} and the normalized augmented adjacency matrix $\dot{\tilde{\adjmatrix}}$ from~\Eqref{chap4:eq-augadj}:
\begin{equation}
    \mS = \alpha(\identity - (1-\alpha)\dot{\tilde{\adjmatrix}})^{-1} \Rightarrow h_\text{page}(\tilde{\lambda}) = \frac{\alpha}{\alpha(1-\tilde{\lambda}) + \tilde{\lambda}} \; ; \label{chap4:eq-page}
\end{equation}
Also note that the authors originally proposed to apply the filter after the affine transformation of the logistic regression or after the MLP that reduces the number of features to the number of classes. It thus differs substantially from~\citep{wu2019simplifying}. We will further discuss this in Section~\ref{chap4:classify_graph_vertices} in order to explicit all the different decisions that are made for each model. We also discuss how to fairly compare them.

\begin{remark}
    While for analysis sake it is better to consider the filters in the spectral form, applying the filters as a diffusion matrix is much more computationally efficient: $O(|\sV|^3)$ vs $O(|\sE|)$. Therefore it is quite recommended to convert the spectral filter into matrix form as a pre-processing step when possible. 
\end{remark}

\subsection{Graph Convolutional Layers (GCL)}\label{chap4:gcn}

Now that we have introduced how to use graph filters to extract relevant features from graph signals, we can extend this to a layer in a neural network. Indeed in the original GCN paper~\citep{kipf2016semi} each layer (that we call GCL in this work) is defined as:
\begin{equation}\label{chap4:eq-gcn}
    f^{\ell}({\vx}) = g(\mS^m \vx \mW^\ell + \vb^\ell)\; ,
\end{equation}
where $f^\ell$ is the layer function of layer $\ell$, $\mW^\ell$ is the weight matrix of the afine transformation of layer $\ell$, $\vb^\ell$ is the bias associated with layer $\ell$, $\vx$ is the input of the layer, $g$ is the nonlinear activation function and $\mS$ is the diffusion matrix presented in the previous section based on the augmented adjacency matrix. Note that a GCN model can have more than one layer and that if we had removed the nonlinear activations we would have the same model as the one presented in the previous section as all the weights and bias would be compressed in a single matrix and a single vector that would be responsible for the logistic regression.

In other words, SGC is a specialization of GCNs where there is no concept of network or layers as all the layers are linear. On the other hand, by introducing the nonlinearity we can only approximate the true graph filter, specially in cases where the nonlinear function aggregates over the nodes of the graph as in~\citep{ruiz2019invariance}. In this case we define the filter using $\mS^m$ and apply the same filter for each layer. 

Note that even if in theory GCNs are able to represent more complex functions than SGC, it is not given which method will perform best. Indeed in~\citep{wu2019simplifying} the authors show that SGC may present equivalent or even better results than GCN in some tasks. In a recent paper~\citep{vignac2020choice}, the author perform a more careful evaluation between GCNs and SGC and show that choosing which method to use depends heavily on a simple metric: the higher is the number of samples $N$ per feature $F$ the better GCN performs in comparison with SGC. In other words, when you have $N>>F$ linear models such as SGC perform very well, on the other hand when $F>>N$ more complex models such as GCN and its improvements are better.

\subsection{How to combine multiple filters in each GCL}\label{chap4:tagcn}

Given the limitations of the GCN model that implements only one filter per layer and that we are not able to compose it to determine the global filter of the entire network given the nonlinearities, authors have proposed alternative ways to compose multiple filters in each GCL of the GCN. For example, one simple solution introducted in TAGCN~\citep{du2017topology} is to perform a sum over different filters in each layer as follows: 
\begin{equation}\label{chap4:eq-tagcn}
    f^{\ell}{\vx} = \sum_{m \in \sM} g(\mS^m \vx \mW^{\ell,m} + \vb^{\ell})\; ,
\end{equation}
where $\sM$ is the set of powers of $\mS$ that we want to use as filter and there is a different affine transformation for each filter. An extension to this framework is presented in~\citep{liao2019lanczosnet} where the authors use the Lanczos polynomial to accelerate the computation of $\mS^m$ in the case where $m$ is very large. 

Note that, as before, GCN may be seen as a specialization of this model that only uses one filter per GCL and that, as before, even if this model is more expressive it is not given that it will be better performing as seen in~\citep{vialatte2018convolution,vignac2020choice}.

\subsection{Modifying the support graph to improve GCNs}\label{chap4:gat}

Another possibility of improvement to GCNs is to change the graph itself instead of the graph filter. In this case one can even use a different graph for each layer of the architecture. One of the most relevant papers that uses this methodology is GAT~\citep{velickovic2018graph} where the authors propose to use multiple attention heads~\citep{vaswani2017attention} in order to learn multiple representations of the graph support. Note that even if the graph support is modified, they do not create new links, they only change the weights of the links that are already present. Their approach may be summarized as follows:
\begin{enumerate}
    \item For each attention head $k$, we compute a new diffusion matrix $\mS$ as $\mS_k$:
    \begin{enumerate}
        \item each weight of $\mS_k$ is equal either to $\alpha_{k_{i,j}}$ or to 0 if $\emS_{i,j} = 0$;
        \item $\alpha_{k_{i,j}}$ is computed with the softmax of $\mathit{e}_{i,j}$ so that each node has a similar contribution;
        \item $\mathit{e}_{i,j}$ is determined using a shared attention mechanism $a$ that takes the affine transformation $\mW_k$ of the features of each node $f_i$ and $f_j$ and outputs an attention coefficient: $\mathit{e}_{i,j} = g(a(\mW_k\vx^\ell_i || \mW_k\vx^\ell_j))$, where $g$ is a non-linear activation function.
    \end{enumerate}
    \item Then, after computing each diffusion matrix the output $\vx^{\ell+1}$ is defined as: $\vx^{\ell+1} = \text{aggregate}(\mS_k \mW_k \vx^\ell)$ where the aggregation function is either concatenation or average over each of the $k$ heads. 
\end{enumerate}

Note that while the way the graph is modified is complex, this methodology allows to use different graphs while still keeping the same overall graph filter that is applied to different eigenvalues. On the other hand, one could argue that keeping the same filter applied to different eigenvalues is not particularly different from applying a different graph filter to the same graph. A similar framework is developed in~\citep{isufi2020edgenets} in order to demonstrate that GAT layers are actually GCN layers with multiple filters per layer that learn the graph as well as the transformations. Note that while this would be the most expressive model of the bunch we presented, it is also the most difficult to train, as both the number of computations and the sensibility of each parameter increases compared to other methods.

\subsection{Learning the graph filter directly}\label{chap4:chebyshev}

Another possibility to improve the GCN/GCL model is to learn the graph filter $h(\lambda)$ directly during the learning of the parameters of the networks. Such methods are mostly constrained by the problem of computing the GFT. Indeed, obtaining the spectral decomposition can quickly become too expensive as the number of vertices in the graph increases. Therefore most of the proposed methods differ in which approximation they use, in ChebNets~\citep{defferrard2016convolutional} the authors use the Chebyshev polynomial to approximate any possible filter~\citep{hammond2011wavelets} without needing to explicitly compute the GFT. In Lanczosnet~\citep{liao2019lanczosnet} the authors propose to use the Lanczos approximation of the orthonormal decomposition in order to learn the filter during training. A major drawback of these methods is the extra time that it is needed to generate approximations good enough to learn interesting filters. Indeed, Chebnets are used as the basis of GCN. As a matter of fact, GCN is a faster and more efficient implementation of Chebnets that uses only a small amount of Chebyshev kernels, and can therefore approximate the filters we desire to extract. 

Finally~\citep{isufi2020edgenets} introduces methods to train edge varying filters, including ones based on ARMA graph filters~\citep{isufi2016autoregressive}. Note that all approaches cited in this section may be summarized as approximations of learning the matrix $\mH$ that represents the filter in~\eqref{chap4:mh-filter}.

\subsection{Using graph translations to generate graph convolutional networks}\label{chap4:translation}

In this section we present the graph convolutional layers proposed in~\citep{pasdeloup2017extending}. The authors propose a method that also learns the graph filter directly, but it vastly differs from the ones presented in the last subsection as the filters here are based on the graph translations instead of on diffusion matrices. In other words the filters described here are way closer to their CNN counterparts than to a spectral definition. Indeed the filters here are obtained as the sum of subfilters applied for each considered translation, which in the case of a 3 by 3 convolutional layer would be \begin{inlinelist} \item center (the original node) \item (node to the) right \item left \item up \item down \item inferior right diagonal \item inferior left diagonal \item superior left diagonal \item superior right diagonal \end{inlinelist}. In this case, we generate $k$ translation functions that when applied to $\adjmatrix$, generate a set of allocation matrices $\sT$. Each translation function follows Definition~\ref{chap3:def-translation}.

In this case the layer function is defined as:
\begin{equation}\label{chap4:eq-translation}
    f^{\ell}({\vx}) = \sum_{\allocationmatrix_k \in \sA} g(\allocationmatrix_k \vx \mW^{\ell,k} + \vb^{\ell})\; ,
\end{equation}
where each element $\allocationmatrix$ in $\sT$ has a different weight matrix $\mW$. This is very similar to the CNN filters, where each filter has a component for each translation. It is also very simple to extend this concept to multiple feature maps, where each feature map implements its own translation filter. Also note that as it is defined as a sum of translations it is very close to the concept we developed in Shift Attention Layers (SAL, c.f. Section~\ref{chap2:SAL}). Indeed in SAL the translations are defined in the 2D grid and the attention kernel chooses which translation should be used for each filter. It would be therefore simple to extend SAL to layers based on graph translations. 

\subsection{Summary of methods presented}

In the previous subsections we have presented the evolution of methods from a simple fixed feature extractor to methods that learn the graph filter or even the graph itself during training. We summarize the methods in Table~\ref{chap4:summary-methods}.

\begin{table}[ht]
    \begin{center}
        \caption{Summary of the methods presented in this section. The last two columns refer to the fact that the method is used or not in the two following sections.}
        \label{chap4:summary-methods}
        \begin{adjustbox}{max width=\linewidth}
            \begin{tabular}{|c|c|c|c|c|c|}
                \hline
                Methods     & DNN & Multiple filters per layer & Learn support graph & Learn Filter & Use translations \\ \hline
                SGC~\citep{wu2019simplifying}         &     &                            &                     &              &                         \\ \hline
                GCN~\citep{kipf2016semi}         & X   &                            &                     &              &                         \\ \hline
                TAGCN~\citep{du2017topology}       & X   & X                          &                     &              &                          \\ \hline
                Lanczosnet~\citep{liao2019lanczosnet}  & X   & X                          &                     & X            &                           \\ \hline
                GAT~\citep{velickovic2018graph}         & X   & X                          & X                   &              &                         \\ \hline
                ChebNet~\citep{defferrard2016convolutional}     & X   & X                          &                     & X            &                           \\ \hline
                DSGCN~\citep{balcilar2020bridging}       & X   &  X                          &                     & X            &                           \\ \hline
                Translation~\citep{pasdeloup2017extending,lassance2018matching} & X   & X                           &                     & X            & X                       \\ \hline
            \end{tabular}
        \end{adjustbox}
    \end{center}
\end{table}

While presenting an extensive comparison of all the presented methods would be desireable it would also be impossible to do so in a fair way given the number of methods and hyperparameters that are to optimize. We therefore rather base our discussion on recent contributions. First in Section~\ref{chap4:classify_graph_signals} we mainly discuss three papers~\citep{he2020deep,Errica2020A,lassance2018matching}, where the first two show that in some use-cases of supervised classification of graphs and graph signals no significant performance gains where found when graph-based machine learning techniques are brought, and the third one is a contribution of our own aiming at improving the performance of GNNs in the regular 2D-Euclidean space and in an irregular 3D-Euclidean space without using priors about the data. 

Then in Section~\ref{chap4:classify_graph_vertices} we focus on some recent findings from~\citep{shchur2018pitfalls,vialatte2018convolution,bontonou2019smoothness}. These papers show that most of the benchmark evaluation that is performed on semi-supervised scenarios is actually not as robust as once thought and that fair comparison is still not available. We build upon this work and propose a framework to verify which one of the graph filters described in Section~\ref{chap4:sgc} is the best performing.

\section{Supervised classification of graph signals}\label{chap4:classify_graph_signals}

Let us discuss some of the applications of the previously presented methods. The first one we consider is supervised classification of graph signals. 

There is a plethora of applications where one may perform this task, with the two most notable mentions being protein-protein interaction and applications on the medical domain. Unfortunately, recent papers have shown that deep learning techniques fail to surpass simpler methods in these scenarios~\citep{he2020deep,Errica2020A}. These recent findings corroborate with our analysis in the previous section that shows that the graph filter is not well defined in the spectral domain. More so, in~\citep{xu2019how} the authors have shown that GCNs lack the expressivity to distinguish some types of graphs. 

Given these recent developments, we choose to limit our discussion on the supervised classification of graph signals to two cases that serve as a ``sanity check'' of the expressivity of graph neural networks: \begin{inlinelist} \item the classification of images (defined as a 2D square), using oracle (grid) or inferred graphs instead of the underlying 2D structure \item the classification of neuroimaging 3D images (that are not defined on a cube), using inferred graphs.\end{inlinelist} Note that in the first case we would like for networks to be able to get results that are similar to the convolutions defined on the 2D space when we use the oracle (grid-graph) structure and to not lose much performance when we use the inferred graph (i.e. we do not have any prior about the data structure). On the second case, we would also desire a performance on the inferred graph that is close to the one using the 3D structure.

Note that the idea is not to surpass traditional 2D/3D convolutions, but to ensure that the graph convolution is able to represent the same type of functions as the 2D convolution. To do so, we propose to test two avenues: \begin{inlinelist} \item use a method to convert the grid/inferred graph structure into a 2D/3D regular structure (a square/cube defined on $\mathbb{Z}$) \item use graph convolutional networks to try to imitate the same functions defined by the CNNs \end{inlinelist}. For the former strategy we will use the embedding methodology we first proposed in~\citep{grelier2018graph}. To the later we will use graph convolutional networks using the convolutions from~\citep{defferrard2016convolutional,lassance2018matching}. Note that we do not include other popular models such as GCN~\citep{kipf2016semi} and SGC~\citep{wu2019simplifying} as it is quite straightforward to see that the models will not be able to express a similar function as they are only able to perform one specific convolutional kernel, while CNNs are normally composed of hundreds of kernels per layer. 

\subsection{Methodology}

In this section we describe the considered methodologies. 

In Section~\ref{chap4:graph_inference_convs}, we present how we infer graphs from the regular 2D/irregular 3D data, then we explain how we can use these graphs in two different approaches from~\citep{grelier2018graph,lassance2018matching} to solve the task of classifying graph signals.

\subsubsection{Graph Inference}\label{chap4:graph_inference_convs}

We have two separate methods for inferring graphs, one for the case of images (regular 2D data) and one for the case of fMRI data (irregular 3D data).

\paragraph{Inferring graphs from images}

In the case of images, a natural choice for the graph structure would be to use a grid-graph (Definition~\ref{chap3:def-gridgraph}). We call this an ``oracle inference'' since we suppose in the following that we ignore the fact we are dealing with images. This setting is only used to provide a best case scenario. Then, we also perform an empirical inference, where we first convert the pixels of all images in the $\trainset$ to a grey-scale. With the pixels now in gray-scale we compute the covariance between all pixels in the image, using each example in $\trainset$ as the different samples. We now have a 2D square matrix of the similarity between the pixels that we use as our adjacency matrix $\adjmatrix$. We then threshold the $\adjmatrix$ so that only the 4 closest neighbors of each vertex (in this case pixel) are connected, symmetrize the resulting matrix and binarize it so that connected elements are connected with weight 1 and unconnected elements have weight 0. Since we are dealing with natural images, we expect that the empirical inference graph is a noisy version of the oracle.

\paragraph{Inferring graphs from fMRI data}

For inferring graphs from fMRI data, we use a simple neighborhood graph where nodes are connected if they have are at most at distance $d$ in the 3D-space. Note that this is possible in this case as data was captured from physical sensors that were then masked on the MNI template and resampled on a 16mm cubic grid. The data resampled on the cubic grid is not regular as it does have a data point per integer coordinate of the cubic grid, but it would be possible to run a 3D CNN on it. 

\subsubsection{Converting a graph to a regular structure defined on $\mathbb{Z}$}\label{chap4:grelier}

In this subsection we present a first method to classify graph signals using DNNs. This idea was first introduced in~\citep{grelier2018graph} where we propose an embedding from the graph structure to $\mathbb{Z}^d$. The main goal was to define the GFT on $\gG$ as a particularization of the classical FT on $\mathbb{Z}^d$ but it may also be used in the context of machine learning to allow the use of 2D CNNs to process data defined on graphs. Indeed, once vertices of the graph have been projected to $\mathbb{Z}^d$, we can use regular convolutions (2D-ones in the case where $d=2$).

The method to embed a graph into $\mathbb{Z}^d$ consists in optimizing a cost. Let us first introduce a weighted graph $\gG = \langle \sV, \sE \rangle$.

\begin{definition}[embedding]\label{chap4:def-embedding}
    We call \textbf{embedding} a function $\phi: \sV \to \mathbb{Z}^d$, where $d\in \mathbb{N}^*$.
\end{definition}

We are specifically interested in embeddings that preserve distances. Specifically, we define the cost $c_\alpha(\phi)$ of an embedding $\phi$ as the following quantity:
\begin{equation}
  c_\alpha(\phi) \triangleq \sum_{v,v' \in \sV}{| \; \alpha\|\phi(v) - \phi(v')\|_1 - d_\gG(v,v') \;|}\;,
\end{equation}
where $d_\gG$ is the shortest path distance in $\gG$. In the remaining of this section, we denote $\delta(v,v') = | \; \alpha\|\phi(v) - \phi(v')\|_1 - d_\gG(v,v') \; |$.

\begin{definition}[optimal embedding]\label{chap4:def-optimalembedding}
    Given a fixed value of $\alpha$, we call \emph{optimal embedding} an embedding with minimum cost.
\end{definition}

The choices in this definition are motivated by 5 main reasons:
\begin{enumerate}
\item We consider all pairs of vertices and not only edges. Consider for example a ring graph where each vertex has exactly two neighbors. Then there are plenty of embeddings that would minimize the cost if considering only edges, but only a few that minimize the sum over all pairs of vertices.
\item We use a sum and not a maximum. This is because small perturbations of grid graphs would lead to dramatic changes in embeddings minimizing the cost if using a maximum. Consider for instance a 2D grid graph in which an arbitrary edge is removed.
\item We choose embedding in $\mathbb{Z}^d$ instead of in $\mathbb{R}^d$, as we want to particularize multidimensional \emph{discrete} Fourier transforms.
\item We use the Manhattan distance, as it is more naturally associated with $\mathbb{Z}^d$ than the Euclidean distance. It also ensures there exists natural embeddings for grid graphs with cost 0.
\item Finally, $\alpha$ is a scaling factor.
\end{enumerate}

Note that the question of finding suitable embeddings for graphs is not novel~\citep{Kobourov04forcedirecteddrawing}. But to our knowledge enforcing the embedding to be in $\mathbb{Z}^d$ was a novel contribution. Even though Definitions~\ref{chap4:def-embedding} and~\ref{chap4:def-optimalembedding} work for any $d$ we are going to focus on the consistency of these definitions by considering the particular case $d=2$, i.e., a regular 2D Euclidean space.

Note that in the case of grid graphs (Definition~\ref{chap3:def-gridgraph}) the embedding should be a perfect square, indeed it follows that: 

\begin{definition}[natural embedding]
    We call \emph{natural embedding} of a grid graph (as defined in Definition~\ref{chap3:def-gridgraph}) the identity function.
\end{definition}

Which leads to the following theorem:

\begin{theorem}
  The natural embedding of a grid graph is its only optimal embedding for $\alpha=1$, up to rotation, translation and symmetry.
  \label{chap4:theorem-grid}
\end{theorem}
\begin{proof}
  The proof is straightforward, as the cost of the natural embedding is clearly 0. Reciprocally, a cost of 0 forces any group of vertices $\{(x,y),(x,y'),(x',y),(x',y')\}$ to be projected to a translation, rotation and/or symmetry of the corresponding rectangle in $\mathbb{Z}^2$. Then any remaining vertex is uniquely defined from these four ones.
\end{proof}

In more general settings where the graph we are dealing with is not a grid-graph, we have to solve an optimization problem with the expectation of finding a relevant embedding. Once the embedding is found, we simply consider graph signals as images and process them with regular CNNs.

\subsubsection{Matching CNNs without priors}\label{chap4:matching}

Another possibility is to build upon translations on graphs we defined in Section~\ref{chap3:graph_translation}, that we used to define convolutions in~\Eqref{chap4:eq-translation}. This is the idea we used in~\citep{lassance2018matching} to show that is is possible to approach the performance of CNNs without priors about the fact we are dealing with regular images. An overview of all the steps of the method is presented in Figure~\ref{chap4:fig-outline}.

\begin{figure}[!ht]
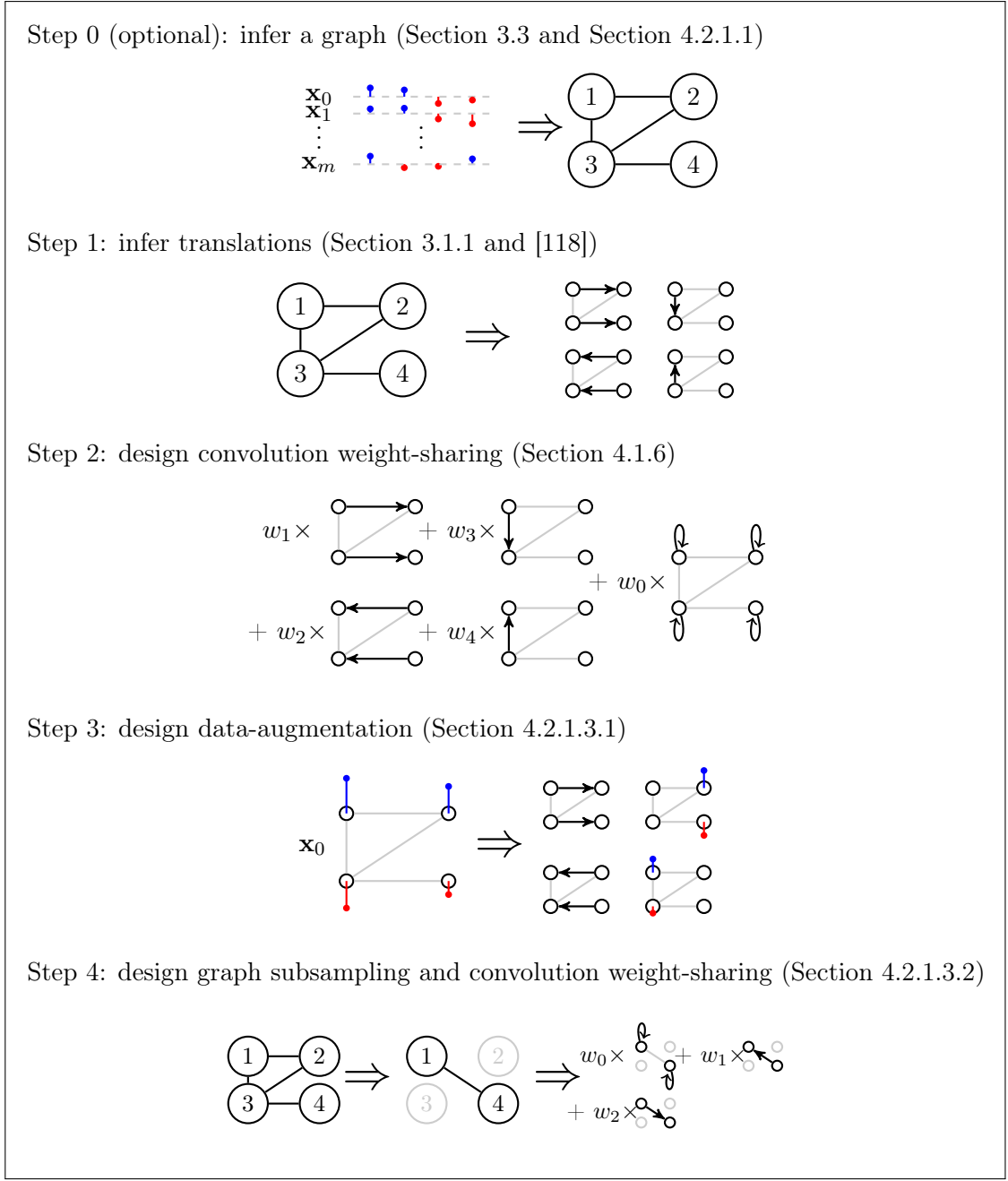

    \begin{framed}
      Step 0 (optional): infer a graph (Section~\ref{chap3:graph_inference} and Section~\ref{chap4:graph_inference_convs})
      \begin{center}
  \tikzsetnextfilename{chapter4/tikz/graphinference}%
  \input{chapter4/tikz/graphinference.tex}%

      \end{center}
      
      Step 1: infer translations (Section~\ref{chap3:graph_translation} and~\citep{pasdeloup2017extending})
      \begin{center}
  \tikzsetnextfilename{chapter4/tikz/translations}%
  \input{chapter4/tikz/translations.tex}%
      
      \end{center}
      
      Step 2: design convolution weight-sharing (Section~\ref{chap4:translation})
    
      \begin{center}
  \tikzsetnextfilename{chapter4/tikz/weight-sharing}%
  \input{chapter4/tikz/weight-sharing.tex}%

      \end{center}
      
      Step 3: design data-augmentation (Section~\ref{chap4:dataaug})
      \begin{center}
  \tikzsetnextfilename{chapter4/tikz/dataaugmentation}%
  \input{chapter4/tikz/dataaugmentation.tex}%

      \end{center}
      Step 4: design graph subsampling and convolution weight-sharing (Section~\ref{chap4:strided-convolutions})
    
      \begin{center}
  \tikzsetnextfilename{chapter4/tikz/subsampling}%
  \input{chapter4/tikz/subsampling.tex}%

      \end{center}
        \end{framed}
      \caption{Outline of the different steps in designing GNNs based on graph translations. Figure extracted from~\citep{lassance2018matching} @2018 IEEE.}\label{chap4:fig-outline}
      \vspace{-.5cm}
\end{figure}

If the next paragraphs, we explain the steps of the method that were not introduced before in the manuscript: data augmentation, strided convolutions and convolutions on the strided graph.
    
\paragraph{Defining data augmentation in graph space}\label{chap4:dataaug}

Once translations are obtained on $\gG$, one can use them to move training vectors, artificially creating new ones. Note that this type of data-augmentation is still not at the same level as the ones for images since no flipping, scaling or rotations are used and that it preceded more recent forms of data augmentation that could be applied to graph data such as~\citep{zhang2018mixup,verma2019graphmix}. Future work combining such approaches could reduce the gap even further between state of the art CNNs and GNNs. 

We depict an example of the translation-based data augmentation in Figure~\ref{chap4:fig-data_aug}. Note that in the case of the grid graph the augmented image closely resembles the image augmented using the original 2D support, and that in the inferred case we start losing information. 

\begin{figure}[ht]
    \begin{center}
      \begin{subfigure}[ht]{.2\linewidth}
        \centering
        \caption{Original image.}
        \resizebox{\linewidth}{!}{\includegraphics{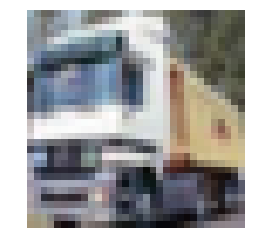}}
    \end{subfigure}
    \hspace{.02\linewidth}
    \begin{subfigure}[ht]{.2\linewidth}
        \centering
        \caption{Random crop.}
        \resizebox{\linewidth}{!}{\includegraphics{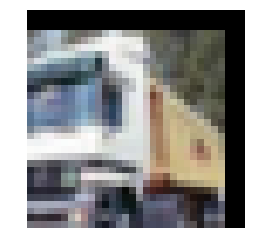}}
    \end{subfigure}        
    \hspace{.02\linewidth}
    \begin{subfigure}[ht]{.2\linewidth}
        \centering
        \caption{Grid graph.}
        \resizebox{\linewidth}{!}{\includegraphics{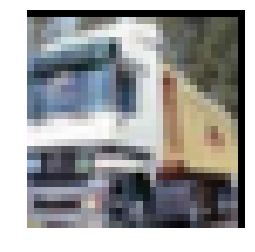}}
    \end{subfigure}        
    \hspace{.02\linewidth}
    \begin{subfigure}[ht]{.2\linewidth}
        \centering
        \caption{Inferred graph.}
        \resizebox{\linewidth}{!}{\includegraphics{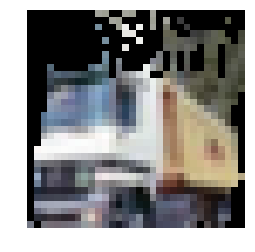}}
    \end{subfigure}        
        \caption{Comparison of image data augmentation and data augmentation using graph translations.}
        \label{chap4:fig-data_aug}
    \end{center}
  \end{figure}

\paragraph{Strided convolutions and convolutions on strided graphs}\label{chap4:strided-convolutions}

Downscaling is a tricky part of the process because it supposes one can somehow regularly sample vectors. As a matter of fact, a nonregular sampling is likely to produce a highly irregular downscaled graph, on which looking for translations irremediably leads to poor accuracy, as we noticed in our experiments. We rather define the translations of the strided graph using the previously found translations on $\gG$.

\noindent\textbf{First step: extended convolution with stride $r$}

Given an arbitrary initial vertex $v_0 \in \sV$, the set of kept vertices $\sV_{\downarrow r}$ is defined inductively as follows:
\begin{itemize}[noitemsep,nolistsep]
\item $\sV_{\downarrow r}^0 = \{v_0\}$,
\item $\forall t \in \mathbb{N}, \sV_{\downarrow r}^{t+1} = \sV_{\downarrow r}^t \cup \{v \in \sV, \forall v' \in \sV_{\downarrow }^t, v \not\in N_{r-1}(v') \land \exists v' \in V_{\downarrow r}^t, v \in N_{r}(v') \}$.
\end{itemize}

This sequence is nondecreasing and bounded by $\sV$, so it eventually becomes stationary and we obtain $\sV_{\downarrow r} = \lim_t{\sV_{\downarrow r}^t}$. Figure~\ref{chap4:fig-downscaling} illustrate the first downscaling $\sV_{\downarrow 2}$ on a grid graph.

\begin{figure}
    \begin{center}
  \tikzsetnextfilename{chapter4/tikz/downscaling}%
  \input{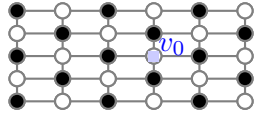}%

    \caption{Downscaling of the grid graph. Disregarded vertices are filled in. Figure and caption extracted from~\citep{lassance2018matching} @2018 IEEE.}
    \label{chap4:fig-downscaling}
    \end{center}
\end{figure}

The output neurons of the extended convolution layer with stride $r$ are $\sV_{\downarrow r}$.\\

\noindent\textbf{Second step: convolutions for the strided graph}

Using the proxy-translations on $\gG$, we move a localized $r$-hop indexing kernel over $\gG$. At each location, we associate the vertices of $\sV_{\downarrow r}$ with indices of the kernel, thus obtaining what we define as induced $_{\downarrow r}$-translations on the set $\sV_{\downarrow r}$. In other words, when the kernel is centered on $v_0$, if $v_1 \in \sV_{\downarrow r}$ is associated with the index $p_0$, we obtain $\phi_{p_0}^{\downarrow r}(v_0) = v_1$. Subsequent convolutions at lower scales are defined using these induced $_{\downarrow r}$-translations similarly to Section~\ref{chap4:translation}.

\subsection{Experiments}

In this section we present two types of experiments, first we try to achieve similar performance as 2D convolutions in the CIFAR-10 dataset, we then use the PINES dataset in order to challenge our methods in a less regular domain.

\subsubsection{CIFAR-10}

On the CIFAR-10 dataset, our models are based on the Resnet-18 architecture. We tested different combinations of graph support and data augmentation and present the results in Table~\ref{chap4:cifar-table}. In particular, it is interesting to note that results obtained without any structure prior (91.07\%) are only 2.7\% away from the baseline using classical CNNs on images (93.80\%). This gap is even smaller (less than 1\%) when using the grid prior. Also, without priors the method from Section~\ref{chap4:matching} significantly outperforms the others and the network defined on the translations of the grid graph obtains a result that is slightly better than the original network. Note that this advantage may come from the fact that the networks defined on the translations have slightly more parameters than the original network. 

\begin{table}[ht]
\begin{center}
\caption{CIFAR-10 result comparison table. Note that the results using the method from Section~\ref{chap4:grelier} and the original CNN are the same in the case of grid graphs (as the embedding of a grid graph is the 2D structure).}\label{chap4:cifar-table}
\resizebox{\columnwidth}{!}{%
\begin{tabular}{|l|c|c||c|c||c|c|}
\hline
\multirow{2}{*}{Support} & \multirow{2}{*}{MLP\citep{lin2015far}} & \multirow{2}{*}{CNN (Resnet-18)} & \multicolumn{2}{c|}{Grid Graph (Oracle)}                                  & \multicolumn{2}{c|}{Covariance Graph (Inferred)}         \\ \cline{4-7} 
                         &                                        &                                  & Chebnet~\citep{defferrard2016convolutional} Section~\ref{chap4:chebyshev} & Section~\ref{chap4:matching} & Section~\ref{chap4:grelier} & Section~\ref{chap4:matching} \\ \hline
Random crop + Flip       & 78.62\%                                & \textbf{93.80\%}                 & 85.13\%                                                                   & 93.94\%                                      & ------                      & 92.57\%                                          \\ \hline
Random crop              & ------                                 & 92.73\%                          & 84.41\%                                                                   & 92.94\%                                      & ------                      & 91.29\%                                          \\ \hline
Graph Data Augmentation  & ------                                 & ----                             & ------                                                                    & 92.81\%                                      & 89.25\%$^a$                 & \textbf{91.07\%}$^a$                             \\ \hline 
None                     & 69.62\%$^a$                                & 87.78\%                          & ------                                                                    & 88.83\%                                      & ----                        & 85.88\%$^a$                                      \\ \hline
\end{tabular}

}
\end{center}
\footnotesize{$^a$ No priors about the structure.}\\
\vspace{-.6cm}
\end{table}

\subsubsection{PINES}

We now test our methods on the previously introduced PINES dataset. We use a shallow network to evaluate our results and compare with a simple MLP, a CNN with 1x1 filters (i.e. that treats each sensor separately before the classification layer) and a 3D CNN with 9x9x9 kernels. In this way we are able to compare with methods that do not take into account the structure (MLP, CNN1x1) and with methods that do, but not optimally (3D CNN applied to irregular inputs). The results are presented in Table~\ref{chap4:tab-pines}, where we see that both methods based on geometric graphs achieve results that are close to the network that uses the original support. Yet neither the graph supported or the 3D supported methods were able to clearly outperform a method that does not use the structure (CNN 1x1).

\begin{table}[ht]
    \centering
    \caption{PINES fMRI dataset accuracy comparison table.}
    \label{chap4:tab-pines}
    \resizebox{\columnwidth}{!}{%
    \begin{tabular}{|l||c|c||c||c|c|c|}
    \hline
    Support                      & \multicolumn{2}{c||}{None} & Irregular 3D structure & \multicolumn{3}{c|}{Neighborhood Graph}     \\ \hline
    Method                       & MLP                        & CNN1x1                 & 3D CNN     & Chebnet~\citep{defferrard2016convolutional} Section~\ref{chap4:chebyshev} & Section~\ref{chap4:matching} & Section~\ref{chap4:grelier}                   \\ \hline
    Accuracy                     & 82.62\%                    & 84.30\%                & \textbf{85.47\%}    & 82.80\%                                                                   & 85.08\% & 84.78\% \\ \hline
    \end{tabular}
    }
\end{table}

\subsection{Conclusion}

In this section based on recent findings~\citep{he2020deep,Errica2020A,xu2019how} we have shown that graph convolution methods are not able to improve over simple baselines, which is inline with the discussion from the previous section. We then have introduced two methods that try to mitigate these problems by either embedding the graph in a $\mathbb{Z}^2$ space or infer graph translations that serve as a proxy to translations in the $\mathbb{Z}^2$ domain. We have demonstrated with experiments that we are able to close the gap between CNNs and GNNs in the case of regular 2D domains, but on the other hand the discussed methods do not seem to generalize well to more irregular graphs (irregular 3D domain). 

\section{Semi supervised classification of vertices}\label{chap4:classify_graph_vertices}

One of the most interesting use of data defined of graph is the semi supervised classification of vertices. Indeed many applications can be resumed into this framework, from social networks (identifying data about the users based on their connections) to citation networks (classifying the article based on its text and citations). It is thus easy to understand why it is the most chosen task to evaluate new methods. Indeed almost all methods discussed in Section~\ref{chap4:definitions} use this task on their benchmarks in order to empirically evaluate their abilities.

Unfortunately there is a two fold problem in the evaluation of these methods using the standard datasets. Before going deep into the problem, we would like to preface by saying that most of the time these problems do not arise from malice or lack of knowledge, they come simply from a lack of computational power/deadline rush from the authors. The first problem is the choice of the $\trainset,\validset,\testset$ split. While using the same split -- as it is done in image applications in order to perform a fair comparison -- seems to be a clear method to follow it has been shown in~\citep{shchur2018pitfalls} to be problematic as the samples are not independent. Indeed, as all the elements of $\dataset$ are in the graph $\gG$, just randomly choosing the samples is not enough: the relationships between them will bias the evaluation. 

Consider a simple example of a graph with multiple rings, where vertices from one ring are connected with the two successive vertices of the same ring, but with atmost one vertex of another ring. In this case the split is very important as methods that give more importance to close nodes would be more effective when vertices from the same ring are present in all sets, while methods that are able to use long connections would be better if only one vertex per ring was chosen for each split. We depict such a ring graph with two different splits in Figure~\ref{chap4:fig-ring}.

\begin{figure}[ht]
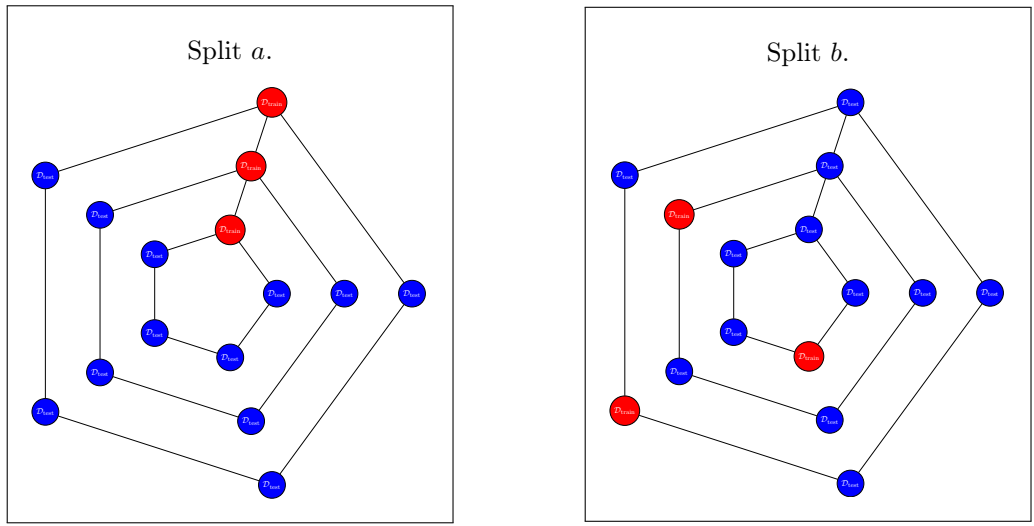

    \centering
    \begin{subfigure}[ht]{.4\linewidth}
        \begin{framed}
            \centering
            \caption{Split $a$.}
            \resizebox{\linewidth}{!}{%
  \tikzsetnextfilename{chapter4/tikz/ring_graph}%
  \input{chapter4/tikz/ring_graph.tex}%
}
        \end{framed}
    \end{subfigure}
    \hspace{0.1\linewidth}
    \begin{subfigure}[ht]{.4\linewidth}
        \begin{framed}
            \centering
            \caption{Split $b$.}
            \resizebox{\linewidth}{!}{%
  \tikzsetnextfilename{chapter4/tikz/ring_graph2}%
  \input{chapter4/tikz/ring_graph2.tex}%
}
        \end{framed}
    \end{subfigure}
    \caption{Two different train/test splits for the same multi-ring graph. Note how the label will propagate very differently depending on the split. Red vertices are on the training set and blue vertices are on the test set.}
    \label{chap4:fig-ring}
\end{figure}

This problem with the $\dataset$ split has led to the proposal of new benchmarks such as~\citep{hu2020open}, that already implement this multi-split setup by default. Unfortunately, this is only half of the problem with benchmarking GNNs.

There is also the problem of fairly comparing methods using the same underlying architecture/regularization in order to ensure that the improvements come from the proposed method and not from a more efficient regularization. This has already been discussed in several papers such as~\citep{shchur2018pitfalls,klicpera2019combining,vialatte2018convolution} where the authors have shown that by adding one simple regularization parameter, that we call edge-dropout (as it consists in adding dropout~\citep{srivastava2014dropout} to the edges of $S$) during training we can have similar performances between models that had been originally very far in performance (GCN~\citep{kipf2016semi} and GAT~\citep{velickovic2018graph}). More-so, this regularization was present in the GAT model and not in the original GCN.

Indeed this problem still happens, for example consider the recent work in~\citep{balcilar2020bridging} and their low-pass filter method. While we agree with the authors that it is more sound than the GCN filter, we do not agree that they were able to outperform the GCN model in equal conditions. If we take the optimized GCN from~\citep{vialatte2018convolution}, where the only difference to the original GCN~\citep{kipf2016semi} is the introduction of edge-dropout, their low-pass filter is not able to improve over the optimized GCN, even if the low-pass filter method also applies edge-dropout.

Note that this is only one example of how unfair comparison happens. Making them fair is hard. The easiest way would be to ensure that one only compares to methods that use the same regularization as oneself, but this could generate too big of a barrier, as it would not allow researchers to fully draw the capabilities of their methods. The harder way would be to correctly perform hyperparameter search to ensure that both the proposed method and the one compared with are on the best of their capabilities. This is not only hard in the theoretical sense of what are the hyperparameters one should optimize, but on the practical sense of executing all the computations needed to perform such a research/optimization. Also note that this hyperparameter problem is not only present in this task, but it is amplified by the split problem, which further biases methods for a set of hyperparameters.

In the following paragraphs we delve further in this two-fold problem, by proposing a framework to benchmark the different graph filters.

\subsection{A framework for comparing graph filters}

An ideal framework for benchmarking the task of semi-supervised classification of vertices in a graph would be one that takes into consideration both problems: \begin{inlinelist} \item models should be compared by their performance over multiple cross-validation splits \item models should be compared either to the maximum of their capabilities or by using the exact same architecture and regularization as the method they compare to \end{inlinelist}. Solving the first problem is easy in theory, but hard in practice as it involves running $s$ times more tests where $s$ is the number of splits. Note that $s$ should be a high number in order to avoid biasing the comparison, how high should it be depends on the confidence interval one expects to obtain from the measure.

In the worst case scenario this would lead to an increase of $sm$ to the overall computational time, where $m$ is the number of methods one is comparing too. Note that it could be relaxed to $s$ if the methods are already benchmarked correctly with a common set of $s$ splits. If we consider $s\ge50$ this would mean that just doing correct split procedure would increase the time needed for running the experiments 50 fold. If we add to this the fact that it is common to rerun the experiments $n$ times to ensure that the results are not obtained thanks to a lucky initialization of parameters, we have even more time of running experiments.

But this is not even the worst part of the computational complexity. Correctly choosing the hyperparameters that one wishes to optimize is a very difficult problem. Here are some questions one should take into account:
\begin{itemize}
    \item do we add graph filters as pre-processing (SGC~\citep{wu2019simplifying})? 
    \item do we add graph filters as post-processing (APPNP~\citep{klicpera2019combining})? 
    \item which neural networks should we use (single-layer, multi-layer, GCN~\citep{kipf2016semi})? 
    \item what size we use for each layer? In the case of multi-layer do we use the same? 
    \item do we combine multiple filters in the same layer?
    \item which graph filter should we use? do we use the same graph filter for each step? what about their hyperparameters?
    \item which value of dropout to use at the input? do we use the same for the input of each layer?
    \item which value of edge-dropout should we use? do we use the same for each graph filter?
    \item do we use L2-regularization? Which value do we choose? The same for all layers?
    \item which learning rate do we use?
\end{itemize}

Note that answering all these 12 questions will generate a very large amount of combinations that one should test (2 choices per combination is already $2^{12}$ possibilities). If we also do it $ns$ times to ensure that it is correctly done in both split and random initialization sides it would be very much intractable.

\subsection{Relaxed framework for comparing graph filters}

Given that it is impossible to tackle all the questions, we thus reduce the number of choices by simplifying the problem with a few simple rules: \begin{inlinelist} \item if one element is applied multiple times it is always applied with the same value \item only one filter is used to isolate their contribution \item only one set of filter hyperparameters is used (SGC with $m=2$ is considered a different filter from SGC with $m=3$) \item we only consider networks with a single hidden layer (hidden size of 64) \item we use the same amount of L2-regularization (0.005) and apply it only to the hidden layer \item we use the same amount of learning rate (0.01) as it is always the same network. \end{inlinelist} Note that by doing so we vastly diminish our ability to ensure that the found solution is the best one. The chosen set parameters for learning rate, hidden layer size and L2-regularization come from~\citep{klicpera2019combining}. This choices leads us to a simplified framework:

\begin{itemize}
    \item which graph filters/graph filter hyperparameters do we test? [$f$ choices]
    \item graph filters as pre-processing or post-processing? [3 choices]  
    \item dropout values for the input? [$d_i$ choices]
    \item dropout values for the edge-dropouts? [$d_e$ choices]
\end{itemize}

This leads to a total of $3 f d_i d_e s n$ choices. We can now set these values to amounts that seem reasonable such as 4 for all the simple choices, $s=50$ and $n=2$ leaving us with 4,800 tests to execute per filter. Even under this simplified conditions a very quick test that takes 1 second to fully execute (from data load to network training to evaluation and saving the results) would take almost an hour and a half to finish. If we consider a more realistic scenario of 30 seconds per test, we are at almost 2 days of training for a single filter. Note that this is for a single filter, in a very simplified scenario where there is only the weights of a logistic regression to train and that a great deal of hyperparameters are taken out. It is not surprising that most papers are not able to do this properly if they have to compare their method with others in the literature. Ideally this could be counter-acted by having a common set of hyperparameters one would optimize in order to compare with other methods. 

Note that in this experiment scenario, when we use a graph filter for pre-processing we are doing a type of feature denoising. On the other hand, when we use a graph filter for post-processing we are doing a type of label propagation. Finally, if we use both, it is akin to a one layer GCN.

\subsection{Graph filter comparison}

In the previous section we have defined the relaxed framework. In this section we present the results using the cora dataset. We chose a split of 20 examples per class for $\trainset$, 30 examples per class for $\validset$ and the rest for the test set as used in~\citep{shchur2018pitfalls}. We use a maximum of 10000 epochs with early stopping if the validation accuracy does not improve after 100 epochs (patience threshold). Code for reproducing the experiments is available at~\url{https://github.com/cadurosar} and a summary of the searched hyperparameters is available in Table~\ref{chap4:hyperparameters-table}.

\begin{table}[ht]
\caption{Summary of the searched hyperparameters.}
\label{chap4:hyperparameters-table}
\begin{center}
\begin{tabular}{|c|c|c|}
\hline
Dropout input         & \multicolumn{2}{c|}{0, 0.25, 0.5, 0.75}                    \\ \hline
Dropout kernel        & \multicolumn{2}{c|}{0, 0.25, 0.5, 0.75}                    \\ \hline
Graph filter position & \multicolumn{2}{c|}{Pre-processing, post-processing, both} \\ \hline
\multirow{5}{*}{Filters}               & $h_\text{SGC}$ & $\alpha=1$ and $m=2$                                     \\
            & $h_\text{Tikhonov}$ & $\alpha=\{10,50\}$                                      \\ 
            & $h_\text{VBL}$ & $\alpha=0.1$ and $m=20$                                     \\ 
            & $h_\text{balcilar-lowpass}$ & $\alpha=1$ and $m=\{5,10\}$                            \\ 
            & $h_\text{Page}$ & $\alpha=0.1$                                     \\ 
\hline
\end{tabular}
\end{center}
\end{table}
We first present the mean validation and test set accuracies, alongside the 95\% confidence interval in Table~\ref{chap4:hyperparameter-result-table}. Note that we only show the results for the hyperparameters that achieved the best validation accuracy for each considered graph filter. In this scenario, the best performing filter of our hyperparameter search was the PageRank filter introduced in~\citep{klicpera2019combining}, but contrary from the original paper, it performed best when used only for feature denoising instead of label propagation. Also note that unfortunately the results are not 100\% conclusive, as seen by the confidence intervals. We believe that adding more initializations per split could allow us to retrieve conclusive results (more than the 95\% confidence interval).

\begin{table}[ht]

\begin{center}
\caption{Performance comparison between different filters using 50 different splits and 2 initializations per split. Results presented as mean accuracy $\pm$ 95\% confidence interval.}
\label{chap4:hyperparameter-result-table}
\begin{adjustbox}{max width=\linewidth}
\begin{tabular}{|c|c|c|c|c|c|}
\hline
Graph filter & Graph filter position & Dropout input & Dropout kernel & Validation Set & Test set      \\ \hline
$h_\text{balcilar-lowpass}(\alpha=1,m=10)$     & Both                  & 0.5           & 0.25           & 83.73 $\pm$ 0.51  & 79.28 $\pm$ 0.31 \\ \hline
$h_\text{VBL}(\alpha=0.1,m=20)$                & Both                  & 0.5           & 0.5            & 84.46 $\pm$ 0.53  & 80.14 $\pm$ 0.34 \\ \hline
$h_\text{Tikhonov}(\alpha=10)$                 & Pre-processing        & 0.5           & 0.5            & 84.78 $\pm$ 0.47  & 80.19 $\pm$ 0.36 \\ \hline
$h_\text{SGC}(\alpha=1,m=2)$                   & Both                  & 0.75          & 0              & 84.96 $\pm$ 0.51  & 80.52 $\pm$ 0.33 \\ \hline
$h_\text{Page}(\alpha=0.1)$                    & Pre-processing        & 0.25          & 0.5            & \textbf{85.24 $\pm$ 0.47}  & \textbf{80.92 $\pm$ 0.3}  \\ \hline
\end{tabular}
\end{adjustbox}
\end{center}
\end{table}

\subsection{Results on the planetoid split}

Now that we have selected the hyperparameters on the cora dataset, we can check the performance of our found hyperparameters on the split defined by~\citep{yang2016revisiting}. This split is used as the de-facto comparison in most papers in the literature. Note that we do this only to verify the performance when compared to other papers as this is not the ideal comparison method. We test 100 different initializations and report the mean test set accuracy alongside the 95\% confidence interval in Table~\ref{chap4:planetoid-cora-result-table}. We note that the relative order of the methods stays the same as per our hyperparameter search and that the performance of the best method (Pagerank) rivals with the best performances found using GCNs and GATs~\citep{kipf2016semi,velickovic2018graph,vialatte2018convolution,shchur2018pitfalls,balcilar2020bridging} while only applying the graph filtering operation as a pre-processing. Finally, the smaller values for the 95\% confidence interval when compared to the difference previous test could be linked to the difference in amount of initializations and splits, which allowed us to have more significant results. 

\begin{table}[ht]
\begin{center}
\caption{Performance comparison between different filters on the split from~\citep{yang2016revisiting}. Results presented as mean accuracy $\pm$ 95\% confidence interval.}
\label{chap4:planetoid-cora-result-table}
\begin{adjustbox}{max width=\linewidth}
\begin{tabular}{|c|c|c|c|c|c|}
\hline
Graph filter & Graph filter position & Dropout input & Dropout kernel & Test set accuracy      \\ \hline
$h_\text{balcilar-lowpass}(\alpha=1,m=10)$     & Both                  & 0.5           & 0.25           & 81.27 $\pm$ 0.20 \\ \hline
$h_\text{VBL}(\alpha=0.1,m=20)$                & Both                  & 0.5           & 0.5            & 82.42 $\pm$ 0.22 \\ \hline
$h_\text{Tikhonov}(\alpha=10)$                 & Pre-processing        & 0.5           & 0.5            & 82.52 $\pm$ 0.19 \\ \hline
$h_\text{SGC}(\alpha=1,m=2)$                   & Both                  & 0.75          & 0              & 82.62 $\pm$ 0.19 \\ \hline
$h_\text{Page}(\alpha=0.1)$                    & Pre-processing        & 0.25          & 0.5            & \textbf{83.53 $\pm$ 0.16}  \\ \hline
\end{tabular}
\end{adjustbox}
\end{center}
\end{table}

\subsection{Ablation results}

We now present ablation results for two of our variables, the dropout possibilities and the position of the graph filter. First concerning the different dropouts possibilities and then for the three different filter behaviours (feature denoising, label propagation and GCN-like).

\subsubsection{Dropout ablation}

We investigate the effect of using dropout on the inputs, the graph kernel or both. Note that the former may be seen as a type of data augmentation of the inputs, while the latter as a form of data augmentation of the graphs (increasing the amount of graphs that we can use on the train set). We display the mean test set accuracy for the best performing hyperparameters of each condition in Table~\ref{chap4:dropout-ablation-table}. The presence of dropout seems to improve the performance on the tested cases, with the dropout on the inputs being more important than on the kernel, but the use of both improves the performance for all graph filters safe for SGC.

\begin{table}[ht]
\begin{center}
\caption{Performance comparison between different filters using 50 different splits and 2 initializations per split for different dropout conditions. Results presented as mean accuracy $\pm$ 95\% confidence interval.}
\label{chap4:dropout-ablation-table}
\begin{adjustbox}{max width=\linewidth}
\begin{tabular}{|c|c|c|c|c|c|}
\hline
Graph filter                  & No dropout                 & Dropout input only     & Dropout kernel only & Dropout both      \\ \hline
$h_\text{balcilar-lowpass}$   & 77.80 $\pm$ 0.35           & 78.90 $\pm$ 0.38          & 78.42 $\pm$ 0.34                & 79.28 $\pm$ 0.31 \\ \hline
$h_\text{VBL}$                & 78.75 $\pm$ 0.37           & 79.90 $\pm$ 0.34          & 79.52 $\pm$ 0.33                & 80.14 $\pm$ 0.34 \\ \hline
$h_\text{Tikhonov}$           & 79.43 $\pm$ 0.32           & 80.14 $\pm$ 0.35          & 80.11 $\pm$ 0.31                & 80.19 $\pm$ 0.36 \\ \hline
$h_\text{SGC}$                & 79.94 $\pm$ 0.35           & 80.52 $\pm$ 0.33          & 80.00 $\pm$ 0.36                & 80.41 $\pm$ 0.34 \\ \hline
$h_\text{Page}$               & 79.91 $\pm$ 0.32           & \textbf{80.76 $\pm$ 0.33} & \textbf{80.51 $\pm$ 0.51}       & \textbf{80.92 $\pm$ 0.3}  \\ \hline
\end{tabular}
\end{adjustbox}
\end{center}
\end{table}

\subsubsection{Graph filter position}

Now we investigate the effect of using the graph filter as a pre-processing (directly at the input as a sort of feature denoising), as a post-processing step (directly at the output as a sort of label propagation) and in both, which in this case is close to a one layer GCN. We display the mean test set accuracy for the best performing hyperparameters of each condition in Table~\ref{chap4:graph-ablation-table}. We can see that for most filters, being applied as a post-processing works slightly better than pre-processing, but more tests would be needed to have any significant conclusion. We also note that surprisingly the PageRank filter had a worse performance when it is used as both pre and post-processing. Further analysis on this effect is needed.

\begin{table}[ht]
\begin{center}
\caption{Performance comparison between different filters using 50 different splits and 2 initializations per split for different dropout conditions. Results presented as mean accuracy $\pm$ 95\% confidence interval. Note that differently from the validation set, in the test set the page rank filter obtains a better result at the post-processing instead of the pre-processing, however for ensuring the correct rigor we use the pre-processing version for all other results as it was the best results on the \emph{validation set}.}
\label{chap4:graph-ablation-table}
\begin{adjustbox}{max width=\linewidth}
\begin{tabular}{|c|c|c|c|c|c|}
\hline
Graph filter                  & Pre-processing only    & Post-processing only         & Pre and post-processing      \\ \hline
$h_\text{balcilar-lowpass}$   & 77.44 $\pm$ 0.33          & 77.29 $\pm$ 0.34                & 79.28 $\pm$ 0.31 \\ \hline
$h_\text{VBL}$                & 78.45 $\pm$ 0.37          & 78.69 $\pm$ 0.33                & 80.15 $\pm$ 0.34 \\ \hline
$h_\text{Tikhonov}$           & 80.18 $\pm$ 0.36          & 80.36 $\pm$ 0.39                & 80.11 $\pm$ 0.36 \\ \hline
$h_\text{SGC}$                & 79.67 $\pm$ 0.33          & 79.73 $\pm$ 0.33                & \textbf{80.52 $\pm$ 0.34} \\ \hline
$h_\text{Page}$               & \textbf{80.92 $\pm$ 0.3} & \textbf{81.09 $\pm$ 0.32}        & 80.05 $\pm$ 0.32  \\ \hline
\end{tabular}
\end{adjustbox}
\end{center}
\end{table}

In the previous paragraphs we have introduced what we believe would be the correct framework to evaluate the different graph filters, but had to use a downgraded version in order to keep it under reasonable timing constraints. We were able to see that the PageRank filter introduced in~\citep{klicpera2019combining} had the best results between the analysed filters, but that further testing would be needed to understand all the effects that we saw in the results.

\section{Summary of the chapter}

In this chapter we have delved into the domain of deep neural networks defined on graphs. We have built upon the concepts from the previous chapters in order to define recent methods in a single graph filter framework that we have presented in increasing order of complexity in Section~\ref{chap4:definitions}. While this framework is not exactly novel, we have extended it to more methods and have introduced a discussion on the drawbacks of these methods.

We then discussed applications of DNNs defined on graphs in the context of the supervised classification of graph signals in Section~\ref{chap4:classify_graph_signals}. We have discussed recent contributions that show the drawbacks of the current approaches in this domain and then introduced two of our contributions. Their aim is to close the gap between graph convolutions and classic 2D/3D convolutions. Our two introduced contributions were published in conferences as follows:
\begin{itemize}
    \item \bibentry{grelier2018graph}
    \item \bibentry{lassance2018matching}
\end{itemize}

Finally, we have discussed the applications on the context of semi-supervised classification of vertices. We have first discussed the problem of fairly evaluating the different GNN methods on this task. While it is not a novel problem with the domain, recent works still commit the two most common pitfalls: \begin{inlinelist}\item using only one train/valid/test split that has been already shown to bias results in~\citep{shchur2018pitfalls} \item not comparing the methods fairly, e.g., method A performs better than method B, but it is mostly due to adding dropout than due to the method itself\end{inlinelist}. Note that these problems are not necessarily due to malpractice or to malice, but mostly from the sheer computational problem that would be necessary to perform everything correctly. Indeed, we propose a framework in order to solve both i) and ii), but show that we would never be able to execute the optimal version in reasonable delay. We added a relaxed framework and present our results on the cora dataset.
\chapter{Deep Neural Networks latent spaces supported on graphs}\label{chap5}
\localtableofcontents
\vspace{1.0cm}

In the first two chapters we introduced the concepts of DNNs and GSP. We then combined these concepts in the previous chapter to introduce neural networks on inputs defined on graphs. In this chapter we present applications where graphs can be used to represent the topology of intermediate DNN activations, even if the inputs are not defined on a graph support. 

We mainly use the smoothness of intermediate representation graphs to characterize DNN behavior in this chapter. We consider multiple goals. The first one is to determine whether a DNN is overfitted, underfitted or achieves a good fit in Section~\ref{chap5:smoothness}. We then focus on applying this concept during the training of DNNs, showing that we can obtain performing feature extractors in Section~\ref{chap5:smoothness_loss} or that we can improve robustness to attacks and deviations in Section~\ref{chap5:laplacian}. Finally, we show that using intermediate representation graphs to define the topology of intermediate spaces we are able to improve knowledge distillation performance in Section~\ref{chap5:gkd}. 

\section{Characterizing DNN behavior via smoothness of intermediate representations}\label{chap5:smoothness}

We have previously discussed the problem of overfitting in Chapter~\ref{chap2} and have defined our view in Definition~\ref{chap2:overfitting}. One of the problems that arises from overfitting is that most of the ways to address it revolve around two main concepts: \begin{inlinelist}\item better understanding of the domain in order to generate artificial data points, c.f. Section~\ref{chap2:data_augmentation} \item removing data points from the trainable ($\trainset$) subset of $\dataset$ and allocating them to $\validset$ and $\testset$.\end{inlinelist} While the former strategy has been shown to help prevent overfitting and improve generalization performance, the latter has been demonstrated to reduce overall generalization performance, leading to the new paradigm of only separating data into $\trainset$ and $\testset$ which has lead to improved generalization performance, but may lead to overfitting to the $\testset$ (i.e. the network may lose performance if we resample $\testset$ from $\sD$).

In order to address this problem, one possible solution is to develop analytical tools that try to infer overfitting using only data from $\trainset$. While it is hard to develop such tools with a very strong scientific background, we develop in this Section two empirical frameworks in order to analyse the state of DNNs a posteriori: \begin{inlinelist} \item use the evolution of the smoothness of intermediate representations to characterize the state of a DNN into underfitted, ok, overfitted and strongly overfitted. \item use the lessons learned from (i) to find a metric that is correlated to the performance on the $\testset$ using only data on the $\trainset$.\end{inlinelist}

\subsection{Predicting DNN behavior using GSP}

In this section we present the idea from~\citep{gripon2018insidelook} that was the cornerstone for our interest in using graphs and GSP to analyze the intermediate representations of DNNs. The authors propose to analyse the evolution of the smoothness of the binary label indicator signal (c.f. Definition~\ref{chap2:label_indicator_vector}) on graphs that are inferred at the output of each layer of a DNN at each training epoch. We recall that the smoothness of a signal on a graph represents how the signal and the graph are aligned. Also, using a binary label indicator vector as our signal means that the smoothness metric can be considered as the sum of the edges of nodes in different classes. Finally, as our graphs are inferred using the similarity between nodes and are thresholded to form a $k$-NN graph, a perfectly smooth graph would be one where either the similarity between examples of different classes is zero or there are no connections between examples of different classes.

The goal is to try to find if there are differences in behavior when analyzing the smoothness metric for each case. Indeed, if there is a clear difference in behavior, this metric could potentially be used to identify, using only $\trainset$, which condition the network seems to favor and take the correct steps in order to transform an overfitted network to a ``good fit'' network, without needing to use $\validset$ or $\testset$. In the following paragraphs, we present and discuss their results.

\subsubsection{Experiments}

The goal of the experiments is to analyze the evolution of the smoothness based metric over the training epochs. To do so, a Resnet-18 with preact blocks is used as the architecture, data augmentation is used as well in order to try to avoid overfitting to $\trainset$ and the experiments are performed on the CIFAR-10 dataset. There are four main different conditions that are analyzed: 
\begin{enumerate}
    \item ``Good fit''/Reference: The reference training procedure;
    \item Underfitting: An underfitted architecture where we divide by 10 the number of feature maps in each convolutional layer.
    \item Extreme overfit, In this case we randomize the labels on $\trainset$. The accuracy on the randomized $\trainset$ is therefore 100\% but on $\testset$ it is as good as a random guess). This is inline with the findings from~\citep{zhang2017generalization};
    \item Overfitting: A slightly overfitted condition in which data augmentation and regularization are removed.
\end{enumerate}

In Figure~\ref{chap5:fig-accuracy-evolution} we depict the evolution of train and test accuracy during the training of the networks. Note that as expected, the best test accuracy is detected on the ``reference'' case and that we are able to train the extreme overfit very quickly when compared with the other networks. 

\begin{figure}[ht]
    \begin{center}
        \begin{subfigure}[ht]{0.48\linewidth}
            \centering
            \caption{Reference}
            \begin{adjustbox}{max width=\linewidth}
                \includegraphics{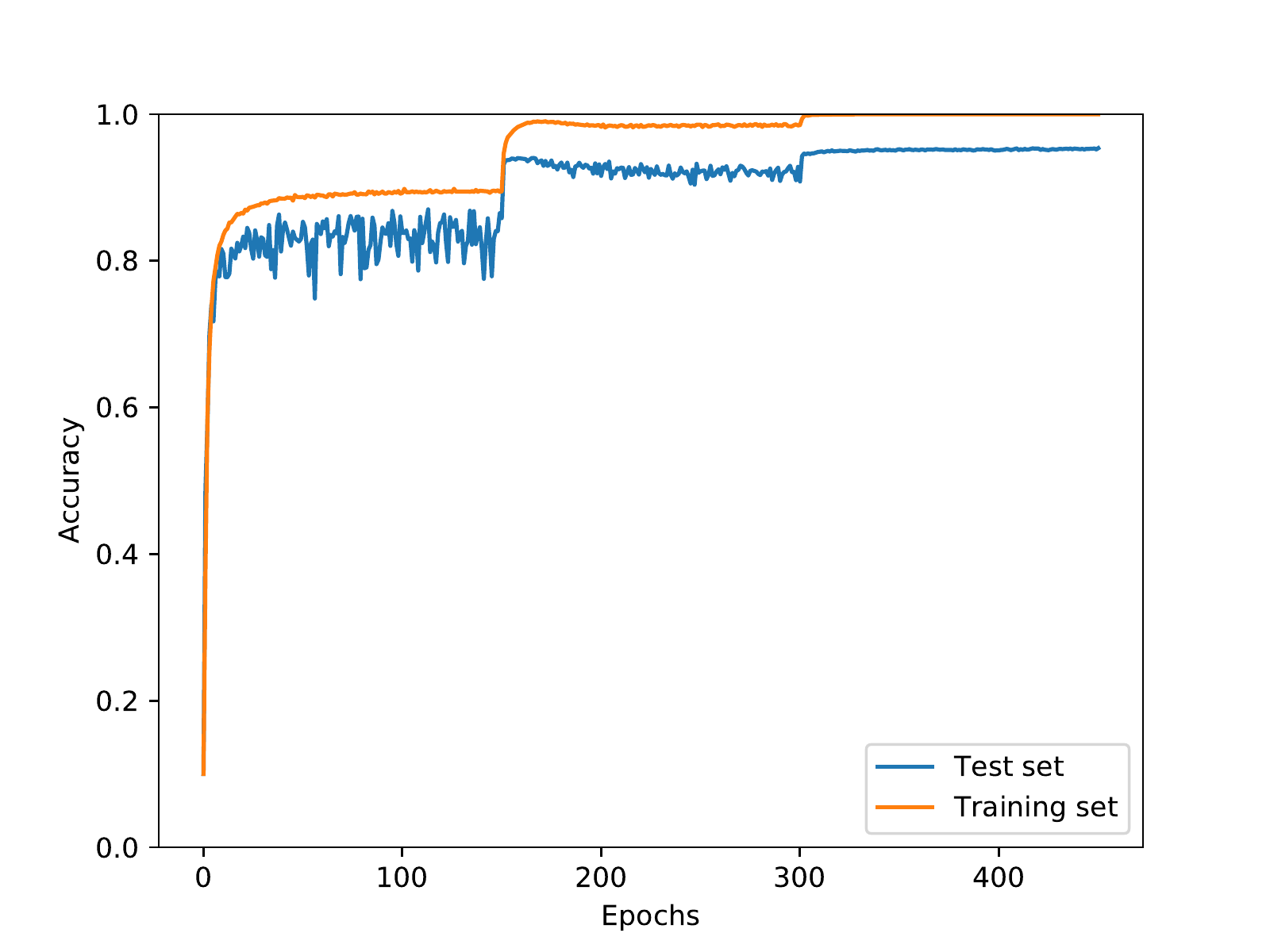}
            \end{adjustbox}
        \end{subfigure}
        \begin{subfigure}[ht]{0.48\linewidth}
            \centering
            \caption{Slight overfit}
            \begin{adjustbox}{max width=\linewidth}
                \includegraphics{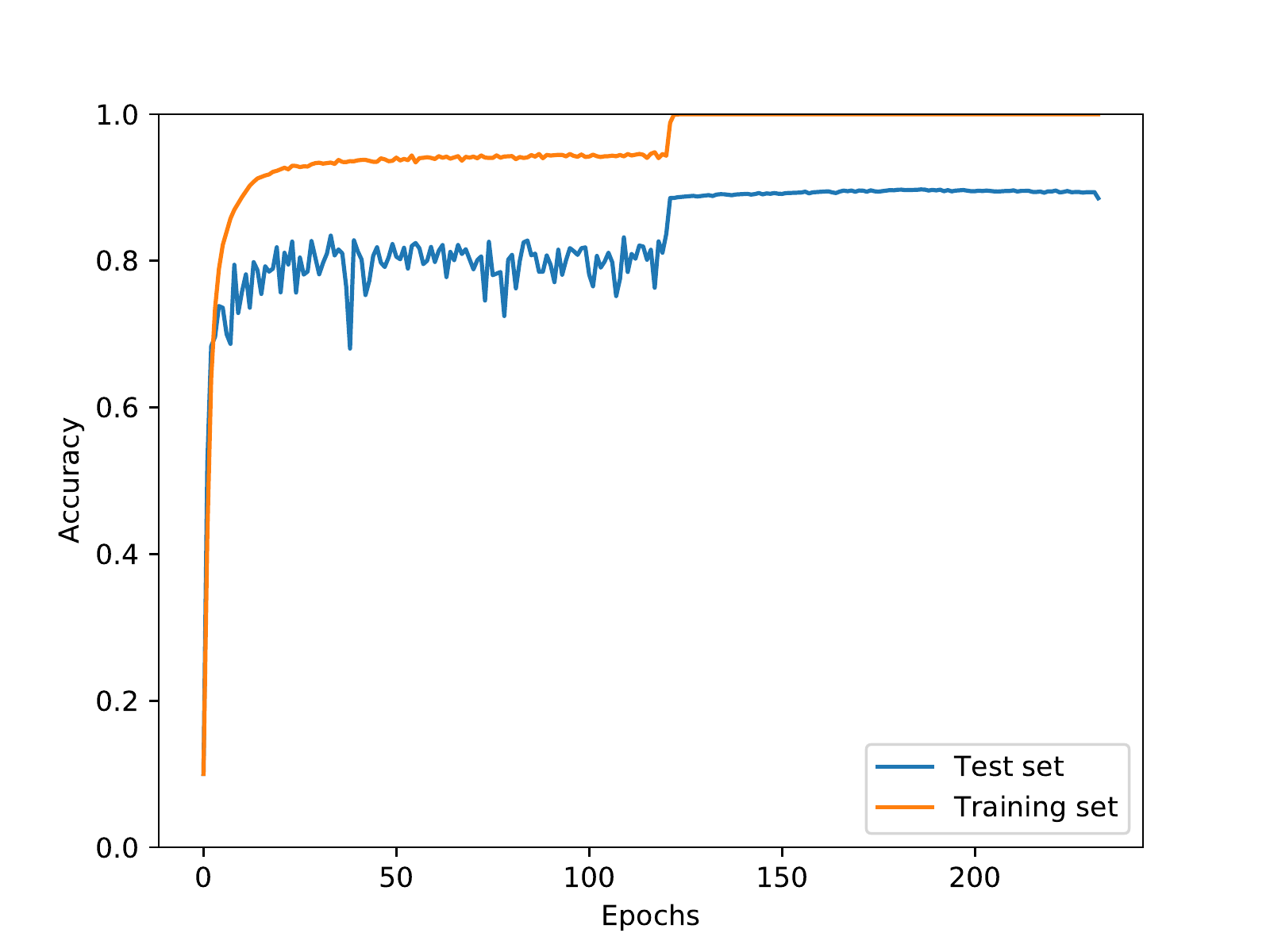}
            \end{adjustbox}
        \end{subfigure}
        \begin{subfigure}[ht]{0.48\linewidth}
            \centering
            \caption{Underfit}
            \begin{adjustbox}{max width=\linewidth}
                \includegraphics{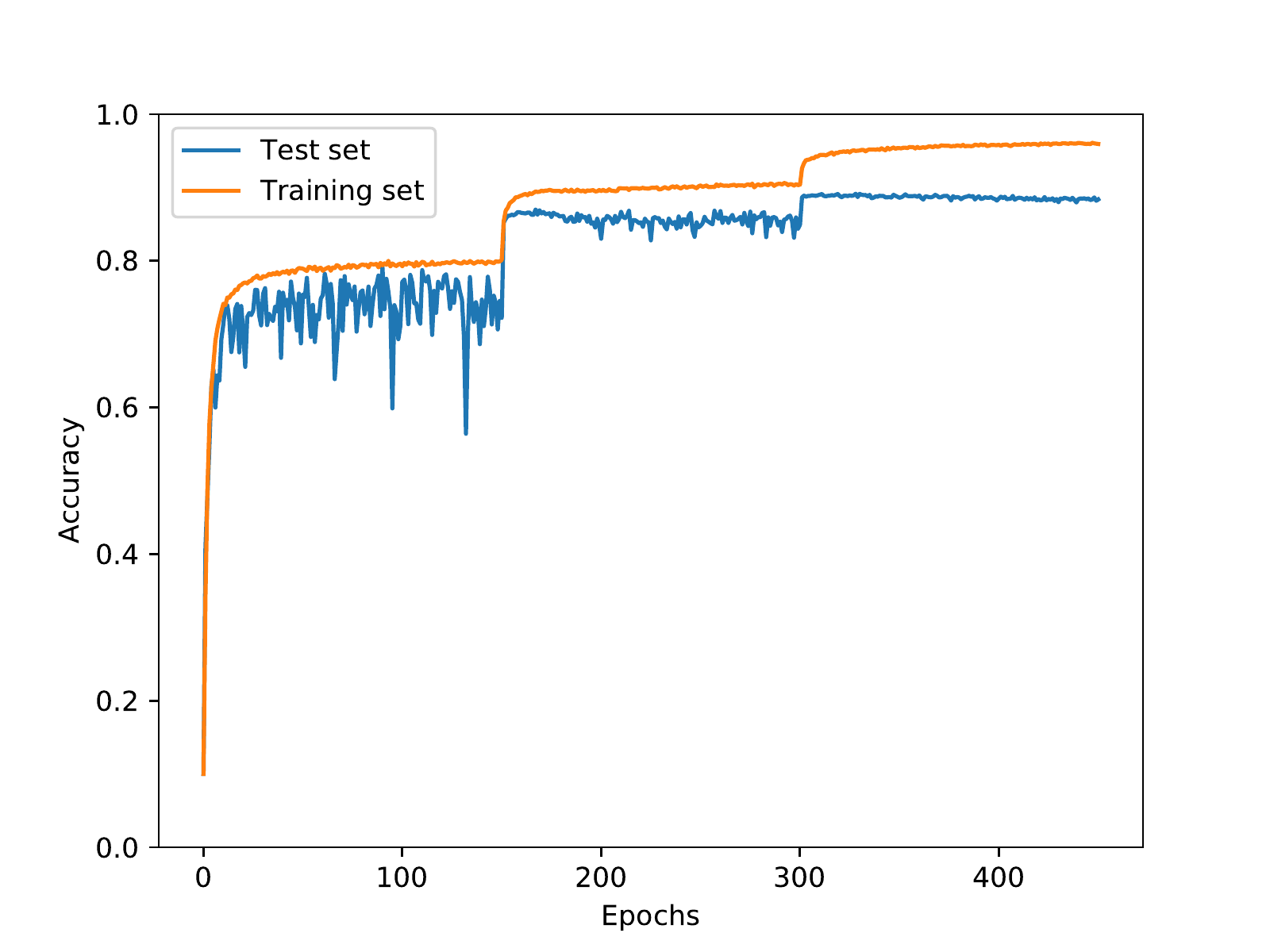}
            \end{adjustbox}
        \end{subfigure}
        \begin{subfigure}[ht]{0.48\linewidth}
            \centering
            \caption{Extreme overfit}
            \begin{adjustbox}{max width=\linewidth}
                \includegraphics{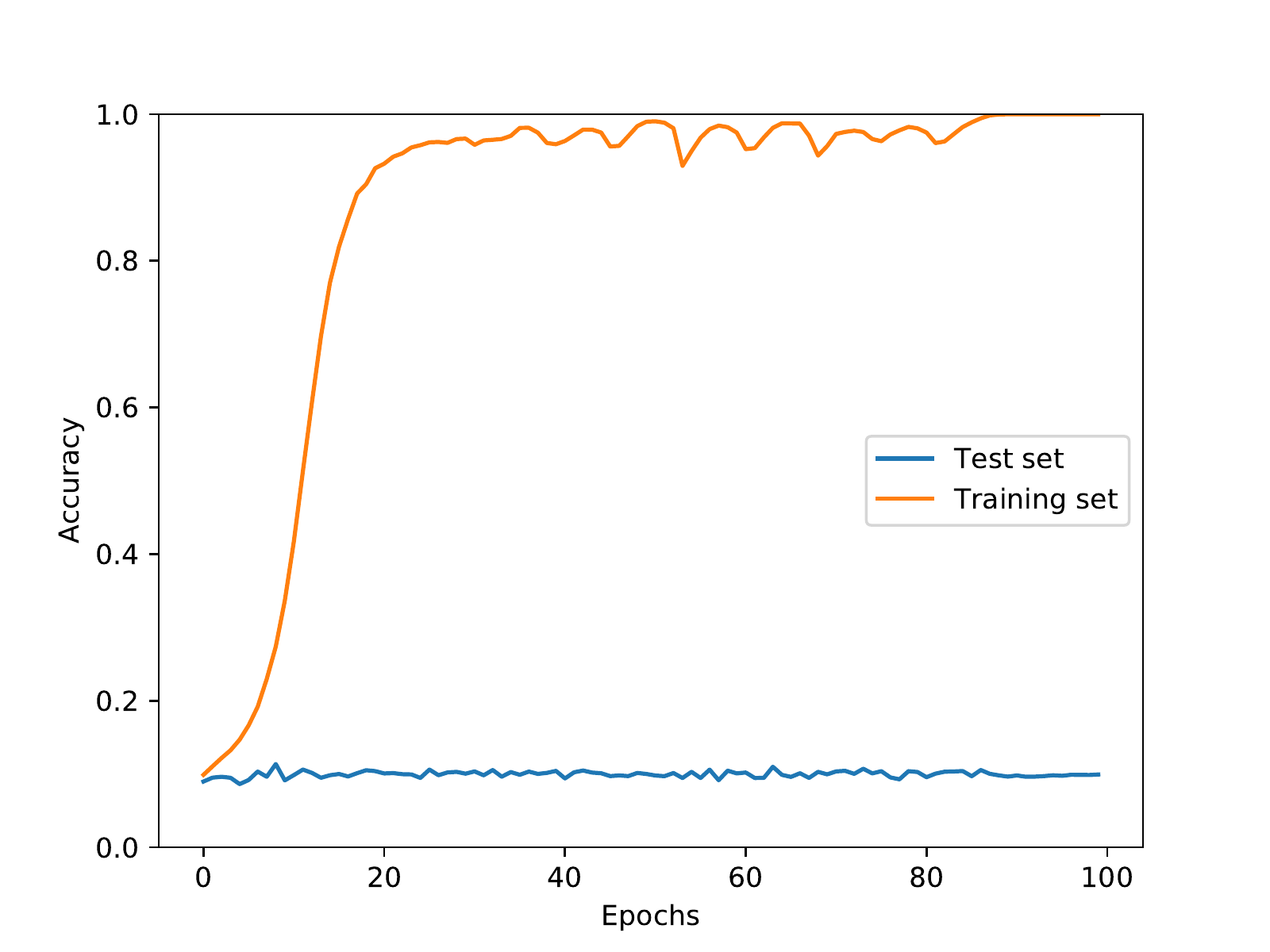}
            \end{adjustbox}
        \end{subfigure}
    \end{center}
    \caption{Train and test accuracy evolution on a Resnet18 under different conditions. Figure adapted from~\citep{gripon2018insidelook} ©2018 IEEE.}
    \label{chap5:fig-accuracy-evolution}
\end{figure}

We then present the smoothness evolution in Figure~\ref{chap5:fig-smoothness-evolution}, note how there is a very distinct behavior over the 4 conditions. Each representation is the output of a Resnet block (c.f. Section~\ref{chap2:subsection_resnet}). That difference in behavior could indicate that it is indeed possible to evaluate the status of the network, using only information from $\trainset$. Note that there are important changes of dimension occurring multiple times throughout the process, which seems to be inline with the concept of group of blocks of the Resnet, in representations 3, 5 and 7. As a result, it is better to consider the representations as groups: \begin{inlinelist} \item 1, 2, and 3 \item 4 and 5 \item 6 and 7 \item 8, 9, and 10.\end{inlinelist} 

\begin{figure}[ht]
    \begin{center}
        \begin{subfigure}[ht]{0.48\linewidth}
            \centering
            \caption{Reference}
            \begin{adjustbox}{max width=\linewidth}
                \includegraphics{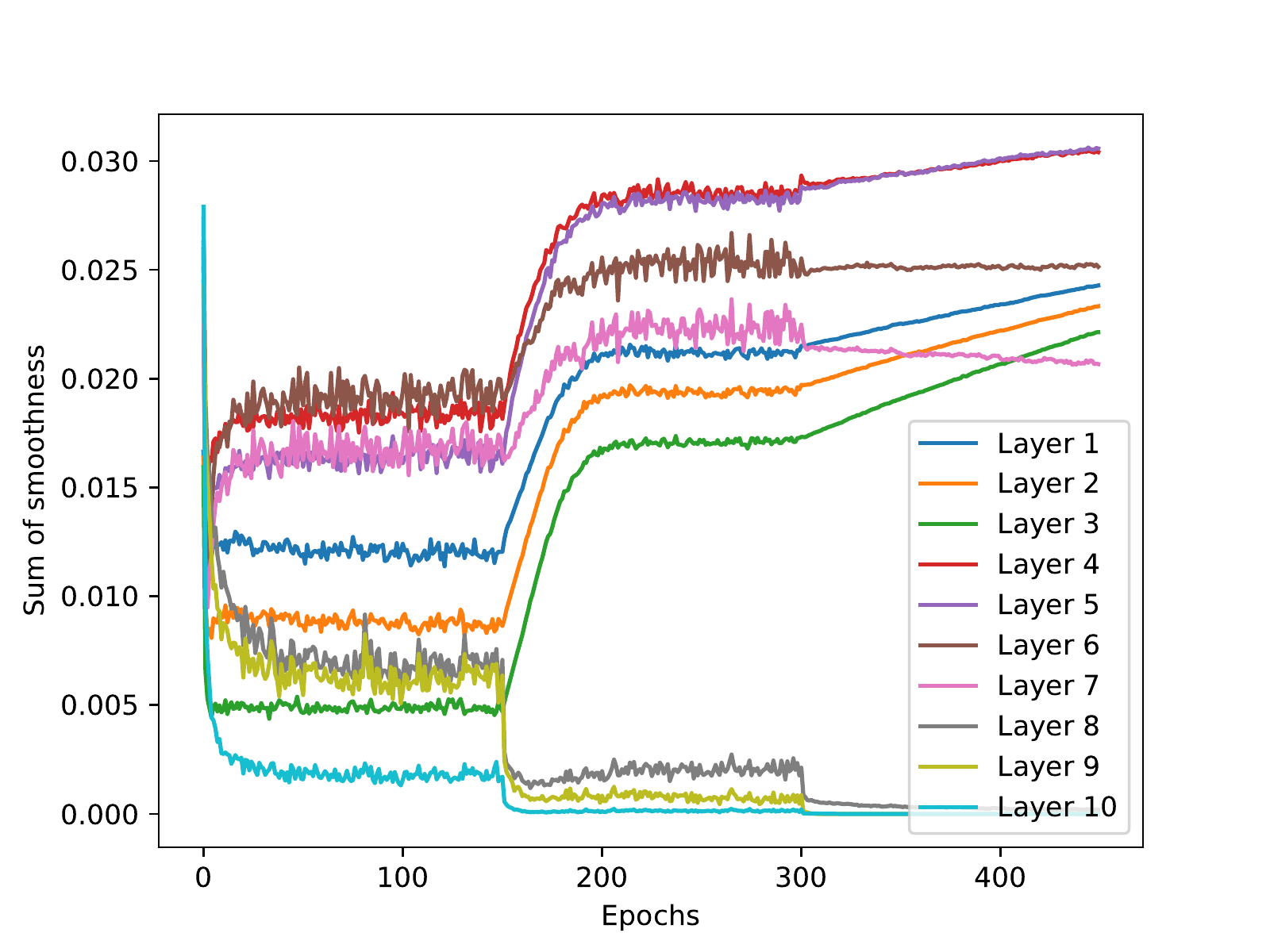}
            \end{adjustbox}
        \end{subfigure}
        \begin{subfigure}[ht]{0.48\linewidth}
            \centering
            \caption{Slight overfit}
            \begin{adjustbox}{max width=\linewidth}
                \includegraphics{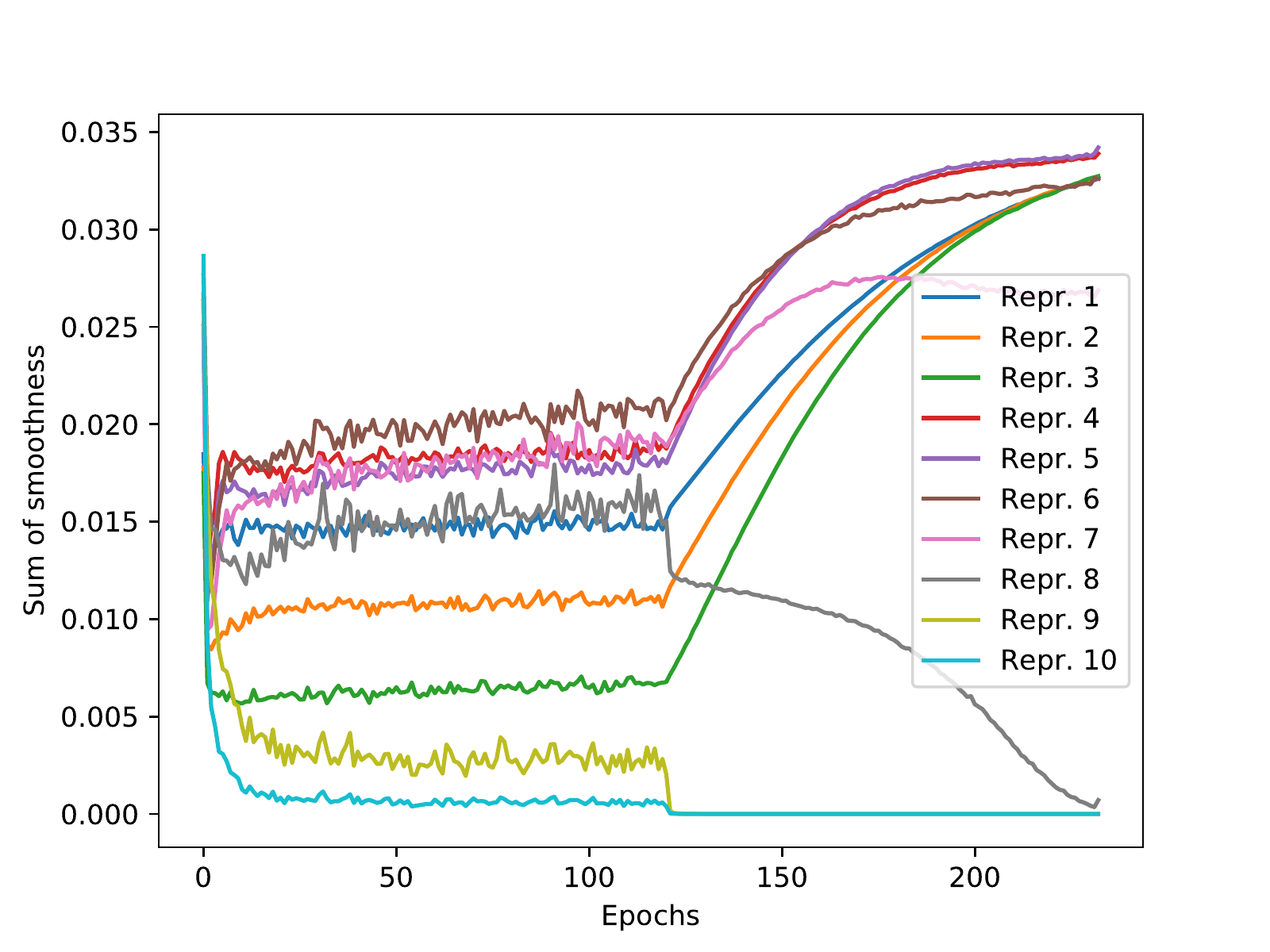}
            \end{adjustbox}
        \end{subfigure}
        \begin{subfigure}[ht]{0.48\linewidth}
            \centering
            \caption{Underfit}
            \begin{adjustbox}{max width=\linewidth}
                \includegraphics{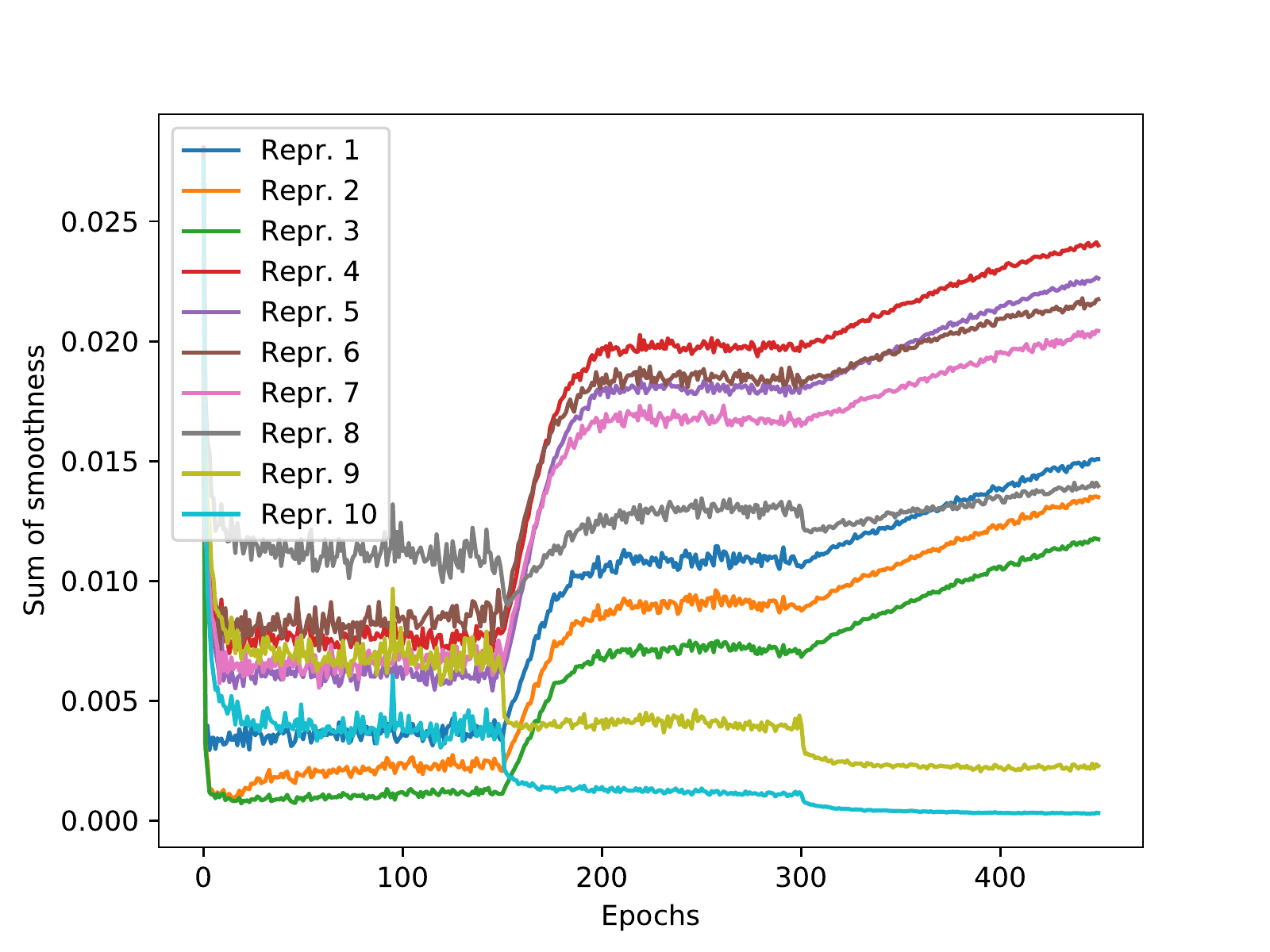}
            \end{adjustbox}
        \end{subfigure}
        \begin{subfigure}[ht]{0.48\linewidth}
            \centering
            \caption{Extreme overfit}
            \begin{adjustbox}{max width=\linewidth}
                \includegraphics{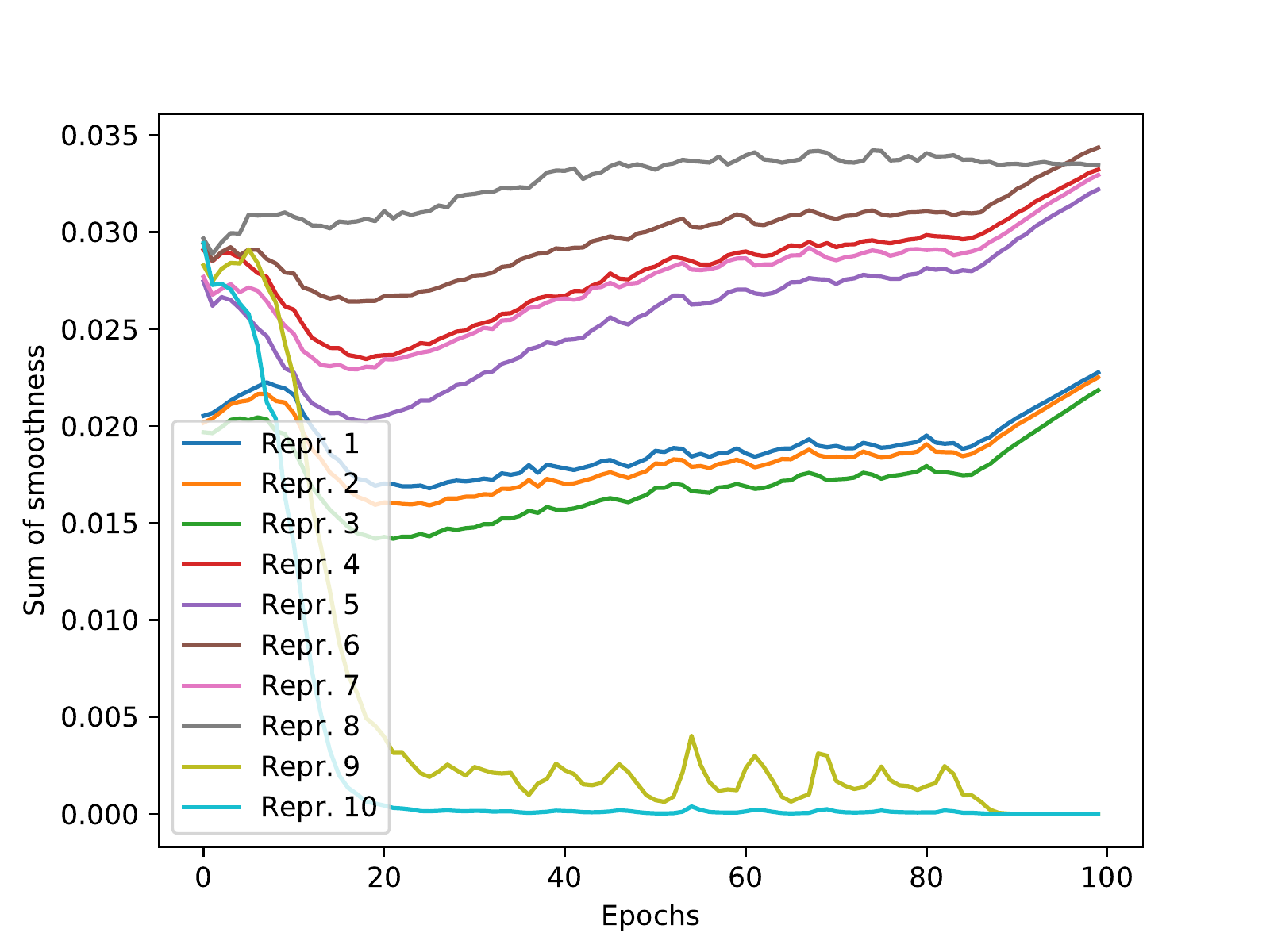}
            \end{adjustbox}
        \end{subfigure}
    \end{center}
    \caption{Smoothness evolution over the different representations of a Resnet-18 under different conditions. Figure adapted from~\citep{gripon2018insidelook} ©2018 IEEE.}
    \label{chap5:fig-smoothness-evolution}
\end{figure}

Interestingly, we are able to observe significant changes between the different conditions, in particular in the last group. In the reference case the three layers have reasonably similar label smoothness, whereas in other conditions we see important gaps, in particular between representations 8 and 9. Indeed, the lack of a gap between label smoothness between the penultimate representation and the output one seems to be a good indicator of overfitting. 

On the other hand, it is also fair to say that the reference condition is slightly overfitting, considering the very high score on the training set. Recall that DNNs are normally trained with a learning rate that decays very quickly, and it seems that these sudden changes in the learning rate specializes the last group in clustering properly the examples whereas it effects the contrary for previous representations. Another interesting observation is the fact that the label smoothness continues to evolve even if training and test accuracies remain constant. Indeed, this motivates the idea that even if the training accuracy is optimal, there are some representations that are still changing, and therefore, the label smoothness could also be used as a secondary measure to see if the network has converged. In the following section we present a continuation of this work, where we use the same smoothness metric in order to try to predict the test accuracy of the network.

\subsection{Predicting under/overfitting using graph signal smoothness}

To avoid underfitting and overfitting often boils down to performing crossvalidation, where one needs to split $\dataset$ into $\trainset$ and $\validset$ in order to assess the generalization performance of a DNN trained on $\trainset$. However, this requires to reduce the size of the training set, thus leading to globally poorer accuracy, and it does not guarantee the architecture will be the best for a set distinct from the validation one. In this subsection we use the GSP framework to develop a measure that we call ``Smoothness Gap'' in order to analyse the state of a DNN into underfitted, ok or overfitted. We do so using the knowledge acquired from the work detailed in the previous subsection~\citep{gripon2018insidelook} which shows that a neural network that generalizes well seems to have a very smooth $k$-neighbours graph using the features from its last layers. This is in line with previous works in the regularization of DNNs~\citep{lee2015deeply}, which use the classification results based on the features of the intermediate layers of a network as a regularizer. The ideas of this section were accepted to a workshop without published procedings~\citep{lassance2018predicting}.

\subsubsection{Smoothness Gap}

First we formally define what we call smoothness gap. Let us consider $M$ example inputs for each of the $C$ classes which we then use to generate intermediate representations across a given trained DNN. We then use the same strategy from Section~\ref{chap3:reducing_noise} to generate $k$-nearest neighbor graphs for each intermediate representation of the DNN and consider the binary label indicator vector from Definition~\ref{chap2:label_indicator_vector} as our graph signal. Note that by using $k$-NN graphs, each $\adjmatrix$ contains less than $2MCk$ nonzero elements and that choosing the correct value of $k$ is a well known problem as we have previously discussed in the previous chapters.

With a graph for each intermediate representation and a graph signal we can then compute the smoothness ($\sigma$) of the graph signal for each intermediate representation. We recall that the smoothness of a label signal is a direct measure of how well the examples of this class are separated from the other classes, and that a global label smoothness of 0 indicates pairs of examples belonging to distinct classes are not connected in the graph or are completely orthogonal. 

We define the smoothness gap as the difference between the smoothness of the representations on the last layer of the network (i.e. the classification layer) and the representations of the penultimate layer (i.e. the representations after the global average pooling). Note that this is influenced by our architecture, in our case we use Resnet-18 as defined in Chapter~\ref{chap2}. Finally, in order to compare smoothness of a given signal on various graphs with possibly very different weightings, we choose to normalize smoothness by its maximum possible value. In our case, we rather use an upper-bound which is $2 MCk$.

\subsubsection{Experiments}

We train our DNNs on a portion of the CIFAR-10 dataset~\citep{krizhevsky2009learning}. To estimate label smoothness, we sample 50 examples from each class to generate our graphs~($M = 50$ and $C = 10$). We repeat this sampling $10$ times. We evaluate our measure using graphs and we also compute the $R^2$ coefficient obtained by a linear regression over our measures to further stress the correlation. The results reported are the mean label smoothness over the 10 graphs. In order to controlably generate our underfitting and overfitting conditions, we proceed as follows:

\begin{enumerate}

\item \textbf{Overfitting:} we use only a portion of the training set ranging from $21\%$ to $99\%$ by $2\%$ increments;

\item \textbf{Underfitting:} convolutional layers come with a hyperparameter which is the number of feature maps. In order to easily vary the number of trainable parameters without changing the global architecture, we thus vary the number of feature maps. In the chosen architecture, the number of feature maps on the first convolutional layer determines all the others. We thus vary it from $5$ to $64$, its default value.
\end{enumerate}

We considered various values of $k$ ($10, 20, M = 50, MC = 500$). Most consistent results were obtained with a value of $20$. Using $10$ would incur on a lot of points being concentrated with approximately zero smoothness. This is not surprising as it tends to select only the very nearest neighbors. Using $M$ or $MC$ leads to a lot of noise in the measures as there are many more pairwise distances to take into consideration.

In Figure~\ref{chap5:fig-datsetsize-acc} we show that by varying the size of the dataset we can generate highly overfitted DNNs. Moreover, there is a correlation between generalization abilities reported by the test accuracy score and the smoothness gap $\delta_s$. We stressed this fact by computing a linear regression and obtained a $R^2$ coefficient of $68\%$.

\begin{figure}[ht]
    \begin{center}
        \begin{subfigure}[ht]{\linewidth}
            \centering
            \caption{$R^2=68\%$}
            \begin{adjustbox}{max width=\linewidth}
  \tikzsetnextfilename{chapter5/tikz/varyingsize}%
  \input{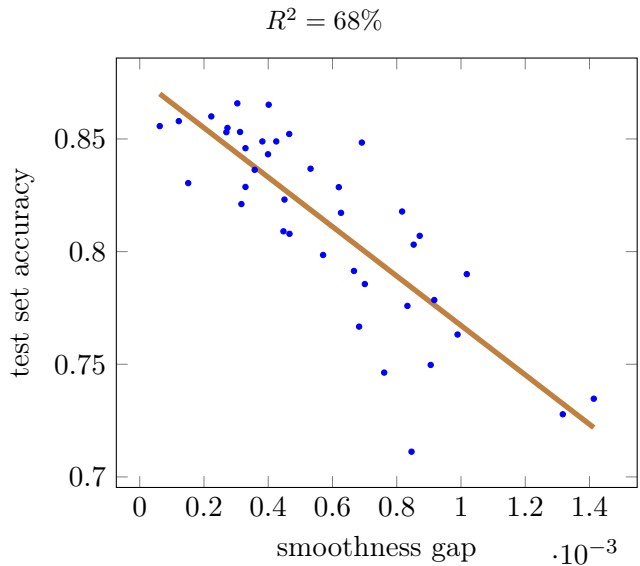}%

            \end{adjustbox}
        \end{subfigure}
    \end{center}
    \caption{Results generated by varying the size of the dataset. Figure extracted from~\citep{lassance2018predicting}.}
    \label{chap5:fig-datsetsize-acc}
  \end{figure}
  
Now we study the case of underfitted/properly fitted DNNs. First we test the case where we vary the amount of parameters on the network following the traditional scaling of Resnet-18 (c.f. Section~\ref{chap2:subsection_resnet}) and we depict the experiments in Figure~\ref{chap5:fig-varying-firstlayer}, where we show that we can also obtain a strong correlation between the test accuracy and the smoothness gap $\delta_s$, as seen by the $R^2=84\%$ coefficient of its linear regression. These results show a very high predictability of the test error given the smoothness gap $\delta_s$. It is very interesting to see this high predictability as the computation of the smoothness gap does not require any knowledge about the test set. 

\begin{figure}[ht]
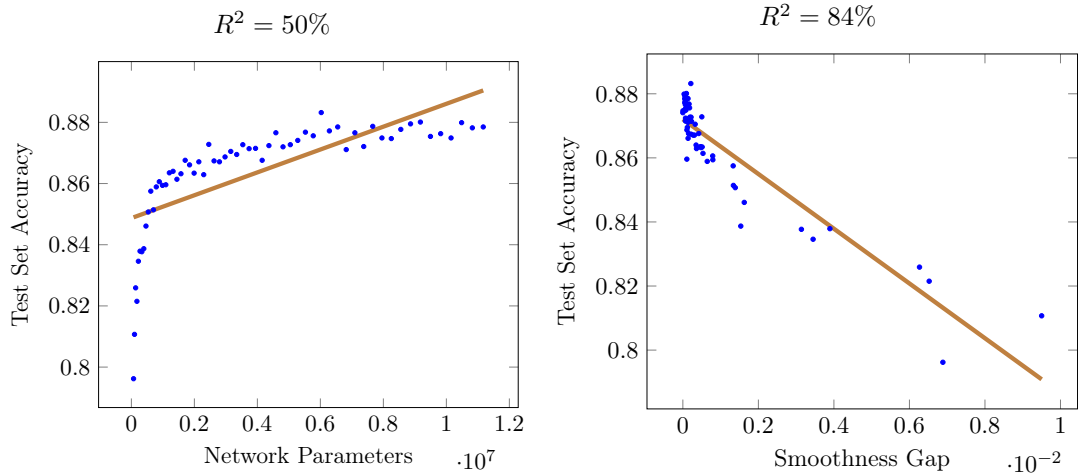

    \begin{center}
        \begin{subfigure}[ht]{.48\linewidth}
            \centering
            \caption{$R^2=50\%$}
            \begin{adjustbox}{max width=\linewidth}
  \tikzsetnextfilename{chapter5/tikz/firstlayersize}%
  \input{chapter5/tikz/firstlayersize.tex}%

            \end{adjustbox}
        \end{subfigure}
        \begin{subfigure}[ht]{.48\linewidth}
            \centering
            \caption{$R^2=84\%$}
            \begin{adjustbox}{max width=\linewidth}
  \tikzsetnextfilename{chapter5/tikz/firstlayersmoothness}%
  \input{chapter5/tikz/firstlayersmoothness.tex}%

            \end{adjustbox}
        \end{subfigure}    
    \end{center}
    \caption{Results generated by varying the size of DNN following the traditional Resnet scaling pattern of the DNNs. In the left we depict the correlation between network size and test set performance, while on the right we show the correlation between the smoothness gap and test set performance. Figure extracted from~\citep{lassance2018predicting}.}
    \label{chap5:fig-varying-firstlayer}
\end{figure}

However, we note that for the underfitted condition, the performance of the network is also correlated to the number of parameters~($R^2 = 50\%$), even if this relation is less strict. In order to be sure that our measure was not only an indirect measure of the size of the network, we performed additional experiments where the number of feature maps at each layer is changed independently, resulting in almost no correlation between the global number of parameters and test accuracy~($R^2 = 14\%$), while still maintaining a good correlation between the smoothness gap and test accuracy~($R^2 = 67\%$). We depict the relationships between size, smoothness and test set accuracy in~Figure~\ref{chap5:fig-varying-all}.

\begin{figure}[ht]
    \begin{center}
        \begin{subfigure}[ht]{.48\linewidth}
            \centering
            \caption{$R^2=14\%$}
            \begin{adjustbox}{max width=\linewidth}
  \tikzsetnextfilename{chapter5/tikz/sizefeatures}%
  \input{chapter5/tikz/sizefeatures.tex}%

            \end{adjustbox}
        \end{subfigure}
        \begin{subfigure}[ht]{.48\linewidth}
            \centering
            \caption{$R^2=67\%$}
            \begin{adjustbox}{max width=\linewidth}
  \tikzsetnextfilename{chapter5/tikz/smoothnessfeatures}%
  \input{chapter5/tikz/smoothnessfeatures.tex}%

            \end{adjustbox}
        \end{subfigure}    
    \end{center}
    \caption{Results generated by varying the size of DNN without following the traditional Resnet scaling pattern of the DNNs. In the left we depict the correlation between network size and test set performance, while on the right we show the correlation between the smoothness gap and test set performance. Figure extracted from~\citep{lassance2018predicting}.}
    \label{chap5:fig-varying-all}
\end{figure}

In this subsection we have proposed to measure the smoothness gap and have shown via experiments that there exists a strong correlation between this measure and the generalization of DNNs. While this seems very promising, it is of upmost importance to be careful and not overpromise as further study is still needed to see if this is an useful correlation or if it is a subproduct of a possibly poor experimental design. Future work includes developing an understanding of why the training set smoothness is correlated with the test set accuracy, using this measure explicitly when performing hyperparameter search, and studying how to use this measure during the training phase.

\section{Smoothness as an objective function for DNNs}\label{chap5:smoothness_loss}

In the previous section we described how the smoothness of graphs generated by the intermediate features of DNNs may be linked with their generalization abilities. In this section we inverse this concept, by training the network to directly minimize the smoothness generated by a DNN in order to train a feature extractor, which was the central theme of our contribution~\citep{bontonou2019smoothness}. 

In machine learning, classification is one of the most studied problems and cross-entropy is the most popular loss function for computer vision tasks. Cross-entropy is often preferred over mean squared error because it converges faster and tends to reach better accuracy. However, cross-entropy requires the outputs of the network to be label indicator vectors of the classes. We believe this decision comes with noticeable drawbacks:
\begin{itemize}
    \item The dimension of the output vectors has to be equal to the number of classes, preventing an easy adaptation to the introduction of new classes. In scenarios where the number of classes is large, this also causes the last layer of the network to contain a lot of parameters.
    \item Inputs of the same class are forced to be mapped to the same output, even if they belong to distinct clusters in the input space. This might cause severe distortions in the topological space that are likely to create vulnerabilities to small deviations of the inputs.
    \item The arbitrary choice of the one-hot-bit encoding is independent of the distribution of the input and of the initialization of the network parameters, which can slow and harden the training process.
\end{itemize}

To overcome these drawbacks, authors have proposed several solutions, some of which we present here:
\begin{enumerate}
    \item Some works propose to train DNNs solely as feature extractors, one example is~\citep{hermans2017defense}, where the authors use triplets, where the first element is the example to train, the second belongs to the same class and the last to another class. They then enforce that the first is closer to the second than to the last. 
    \item Other authors try to smooth the outputs of the DNN, either using smoother representations~\citep{szegedy2016rethinking} or by using a teacher network to define the output vector~\citep{hinton2014distillation} or just by using smoother vectors. 
    \item Finally, works such as~\citep{dietterich1994solving,tigreat2016assembly} propose to use error correcting codes to generate outputs of the network.
\end{enumerate}

In this section, we tackle the problem of training deep learning architectures to generate features that are easy to classify without relying on arbitrary choices for the representation of the output. We introduce a loss function that aims at maximizing the distances between outputs of different classes. It is expressed using the smoothness of a label signal on similarity graphs built at the output of the network. The proposed criterion does not force the output dimension to match the number of classes, can result in distinct clusters in the output domain for a same class, and builds upon the distribution of the inputs and the initialization of the network parameters. We demonstrate the ability of the proposed loss function to train networks with state-of-the-art accuracy on common computer vision benchmarks and its ability to yield increased robustness to deviations of the inputs.

\subsection{Methodology}\label{methodo}

Consider the problem of training a classification function $f$ using $\trainset$, where $\trainset$ is made of $n$ elements. In this case we can say that $\vx_\mu$ refers to an input tensor and $\vy_\mu$ to the corresponding output vector. We denote $C$ the number of classes. In the context of deep learning, $\vy_\mu$ is typically a binary label indicator vector of its class ($\vy_\mu \in \R^C$) and the network function $f$ is trained to minimize the {\em cross-entropy loss} as previously defined in Equation~\ref{chap2:simplified_ce}.

We will consider graphs $\gG$, defined by their weighted adjacency matrix $\adjmatrix$, where $\emAdjacency_{\mu,\nu}$ is the weight of the edge between vertices $\mu$ and $\nu$, or 0 if no such edge exists. We also define the Laplacian $\mL$ as: $\mL = \mD - \adjmatrix$ where  $\mD$ is the degree matrix of the graph.

Given a graph $\gG$ and a graph signal vector $\vs \in \R^{V}$, we can compute the {\em graph signal smoothness}, c.f. Section~\ref{chap3:smoothness}. Finally, we call \emph{label signal} associated with the class $c$ the binary indicator vector $\vs_c$ of elements of class $c$. Hence, $\vs_{c_{\mu}} = 1$ if and only if $\vx_\mu$ is in class $c$ (c.f. Definition~\ref{chap2:label_indicator_vector}).

\subsection{Proposed graph smoothness loss}
We propose to replace the cross-entropy loss with a graph smoothness loss. Consider a fixed metric $\|\cdot\|$. We can compute the distances/similarities between the representations $f(\vx_\mu), \forall \mu \in \trainset$. Using this information, we build a $k$-nearest neighbor graph. 

To generate this graph, denoted $\gG$, we compute the similarity between each representation using an RBF-kernel parameterized by $\alpha$, and then threshold to the closest $k$-neighbors. This leads us to the following $\adjmatrix$: 
\begin{equation}
    \emAdjacency_{\mu,\nu} \neq 0 \Rightarrow \emAdjacency_{\mu,\nu} = \exp{\left(-\alpha\|f(\vx_\mu) - f(\vx_\nu)\|\right)},\forall \mu, \forall \nu\;.
\end{equation}

We can then define our graph smoothness loss as follows.

\begin{definition}[graph smoothness loss]\label{chap5:def-graphsmoothnessloss}
   We call $\emph{graph smoothness loss}$ of $f$ the quantity:
   \begin{eqnarray*}
       \mathcal{L}_{\gG} &=& \vs^{\top} \mL \vs\\
       &=& \hspace{-1.2cm}\underbrace{\sum_{\genfrac{}{}{0pt}{2}{\mathbf{x}_\mu, \mathbf{x}_{\nu}, \mathbf{W}_k[\mu \nu] \neq 0}{ \mathbf{s}_c[\mu] \mathbf{s}_{c}[\nu] = 0, \forall c}}}_{\text{sum over inputs of distinct classes}}{\hspace{-1cm}\exp{\left(-\alpha\|f(\mathbf{x}_{\mu}) - f(\mathbf{x}_{\nu})\|\right)}}
       \;.
   \end{eqnarray*}
\end{definition}

In the following subsection, we motivate the use of this loss.

\subsection{Properties of the graph smoothness loss}
The cross-entropy loss introduced in Equation~\ref{chap2:simplified_ce} aims at mapping inputs of the network to arbitrarily chosen one-hot-bit encoded vectors representing the corresponding classes. Our proposed loss function differs from the cross-entropy loss in three main aspects:
\begin{itemize}
    \item The cross-entropy loss forces a mapping from the input to a single point for each class. This might force the network to considerably distort space, for example in the case where a class is made of several disjoint clusters. The use of $k$-nearest neighbors gives more flexibility to the proposed loss: using a small value of $k$, it is possible to minimize the graph smoothness loss with multiple clusters of points for each class;
    \item The cross-entropy loss requires to arbitrarily choose the outputs of the network, disregarding the dataset and the initialization of the network. In contrast, the proposed loss is only interested in relative positioning of outputs with regards to one another, and can therefore build upon the initial distribution yielded by the network;
    \item To use the cross-entropy loss we are obliged to use an output vector whose dimension is the number of classes of the problem at hand. It is thus required to modify the network to accommodate for new classes (e.g. in an incremental scenario). The dimension of the network output $d$ is less tightly tied to the number of classes with the proposed loss.
\end{itemize}

It is important to note that the capacity and the dimension of the output space need to be bounded by the problem at hand. If the output dimension is too small, it is likely that the network will not be able to converge (i.e., underfit): consider a toy example, where we try to separate $n$ samples so that they are all at the same distance in the output space. It is only possible to suffice this condition if the output dimension is at least $n-1$. On the other hand, if the capacity is too large, we can simply scatter each point in the output space, so that the distance between the image of any two inputs is large, but not significant as it has the same behavior for any input. This relation between the dimension of the output space and the ability of the network to classify is further discussed in the experiments.

\subsection{Experiments}

We evaluate the performance of the proposed loss using three common datasets of image classification: \begin{inlinelist}
    \item CIFAR-10 \item CIFAR-100 \item SVHN
\end{inlinelist}. For each dataset, we follow the same experimental process: \begin{inlinelist}
    \item We define the architecture we are going to use, following networks that are known to provide a good result when using the cross-entropy loss
    \item We train two networks of the same architecture (number of layers, number of features per layer) and hyperparameters (number of epochs, learning rate, gradient descent algorithm,  mini-batch size,  weight decay, weight normalization), but one is trained with  the cross-entropy loss and the other with the proposed graph smoothness loss
    \item We then tune the additional hyperparameters of the proposed loss ($k$, $\alpha$, $d$). When performing classification, we train a simple classifier on top of the network to measure its accuracy.
\end{inlinelist} Note that all input images are normalized before being processed. It is important to keep in mind that by choosing this methodology, we bias the experiments in favor of using the cross-entropy loss, since the chosen architectures have been designed for its use.

The network architecture we use is Resnet-18~\citep{he2016deep}, as previously defined in Chapter~\ref{chap2}. The network is trained for 200 epochs using 100 examples per mini-batch. SVHN and CIFAR-10 networks are trained with SGD, using a learning rate that starts at 0.1 and is divided by 10 at epochs 100 and 150, with a weight decay factor of $10^{-4}$ and a Nesterov momentum of 0.9. On the other hand, CIFAR-100 is trained with the Adam optimizer~\citep{kingma2015adam}, using a learning rate that starts at 0.001 and is divided by 10 at epochs 100 and 150.  Note that due to computational constraints we built a graph for each mini-batch (i.e., graph smoothness is calculated on a graph of 100 vertices that changes at each mini-batch).

In the original version of the chosen architecture, the linear function of the last layer outputs a $C$ dimensional vector on which a softmax function is applied. When using the proposed loss, the linear function outputs a $d$ dimensional vector, where $d$ is an hyperparameter, normalized with respect to the $\normltwo$ norm. We use this normalization to constrain the outputs to remain in a compact subset of the output space. As previously discussed in Section~\ref{methodo}, if we did not normalize the output, and since we use the $\normltwo$ metric to build the graphs in our experiments, the network would likely converge to a trivial solution that would scatter the outputs far away from each other in the output domain, regardless of their class.

\subsubsection{Visualization}
We first compare the embedding obtained using the proposed loss and $d=2$ with the one obtained when putting a bottleneck layer of the same dimension $d=2$ using the cross-entropy loss. Results on CIFAR-10 are depicted in Figure~\ref{chap5:fig-embedding-smoothness}. Exceptionally, for this experiment we do normalize the output of the last layer of the network using batch norm instead of $\normltwo$ norm. This is because using $d=2$ with a $\normltwo$ normalization would reduce the output space dimension to 1, which would likely be too small to allow the training loss to descend to 0. We observe that in the third column of Figure~\ref{chap5:fig-embedding-smoothness}, our method creates clusters whereas the baseline method creates lines. This reflects the choice of the distance metric: our method uses the $\normltwo$ distance, whereas the baseline seems to use the cosine distance instead. Figure~\ref{chap5:fig-embedding-smoothness} shows that training examples are better clustered at the end of the training process when using the proposed loss than with the cross-entropy loss.

\begin{figure}[ht]
    \begin{center}
        \begin{adjustbox}{max width=1.1\linewidth}
            \includegraphics[scale=0.45]{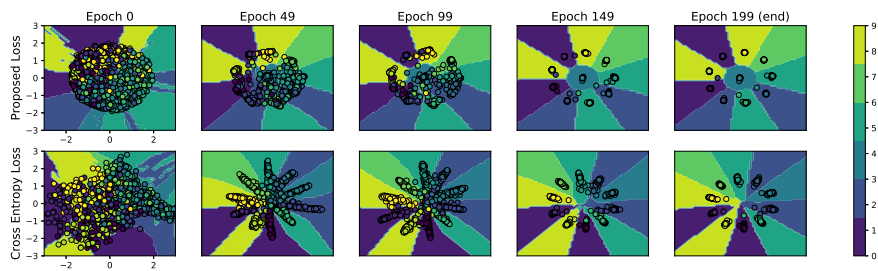}
        \end{adjustbox}
        \caption{Embeddings of CIFAR-10 training set learned using the proposed graph smoothness loss with $d=2$ (top row) compared with the ones obtained using a bottleneck layer and cross-entropy with the same architecture (bottom row). Figure and caption extracted from~\citep{bontonou2019smoothness}. ©2019, IEEE.}\label{chap5:fig-embedding-smoothness}
    \end{center}
\end{figure}    

\subsubsection{Classification}
We evaluate the influence on classification performance of the three hyperparameters of the proposed loss: the number of neighbors $k$ to consider in the similarity graph $\gG$, the number of dimensions $d$ coming out of the network and the scaling parameter $\alpha$ used to define the weights of the graph. When varying $k$, we fix $d$ to be the number of classes and $\alpha = 2$; when varying $d$, we fix $k$ to the maximum value and $\alpha = 2$ and when varying $\alpha$, we fix $d$ to be the number of classes and $k$ to the maximum value. The results are summarized in Figure~\ref{chap5:fig-tests-smoothness}. Note that a 10-NN classifier was used to obtain the accuracy. We observe that the higher $k$ is, the higher the test accuracy is, even if the sensitivity to $k$ is lower when $k$ is larger than the number of classes. As soon as $d$ becomes large enough to accommodate for the number of classes, we observe that the test accuracy starts dropping slowly. Therefore, because using a larger value of $d$ does not seem particularly harmful, applications where the number of classes is unknown (such as in incremental learning) should use a high $d$. Similarly, there is almost no dependence to $\alpha$ as long as its value is small enough. Indeed, when $\alpha$ is large, the loss tends to be close to 0 even if the corresponding distances are still relatively small.

\begin{figure}[ht]
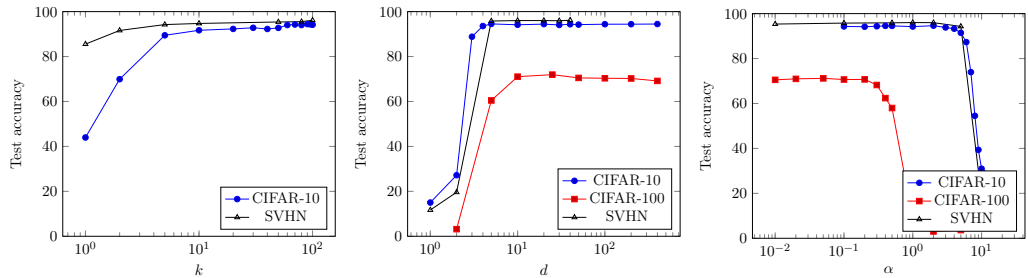

\begin{center}
    \begin{subfigure}[ht]{.3\linewidth}
        \centering
        \begin{adjustbox}{max width=\linewidth}
  \tikzsetnextfilename{chapter5/tikz/testxk}%
  \input{chapter5/tikz/testxk.tex}%

        \end{adjustbox}    
    \end{subfigure}
    \begin{subfigure}[ht]{.3\linewidth}
        \centering
        \begin{adjustbox}{max width=\linewidth}
  \tikzsetnextfilename{chapter5/tikz/testxd}%
  \input{chapter5/tikz/testxd.tex}%

        \end{adjustbox}    
    \end{subfigure}
    \begin{subfigure}[ht]{.3\linewidth}
        \centering
        \begin{adjustbox}{max width=\linewidth}
  \tikzsetnextfilename{chapter5/tikz/testxalpha}%
  \input{chapter5/tikz/testxalpha.tex}%

        \end{adjustbox}    
    \end{subfigure}
   \caption{Test set accuracy as a function of the different parameters $k$, $d$ and $\alpha$. Plots extracted from~\citep{bontonou2019smoothness} ©2019, IEEE.}\label{chap5:fig-tests-smoothness}
\end{center}
\end{figure}

We next evaluate the performance of the graph smoothness loss for classification. To this end, we compare its accuracy to that achieved with optimized network architectures using a cross-entropy loss (CE). We use various classifiers on top of the graph smoothness loss-trained architectures: a $1$-nearest neighbors classifier ($1$-NN), a $10$-nearest neighbors classifier ($10$-NN) and a support vector classifier (SVC) using radial basis functions. The results are summarized in Table~\ref{chap5:table-performance}. We observe that the test error obtained with the proposed loss is close to the CE test error, suggesting that the proposed loss is able to compete in terms of accuracy with the cross-entropy. Interestingly, we do not observe a significant difference in accuracy between the classifiers. Besides, both losses require the same training time.

\begin{table}[ht]
\begin{center}
\caption{Test errors on CIFAR-10, CIFAR-100 and SVHN datasets. The top contains the test error of the optimized network architectures for a cross-entropy loss (CE). The bottom contains the test error of the same network architectures for our proposed graph smoothness loss, associated with three different classifiers. Table and caption extracted from~\citep{bontonou2019smoothness} ©2019, IEEE.}
\label{chap5:table-performance}
\begin{tabular}{c|ccc}
\hline
Loss - Classifier          & CIFAR-10 & CIFAR-100 & SVHN \\
    \Xhline{2\arrayrulewidth}
CE - Argmax                & \textbf{5.06\%}   & \textbf{27.92\%}   &  3.69\%        \\
Proposed      - 1-NN       & 5.63\%   & 29.17\%   &  3.84\%         \\
Proposed      - 10-NN      & 5.48\%   & 28.82\%   &  \textbf{3.34}\%            \\
Proposed      - RBF SVC    & 5.50\%   & 30.55\%   &    3.40\%   \\ \hline 

\end{tabular}
\end{center}
\end{table}

\subsubsection{Robustness}

We now evaluate the robustness of the trained architectures using the robustness bechmark defined in Section~\ref{chap2:robustness_benchmark}. We report the results in Table~\ref{chap5:table-robustness}. We first report the error rate on the clean test set for which we observe a small drop in performance when using the proposed loss. However, this drop is compensated by a better accommodation to deviations of the inputs, as reported by the Mean Corruption Error (MCE) scores (see~\citep{hendrycks2019robustness}). Such a trade-off between accuracy and robustness has been discussed in~\citep{fawzi2018analysis}. For this experiment, we fixed $k$ to its maximum value, $d=200$, $\alpha=2$ and we used 10-NN as a classifier when using the graph smoothness loss.

\begin{table}[ht]
\centering
\caption{Robustness comparison on the 15 corruptions benchmarks from~\citep{hendrycks2019robustness} on the CIFAR-10 dataset. Table and caption extracted from~\citep{bontonou2019smoothness} ©2019, IEEE.}
\label{chap5:table-robustness}
\begin{tabular}{lc|cc}
\hline
Method               & Clean test error & MCE   & relative MCE \\         \Xhline{2\arrayrulewidth}
Cross-entropy              & {\bf 5.06}\%           & 100   & 100          \\
Proposed             & 5.60\%           & {\bf 95.28} & {\bf 90.33}  \\ \hline     
\end{tabular}
\end{table}

In the previous paragraphs, we have introduced a loss function that consists in minimizing the graph smoothness of label signals on similarity graphs built at the output of a deep learning architecture. We discussed several interesting properties of this loss when compared to using the classical cross-entropy. We have shown empirically that the proposed loss can reach similar performance as cross-entropy, while providing more degrees of freedom and increased robustness to deviations of the inputs.

\section{Controling DNN smoothness to improve robustness}\label{chap5:laplacian}

In the previous section we concentrated in the effects of changing the whole training objective of a DNN to the minimization of graph signal signal smoothness. Now, we go back to the ideas presented in Section~\ref{chap5:smoothness} and study the robustness effect of controling the evolution of DNN smoothness. The contents described in this section were made available in our archival contribution~\citep{lassance2018laplacian}, and were the starting stone for defining the robustness metric in~\citep{lassance2019robustness}.

As we have previously discussed, the ability of DNNs to achieve good generalization is closely related to the amount of data available. This strong dependency on data may lead to selection of biased features of the training dataset, resulting in a lack of robustness in classification performance. In this work robustness has been defined to be the ability of a classifier to infer correctly even when the inputs (or the parameters of the classifier) are subject to perturbations. These perturbations can be due to general factors --such as noise, quantization of inputs or parameters, and adversarial attacks-- as well as application specific ones --such as the use of a different camera lens, brightness exposure, or weather, in an imaging task, c.f. Section~\ref{chap2:robustness} for a more in-depth discussion.

In this section we propose to introduce a regularizer that penalizes large deformations of the class boundaries throughout the network architecture, independently of the types of perturbations that we expect to face when the system is deployed. It also enforces a large margin $r$ (i.e., mid-distance between examples of distinct classes) at each layer of the architecture. Note that we have already discussed some of the properties and results of this regularizer in Section~\ref{chap2:robustness}, but here we detail its methodology and provide experiments to support our claims. 

To understand the intuition behind our proposed regularizer, first recall that networks are typically trained with the objective of yielding zero error for the training set. If error on the training set is (approximately) zero then any two examples with different labels can be separated by the network, even if these examples are close to each other in the original domain. This means that the network function can create significant deformations of the space (i.e., small distances in the original domain map to larger distances in the final layers) and explains how an adversarial attack with small changes to the input can lead to class label changes. Our proposed regularizer penalizes big changes at the boundaries between classes. By forcing boundary deformations to evolve smoothly across the architecture, and at the same time by maintaining a large margin, the proposed regularizer therefore favors smooth variations. We argue that favoring smooth variations leads to better robustness, as per Definition~\ref{chap2:def_robustness}. We will empirically demonstrate this claim on classical vision datasets.

The proposed regularizer is based on a series of graphs, one for each layer of the DL architecture, where each graph captures the similarity between training examples given their intermediate representation at that layer. Our regularizer favors small changes, from one layer to the next, in the distances between pairs of examples in different classes. Note that the distance between any two examples at a certain layer depends on their positions in the original domain and the network function applied up to that layer. Thus,  constraints on the distances lead to constraints on the parameters of the network function. It achieves so by penalizing large changes in the smoothness (computed using the Laplacian quadratic form) of the class indicator vectors (viewed as ``graph signals''). As a result, the margin is kept almost constant across layers, and the deformations of space are controlled at the boundary regions, as illustrated in Figure~\ref{chap5:fig-examples_regularizer}. This regularizer draws heavily from the analysis derived in Section~\ref{chap5:smoothness}, and uses the robustness definition that was previously introduced in Section~\ref{chap2:robustness}.

\begin{figure}[ht]
    \centering
    \begin{adjustbox}{max width=\columnwidth}    
  \tikzsetnextfilename{chapter5/tikz/example_regularizer}%
  \input{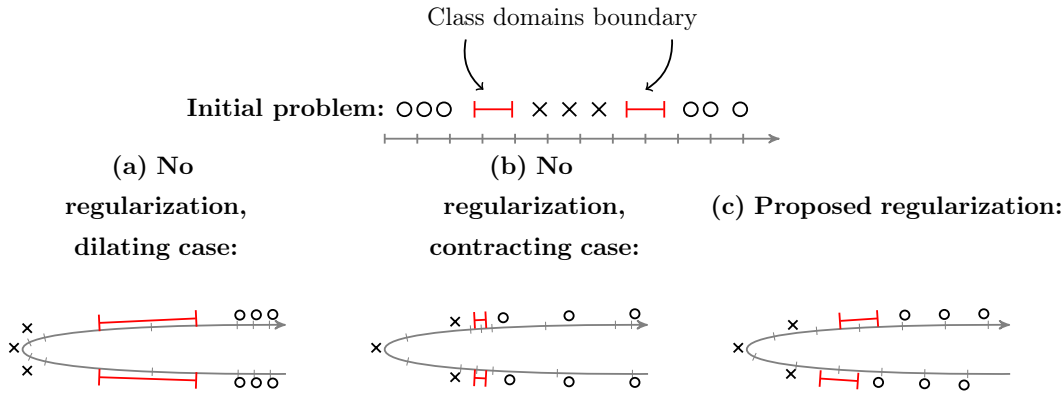}%

    \end{adjustbox}
    \caption{Illustration of the effect of our proposed regularizer. In this example, the goal is to classify circles and crosses (top).  Without use of regularizers (bottom left), the resulting embedding may considerably stretch the boundary regions. Consequently the risk is to obtain sharp transitions in the network function (that would correspond to a large value of $\alpha$ in Equation~(\ref{chap2:eq_robustness})). Another possible issue would be to push inputs closer to the boundary (bottom center), thus reducing the margin (that would correspond to a small value of $r$ in Equation~(\ref{chap2:eq_robustness})). Forcing small variations of smoothness of label signals (bottom right), we ensure the topology is not dramatically changed in the boundary regions. Figure and caption extracted from~\citep{lassance2018laplacian}.}
    \label{chap5:fig-examples_regularizer}
\end{figure}

Another example of prior work that is related to our regularizer is~\citep{svoboda2019peernets} where the authors exploit graph convolutional layers. This leads to smoothing the latent representations of the inferred images using similar images from the training set, in order to increase the robustness of the network. Note that this could be described as a denoising of the inference (test) image, using the training ones. However this differs from the proposed regularizer as our work focuses on generating a smooth network function and their work focuses on combining inputs in order to generate a smooth network function.

In the following of this section we first present a quick recall of our robustness definition and then introduce our regularizer that enforce this property using similarity graphs. We then demonstrate, using readily-available image classification datasets, the robustness of the proposed regularizer to the following common perturbations:
\begin{inlinelist}
\item noise~\citep{hendrycks2019robustness,mallat2016understanding}, for which we show reductions in relative error increase
\item adversarial attacks~\citep{szegedy2014intriguing,goodfellow2014adversarial}, for which the median defense radius ~\citep{moosavi2016deepfool} is increased by 50\% in comparison with the baseline and by 12\% in comparison with another method in the literature~\citep{cisse2017parseval}
\item implementation defects, which result in only approximately correct computations~\citep{hubara2017quantized}, for which we increase the median accuracy by 48\% relative to the baseline and 26\% relative to another method in the literature~\citep{cisse2017parseval}.
\end{inlinelist}

\subsection{Methodology}

In this subsection, we first recall our robustness definition and then present the proposed regularizer.

\subsubsection{Robustness definition}
\label{motivation2}

Recall that a deep neural network architecture can be entirely described by its associated ``network function''. In most cases, the network function $f$ receives an input $\vx$ and outputs a class-wise classification score $f(\vx)$. Typically this output is a vector with as many coordinates as the number of classes in the problem, where the highest valued coordinate is the decision of the network (i.e $\arg\max$ classifier). This function is constructed via the composition of multiple intermediate functions $f^\ell$:
\begin{equation}
    f = f^{L}\circ f^{L-1} \circ \dots \circ f^1,
\end{equation}
where each function $f^\ell$ is highly constrained, typically as the concatenation of a parameter-free nonlinear function with a parameterized linear function.

The function $f$ is typically obtained based on a very large number of parameters, which are tuned during the learning phase. During this phase, a loss function is minimized over a set of training examples using a variant of the stochastic gradient descent algorithm. At the end of the training process, each training example is associated through $f$ with a vector whose largest value is the actual class of that example, leading to an accuracy close to 100\% on the training set. Importantly, the loss function usually targets a specific margin in the output domain. For example, when using the classical cross-entropy loss, the loss function is minimized when the output of the training examples are the one-hot-bit vectors of their corresponding class~\citep{dlbook}, which corresponds to a margin in the output domain of about $\sqrt{2}/2$ for the $\normltwo$ norm.

We use the $\alpha$-robust concept introduced in Definition~\ref{chap2:def_robustness}. Recall that we can say that a network is $\alpha$-robust if $f$ is locally $\alpha$-Lipschitz within a radius $r$ of any point in domain $R$. Obviously, we would like to obtain a function $f$ that is $\alpha$-robust for any valid input. But since we only have access to training samples, we only enforce the property over the training set. 

This definition captures a compromise between margin (represented by $r$) and slope (represented by $\alpha$) of the network function. This is in contrast to other works~\citep{cisse2017parseval} where robustness is directly linked to the Lipschitz constant of the network function. The main motivation for introducing this weaker definition of robustness is that we do not want network functions to be contractive \emph{everywhere}. Indeed, if all mappings are contractive everywhere we cannot hope to separate some samples in different classes. A more in-depth discussion of this is available in Section~\ref{chap2:robustness}.

In what follows, we introduce regularizers that enforce this property using similarity graphs.

\subsubsection{Intermediate representation graphs}

First let us recall the concept of intermediate representation graphs and its notations. Consider a deep learning network architecture. Such a network is obtained by assembling layers of various types. A layer can be represented by a function $f^\ell: \vx^\ell\mapsto \vx^{\ell+1}$ where $\vx^\ell$ is the intermediate representation of the input at layer $\ell$. Assembling can be achieved in various ways: composition, concatenation, sums, etc so that we obtain a global function $f$ that associates an input tensor $\vx$ to an output tensor $\vy = f(\vx)$. In practice a batch of $b$ inputs $\mathcal{X}=\{\mathbf{x}_1,\dots,\mathbf{x}_b\}$ is processed concurrently.

Given a (meaningful) similarity measure $\text{sim}$ on tensors, we can define the similarity matrix of the intermediate representations at layer $\ell$ as: 
\begin{equation}
    \emAdjacency^\ell_{i,j} = \text{sim}(\vx_i^{\ell+1}, \vx_j^{\ell+1}), \forall 1\leq i,j \leq b,
\end{equation} where $\emAdjacency^{\ell}[i,j]$ denotes the element at line $i$ and column $j$ in $\adjmatrix^{\ell}$. In our experiments we mostly focus on the use of cosine similarity, which is widely used in computer vision. It is often the case that the output $\vx^{\ell+1}$ is obtained right after using a ReLU function, that forces all its values to be nonnegative, so that all values in $\adjmatrix^{\ell}$ are also nonnegative. We then use $\adjmatrix^{\ell}$ to define a weighted graph $\gG^\ell = \langle \sV, \sE^\ell\rangle$, where $\sV = \{1,\dots,b\}$ is the set of vertices and $\sE$ the set of edges defined with $\adjmatrix$.

\subsubsection{Smoothness of label signals}

Given a weighted graph: $\gG^\ell$, the Laplacian of $\gG^\ell$ is the matrix: 
\begin{equation}
\mL^\ell = \mD^\ell - \adjmatrix^\ell \;.
\end{equation}
Consider a graph signal $\vs \in \R^b$, we define $\hat{\vs}$ the Graph Fourier Transform (GFT) of $\vs$ on $\gG^\ell$ as~\citep{shuman2013emerging}:  
\begin{equation}
\hat{\vs} = \mF^\top \vs. 
\end{equation}
Assume the order of the eigenvectors is chosen so that the corresponding eigenvalues are in ascending order. If only the first few entries of $\hat{\vs}$ are nonzero then $\vs$ is said to be low frequency (i.e., smooth) on the graph. In the extreme case where only the first entry of $\hat{\mathbf{s}}$ is nonzero we have that $\mathbf{s}$ is constant (maximum smoothness). Recall that the smoothness $\sigma^\ell(\vs)$ of a signal $\vs$ can be measured using the Laplacian quadratic form:
\begin{equation}
    \sigma^\ell(\vs) = \vs^\top \mL^\ell \vs = \sum_{i,j = 1}^{b}{\emAdjacency^\ell_{i,j}(\vs_i - \vs_j)^2}\;.
\end{equation}

In this section, we are particularly interested in smoothness of the label signals. Label signals are also called binary label indicator vector, as we have previously defined in Definition~\ref{chap2:label_indicator_vector}. Recall that when we are dealing with binary signals, the smoothness of the signal is given by the sum of similarities between examples in distinct classes (since $\vs_i - \vs_j$ is zero when $i$ and $j$ have the same label). Thus, a total smoothness of 0 means that all examples in distinct classes have 0 similarity. 

Next we introduce a regularizer that limits how much $\sigma^\ell$ can vary from one layer to the next, thus leading to a network that is more inline with Definition~\ref{chap2:def_robustness}. This will be shown later to improve robustness in Section~\ref{experiments}.

\subsubsection{Proposed regularizer}

\paragraph{Definition: }

We propose to measure the deformation induced by a given layer $\ell$ by computing the difference between label signal smoothness before and after the layer for all labels:
\begin{equation}
    \delta_\sigma^\ell = \sum_c{\left|\sigma^\ell(\vs_c) - \sigma^{\ell-1}(\vs_c)\right|}.
\end{equation}

These quantities are used to regularize modifications made by each of the layers during the learning process. The pseudo-code of Algorithm~\ref{chap5:loss-algo} describes how we use the proposed regularizer to compute the loss.

\paragraph{Illustrative example:}

In Figure~\ref{chap5:fig-examples_regularizer} we depicted a toy illustrative example to motivate the proposed regularizer. We consider here a one-dimensional two-class problem. To linearly separate circles and crosses, it is necessary to group all circles. Without regularization the resulting embedding is likely to either considerably increase the distance between examples in different classes (case (a)), thus producing sharp transitions in the network function, or to reduce the margin (case (b)). In contrast, by penalizing large variations of the smoothness of label signals (case (c)), the average distance between examples in different classes must be preserved in the embedding domain, resulting in a more precise control of distances within the boundary region. 

\begin{remark}
Since we only consider label signals, we solely depend on the similarities between examples of distinct classes. As such, the regularizer only focuses on the boundary, and does not vary if the distance between examples of the same label grows or shrinks.
\end{remark}

\begin{remark}
Compared with~\citep{cisse2017parseval}, there are key differences that characterize the proposed regularizer:
\begin{enumerate}
    \item Only pairwise distances between examples are taken into account. This has the effect of controlling space deformations only in the directions of training examples;
    \item The network is forced to maintain a minimum margin by keeping the smoothness small at each layer of the architecture, thus controlling both contraction and dilatation of space at the boundary. This is illustrated in Figure~\ref{chap5:fig-examples_regularizer}, where~\citep{cisse2017parseval} is represented by b) and our method by c);
    \item The proposed criterion is an average (sum) over all distances, rather than a stricter criterion (e.g. maintaining a small Lipschitz constant), which would force each pair of vectors $(\vx_i, \vx_j)$ to obey the constraint.
\end{enumerate}
\end{remark}

In summary, by enforcing small variations of smoothness across the layers of the network, the proposed regularizer maintains a large enough $r$ so that Equation~(\ref{chap2:eq_robustness}) can hold, while also controlling dilatation. Combining it with Parseval~\citep{cisse2017parseval} would allow for a better control of the $\alpha$ parameter in the other directions of the input space.

\begin{algorithm}
\caption{Loss function of the regularized network}
\label{chap5:loss-algo}
\textbf{Inputs:} \\
$\vx$: list of all the representations of the network.\\
$\text{ReLUs}$, the list containing the positions of all the ReLU activations on $f$.\\
$\vy$, the output of the network \\
$\vs$, the label signal of the batch, i.e.,  the ground truth labels of the examples of the batch \\
$m$, the power of the Laplacian for which we wish to compute the smoothness; \\
$\gamma$, the scaling coefficient of the regularizer loss.

\begin{algorithmic}

\Procedure{Loss}{$\vx, \text{ReLUs},\vy,\vs,m,\gamma$}

\For{$\ell \in \text{ReLUs}$}
\State $\sigma^\ell \gets$ Smoothness$(\vx^\ell,\vs,m)$
\EndFor

\State        $\Delta \gets \frac{\sum_{\ell \in \text{ReLUs}} |\sigma^\ell - \sigma^{\ell-1}|}{|| \text{ReLUs} ||-1}$ 
\State        \Return CategoricalCrossEntropy$(\vs,\vy) + \gamma  \Delta$  
\EndProcedure

\Procedure{Smoothness}{$\vx^\ell,\vs,m$}

\State        $\adjmatrix^\ell \gets$ Pairwise similarity of $\vx^\ell$ (we use cosine similarity in our work)
\State        $\mD^\ell \gets$ Diagonal degree matrix of $\adjmatrix^\ell$
\State        $\mL^\ell \gets  \mathbf{D}^\ell-\mathbf{M}^\ell$ 
\State        $\mathbf{\sigma}^\ell \gets$ Trace$(\vs^\top (L^\ell)^m \vs)$
\State        \Return $\mathbf{\sigma}^\ell$

\EndProcedure

\end{algorithmic}
\end{algorithm}

\subsection{Experiments}
\label{experiments}

In the following subsections we evaluate the proposed method using various tests. We use the well known CIFAR-10 dataset~\citep{krizhevsky2009learning} as a first benchmark and we demonstrate that our proposed regularizer can improve robustness as defined in Section~\ref{chap2:robustness}. 

In summary, in Section~\ref{chap5:subsection_alphar} we first verify that the proposed regularizer favors Definition~\ref{chap2:def_robustness}. We then show in Section~\ref{chap5:no_data_augmentation} that by using the proposed regularizer we are able to increase robustness for random perturbations and weak adversarial attacks. In Section~\ref{chap5:subsection_advdataaug}, we challenge our method on more competitive benchmarks. Finally, in Section~\ref{chap5:other_datasets} we extend the analysis to CIFAR-100~\citep{krizhevsky2009learning} and Imagenet32x32~\citep{chrabaszcz2017downsampled} to validate the generality of the method. These experiments demonstrate that DNNs trained with the proposed regularizer lead to improved robustness.

To measure accuracy, we average over 10 runs each time, unless mentioned otherwise. In all reports, $P$ stands for Parseval~\citep{cisse2017parseval} trained networks, $R$ for networks trained with the proposed regularizer and $V$ for vanilla (i.e. baseline) networks. The corresponding code is available at \url{https://github.com/cadurosar/laplacian_networks}.

\subsubsection{Robustness of trained architectures}
\label{chap5:subsection_alphar}

First we verify that the proposed regularizer improves robustness as defined in Section~\ref{chap2:robustness}. For various values of $r$, we estimate $\alpha_{\min}(r) = \arg\min_{\alpha}{\{f\in Robust_{\alpha}(r)\}}$. We use 1000 training examples and generate 100 uniform noises to estimate $\alpha_{\min}(\cdot)$. Results are shown in Figure~\ref{chap5:fig-ralpha}. We observe that networks trained with the proposed regularizer allow for smaller $\alpha$ values when the radius $r$ increases. The Parseval method achieves better (smaller) Lipschitz constant than Vanilla, as suggested by the large values of $r$. However, we observe that $\alpha_{\min}$ grows fast when using Parseval, suggesting that sharp transitions are allowed in the vicinity of trained examples.

\begin{figure}[ht]
 \begin{center}
  \tikzsetnextfilename{chapter5/tikz/ralpha}%
  \input{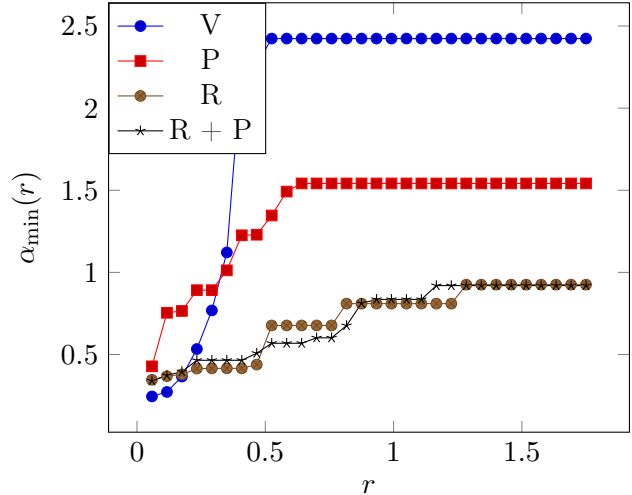}%

    \caption{Estimations of $\alpha_{\min}(r)$ obtained for different radius $r$ over training examples. The proposed regularizer allows for smaller $\alpha$ values when $r$ increases. Figure and caption extracted from~\citep{lassance2018laplacian}.}    \label{chap5:fig-ralpha}
    \end{center}
\end{figure}

\subsubsection{Experiments on perturbations and adversarial attacks}
\label{chap5:no_data_augmentation}
In this subsection we verify the ability of the proposed regularizer to increase robustness, while retaining acceptable accuracy on the clean test set, on the CIFAR-10 dataset without any type of data augmentation.

\paragraph{Clean test set}

Before checking the robustness of the network, we first test the performance on clean examples. In the second column of Table~\ref{chap5:tab-corruption}, we show the baseline accuracy of the models on the clean CIFAR-10 test set (no perturbation is added at this point). These experiments agree with the claim from~\citep{cisse2017parseval} where the authors show that they are able to increase the performance of the network on the clean test set. We observe that the proposed method leads to a minor decrease of performance on this test. However, we see in the following experiments that this is compensated by an increased robustness to perturbations. 
Such a trade-off between robustness and accuracy has already been discussed in the literature~\citep{fawzi2018analysis}.

\begin{table}[ht]
\begin{center}
\caption{Network mean Relative Error Inflation (mREI) under different types of perturbation. Bottom line represents the corresponding median Cosine Distance (mCD) (at the highest perturbation severity) between corrupted and clean images. Table and caption extracted from~\citep{lassance2018laplacian}.}\label{chap5:tab-corruption}
\begin{adjustbox}{max width=\linewidth}

{\setlength\tabcolsep{1.3pt}%
\begin{tabular}{@{}l c c| c c c | c c c c | c c c  c | c c c c@{}}
\multicolumn{3}{c}{} & \multicolumn{3}{c}{Noise} & \multicolumn{4}{c}{Blur} & \multicolumn{4}{c}{Weather} & \multicolumn{4}{c}{Digital} \\ \hline 
\scriptsize{Network} & \scriptsize{Clean set}  & \multicolumn{1}{c|}{\,\textbf{\scriptsize{$mREI$}}\,} & \scriptsize{Gauss.}
    & \scriptsize{Shot} & \scriptsize{Impulse} & \scriptsize{Defocus} & \scriptsize{Glass} & \scriptsize{Motion} & \scriptsize{Zoom} & \scriptsize{Snow} & \scriptsize{Frost} & \scriptsize{Fog} & \scriptsize{Bright} & \scriptsize{Contrast} & \scriptsize{Elastic} & \scriptsize{Pixel} & \scriptsize{JPEG}\\ \hline 
Vanilla (V)  & 11.9\% & 0.00  & 0.00  & 0.00  & 0.00 & 0.00  & 0.00  & 0.00  & 0.00  & 0.00  & 0.00  & 0.00  & 0.00 & 0.00 & 0.00  & 0.00  & 0.00  \\
Parseval (P)  & 10.3\% & 0.29  & 0.71  & 0.48  & 0.57 & 0.10  & 1.01  & 0.15  & 0.11  & 0.16  & 0.17  & -0.02 & 0.04 & 0.13 & 0.14  & 0.25  & 0.28  \\
Regularizer (R)  & 13.2\% & -0.29 & -1.12 & -0.86 & 0.10 & -0.03 & -0.65 & -0.09 & -0.19 & -0.30 & -0.61 & 0.17  & 0.01 & 0.40 & -0.15 & -0.50 & -0.50 \\
P and R & 12.8\% & -0.35 & -1.33 & -1.00 & 0.05 & -0.09 & -0.75 & -0.18 & -0.31 & -0.41 & -0.67 & 0.13  & 0.04 & 0.48 & -0.17 & -0.47 & -0.53 \\ \hline

\multicolumn{3}{c|}{mCD $10^{-3}$} & 18 & 16 & 37 & 5 & 24 & 15 & 17 & 15 & 20 & 51  & 14 & 57 & 14 & 6 & 3 \\
\end{tabular}}
\end{adjustbox}
\end{center}
\end{table}

\paragraph{Perturbation robustness}

In order to assess the effectiveness of the various methods when subject to perturbations,  we use the benchmark proposed in~\citep{hendrycks2019robustness}, and previously described in Section~\ref{chap2:robustness_benchmark}. The benchmark consists of 15 different perturbations, with 5 levels of severity each (note that they are referred to as ``corruptions'' in~\citep{hendrycks2019robustness}). Perturbations test the robustness of the network to noise when compared to its clean test set performance.

In more details, we are interested in the mean Relative Error Inflation (mREI). To define it, consider $E^{\texttt{per},\texttt{sev}}_{\texttt{net}}$ the error rate of a network \texttt{net} (\texttt{V},\texttt{P},\texttt{R} or \texttt{P+R}), under perturbation type \texttt{per} and severity \texttt{sev}. Denote $E_{\texttt{net}}$ the error rate of the network $\texttt{net}$ on the clean set. We first define Error Inflation (EI) as: $$EI_\texttt{net}^{\texttt{per},\texttt{sev}} = \frac{E^{\texttt{per},\texttt{sev}}_\texttt{net}}{E_\texttt{net}}.$$ Then the Relative Error Inflation REI is defined as: $$REI_\texttt{net}^{\texttt{per},\texttt{sev}}= EI_\texttt{net}^{\texttt{per},\texttt{sev}} - EI_\texttt{V}^{\texttt{per},\texttt{sev}}.$$ Finally, mREI is obtained by averaging over all severities. Note that this is different from the traditional MCE metric, but we believe that this metric is more inline with our objective here. 

The results are described in~\ref{chap5:tab-corruption} for the CIFAR-10 dataset. The raw error rates under each type of perturbations can be found in the original paper. We observe that Parseval alone is not able to help with the mREI, despite reducing the clean set error. On the other hand, the proposed regularizer and its combination with Parseval training decreases the clean set accuracy but increases the relative performance under perturbations by a significant amount.

This experiment supports the fact that the proposed regularizer can significantly improve robustness to most types of perturbations introduced in~\citep{hendrycks2019robustness}. It is worth pointing out that this finding does not hold for Impulse Noise, Fog, and Contrast. Looking more into details, we observe that Impulse noise shifts some values on the image to either its maximum possible value or the minimum possible value, while Fog and Contrast perform a re-normalization of the image. In those cases perturbations have the effect of creating noisy inputs that are far away (in terms of the cosine distance) from the original images, as supported by the last line of the table. This is in contrast to the other types of perturbations in the experiment. Because they can be far away, these perturbations do not fulfill Definition~\ref{chap2:def_robustness}, where there is a maximum radius $r$ for which robustness is enforced around the examples. In other words, our robustness definition is focusing on small deviations/distances as those are more likely to characterize noise (i.e., we focus on distances that are too small to change the class of the image). 

\paragraph{Adversarial Robustness}

We next evaluate robustness to adversarial inputs, which are specifically built to fool the network function. Such adversarial inputs can be generated and evaluated in multiple ways. Here we implement three approaches: \begin{inlinelist}
\item a mean case of adversarial noise, where the adversary can only use one forward and one backward pass to generate the perturbations \item a worst case scenario, where the adversary can use multiple forward and backward passes to try to find the smallest perturbation that will fool the network \item a compromise between the mean case and the worst case, where the adversary can do a predefined number of forward and backward passes with a perturbation threshold limit.
\end{inlinelist} 

For the first approach, we add the scaled gradient sign (FGSM attack) to the input~\citep{kurakin2017adversarial}, so that we obtain a target SNR of 33. This is inline with previous works~\citep{cisse2017parseval}. Obtained results are introduced in the left and center plots of Figure~\ref{chap5:fig-advNoise1}. In the left plot the noise is added after normalizing the input, whereas on the middle plot it is added before normalizing it. As with the perturbation tests, a combination of the Parseval method and our proposed approach yields the most robust architecture.

\begin{figure}[t]
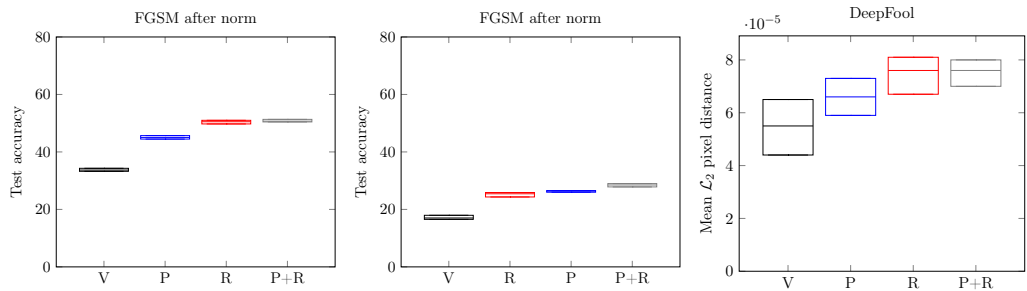

    \begin{center}
        \begin{subfigure}[ht]{0.3\linewidth}
            \centering
            \begin{adjustbox}{max width=\linewidth}
  \tikzsetnextfilename{chapter5/tikz/adv_noise1_FGSMAfter}%
  \input{chapter5/tikz/adv_noise1_FGSMAfter.tex}%

            \end{adjustbox}                
        \end{subfigure}
        \begin{subfigure}[ht]{0.3\linewidth}
            \centering
            \begin{adjustbox}{max width=\linewidth}
  \tikzsetnextfilename{chapter5/tikz/adv_noise1_FGSMBefore}%
  \input{chapter5/tikz/adv_noise1_FGSMBefore.tex}%

            \end{adjustbox}                
        \end{subfigure}
        \begin{subfigure}[ht]{0.3\linewidth}
            \centering
            \begin{adjustbox}{max width=\linewidth}
  \tikzsetnextfilename{chapter5/tikz/adv_noise1_DeepFool}%
  \input{chapter5/tikz/adv_noise1_DeepFool.tex}%

            \end{adjustbox}                
        \end{subfigure}
        \caption{Robustness against an adversary measured by the test set accuracy under FGSM attack in the left and center plots and by the mean $\mathcal{L}_2$ pixel distance needed to fool the network using DeepFool on the right plot. Figure and caption extracted from~\citep{lassance2018laplacian}.}\label{chap5:fig-advNoise1}
    \end{center}
\end{figure}

In regards to the second approach, where a worst case scenario is considered, we use the Foolbox~\citep{rauber2017foolbox} implementation of DeepFool~\citep{moosavi2016deepfool}. Due to time constraints we sample only $\frac{1}{10}$ of the test set images for this test. The conclusions we can draw are similar (right plot of Figure~\ref{chap5:fig-advNoise1}) to those obtained for the first adversarial attack approach.
Finally, for the third approach we use the PGD (Projected Gradient Descent) attack introduced in~\citep{madry2018towards}. PGD is an iterative version of FGSM, which loops for a maximum number of $it$ iterations. For each iteration it moves by a distance of $step$ in the direction of the gradient, provided it does not move away from the original image by a distance greater than $\epsilon$.
Our experiments, described in Table~\ref{chap5:PGDAttack-weak}, show that the proposed regularizer increases robustness against a PGD attack, for an epsilon corresponding to an SNR of about 33 ($it=20,step=0.002,\epsilon=0.01$).

\begin{table}[ht]
\begin{center}
\caption{Median test set accuracy on the CIFAR-10 dataset against the PGD attack. Table and caption extracted from~\citep{lassance2018laplacian}.}
\begin{tabular}{l|cc}
\hline
Model          & PGD Accuracy \\ \Xhline{2\arrayrulewidth}
V              & 1.18\%  \\                                                     
P              & 1.72\%      \\                                                
R              & 5.2\% \\
P+R             & \textbf{5.6\%}      \\ 
\Xhline{2\arrayrulewidth}

\end{tabular}
\label{chap5:PGDAttack-weak}
\end{center}
\end{table}

A common pitfall in evaluating robustness to adversarial attacks comes from the fact the gradient of the architecture can be masked due to the introduced method. As a consequence, generated attacks become weaker compared to those on the vanilla architecture. So, to further verify that the obtained results are not only due to gradient masking, we perform tests with black box FGSM, where the target attacked network is not the same as the source of the adversarial noise. This way, all networks are tested against the same attacks.

For this test we continue to use an SNR of about 33 with the FGSM method. We choose the network with the best performance for each of the tested methods. The results are depicted in Table~\ref{black-box}. In our experiments, we found that the combination of our method with Parseval is the most robust to noise coming from other sources. This demonstrates that the improvements are not caused by gradient masking, but are caused by the increased robustness of the proposed method and Parseval's. Interestingly, the noise created by both Parseval and our method did not challenge the other methods as well as the one created by Vanilla, justifying a posteriori the interest of this experiment.

\begin{table}[ht]
\begin{center}
\caption{Comparison of CIFAR-10 test set accuracy under the black box FGSM attack. The most robust target for a given source is bolded, while the strongest source for a target is in italic. Table and caption extracted from~\citep{lassance2018laplacian}.}
\begin{tabular}{l|c|c|c|c}
\hline
\multirow{2}{*}{Target} & \multicolumn{4}{c}{Source}                                                \\ 
                        & V              & P              & R                       & P+R            \\ \hline
V             & X              & \textit{60.74} & 61.49                   & 72.51          \\ \hline
P            & \textit{57.82} & X              & 68.21                   & \textbf{73.87} \\ \hline
R         & \textit{69.72} & 74.96          & X                       & 73.56          \\ \hline
P+R                     & \textbf{75.35} & \textbf{76.11} & \textit{\textbf{70.22}} & X              \\ \Xhline{2\arrayrulewidth}

\end{tabular}
\label{black-box}
\end{center}
\end{table}

\paragraph{Robustness to parameter and activation noises}

In a third series of experiments we aim at evaluating the robustness of the architecture to noise on parameters and activations. We consider two types of noises: \begin{inlinelist} \item erasures of the memory (dropout) \item quantization of the weights~\citep{hubara2017quantized}.\end{inlinelist} 

In the dropout case, we compute the test set accuracy when the network has a probability of either $25\%$ or $40\%$ of dropping an intermediate representation value after each block of computation in the architecture. We average over a run of $40$ experiments. Results are depicted in the left and center plots of Figure~\ref{chap5:dropout_quantized}. It is interesting to note that the Parseval trained functions seem to collapse as soon as we reach $40\%$ probability of dropout, providing an average accuracy smaller than the vanilla networks. In contrast, the proposed method is the most robust to these perturbations.

\begin{figure}[t]
    \begin{center}
        \begin{adjustbox}{max width=\linewidth}
  \tikzsetnextfilename{chapter5/tikz/dropout_quantized}%
  \input{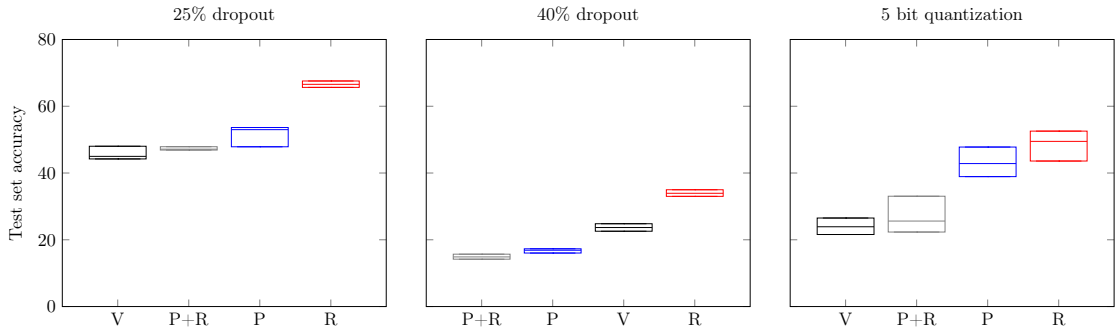}%

        \end{adjustbox}        
        \caption{CIFAR-10 test set accuracy under different types of implementation related noise. Figure and caption extracted from~\citep{lassance2018laplacian}.}\label{chap5:dropout_quantized}
    \end{center}
\end{figure}

For the quantization of the weights, we aim at compressing the network size in memory by a factor of 6. We therefore quantize the weights using 5 bits (instead of 32) and re-evaluate the test set accuracy. The right plot of Figure~\ref{chap5:dropout_quantized} shows that the proposed method is providing a better robustness to this perturbation than the tested counterparts.

Overall, these experiments confirm previous ones in the conclusion that the proposed regularizer obtains the best robustness compared to Parseval and Vanilla architectures.

\subsubsection{Experiments on challenging benchmarks}
\label{chap5:subsection_advdataaug}

In this subsection we verify the ability of the proposed regularizer to increase robustness on the CIFAR-10 dataset while being combined with recent techniques of adversarial data augmentation. This is important as those methods are seen as the state of the art for adversarial robustness. We recall that adversarial data augmentation consists in augmenting the training set during the training stage by using the same kind of attacks as those described in the last subsection. We refer to techniques using adversarial data augmentation using the letter A.

\paragraph{Tests with FGSM adversarial data augmentation}

We first perform experiments with adversarial data augmentation as suggested in~\citep{kurakin2017adversarial}. To be more precise we use the method they advise which is called ``step1.1'' using $\epsilon = \frac{8}{255}$. 

A first test consists in measuring the accuracy of these methods when the test set inputs are modified with additive Gaussian noise with various SNRs. As expected, we observe in Figure~\ref{chap5:fig-gaussian-adv} that training with adversarial examples helps in this case, as it adds more variation to the training set. Yet it reduces the accuracy on the clean set (left plot). Note that combining our method with adversarial training results in the best median accuracy.

\begin{figure}[ht]
    \begin{center}
        \begin{adjustbox}{max width=\linewidth}
  \tikzsetnextfilename{chapter5/tikz/gaussian-adv}%
  \input{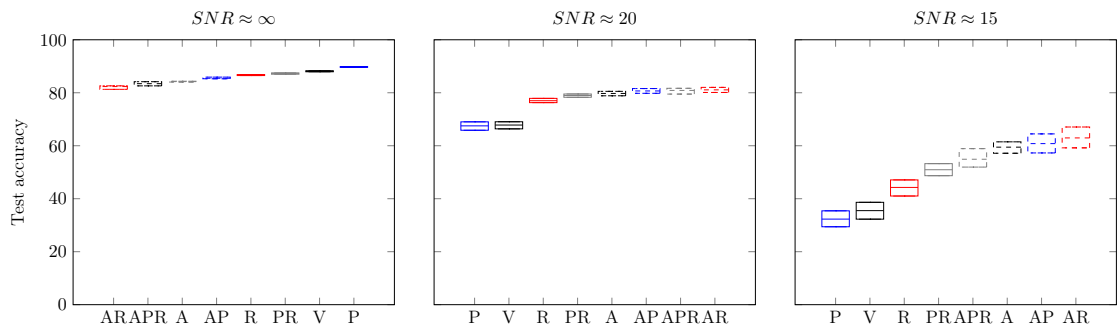}%

        \end{adjustbox}
        \caption{Test set accuracy under Gaussian noise with varying Signal-to-Noise Ratio (SNR). Figure and caption extracted from~\citep{lassance2018laplacian}.}\label{chap5:fig-gaussian-adv}
    \end{center}
\end{figure}

About robustness to adversarial attacks, the obtained results are depicted in Figure~\ref{chap5:fig-adversarialNoise-adv}. We observe that adding FGSM adversarial training does not generalize well to other types of attack (which is readily seen in the literature~\citep{madry2018towards}). Overall, the models using the proposed regularizer are the most robust again.

\begin{figure}[ht]
    \begin{center}
        \begin{adjustbox}{max width=\linewidth}
  \tikzsetnextfilename{chapter5/tikz/adversarialnoise-adv}%
  \input{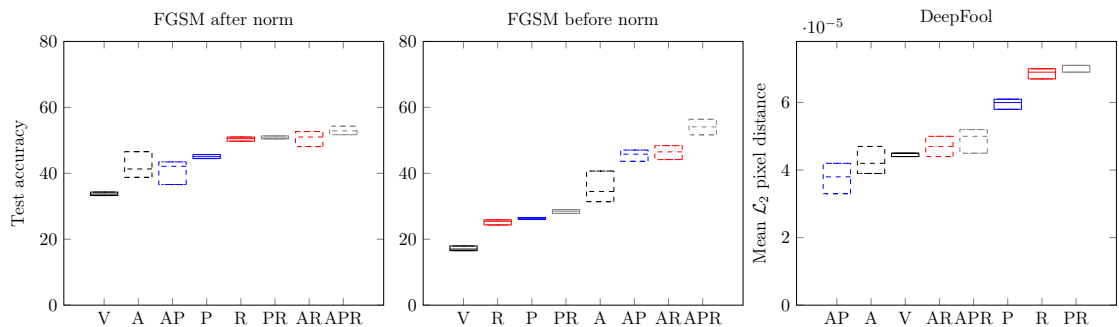}%

        \end{adjustbox}
        \caption{Robustness against an adversary measured by the test set accuracy under FGSM attack in the left and center plots and by the mean $\mathcal{L}_2$ pixel distance needed to fool the network using DeepFool on the right plot. Figure and caption extracted from~\citep{lassance2018laplacian}.}\label{chap5:fig-adversarialNoise-adv}
    \end{center}
\end{figure}

Finally, when considering implementation related perturbations, the results depicted in Figure~\ref{chap5:fig-dropout-quantized-adv} are consistent with the ones from the previous section, in which is shown that the proposed regularizer helps improving robustness to this type of noise.

\begin{figure}[ht]
    \begin{center}
        \begin{adjustbox}{max width=\linewidth}
  \tikzsetnextfilename{chapter5/tikz/dropout_quantized_adv}%
  \input{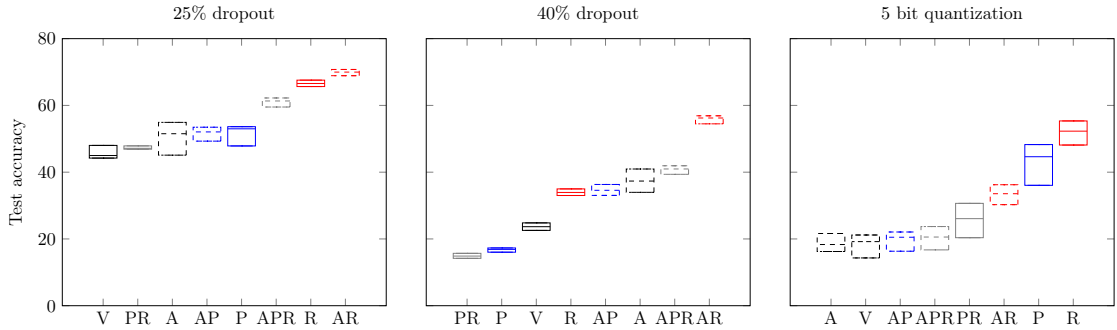}%

        \end{adjustbox}
    \caption{Test set accuracy under different types of implementation related noise. Figure and caption extracted from~\citep{lassance2018laplacian}.}
    \label{chap5:fig-dropout-quantized-adv}
    \end{center}
\end{figure}

In summary, even when adding adversarial training, the proposed regularizer is either the most robust in median, or capable of improving the robustness when combined with the other methods.

\paragraph{Tests with PGD adversarial data augmentation}

Most of our adversarial tests are performed with FGSM because of its simplicity and speed, even though it has already been shown~(e.g: ~\citep{madry2018towards}) that FGSM is weak as an attack and as a defense mechanism. Despite the fact we do not only target adversarial defense, we further stress the ability of the proposed regularizer to improve it and to combine with other methods. To this end we perform experiments against the PGD (Projected Gradient Descent) attack.

As the proposed regularizer can be combined with FGSM defense, it is natural to also test it alongside PGD training. We use the parameters advised in~\citep{madry2018towards}: 7 iterations with $step = 2/255$, and $\epsilon=8/255$. The results depicted in Table~\ref{chap5:PGDtraining} show that using our regularizer increases robustness of networks trained with PGD. Note that Dropout and Gaussian Noise were applied ten times to each of the networks and the results are displayed as the mean test set accuracy under these perturbations. A rate of 40\% was used for dropout. The PGD attack uses the following parameters: $it=20,step=\frac{2}{255},\epsilon=\frac{8}{255}\;.$

\begin{table}[ht]
\begin{center}
\caption{Test set accuracy results on the CIFAR-10 dataset with PGD training. Table and caption extracted from~\citep{lassance2018laplacian}.}
\begin{tabular}{cc|ccc}
\hline
 & Clean & Gaussian & PGD & Dropout                                            \\ \Xhline{2\arrayrulewidth}
A               & 76.39\% & 71.25\%          & 32.78\%           & 35.20\%           \\ 
A + R & 76.36\% & \textbf{72.26\%}          & \textbf{33.72}\%  & \textbf{55.63\%}          \\ \Xhline{2\arrayrulewidth}
\end{tabular}
\label{chap5:PGDtraining}
\end{center}
\end{table}

\subsubsection{Experiments with other datasets}
\label{chap5:other_datasets}

In this final subsection, we test the generality of the method using the CIFAR-100 and ImageNet32x32 datasets, with a subset of the perturbations used for CIFAR-10. Gaussian Noise is applied ten times to each of the networks for a total of 30 different runs. A SNR of $33$ is used for FGSM and $15$ for Gaussian Noise. Images are normalized in the same way as the experiments with CIFAR-10. Standard data augmentation is used for CIFAR-100.  

Results on CIFAR-100 are shown in Table~\ref{chap5:table_cifar_100_wide} as the mean over three different initializations. We observe that as it was the case on CIFAR-10, the proposed method and the combination of the methods is the most robust on these test cases.

\begin{table}[ht]
\begin{center}
\caption{Test set accuracy results on the CIFAR-100 dataset. Table and caption extracted from~\citep{lassance2018laplacian}.}
\begin{tabular}{cc|cc}
\hline
Model & Clean Set       & Gaussian Noise  & FGSM                                             \\ \Xhline{2\arrayrulewidth}
Vanilla (V)   & 78.7\%   & 12.6\%          & 20.5\%           \\ 
Parseval (P)   & \textbf{80.1\%}            & 14.8\%          & 22.0\%           \\ 
Regularizer (R)   & 79.4\%            & 15.9\%          & 23.0\%           \\ 
P+R  & 79.5\%            & \textbf{19.1\%} & \textbf{24.4\%}  \\ \hline                                            
\end{tabular}
\label{chap5:table_cifar_100_wide}
\end{center}
\end{table}

We then use Imagenet32x32, a downscaled version of Imagenet~\citep{chrabaszcz2017downsampled} which can be used as an alternative to CIFAR-10 while maintaining a similar computational budget~\citep{chrabaszcz2017downsampled}. We use the same network and training hyperparameters of the original paper. Gaussian Noise and Dropout are applied 40 times to each of the networks. Gaussian noise is applied with SNR=33 whereas Dropout is applied with 15\%.

Results are shown in Table~\ref{chap5:imagenet32x32}. We observe that as it was the case on CIFAR-10 and CIFAR-100, the proposed method provides more robustness in all of these test cases. Note that we had trouble fine-tuning the $\beta$ parameter for the Parseval criterion, explaining the poor performance of Parseval and its combination with our proposed regularizer.

\begin{table}[ht]
\begin{center}
\caption{Test set accuracy results on the Imagenet32x32 dataset. Table and caption extracted from~\citep{lassance2018laplacian}.}
\begin{tabular}{cc|ccc}
\hline 
Model & Clean & Gaussian Noise  &  Dropout    \\ \Xhline{2\arrayrulewidth}
Vanilla (V)     & 52.1\%    & 36.8\%          & 2.3\% \\ 
Parseval (P)     & 48.1\%    & 34.10\%                & 3.71\%\\ 
Regularizer (R)     & \textbf{52.4\%}    & \textbf{37.4\%}          & \textbf{7.0\%}      \\ 
P+R    & 43.80\%   & 29.87\%                &  5.0\%\\ \hline                                            
\end{tabular}
\label{chap5:imagenet32x32}
\end{center}
\end{table}

In this Section we have introduced a definition of robustness alongside an associated regularizer. The former takes into account both small variations around the training set examples and the margin. The latter enforces small variations of the smoothness of label signals on similarity graphs obtained at intermediate layers of a deep learning network architecture.
We have empirically shown with our tests that the proposed regularizer can lead to improved robustness in various conditions compared to existing counterparts. We also demonstrated that combining the proposed regularizer with existing methods can result in even better robustness for some conditions. Future work includes a more systematic study of the effectiveness of the method with regards to other datasets, models and perturbations. Recent works shown adversarial noise is partially transferable between models and dataset and therefore we are confident about the generality of the method in terms of models and datasets.

\section{Using intermediate representation graphs to compress DNNs}\label{chap5:gkd}

In the previous sections we have shown the interest of using the concepts of GSP in order to analyze and improve DNNs. In this section we present an introductory work that specializes a knowledge distillation framework, called Relational Knowledge Distillation (RKD), to the graph domain. We call this new framework Graph Knowledge Distillation (GKD). In other words, we present a technique that allows us to compress DNNs by using graphs to represent the intermediate spaces of neural networks. We presented this work in a recent contribution~\citep{lassance2020deep} and note two works from the same period that proposed similar ideas~\citep{liu2019knowledge,lee2019graph}. 

As we have previously discussed in Section~\ref{chap2:compression}, the success of DNNs is heavily linked to the availability of large amounts of data and special purpose hardware, e.g., graphics processing units (GPUs) allowing significant levels of parallelism. However, this need for a significant amount of computation is a limitation in the context of embedded systems, where energy and memory are constrained. As a result, numerous recent works have focused on compressing deep learning architectures, some of them using the distillation technique. 

In a quick recall from Section~\ref{chap2:distillation}, one approach to distillation is performing Individual Knowledge Distillation (IKD)~\citep{ba2014deep,hinton2014distillation,romero2015fitnets}. Initial IKD techniques~\citep{hinton2014distillation} focused on using the output representations of the teacher as a target for the smaller architecture, while more recent works have reached better accuracy by performing this process layer-wise, or block-wise for complex architectures~\citep{romero2015fitnets,koratana2019lit}. However, IKD can be directly performed layer-wise only if the student and the teacher have inner data representations with the same dimension~\citep{koratana2019lit}, or if transformations are added~\citep{romero2015fitnets}.

In an effort to allow distillation to be performed layer-wise on architectures with varying dimensions, recent works~\citep{park2019rkd} have introduced distillation in a dimension-agnostic manner. To do so, these methods focus on the relative distances of the intermediate representations of training examples, rather than on the exact positions of each example in their corresponding domains. These methods are referred to as relational knowledge distillation (RKD) in the literature.

In this section, we present our work in which we extend this notion of RKD by introducing {\em graph knowledge distillation} (GKD). As in the previous sections of this chapter, we construct graphs where vertices represent training examples, and the edge weight between two vertices is a function of the similarity between the representations of the corresponding  examples at a given layer of the network architecture. The main motivation for this choice is that even though representations generally have different dimensions in each architecture, the size of the corresponding graphs is always the same  (since the number of nodes is equal to the number of training examples). Thus, information from graphs generated from the teacher architecture can be used to train the student architecture by introducing a discrepancy loss between their respective adjacency matrices during training.

In other words, we introduce a layer-wise distillation process using graphs, extending the RKD framework, and we demonstrate that this method can improve the accuracy of students trained in the context of distillation, using standard vision benchmarks. The reported gains are about twice as important as those obtained by using standard RKD instead of no distillation.

\subsection{Methodology}
\label{methodology}

In this section we first introduce RKD and recall some of the notations from Section~\ref{chap2:distillation}, then we introduce the methodology used to define GKD.

\subsubsection{Relational Knowledge Distillation (RKD) }

Recall that $T$ and $S$ denote teacher and student architectures, respectively. The goal of distilation is to transfer knowledge from $T$ to $S$, where $S$ typically contains fewer parameters than $T$. For presentation simplicity, we assume that both architectures generate the same number of inner representations. In the context of distillation, we consider that the teacher has already been trained, and that we want to use both the training set and the inner representations of the teacher in order to train the student. This is an alternative to directly training the student using only the training data (which we refer to as ``baseline'' in our experiments). Also recall, that we use the following loss to train the student:
\begin{equation}
    \mathcal{L} = \mathcal{L}_\text{task} + \lambda_{\text{KD}} \cdot \mathcal{L}_\text{KD}\;.
    \label{chap5:loss-distilation}
\end{equation}

We denote $\mX \in \trainset$ a batch of input examples and $\sX'$ the set of intermediate representations generated using $\mX$ that are used for inferring knowledge. RKD approaches consider relative metrics between the respective inner representations of the networks to be compared. In the specific case of RKD-D~\citep{park2019rkd}, the mathematical formulation is: 
\begin{equation}
\mathcal{L}_\text{RKD-D} = \sum_{X' \in \sX'}\sum_{(\vx_i,\vx_j)\in \dot{X'}}{\mathcal{L}_d\left(\frac{\|\vx^S_i-\vx^S_j\|_2}{\Delta'^S},\frac{\|\vx^T_i - \vx^T_j\|_2}{{\Delta'^T}}\right)},
\end{equation}
where $\dot{X'}$ is the set of all possible pairs from $X'$, $\Delta'^A$ is the average distance between all couples $(\vx^A_i,\vx^A_j)$ for each $X' \in \sX'$ for the given architecture, and $\mathcal{L}_d$ is the Huber loss~\citep{huber1992robust}. The main advantage of using RKD is that it allows to distillate knowledge from an inner representation of the teacher to one of the student, even if their respective dimensions are different.

\subsubsection{Proposed Approach: Graph Knowledge Distillation (GKD)}\label{chap5:gkd_section}

We now introduce our proposed approach. Instead of directly trying to make the distances between data points in the student match those of the teacher, we consider the problem from a graph perspective. Given an architecture $A$, a batch of inputs $X$, we compute the corresponding inner representations $X'^A = f'^A([\vx, \vx \in X])$. We can then choose a set of layers that we want to consider and create a set $\sX'$ containing these intermediate representations. These representations are then used to define a similarity graph $\gG^A(X')$, for each $X' \in \sX'$. The graph contains a node for each input in the batch, and the edge weight $\emAdjacency^A(X')_{i,j}$ represents the similarity between the $i$-th and the $j$-th elements of $X'$ from architecture $A$. In this work, we use the cosine similarity. Finally, in order to control the importance of outliers, we also normalize the adjacency matrix.

While training the student, we input our training batch into both the student architecture and the (now fixed) previously trained teacher architecture. This provides a similarity graph for each representation $X'$ from the set of representations to consider $\sX'$. The loss we aim to minimize combines the task loss, as expressed in Equation~\ref{chap5:loss-distilation}, with the following graph knowledge distillation (GKD) loss: 
\begin{equation}
    \mathcal{L}_{\text{GKD}} = \sum_{X' \in \sX'}{\mathcal{L}_d(\gG^S(X'),\gG^T(X'))}\;.
\end{equation}

In our work, we mainly consider the case where $\mathcal{L}_d$ is the Frobenius norm between the adjacency matrices. The GKD loss measures the discrepancy between the adjacency matrices of teacher and student graphs. In this way the geometry of the latent representations of the student will be forced to converge to that of the teacher. Our intuition is that since the teacher network is expected to generalize well to the test, mimicking its latent representation geometry should allow for better generalization of the student network as well. An equivalent definition of our proposed loss is:
\begin{equation}
    \mathcal{L}_{\text{GKD}} = \sum_{X' \in \sX'}{\|\adjmatrix^S(X')-\adjmatrix^T(X')\|_2^2}\;.
\end{equation}

A first obvious advantage of GKD with respect to RKD-D is the fact it has a more natural normalization over the batch of inputs, yielding to a more robust process. This is discussed in Section~\ref{chap5:normalization}. Amongst other degrees of freedom that become available when using graphs, we focus on three possible variations of the method:
\begin{enumerate}
    \item Task specific: considering only examples of the same (resp. distinct) classes when creating the edges of the graph, thus focusing on the clustering (resp. margin) of classes,
    \item Localized: weighting differently the closest and furthest neighbors of each node in the graph, in an effort to focus on locality, or to the contrary on remoteness,
    \item Smoothed: taking powers $p$ of the normalized adjacency matrix of considered graphs before computing the loss. By considering higher powers of $\adjmatrix$, we consider smoothed relations between inner representations of inputs.
\end{enumerate}

\subsection{Experiments}

In this section, we perform two types of experiments. First we compare the accuracy of RKD-D and GKD using the CIFAR-10 and CIFAR-100 datasets~\citep{krizhevsky2009learning}, analyze the impact of the normalization of the similarities, compare the consistency with the teacher and perform spectral analysis of the different graphs and graph signals. We then look at proposed variations of GKD: task specific, localized and smoothed.

\subsubsection{Hyperparameters}

We train our CIFAR-10/100 networks for 200 epochs, using standard Stochastic Gradient Descent (SGD) with batches of size 128 ($|X|=128$) and an initial learning rate of $0.1$ that is decayed by a factor of $0.2$ at epochs $60$, $120$ and $160$. We also add a momentum of $0.9$ and follow standard data augmentation procedure. We use a ResNet26-1 architecture for our teacher network, while the student network uses a Resnet26-0.5. In terms of scale, ResNet26-0.5 has approximately 27\% of the operations and parameters of ResNet26-1. All these architectures are particularly small compared to the ones achieving state-of-the-art performance. We use a network of same size of the students but trained without a teacher as a baseline that we call Vanilla. Our RKD-D~\citep{park2019rkd} students are trained with the parameters from~\citep{park2019rkd}, $\lambda_\text{RKD-D}=25$ and applied to the output of each block. We applied the same values for GKD. Note that all these choices were made to remain as consistent as possible with existing literature. For each student network we run either 10 (CIFAR-10) or 3 (CIFAR-100) tests and report the median value. The code for reproducing the experiments and boxplots for each experiment is available at~\url{https://github.com/cadurosar/graph_kd}.

\subsubsection{Direct comparison between GKD and RKD-D}

In a first experiment we simply evaluate the test set error rate when performing distillation. Results are summarized in Table~\ref{chap5:errorrate}. We compare student sized networks trained without distillation, that we call baseline, with GKD and RKD-D~\citep{park2019rkd} trained networks. We also report the performance of the teacher. We note that RKD-D~\citep{park2019rkd} by itself provides a small gain in error rate with respect to the Baseline approach, while GKD outperforms RKD-D by almost the same gain.

\begin{table}[ht]
\centering
\caption{Error rate comparison of GKD and RKD-D. Table and caption extracted from~\citep{lassance2020deep} @2020 IEEE.}
\label{chap5:errorrate}
\begin{tabular}{|c|c|c|c|}
\hline
Method                   & CIFAR-10 & CIFAR-100 & Relative size   \\ \hline 
Teacher                  & 7.27\% ($\pm$ 0.26)     & 31.26\%  & 100\% \\ \hline
Baseline                 & 10.34\% ($\pm$ 0.27)   & 38.50\% & 27\%  \\ \hline
\hline
RKD-D~\citep{park2019rkd}         & 10.05\% ($\pm$ 0.28)    & 38.26\% & 27\%  \\ \hline
GKD                      & \textbf{9.71\% ($\pm$ 0.27) }     & \textbf{38.17\%} & 27\%  \\ \hline

\end{tabular}
\end{table}

\subsubsection{Effect of the normalization}
\label{chap5:normalization}

To better understand why GKD performed better than RKD-D we analyze the contribution of each example in a batch in both the GKD loss and the RKD-D one. If our premise from Section~\ref{chap5:gkd_section} is correct, by using a degree normalized adjacency matrix instead of the distance pairs directly, most examples will be able to contribute to the optimization. To do so, we compute the respective loss, for each block, using 50 batches of 1000 training set examples and analyze the median amount of examples that are responsible for 90\% of the loss at each block.  In Table~\ref{chap5:normalization_table}, we present the results. As we suspected for GKD, it shows a significant advantage on the number of examples responsible for 90\% of the loss.

\begin{table}[ht]
\centering
\caption{Comparison of the effect of the normalization on the amount of examples that it takes to achieve 90\% of the total loss value. Table and caption extracted from~\citep{lassance2020deep} @2020 IEEE.}
\begin{tabular}{|c|c|c|}
\hline
Block position in the architecture & RKD-D            & GKD              \\ \hline\hline
Middle & 83.70\%          & \textbf{86.50\%} \\ \hline
Final & 82.05\%          & \textbf{83.60\%} \\ \hline
\end{tabular}
\label{chap5:normalization_table}
\end{table}

\subsubsection{Classification consistency}

We now take our trained students and compare their outputs to the trained teacher's outputs. For the output of each WideResNet block we compute the classification of a simple Logistic Regression, while the network's final output is already a classifier. The ideal scenario would be one where the student is 100\% consistent with the teacher's decision on the test set, as this would greatly improve the classification performance when compared to the baseline. The results are depicted in Figure~\ref{chap5:fig-consistency}. As expected the GKD was able to be more consistent with the teacher than the RKD-D.

\begin{figure}[ht]
    \begin{center}
        \begin{adjustbox}{max width=\linewidth}
  \tikzsetnextfilename{chapter5/tikz/consistency}%
  \input{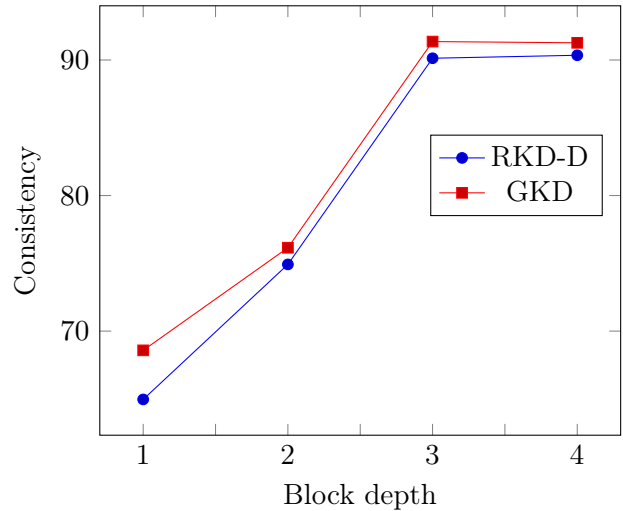}%

        \end{adjustbox}    
        \caption{Analysis of the consistency of classification compared to the teacher, across blocks of RKD-D and GKD students. We consider the classification layer of the network as the ``fourth block''. Figure and caption adapted from~\citep{lassance2020deep} @2020 IEEE.}
        \label{chap5:fig-consistency}
    \end{center}
\end{figure}

\subsubsection{Spectral analysis}

Given that we have introduced intermediate representation graphs, it is quite natural to analyze performance from a GSP perspective~\citep{shuman2013emerging}. We propose to do so by considering specific graph signals $\vs$ and computing their respective smoothness on each of the two graphs. We create graphs with 1000 examples chosen at random from the training set. The signals that we consider are \begin{inlinelist} \item the label binary indicator signal \item the Fiedler eigenvectors from each intermediate representation in the teacher, which allow us to compare the clustering of both networks and how they evolve over successive blocks.\end{inlinelist} The results are depicted in Figure~\ref{chap5:fig-spectral}. We can see that both signals have more smoothness in the graphs generated by GKD. This means that the geometry of the latent spaces from GKD are more aligned to those of the teacher. 

\begin{figure}[ht]
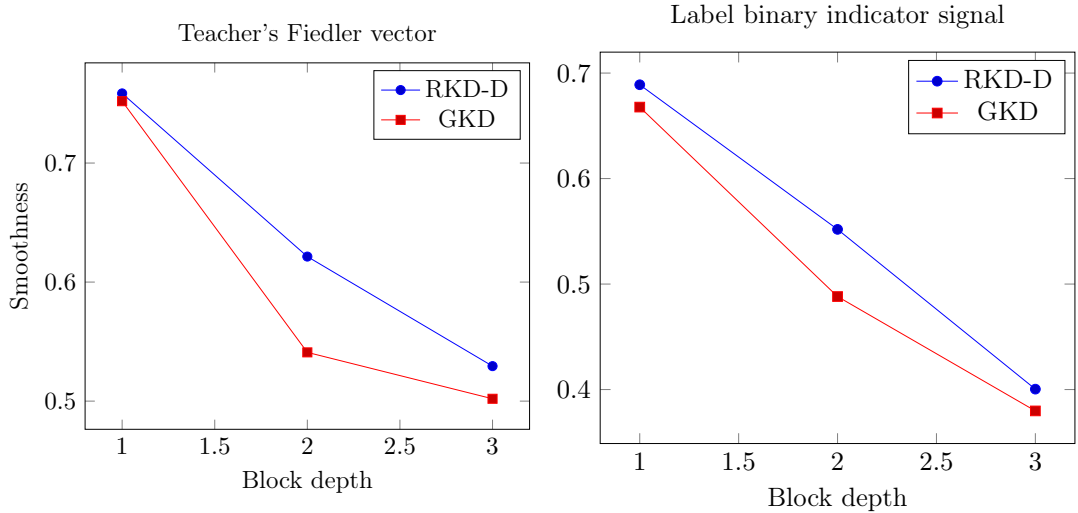

    \begin{center}
        \begin{subfigure}[ht]{0.48\linewidth}
            \centering
            \begin{adjustbox}{max width=\linewidth}
  \tikzsetnextfilename{chapter5/tikz/gkd_fiedler}%
  \input{chapter5/tikz/gkd_fiedler.tex}%

            \end{adjustbox}
        \end{subfigure}
        \begin{subfigure}[ht]{0.48\linewidth}
            \centering
            \begin{adjustbox}{max width=\linewidth}
  \tikzsetnextfilename{chapter5/tikz/gkd_smoothness}%
  \input{chapter5/tikz/gkd_smoothness.tex}%

            \end{adjustbox}
        \end{subfigure}
        \caption{Analysis of the smoothness evolution across blocks of the students. In the right we have the label binary indicator signal and in the left we use the Teacher's Fiedler vector as a signal. Figure and caption extracted from~\citep{lassance2020deep} @2020 IEEE.}
        \label{chap5:fig-spectral}
    \end{center}
\end{figure}

\subsubsection{Task specific graph signals}

We now consider variations of the proposed GKD method. The first one are the effects of considering only intra or inter-class distances. If we consider only inter-class distances we can focus mostly on having a similar margin in both teacher and student. On the other hand, considering only intra-class distances would force both networks to perform the same type of clustering on the classes. The results are presented in Figure~\ref{chap5:fig-signal}. In this case, focusing on the margin helped decrease both median test error rate and its standard deviation, while concentrating on the clustering was not effective. This result is similar to what we found in our prior work (Section~\ref{chap5:laplacian}), which shows that the margin is a better tool to interpret the network results than the class clustering.

\begin{figure}[ht]
    \begin{center}
        \begin{adjustbox}{max width=\linewidth}
  \tikzsetnextfilename{chapter5/tikz/gkd_task_specific}%
  \input{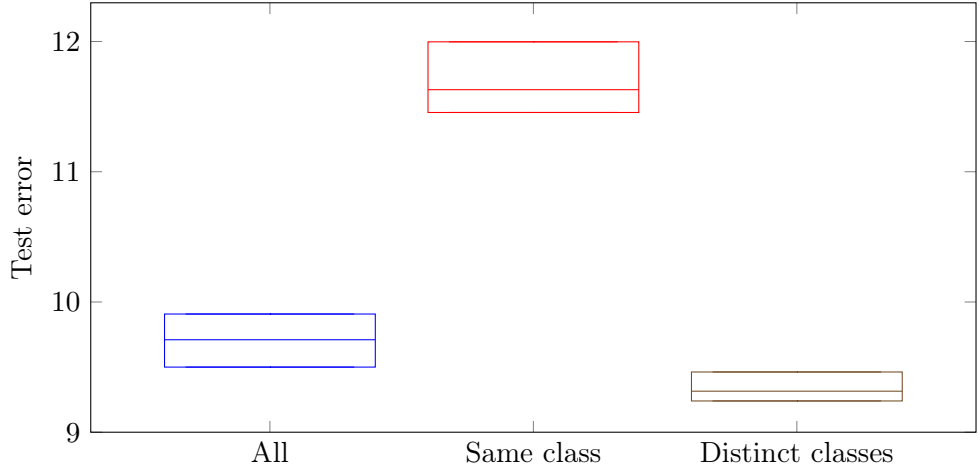}%

        \end{adjustbox}
        \caption{Analysis of the effect of task specific graph signals. Figure and caption extracted from~\citep{lassance2020deep} @2020 IEEE.}
        \label{chap5:fig-signal}
    \end{center}
\end{figure}

\subsubsection{Effect of locality}

To study the effect of locality, we partition the graph edges in two parts: 1) the ones corresponding to the $k$ nearest neighbors of each vertex $\adjmatrix^A(k)$ and 2) the other ones $\overline{\adjmatrix^A(k)}$. Consequently, we can write $\adjmatrix^A = \adjmatrix^A(k) + \overline{\adjmatrix^A(k)}$. We then introduce the new adjacency matrix $\adjmatrix^A(k,\alpha)=\alpha\adjmatrix^A(k) + (1-\alpha)\overline{\adjmatrix^A(k)}$, where $\alpha$ scales the importance of 1) with respect to 2). So choosing $\alpha = 0$ means to disregard nearest neighbors while $\alpha=1$ corresponds to focusing only on them. Results are summarized in Figure~\ref{chap5:fig-neighbors}. We observe that for small value of $k$, small values of $\alpha$ lead to the best performance, whereas for $k=|X|/2$ larger values of $\alpha$ are better. This is similar to results such as~\citep{hermans2017defense}, where the authors show that one should not concentrate on the hardest/easiest cases, but on the intermediate cases.

\begin{figure}[ht]
    \begin{center}
        \begin{adjustbox}{max width=\linewidth}
  \tikzsetnextfilename{chapter5/tikz/gkd_knn}%
  \input{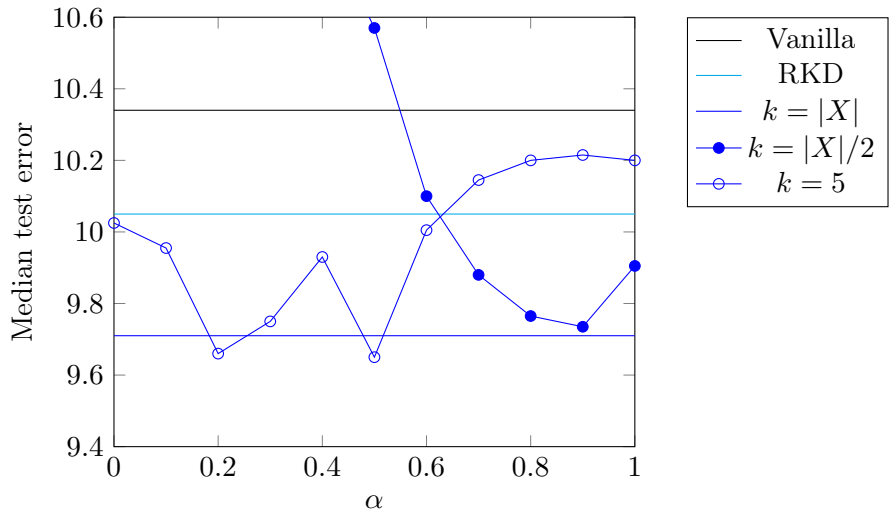}%

        \end{adjustbox}

        \caption{Median test error for different values of $k$ and $\alpha$. Figure and caption extracted from~\citep{lassance2020deep} @2020 IEEE.}
        \label{chap5:fig-neighbors}
    \end{center}
\end{figure}

\subsubsection{Smoothed representations}

Finally, we study the effect of varying the power of adjacency matrices $p$. This allows us to consider smoothed relations between inner representation of inputs when compared to fixing $p$ to 1. The results are presented in Table~\ref{chap5:tab-powers}. Smoothed relations do not seem to help the transfer of knowledge. One possible reason is that larger powers have the effect of drowning out the information.

\begin{table}[ht]
\centering
\caption{Analysis of the effect of varying $p$ on the error rate. Table and caption extracted from~\citep{lassance2020deep} @2020 IEEE.}
\vspace{.5cm}

\begin{tabular}{|c||c|c|c|}
\hline
$p$        & 1 & 2 & 3 \\ \hline
Error Rate & \textbf{9.71\%} & 10.20\% & 10.07\%\\\hline
\end{tabular}
\label{chap5:tab-powers}
\end{table}

In the previous paragraphs we have introduced graph knowledge distillation (GKD), a method using graphs to transfer knowledge from a teacher architecture to a student one. By using graphs, the method opens the way to numerous variations that can significantly benefit the accuracy of the student, as demonstrated by our experiments. We note that we are not the first to propose such an extension, but that nonetheless this is an interesting research direction. In future work we consider: \begin{inlinelist} \item using more appropriate graph distances, such as in~\citep{bunke1998graph,segarra2015diffusion} \item doing a more in-depth exploration of how to properly scale the student network, e.g. following~\citep{tan2019efficientnet} \item combining with approaches such as~\citep{hermans2017defense,bontonou2019smoothness} to train a teacher network in a layer-wise fashion. \end{inlinelist}

\section{Summary of the chapter}

Differently from the previous ones, in this chapter we have mainly presented our contributions in the domain of ``Deep Neural Networks latent spaces supported on graphs''. While this domain is not very developed, we hope that our contributions may shine a light and allow for more development on it, as we believe there are a lot of interesting contributions to pursue. 

We have first introduced the work that we believe was the cornerstone for our interest in the domain~\citep{gripon2018insidelook}, in which the authors have shown that it was possible to characterize different DNN behaviors by analyzing the evolution of the graph signal smoothness over their representations. We then built upon this work to propose a measure that is empirically correlated (but that we are not able to ensure causation) with the generalization performance of DNNs. This was subject of a contribution to a non-archival conference:

\begin{itemize}
    \item \bibentry{lassance2018predicting}
\end{itemize}

Then we concentrated on possible uses of graph signal smoothness during the training of neural networks. First we showed that we are able to train good feature extractors by training the network to minimize the smoothness of the label indicator signals on the graphs generated by their outputs. This new objective function has three important features that are not present in the traditional cross entropy loss and we demonstrate using experiments that we are able to obtain networks that are more robust, without losing too much generalization performance. Second, we propose to use a regularizer in order to control how the smoothness of the label indicator signals evolve over graphs that are generated by the intermediate representations of DNNs. We show that these regularizers are not only theoretically inline with our definition of robustness (Definition~\ref{chap2:def_robustness}), but also that we can demonstrate empirically their efficacy when compared (or added) to other methods in the literature. These two uses of graph signal smoothness were subjects of contributions, one to a conference and the other is under the review process of a journal:

\begin{enumerate}
    \item \bibentry{bontonou2019smoothness}
    \item \bibentry{lassance2018laplacian}
\end{enumerate}

Finally we have presented a method that does not build from the GSP framework, but that allows us to use the GSP framework on previously defined techniques. In other words, we have specialized the RKD framework as GKD, which we have shown empirically and analytically to improve the performance of the compressed networks. This introductory work was published at a conference as:

\begin{itemize}
    \item \bibentry{lassance2020deep}
\end{itemize}

In summary, the works in this chapter aim at proposing and expanding the domain of ``Deep Neural Networks latent spaces supported on graphs'', while showing the possible improvements this domain can bring to the overall Deep Learning community. In the next chapter we provide a summary of the overall thesis, in order to conclude the work and present the research directions that we have opened but were not able to explore yet.
\chapter{Conclusion}\label{conclusion}
\localtableofcontents

In this section, we first present a quick summary of the thesis. We then present a summary of the contributions presented in this document and the perspectives/research directions that we consider now open in the context of this work. We close this document with a discussion and considerations on the overall field of deep learning. The main idea that we pursued during the last three years was to tackle some shortcomings of deep learning architectures by looking at their intermediate representations. 

To perform our analyses, we used the framework of Graph Signal Processing, in which graphs are used to represent the topology of a complex domain (here: latent spaces of deep learning architectures). We have considered deep learning applications within three machine learning domains: \begin{inlinelist} \item representation/transfer learning \item compression of deep learning architectures \item study of overfitting (generalization and robustness) \end{inlinelist}. In the following paragraphs, we present a summary of the contributions introduced in the document, grouped according to the above-mentioned domains.

\section{Summary of contributions}

\subsection{Representation and transfer learning}

In Section~\ref{chap3:graph_inference} we introduced a benchmark to allow the comparison of graph inference methods. Being able to measure and determine the most effective technique for graph inference is of utmost importance, given that most of the analysis we performed during the PhD depends on such inferred graphs. Our findings are that the ``naive'' baseline ($k$-nn symmetric similarity graphs) can achieve good enough performance when well-tuned compared to more principled approaches. This is even more important given the fact these naive baselines can typically be computed very efficiently, allowing for deployment during the learning of deep neural networks. We consider that the experiments and findings presented in this document could be the starting point to interesting extensions, including taking into account the specific task that graphs are inferred for when they are created. We shall discuss this point later in this chapter.

Then in Section~\ref{chap4:classify_graph_signals}, we introduced two techniques that allow us to perform the supervised classification of graph signals. First, we presented an embedding technique which aim at representing a graph in a 2D Euclidean space. As most conventional deep learning architectures are already well adapted to inputs on the 2D Euclidean space, it is quite straight-forward then to use these representations in a classical deep convolutional neural network. We then introduced a set of methods that attack the question of whether one can match the performance of CNNs without using priors about the data structure. Three ideas were discussed in this context: \begin{inlinelist} \item using a graph convolution scheme based on translations, which were first defined in Section~\ref{chap4:translation} \item introducing a method for performing downsampling on those graphs \item introducing a data augmentation scheme based on graph translations\end{inlinelist}. These three improvements allowed us to reduce the gap from convolutions in a 2D space to graph convolutions to a drop in accuracy of only about 2.7\% in a competitive image classification benchmark.   

We also considered the task of learning representations that are suited for the classification task. Namely, we introduced a graph smoothness loss for deep learning architectures in Section~\ref{chap5:smoothness_loss}. Using this loss to train a deep learning architecture enforces that the output should be able to generate graphs that are smooth with regards to the label signal. This graph smoothness loss has three properties that are interesting in the context of representation learning: \begin{inlinelist} \item no contraction, as we only force elements of different classes to be distant, but not all elements of the same class to be close \item no arbitrary decision of where the points of each class should be \item no restriction on the amount of dimensions of the output \end{inlinelist}. We demonstrate via experiments that networks trained with the proposed loss are more robust while achieving similar accuracy than networks trained with the standard cross-entropy loss.  

Finally, we considered transfer learning. In the case of transfer learning, the goal is to use representations learned in a first task where data is abundant in a second problem where data tends to be scarce. In Section~\ref{chap3:graph_filter} we presented graph filters and how they can be used to improve the pre-trained representations by either combining additional information (Section~\ref{chap3:vbl}) or by denoising the learned features in the new class domain (Section~\ref{chap3:reducing_noise}). We showed through our experiments that graph filters are well suited to the transfer learning task, allowing us to improve the performance in various tasks ranging from visual-based localization to few-shot learning. 

\subsection{Neural network compression}

In the context of neural network compression, we presented two types of contributions: \begin{inlinelist} \item efficient layers \item distillation\end{inlinelist}. First, in Section~\ref{chap2:SAL}, we presented Shift Attention Layers (SAL). This efficient layer scheme reduces the amount of weights per convolutional filter using the concept of attention so that only the most important weights are kept per filter at the end of the training. The goal of SAL is to start with a vast optimization space to optimize our network (more parameters). As the network evolves, we thin out the optimization space, pruning out the less important parameters. We compared with similar shift layers and showed that SAL can improve the performance of the compressed networks. 

We then focused on distillation. In Section~\ref{chap5:gkd}, we presented the Graph Knowledge Distillation (GKD) framework that uses graphs to represent latent spaces. GKD has the advantage of disregarding the dimensionality difference between teacher and student's latent spaces while distilling the learned structure from the teacher to the student. Using graphs also allowed us to propose newer additions to the distillation framework. For example, we proposed to consider only the edges between elements of distinct classes, which we showed to improve the performance of GKD.

\subsection{Generalization and robustness}

In this manuscript, we also considered the concept of overfitting, as stated in Definition~\ref{chap2:overfitting}. Per our definition, overfitting is closely linked to both robustness (the network is considered to be overfitted to the $\dataset$) and generalization (we say a network generalizes well if it is not overfitted to $\trainset$ or $\validset$).

We first formally defined what we call robustness in Section~\ref{chap2:robustness}. We also introduced the concept of $\alpha$-robustness. The principle is to not only focus on controlling the maximum perturbation a noise $\epsilon$ applied to the input may cause to the output but also on bounding the radius $r$ where this control should be applied. The goal is to enforce a small $\alpha$ around the examples (small $r$), while allowing more significant transitions on the parts of the space that are not represented in $\dataset$. We analyzed the representations of four recently proposed methods to increase the robustness of DNNs and found out that the methods that follow our definition (i.e., are more $\alpha$-robust) are the ones that are more empirically robust.

We then presented in Section~\ref{chap5:laplacian} a regularizer that, when applied to networks, creates Laplacian networks. In a Laplacian network, the smoothness of the label signal should evolve slowly across sequential intermediate representations. Therefore, a small perturbation applied to the input or one of the representations should not be able to impact the overall classification significantly. We analyzed how this is linked to the previously presented robustness definition and showed through experiments how it empirically improves the performance under various perturbations.  

Finally, we also studied the problem of analyzing the generalization of DNNs, when no extra labeled data (or a $\validset$) is available, in Section~\ref{chap5:smoothness}. We used the smoothness of the label signal as our metric. We first analyzed qualitatively (i.e., via graphs) if there is a difference in behavior from a reference network to controlled scenarios where we know the network will be underfitted or overfitted. We verified that the smoothness gap between the last layers of the network is a good indicator of generalization. We then performed a more quantitative analysis. We trained networks varying multiple hyperparameters and verified that there is indeed a correlation between the smoothness gap and the generalization of the network.

\section{Perspectives and future work}

In this thesis, we have introduced many contributions, each opening their research directions and perspectives. We have already described these perspectives during the introduction of each work. Therefore, we now present more high-level point of view. There are four main perspectives that we discuss: \begin{inlinelist} \item graph inference \item extending our framework to other tasks \item graphs and data acquisition \item communication of results \end{inlinelist}.

\subsection{Graph inference}

First, recall that our goal was to study the intermediate representations of DNNs. In most cases, we have introduced graphs that are inferred from data via a similarity metric. We have discussed such construction in Section~\ref{chap3:graph_inference}, where we have shown that, when correctly tuned, these graph representations can rival more principled techniques. On the other hand, it is fair to say that improvements in graph construction should lead to improvements in most methods presented in this thesis. One such improvement would be to take into account the task at hand when creating the graph, e.g., considering the label information during graph inference allowed us to improve our graph distillation in Section~\ref{chap5:gkd}. 

We already started working on this problem, and thinking of possible solutions to infer graphs given two objectives: matching the representations and helping in the considered downstream task. As a matter of fact, very often authors in the domain of graph inference introduce priors to solve what is in-fine an ill-posed problem: there are infinitely many graphs that would correspond to a dataset. For example, in~\citep{pasdeloup2017characterization}, the authors showed that when the prior is that signals are stationary on the graph, there is a polytope of possible graph structures that would fit the provided data. Choosing a point in the polytope boils down to favorizing a specific key property of the seeked graph structure. In their work, the authors consider sparsity or simplicity for example. We believe that using the task as a prior could lead to an interesting tradeoff between matching the signals and helping in solving the considered task.

\subsection{Extension to other tasks}

Note that in this manuscript, we mainly focused on semi-supervised and supervised classification. Extending our framework, which uses graphs that represent latent spaces, to other tasks, would be interesting future work. For example, recent literature in self-supervised learning uses a technique close to the presented graph smoothness loss, where augmented examples from the same original image should be closer in latent space than examples from different images~\citep{grill2020bootstrap,chen2020simple}. 

More generally, the consideration of hyperbolic spaces in the design and the optimization of deep learning architectures has become increasingly popular~\citep{verma2019manifold,liu2019hyperbolic}. Graphs could be considered as a natural way to reformulate or improve these methods. 

Consider manifold mixup for instance, that is being used in the training of many modern state-of-the-art classifiers~\citep{mangla2019manifold}. The principle is to interpolate inputs, outputs and intermediate representations to augment the training set. Instead of using a naive linear interpolation, using graphs instead could lead to a more accurate generation of augmented inputs.

\subsection{Graphs and data acquisition}

In the previous subsection, we described how our framework may be adapted to tasks that we did not consider in this document. An orthogonal problem would be to use the GSP framework during the data acquisition and or labeling phase. 

Consider the few-shot learning task. In this task, the goal is to learn from a few labeled samples. There are already many techniques to tackle this problem, including the ones presented in this thesis. The labeled samples are either chosen at random or predefined by the benchmark. This procedure leads to two drawbacks: 
\begin{enumerate}
    \item The algorithm's results are very dependent on exactly which are the labeled samples, requiring a large number of random initializations (i.e., drawing the few labeled samples) to have a good enough confidence interval for comparing two methods;  
    \item The procedure is not inline with real-world scenarios. A more realistic scenario is to acquire a collection of unlabelled samples. One can then either: \begin{inlinelist} \item choose in which order it should label the samples (active learning) \item receive a small subset of labeled examples and be able to exploit both labeled and unlabelled data (semi-supervised learning).\end{inlinelist} 
\end{enumerate}

Note that it is common to call ``few labels'' the scenario of ``few-shot semi-supervised learning''. Even if we do not explicitly treat the few-labels task in this thesis, we note that using similarity graphs (akin to those we use in this thesis) improves accuracy on the few-labels scenario~\citep{hu2020exploiting}. Moreover, we believe that extending this framework to the active learning scenario\footnote{more precisely the ``pool setting'', where the learner is given the set of unlabelled samples and can then iteratively choose which points that it wants to label.} should lead to improvements when compared to the traditional semi-supervised setting.

Indeed, imagine that the similarity graph we generate is not well-behaved (i.e., either disconnected or with a high variance of degrees). In this case, correctly choosing which node to label is of significant importance because some label information may be lost in the case of an unlucky random sample. To mitigate this problem, we believe that using a graph sampling algorithm~\citep{tanaka2020sampling} should allow us to accurately select the correct nodes to label, reducing the number of labels needed for adequate performance and reducing the variance caused by the random sampling. 

\subsection{Communication of results}

Finally, we have to analyze the diffusion of the ideas presented here. Parallel to the writing of this thesis, we are also preparing a book chapter that presents the domain of ``Graphs for deep learning latent representations'' in a more concise way, focusing less on our contributions and more on the domain itself. The goal is to create a more straightforward introduction of the concepts described here to diffuse our contributions and inspire more interest in the presented domain.  Further diffusing our findings with introductory courses to deep learning and GSP would be advisable as well.

\section{Discussion and considerations about the field}

Deep Learning has attracted a lot of attention in the past few years, and the trend is increasing. Consider for instance the number of papers submitted to NeurIPS, which is an iconic conference of the domain, in the past few years, showed in Figure~\ref{chap6:neurips}. We can clearly see how popular the domain has become.

\begin{figure}[ht]
    \centering
  \tikzsetnextfilename{conclusion/tikz/neurips}%
  \input{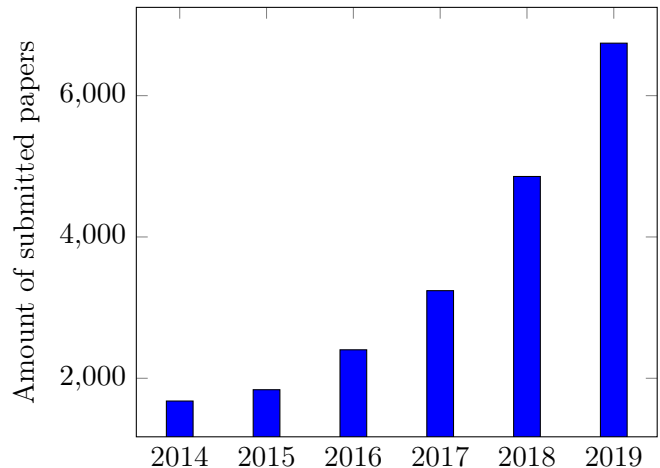}%

    \caption{NeurIPS paper submission evolution from 2014 to 2019. Data extracted from~\url{https://www.openresearch.org/wiki/NIPS} and~\url{https://medium.com/@NeurIPSConf/what-we-learned-from-neurips-2019-data-111ab996462c}}
    \label{chap6:neurips}
\end{figure}

Such a sudden popularity is not without drawbacks. For example, we noticed that despite being mostly driven by experiments, publications to the top-tier conferences in machine learning tend to provide weak statistical guarantees about the improvements they claim to obtain in their work. Too often we found papers in which the confidence interval is not given, the code made available (if made available) does not reproduce the presented results, and the mathematical formulation is disconnected from the presented experimental results.

This is not surprising given the fact that acceptance of papers in these venues is becoming increasingly important for applying to some companies or even academic positions. The problem is that the more papers are submitted, the more reviewers are needed, and of course when the numbers explode that fast, it is not possible to ensure a fair and balanced process.

In this context, most of the references that we cited in this document are very recent (> 2017) and they might contain contradictory results and claims. This is highly problematic in the context of a PhD in science, where we should be more focused on reproducibility and generalization of the results introduced in our contributions.

One such example happened at the start of this PhD thesis. One of our initial goals was to study graph neural networks. Very quickly we started to realize that there was a problem in the way that papers compared with each other. This greatly impacted our vision of the domain and was further confirmed by studies such as~\citep{vialatte2018convolution,shchur2018pitfalls,Errica2020A}. We have previously discussed this in Section~\ref{chap4:classify_graph_vertices}, but I believe that this required a more in-depth discussion in this conclusion. The fact that the problems come not only from the benchmarks (that is tackled by recent contributions such as~\citep{hu2020open}) but from the experimental design is a clear signal that the domain could be evolving too fast. Note that this is not exclusive to GNNs. The same problems have been found in deep metric learning, where in~\citep{roth2020revisiting} the authors show that the gap between older methods and more recent contributions is smaller than it is advertised in the recent papers. This is due mostly to improvements that are not linked with the more recent methods themselves, but in data acquisition/pre-processing. 

We tackled many important problems in this thesis. These problems are very relevant for a safe and trustworthy deployment of deep learning solutions in the society. Of course we would not pretend that the work that we did was not without failures or better than the rest of the literature. Indeed, some of our critics, also apply to some of our work. Yet depending on the opportunities that are going to appear in the continuity of my career, I hope I will be able to continue in this direction of research always striving to perform my work with scientific rigor.

\small{
\bibliography{references}
\bibliographystyle{natbib}
\clearpage
}
\clearpage
\includepdf{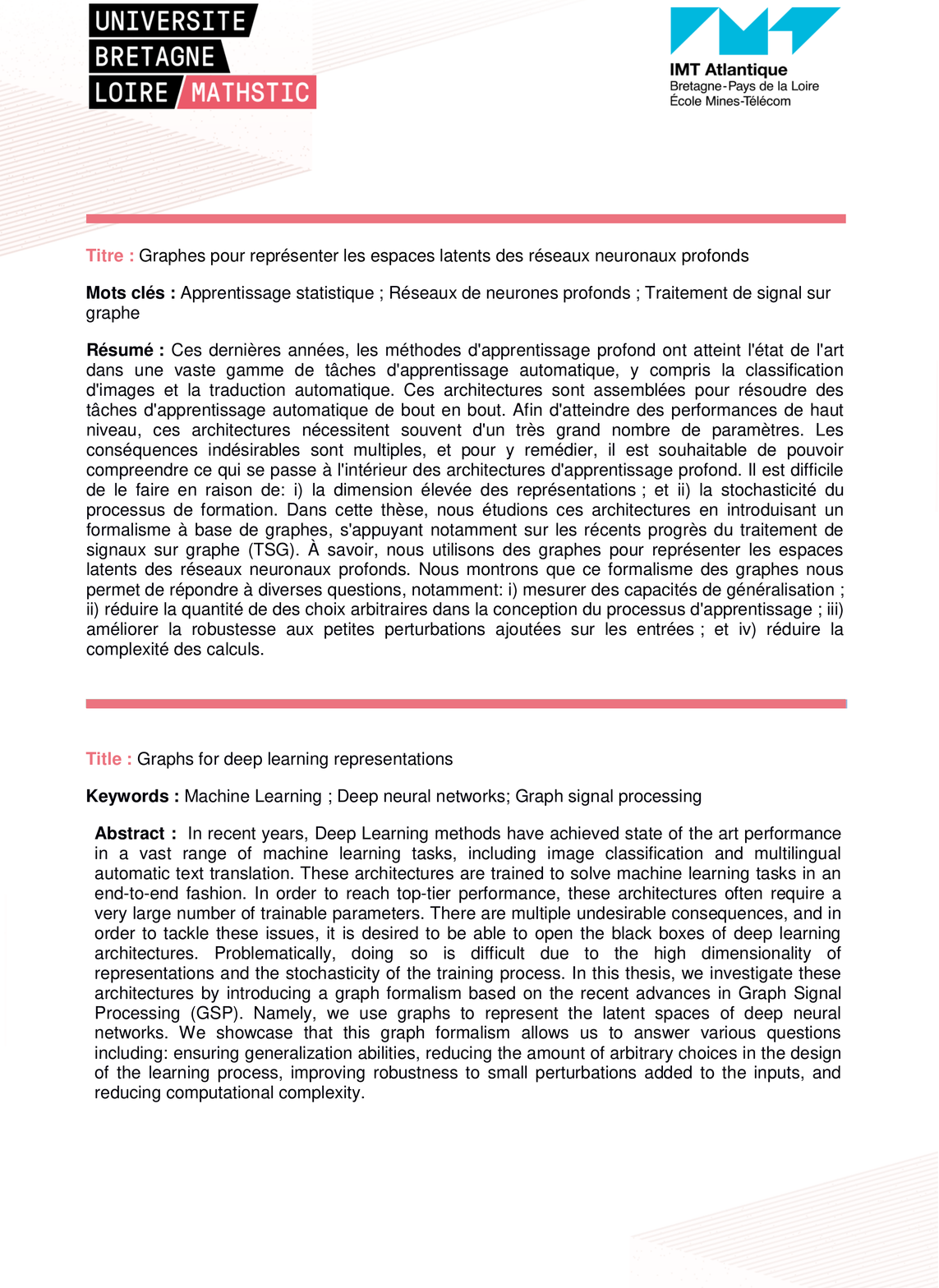}

\end{document}